\documentclass{article}
\usepackage{latexml}

\usepackage[main,
  final,
  nonatbib
]{neurips_2026}

\usepackage[T1]{fontenc}
\usepackage{microtype}
\usepackage{booktabs}
\usepackage{xltabular}
\usepackage{xcolor}
\usepackage{graphicx}
\usepackage{caption}
\usepackage{url}
\usepackage{fontawesome5}

\AtBeginDocument{\setlength{\textwidth}{5.5in}}

\definecolor{primary-red}{RGB}{165,30,55}
\definecolor{secondary-blue}{RGB}{0,105,170}

\usepackage{acro}
\DeclareAcronym{bce}{
  short = BCE,
  long  = binary cross-entropy
}

\DeclareAcronym{mrnp}{
  short = MRNP,
  long  = minimum-representation-norm parameterization
}

\DeclareAcronym{mwnp}{
  short = MWNP,
  long  = minimum-weight-norm parameterization
}

\DeclareAcronym{rsm}{
  short       = RSM,
  long        = representational similarity matrix,
  long-plural = representational similarity matrices
}

\DeclareAcronym{xor}{
  short = XOR,
  long  = exclusive or
}

\usepackage{enumitem}
\setlist[itemize,enumerate]{
  topsep=2pt,
  partopsep=0pt,
  leftmargin=1.5em,
  parsep=0pt,
  itemsep=2pt
}
\setlist[enumerate]{label=(\roman*)}

\usepackage{mathtools}
\usepackage{amssymb}
\usepackage{amsthm}
\usepackage{mleftright}
\usepackage{nicefrac}

\usepackage[
  style=numeric,
  sortcites=true,
]{biblatex}

\usepackage{hyperref}
\hypersetup{
  colorlinks=true,
  linkcolor=secondary-blue,
  urlcolor=secondary-blue,
  citecolor=primary-red,
}

\DeclareFieldFormat[article]{volume}{\textbf{#1}}

\renewbibmacro*{volume+number+eid}{\printfield{volume}\iffieldundef{number}
    {}
    {\printtext[parens]{\printfield{number}}}\setunit{\addcomma\space}\printfield{eid}}

\renewbibmacro*{author}{\ifboolexpr{ test \ifuseauthor and not test {\ifnameundef{author}} }
    {\printnames{author}\setunit*{\addperiod\space}\iffieldundef{year}{}{\printtext[parens]{\printfield{year}}}}{}\clearfield{year}\clearfield{month}\clearfield{day}\clearfield{date}\clearfield{endyear}\clearfield{endmonth}\clearfield{endday}}

\renewbibmacro*{date}{}
\renewbibmacro*{issue+date}{\iffieldundef{issue}
    {}
    {\printtext[parens]{\printfield{issue}}\newunit}}

\newbibmacro*{doi-link}{\href{https://doi.org/\thefield{doi}}{\raisebox{0.067ex}{\small\faIcon[regular]{file-alt}}}}

\newbibmacro*{url-link}{\href{\thefield{url}}{\raisebox{0.067ex}{\small\faIcon[solid]{external-link-alt}}}}

\newbibmacro*{arxiv-link}{\href{https://arxiv.org/abs/\thefield{eprint}}{\raisebox{0.067ex}{\small\faIcon[solid]{external-link-alt}}}}

\newbibmacro*{biorxiv-link}{\href{https://biorxiv.org/content/\thefield{eprint}}{\raisebox{0.067ex}{\small\faIcon[solid]{external-link-alt}}}}

\newbibmacro*{psyarxiv-link}{\href{https://psyarxiv.com/\thefield{eprint}}{\raisebox{0.067ex}{\small\faIcon[solid]{external-link-alt}}}}

\newbibmacro*{trailing-link}{\iffieldundef{eprint}{\iffieldundef{doi}{\iffieldundef{url}{}{\usebibmacro{url-link}}}{\usebibmacro{doi-link}}}{\iffieldequalstr{eprinttype}{arXiv}{\usebibmacro{arxiv-link}}{\iffieldequalstr{eprinttype}{bioRxiv}{\usebibmacro{biorxiv-link}}{\iffieldequalstr{eprinttype}{PsyArXiv}{\usebibmacro{psyarxiv-link}}{\iffieldundef{doi}{\iffieldundef{url}{}{\usebibmacro{url-link}}}{\usebibmacro{doi-link}}}}}}}

\renewbibmacro*{finentry}{\finentry
  \enspace\usebibmacro{trailing-link}}

\DeclareFieldFormat{doi}{}
\DeclareFieldFormat{url}{}
\renewbibmacro*{eprint}{\iffieldundef{eprint}{}{\printtext{\iffieldundef{eprinttype}{Preprint}{\thefield{eprinttype}\addspace preprint}}}}

\newcommand{\codeurl}{https://github.com/mrvnthss/symmetries-representational-geometry}
\newcommand{\codetext}{github.com/mrvnthss/symmetries-representational-geometry}
\newcommand{\codelink}{\href{\codeurl}{\codetext}}

\usepackage{pifont}

\definecolor{dark-teal}{HTML}{1f6f6f}
\definecolor{dark-red}{HTML}{a00000}

\newcommand{\cmark}{\textcolor{dark-teal}{\ding{51}}}
\newcommand{\xmark}{\textcolor{dark-red}{\ding{55}}}

\usepackage[
  capitalize,
  noabbrev,
  nameinlink
]{cleveref}

\theoremstyle{plain}
\newtheorem{theorem}{Theorem}[section]
\newtheorem{proposition}[theorem]{Proposition}
\newtheorem{lemma}[theorem]{Lemma}
\newtheorem{corollary}[theorem]{Corollary}

\theoremstyle{definition}
\newtheorem{definition}[theorem]{Definition}
\newtheorem{example}[theorem]{Example}

\theoremstyle{remark}
\newtheorem{remark}[theorem]{Remark}

\newcommand{\proofdeferred}[1]{\hfill\textnormal{(Proof in \cref{#1}.)}}

\DeclareMathOperator{\spanop}{span}

\newcommand{\cl}{\operatorname{cl}}
\newcommand{\diag}{\operatorname{diag}}
\newcommand{\tr}{\operatorname{tr}}

\newcommand{\cone}{\operatorname{cone}}
\newcommand{\Hol}{\mathcal{H}}

\providecommand{\given}{}
\DeclarePairedDelimiterXPP{\set}[1]{}{\{}{\}}{}{
  \renewcommand{\given}{\nonscript\;\delimsize\vert\nonscript\;\mathopen{}}
  #1
}

\DeclarePairedDelimiter{\abs}{\lvert}{\rvert}

\DeclarePairedDelimiter{\norm}{\lVert}{\rVert}

\DeclareMathOperator{\hardSigmoid}{hard\_sigmoid}
\DeclareMathOperator{\hardSilu}{hard\_silu}
\DeclareMathOperator{\reluSix}{relu6}
\DeclareMathOperator{\sigmoid}{sigmoid}
\DeclareMathOperator{\silu}{silu}
\DeclareMathOperator{\softplus}{softplus}

\hypersetup{pdftitle={Parameter Symmetries Determine Representational Geometry in Overparameterized Nonlinear Networks},
  pdfauthor={Marvin Theiss, Lukas Braun, Andrew M. Saxe, Erin Grant},
  pdfsubject={cs.LG},
}

\title{Parameter Symmetries Determine Representational Geometry in Overparameterized Nonlinear Networks
}

\newif\ifgridauthors\gridauthorstrue
\ifgridauthors

\makeatletter
\newlength{\authorcolwidth}
\newcount\author@gridcol

\def\author@nextcell{\ifnum\author@gridcol=\@ne
    \global\author@gridcol\tw@
    \expandafter\@firstoftwo
  \else
    \global\author@gridcol\@ne
    \expandafter\@secondoftwo
  \fi
  {&}{\\}}

\renewcommand{\@maketitle}{\vbox{\hsize\textwidth
    \linewidth\hsize
    \vskip 0.1in
    \@toptitlebar
    \centering
    {\LARGE\bf \@title\par}
    \@bottomtitlebar
    \if@anonymous
      \begin{tabular}[t]{c}\bf\rule{\z@}{24\p@}
        Anonymous Author(s) \\
        Affiliation \\
        Address \\
        \texttt{email}
      \end{tabular}\else
      \def\And{\end{tabular}\egroup
        \author@nextcell
        \vtop\bgroup\hsize\authorcolwidth\centering
          \begin{tabular}[t]{c}\bf\rule{\z@}{24\p@}\ignorespaces
      }\def\AND{\end{tabular}\egroup
        \global\author@gridcol\@ne
        \\\vtop\bgroup\hsize\authorcolwidth\centering
          \begin{tabular}[t]{c}\bf\rule{\z@}{24\p@}\ignorespaces
      }\global\author@gridcol\@ne
      \begin{tabular}[t]{@{}c@{}c@{}}
        \vtop\bgroup\hsize\authorcolwidth\centering
        \begin{tabular}[t]{c}\bf\rule{\z@}{24\p@}\@author
        \end{tabular}\egroup
      \end{tabular}\fi
    \vskip 0.3in \@minus 0.1in
  }}
\makeatother
 \fi

\author{Marvin Theiss\thanks{Work partially done while a project research intern at the Gatsby Computational Neuroscience Unit, UCL.
  }\\
  University of Tübingen \& IMPRS-IS\\
  \texttt{marvin.theiss@uni-tuebingen.de}
  \And
  Lukas Braun\\
  Allen Institute for Neural Dynamics\\
  \texttt{lukas.braun@alleninstitute.org}
  \AND
  Andrew M.~Saxe\\
  Gatsby Unit \& SWC, UCL\\
  \texttt{a.saxe@ucl.ac.uk}
  \And
  Erin Grant\\
  University of Alberta \& Amii\\
  \texttt{eringrant@ualberta.ca}
}

\begin{document}

\maketitle

\addtocontents{toc}{\protect\setcounter{tocdepth}{-1}}

\begin{abstract}
  Representations are routinely used across machine learning, psychology, and neuroscience to draw inferences about the computations of biological and artificial systems.
Such inferences presume that the geometry of these representations reflects the computation being performed.
For artificial neural networks, however, the extent to which the function a network computes constrains this geometry remains unclear.
One key obstacle is that these networks admit \emph{parameter symmetries}: changes in parameterization that preserve function exactly while reshaping representational geometry.
Here, we show that a broad class of parameter symmetries acts on representations through just three primitive feature transformations: addition, duplication, and scaling.
This feature-level characterization yields a closed-form decomposition of representational geometry into \emph{essential} and \emph{auxiliary} components, revealing how representational degeneracy can exist without and grows with overparameterization even when function is held fixed.
Finally, we establish conditions under which implementation-level selection rules resolve this degeneracy, yielding identifiable geometries in which features are weighted according to their contributions to the network's function.
Together, our results delineate when representational geometry can support inferences about computation, and when it cannot.
\looseness=-1
\space\end{abstract}

\section{Introduction}
\label{sec:introduction}

What does the geometry of the internal representations of neural networks reveal about the computations these networks implement?
Techniques for interpreting and intervening on the internal activities of artificial neural networks assume that local structure, such as weights or activities within a layer, reveals information about the end-to-end computation these networks perform \parencite{zeiler2014visualizing,olah2020zoom,saphra2024mechanistic,mueller2026quest}.
In neuroscience and cognitive science, representational similarity analysis and related techniques for comparing the internal representations of biological and artificial neural networks share this assumption that the local geometry of neural activity reveals the computations it instantiates, and thus serve as powerful tools for interpreting neural function in terms of neural activity \parencite{haxby2014decoding,kriegeskorte2008rsa,yamins2016goaldriven}.

However, artificial neural networks, like biological neural networks, are degenerate \parencite{fakhar2024downstream,albantakis2024brain}, as many distinct parameter configurations \parencite{kunin2021neural,entezari2022permutationinvariance,simsek2021geometry} and therefore distinct patterns of neural activity \parencite{hermann2020representations,flesch2022orthogonal,farrell2023lazy,chou2025feature,lampinen2026representation} can support the same function.
Further, \textcite{braun2025dissociation} analytically established a double dissociation between function and representation in two-layer linear networks: networks can compute identical functions while exhibiting distinct representational geometries, and conversely can exhibit identical representational geometries while computing distinct functions.
If such a double dissociation were general, it would call into question whether representational geometry can serve as a suitable basis for studying computation at all.
Determining whether this double dissociation persists in \emph{nonlinear} neural networks is therefore consequential both practically, because modern applications rely on nonlinear networks, and theoretically, because a given function admits a wider range of representations while the space of realizable functions is itself more complex.

Here, we study how strongly \emph{parameter symmetries}, transformations of a network's parameters that leave the realized function unchanged \parencite{nielsen1990algebraic,sussmann1992uniqueness,chen1993geometry,neyshabur2015pathsgd,zhao2026symmetry}, can dissociate representation from function in nonlinear networks.
By formalizing and refining parameter symmetries described in prior work \parencite{simsek2021geometry,martinelli2024expandcluster} into a precise taxonomy, and deriving their exact effects on representations, we prove that these symmetries act through compositions of just three primitive feature transformations: \emph{addition}, \emph{duplication}, and \emph{scaling}.
This characterization exposes substantial variability in representational geometry, demonstrating that functional equivalence need not imply representational alignment.

Yet this feature-transformation view also reveals a path to identifiability: suitable implementation-level constraints remove symmetry-induced degrees of freedom and restore unique representational geometries, yielding a nonlinear analog of the linear result \parencite{braun2025dissociation}.
These constraints select parameterizations that calibrate features by computational importance, thereby identifying privileged representational geometries even within the degenerate class of nonlinear neural networks.
Together, our results delineate when representational geometry can support inferences about computation and when function-preserving variation in implementation renders such inferences fundamentally ambiguous.
\looseness=-1

Concretely, our main contributions are as follows:
\begin{itemize}
  \item \textbf{Symmetry-equivalence}. We formalize and refine parameter symmetries in nonlinear one-hidden-layer networks described in prior work into a precise taxonomy, and derive a necessary and sufficient criterion for when two parameterizations are symmetry-equivalent (\cref{sec:functional-parameter-symmetries}).
  \item \textbf{Three feature primitives}. We show that parameter symmetries act on hidden representations through feature addition, duplication, and scaling, and characterize hidden-activation matrices within a symmetry orbit (\cref{sec:parameter-symmetries-act-through-three-feature-primitives}).
  \item \textbf{Representational degeneracy}. We decompose representational geometry into essential and auxiliary components, explaining why and characterizing how degeneracy grows with overparameterization as function stays fixed (\cref{sec:parameter-symmetries-dissociate-function-and-representation}).
  \item \textbf{Representational identifiability}. We establish conditions under which selection by minimum norm restores representational identifiability and show that the selected parameterizations weight features according to their contributions to the realized function (\cref{sec:identifiability-through-minimum-norm-selection}).
\end{itemize}

The code used to generate all figures and experimental results reported in this paper is publicly available at \codelink.

\section{Preliminaries and setting}
\label{sec:preliminaries-and-setting}

\paragraph{Nonlinear networks.}
We consider nonlinear, fully connected one-hidden-layer networks with incoming weights $\mathbf{W} \in \mathbb{R}^{N_{h} \times N_{i}}$, hidden biases $\mathbf{b} \in \mathbb{R}^{N_{h}}$, and readout weights $\mathbf{A} \in \mathbb{R}^{N_{o} \times N_{h}}$.
Let $\mathbf{w}_{j}^{\top}$, $b_{j}$, and $\mathbf{a}_{j}$ denote the $j$th row, entry, and column of $\mathbf{W}$, $\mathbf{b}$, and $\mathbf{A}$, respectively, and define the augmented vectors $\overline{\mathbf{w}}_{j}^{\top} \coloneqq (\mathbf{w}_{j}^{\top}, b_{j})$ and $\overline{\mathbf{x}}^{\top} \coloneqq (\mathbf{x}^{\top}, 1)$.
The network then computes
\begin{equation}
  \label{eq:network-function}
  f_{\boldsymbol{\theta}}(\mathbf{x})
  = \sum_{j=1}^{N_{h}}
  \mathbf{a}_{j} \, \sigma(\mathbf{w}_{j}^{\top} \mathbf{x} + b_{j})
  = \sum_{j=1}^{N_{h}}
  \mathbf{a}_{j} \, \sigma(\overline{\mathbf{w}}_{j}^{\top} \overline{\mathbf{x}}),
  \qquad
  \mathbf{x} \in \mathbb{R}^{N_{i}},
\end{equation}
where $\sigma \colon \mathbb{R} \to \mathbb{R}$ is the nonlinearity and $\boldsymbol{\theta} \coloneqq (\sigma; \mathbf{W}, \mathbf{b}, \mathbf{A})$ denotes the parameterization.
We assume that $\sigma$ is nonaffine, i.e., there exist no $m, c \in \mathbb{R}$ such that $\sigma(z) = mz + c$ for every $z \in \mathbb{R}$.

\paragraph{Hidden activations.}
Given inputs $\mathbf{x}^{\mu} \in \mathbb{R}^{N_{i}}$, $\mu = 1, \dots, P$, let $\mathbf{h}^{\mu} \in \mathbb{R}^{N_{h}}$ denote the corresponding hidden activation.
We collect inputs and hidden activations as
\begin{equation}
  \mathbf{X} \coloneqq [\mathbf{x}^{1}, \dots, \mathbf{x}^{P}] \in \mathbb{R}^{N_{i} \times P},
  \qquad
  \mathbf{H} \coloneqq [\mathbf{h}^{1}, \dots, \mathbf{h}^{P}] \in \mathbb{R}^{N_{h} \times P}.
\end{equation}
For later use, define the augmented matrix $\overline{\mathbf{X}} \coloneqq [\overline{\mathbf{x}}^{1}, \dots, \overline{\mathbf{x}}^{P}]$.
The hidden-activation matrix $\mathbf{H}$ admits two complementary interpretations: its $\mu$th column, $\mathbf{h}^{\mu}$, is the population response to input $\mathbf{x}^{\mu}$, while its $j$th row, $\mathbf{v}_{j}^{\top} \in \mathbb{R}^{1 \times P}$, records the activity of neuron $j$ across inputs.

\paragraph{Representational geometry.}
The hidden-activation matrix $\mathbf{H}$ gives rise to the uncentered Gram matrix $\mathbf{H}^{\top} \mathbf{H} \in \mathbb{R}^{P \times P}$, whose $(\mu, \nu)$th entry is the inner product between the population responses elicited by inputs $\mathbf{x}^{\mu}$ and $\mathbf{x}^{\nu}$ \parencite{edelman1998representation,kriegeskorte2013geometry}.
This matrix, known as the \emph{\ac{rsm}}, is a widely used descriptor of representational geometry and a natural object for studying variation therein, since many measures of representational similarity depend on $\mathbf{H}$ only through $\mathbf{H}^{\top}\mathbf{H}$ \parencite{kriegeskorte2008rsa,kornblith2019cka,williams2024equivalence}.
The two interpretations of $\mathbf{H}$ in the preceding paragraph yield complementary views of the \ac{rsm}: at the population level, it encodes pairwise similarities between input-evoked responses, which is the usual interpretation; at the neuron level, the \ac{rsm} decomposes into rank-one contributions from individual neurons:\footnote{Since neuron permutations leave $\mathbf{H}^{\top}\mathbf{H}$ unchanged, we understand all identities involving hidden-activation matrices $\mathbf{H}$, like that in \cref{eq:rsm-neuron-decomposition}, up to a permutation of the rows of $\mathbf{H}$, with all neuron-indexed quantities relabeled consistently.
}
\begin{equation}
  \label{eq:rsm-neuron-decomposition}
  \mathbf{H}^{\top} \mathbf{H}
  = \sum_{j=1}^{N_{h}} \mathbf{v}_{j} \mathbf{v}_{j}^{\top},
  \qquad
  \mathbf{v}_{j}^{\top} = \sigma(\overline{\mathbf{w}}_{j}^{\top} \overline{\mathbf{X}})
  \in \mathbb{R}^{1 \times P}.
\end{equation}
This neuron-wise decomposition is particularly suited to our analysis: by isolating each hidden neuron's contribution, it allows us to track how these contributions change under function-preserving transformations of the network's parameterization $\boldsymbol{\theta}$.
These transformations are precisely the parameter symmetries considered in the next section.

\section{Functional parameter symmetries in nonlinear networks}
\label{sec:functional-parameter-symmetries}

A given function $f$ may be realized by many distinct parameterizations, each inducing its own hidden-activation matrix and corresponding \ac{rsm}.
Here, we study the resulting variation in representational geometry by considering parameterizations related by \emph{parameter symmetries}, transformations of a network's parameters that leave a specified quantity unchanged.
Such symmetries admit several definitions, differing in both the quantity required to remain invariant and the set of inputs over which invariance is imposed \parencite{zhao2026symmetry}.
We focus throughout on the strictest notion, \emph{functional parameter symmetries}, which leave the network function in \cref{eq:network-function} unchanged over the entire input domain.
In this section, we introduce the symmetries we consider, define symmetry orbits and the related notions of irreducibility and overparameterization, and derive orbit invariants that provide a necessary and sufficient criterion for determining whether two parameterizations are symmetry-equivalent.

\subsection{A catalog of parameter symmetries}
\label{subsec:catalog-of-parameter-symmetries}

The parameter symmetries we consider take qualitatively distinct forms.
Some act on \emph{individual} neurons, including the positive-scaling symmetry of positively $1$-homogeneous activations such as ReLU \parencite{neyshabur2015pathsgd}, which rescales a neuron's incoming parameters and readout weights by reciprocal factors, and the sign-flip symmetry of odd activations such as tanh \parencite{nielsen1990algebraic}.
Others exploit redundancy among \emph{groups} of neurons, including \emph{zero groups}, in which neurons cancel one another through their readouts, and \emph{duplicate groups}, whose copies collectively reproduce the computation of a single neuron \parencite{simsek2021geometry}.
A third class of symmetries couples groups of neurons through algebraic properties of the activation $\sigma$, such as \emph{linear groups} that can arise when $\sigma$ decomposes into the sum of an even function and a linear function, as do ReLU and several other commonly used activations \parencite{martinelli2024expandcluster}.
Panel~A of \cref{fig:dissociation} illustrates all parameter symmetries we consider, building on and refining the taxonomies of \textcite{simsek2021geometry,martinelli2024expandcluster}, with detailed definitions deferred to \cref{app-sec:parameter-symmetries-in-overparameterized-networks}.

The algebraic conditions underlying this third class of parameter symmetries can be characterized by whether $\sigma$ is \emph{even-linear}, $\sigma(z) = e(z) + mz$ with $e$ even; \emph{constant-odd}, $\sigma(z) = c + o(z)$ with $o$ odd; or neither \parencite{simsek2021geometry,martinelli2024expandcluster}.
This classification includes positively $1$-homogeneous activations, such as ReLU, which form a subclass of even-linear activations (\cref{cor:positively-homogeneous-functions-are-even-linear}).
\Cref{tab:symmetries-of-jax-activation-functions} classifies other commonly used activations by the symmetries they admit, with detailed derivations in \cref{app-sec:symmetry-properties-of-common-activation-functions}.

\begin{figure*}[t]
  \centering
  \iflatexml
    \includegraphics[width=\textwidth]{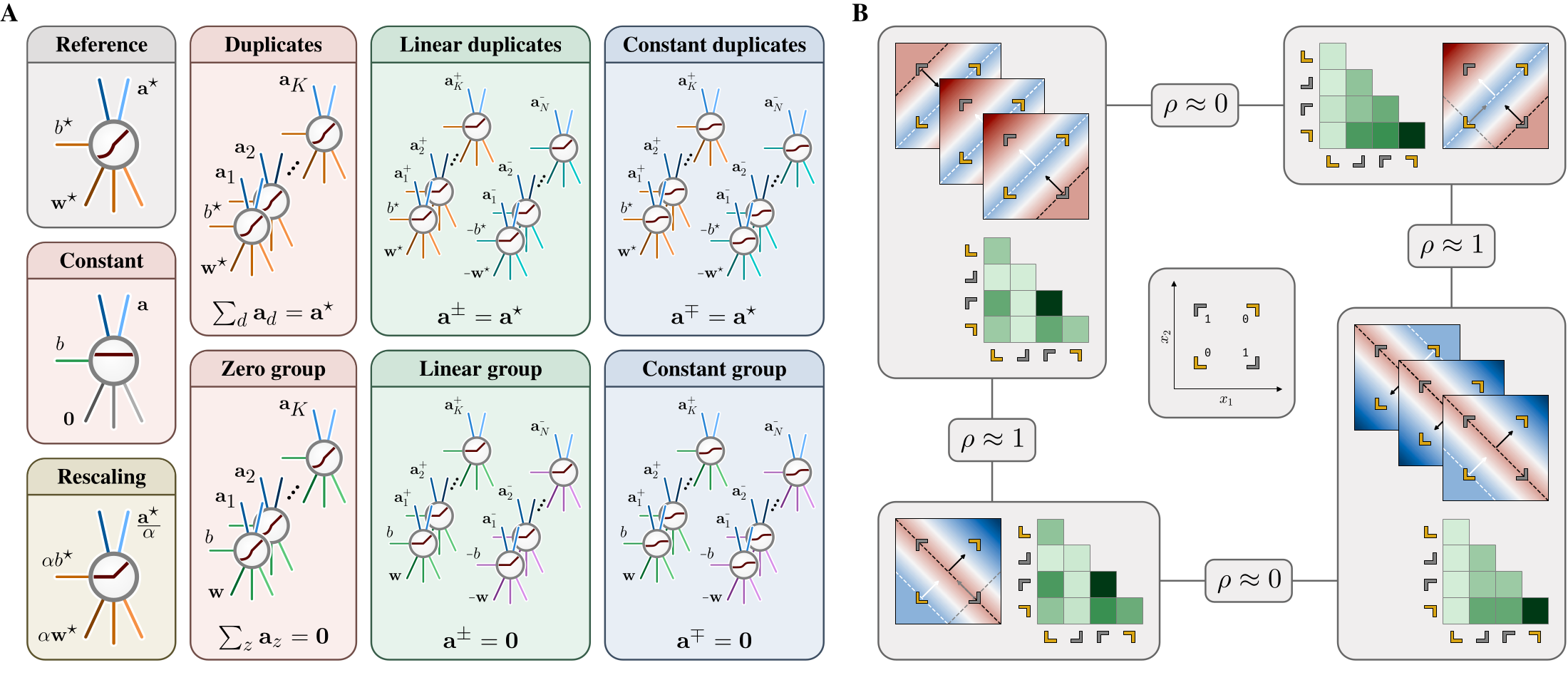}
  \else
    \includegraphics[width=\textwidth]{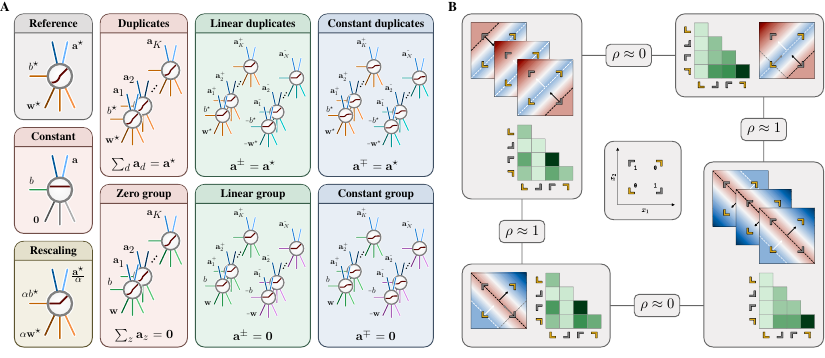}
  \fi
  \caption{\textbf{Parameter symmetries can doubly dissociate function and representation}.
    \textbf{(A)} Parameter symmetries let networks implement the same function via different parameterizations.
    They act on \emph{individual} neurons (e.g., positive scaling for ReLU networks), redistribute or cancel computation across \emph{groups} (duplicates, zero groups), or couple groups of neurons through algebraic structure in the activation $\sigma$ (linear or constant duplicates, linear or constant groups).
    \textbf{(B)} Gradient descent from random initializations yields six distinct two-neuron ReLU networks solving \acs{xor} (center), shown as input-space heatmaps with neuron decision boundaries overlaid.
    They form two families (top left and bottom right) that solve the same task using different functions, with \acp{rsm} of solutions from opposite families being nearly uncorrelated.
    Crucially, such representational variability persists even when the function is held fixed: overparameterized networks implementing \emph{exactly the same} function can have nearly uncorrelated \acp{rsm} (top and bottom).
    Conversely, networks implementing \emph{different} functions can have highly correlated \acp{rsm} (left and right), yielding a double dissociation.
  }
  \label{fig:dissociation}
\end{figure*}

\subsection{Symmetry orbits, irreducibility, and overparameterization}
\label{subsec:symmetry-orbits-irreducibility-and-overparameterization}

The functional parameter symmetries described in \cref{subsec:catalog-of-parameter-symmetries} naturally generate an equivalence relation on network parameterizations, including those of different widths.
We call two parameterizations $\boldsymbol{\theta}$ and $\boldsymbol{\xi}$ \emph{symmetry-equivalent}, and write $\boldsymbol{\theta} \sim \boldsymbol{\xi}$, if they are connected by a finite, function-preserving composition of the parameter symmetries enumerated in \cref{app-subsec:generic-reparameterization-symmetries,app-subsec:activation-independent-symmetries,app-subsec:activation-dependent-symmetries}.

\begin{definition}[Symmetry orbit]
  \label[definition]{def:symmetry-orbit}
  The \emph{symmetry orbit} $\mathcal{O}(\boldsymbol{\theta})$ of a parameterization $\boldsymbol{\theta}$ is the set $\mathcal{O}(\boldsymbol{\theta}) \coloneqq \set{\boldsymbol{\xi} \given \boldsymbol{\xi} \sim \boldsymbol{\theta}}$ consisting of all parameterizations ${\boldsymbol{\xi}}$ that are symmetry-equivalent to $\boldsymbol{\theta}$.
  The \emph{orbit at width $N_h$} is the subset $\mathcal{O}_{N_{h}}(\boldsymbol{\theta}) \coloneqq \set{\boldsymbol{\xi} \in \mathcal{O}(\boldsymbol{\theta}) \given \boldsymbol{\xi} \text{ has width } N_{h}} \subseteq \mathcal{O}(\boldsymbol{\theta})$.
\end{definition}

Minimum-width parameterizations within a symmetry orbit play a central role in our analysis.

\begin{definition}[Irreducibility and overparameterization]
  \label[definition]{def:irreducibility-and-overparameterization}
  A parameterization $\boldsymbol{\theta}$ is called \emph{irreducible} if no other element of its symmetry orbit $\mathcal{O}(\boldsymbol{\theta})$ has smaller width, and \emph{overparameterized} otherwise.
\end{definition}

Irreducible parameterizations within an orbit need \emph{not} be unique.
Neuron permutations generate finite families of irreducible parameterizations.
When admitted by the activation $\sigma$, sign flips and positive scaling generate additional irreducible parameterizations, with the latter producing continuous families thereof.
Perhaps counterintuitively, an orbit can contain distinct irreducible parameterizations that cannot be related by any combination of these symmetries alone (\cref{ex:opposite-irreducible-relu-parameterizations}).

\subsection{Orbit invariants and essential parameter classes}
\label{subsec:orbit-invariants-and-essential-parameter-classes}

The definition of symmetry-equivalence just introduced is not directly operational.
Determining whether two parameterizations are equivalent requires either explicitly constructing a sequence of symmetries that transforms one into the other or proving that no such sequence exists.
Here, we establish a necessary and sufficient criterion characterizing symmetry-equivalence for positively homogeneous activations of degree~$1$, with all remaining activation classes treated in \cref{app-subsec:parameter-classes-orbit-invariants-and-symmetry-equivalence}.

By \cref{cor:positively-homogeneous-functions-are-even-linear}, every nonaffine positively $1$-homogeneous activation $\sigma$ can be expressed as $\sigma(z) = \delta\abs{z} + mz$ with $\delta \neq 0$ and $m \in \mathbb{R}$.
Hence, the contribution of a nonconstant neuron with parameters $(\overline{\mathbf{w}}, \mathbf{a})$ can be decomposed into nonlinear and linear components
\begin{equation}
  \norm{\overline{\mathbf{w}}} \mathbf{a}
  \left(
    \delta
    \frac{\abs{z}}{\norm{\overline{\mathbf{w}}}}
  \right)
  \qquad\text{and}\qquad
  \mathbf{a} mz,
\end{equation}
respectively, where $z \coloneqq \overline{\mathbf{w}}^{\top} \overline{\mathbf{x}}$.
The nonlinear part depends on the neuron's parameterization only through the activation hyperplane defined by $\overline{\mathbf{w}}^{\top} \overline{\mathbf{x}} = 0$ and the rescaled readout $\norm{\overline{\mathbf{w}}} \mathbf{a}$.
We therefore identify incoming-parameter vectors up to nonzero scaling, defining the \emph{parameter class} of $\overline{\mathbf{w}}$ as $[\overline{\mathbf{w}}] \coloneqq \set{\alpha \overline{\mathbf{w}} \given \alpha \in \mathbb{R}_{\neq 0}}$, and we collect the classes represented by $\boldsymbol{\theta}$ in $\mathcal{Q}(\boldsymbol{\theta}) \coloneqq \set{[\overline{\mathbf{w}}_{j}] \given \mathbf{w}_{j} \neq \mathbf{0}}$.
Let $\mathcal{J}_q \coloneqq \set{j \given \overline{\mathbf{w}}_j \in q}$ index the neurons of class $q \in \mathcal{Q}(\boldsymbol{\theta})$, and let $\mathcal{J}_{0} \coloneqq \set{j \given \mathbf{w}_{j} = \mathbf{0}}$ index all constant neurons.
For each class $q = [\overline{\mathbf{w}}] \in \mathcal{Q}(\boldsymbol{\theta})$, we define the \emph{normalized nonlinear feature} $\phi_{q}$ shared by all neurons in $q$ and the corresponding \emph{aggregate rescaled readout} $\boldsymbol{\beta}_{q}(\boldsymbol{\theta})$ as
\begin{equation}
  \phi_{q}(\mathbf{x})
  \coloneqq
  \delta
  \frac{\abs{z}}{\norm{\overline{\mathbf{w}}}},
  \qquad
  \boldsymbol{\beta}_{q}(\boldsymbol{\theta})
  \coloneqq
  \sum_{j \in \mathcal{J}_{q}}
  \norm{\overline{\mathbf{w}}_{j}} \, \mathbf{a}_{j},
\end{equation}
respectively.
Finally, we collect the linear residual $\mathbf{a}_{j} m z_{j}$ of nonconstant neurons and the total contribution $\mathbf{a}_{j} \sigma(b_{j})$ of constant neurons into the map $\mathbf{r}_{\boldsymbol{\theta}} \colon \mathbb{R}^{N_{i}} \to \mathbb{R}^{N_{o}}$ given by
\begin{equation}
  \mathbf{r}_{\boldsymbol{\theta}}(\mathbf{x})
  \coloneqq
  m
  \sum_{j \notin \mathcal{J}_{0}}
  \mathbf{a}_{j} z_{j}
  + \sum_{j \in \mathcal{J}_{0}}
  \mathbf{a}_{j} \sigma(b_{j}),
  \qquad
  z_{j}
  \coloneqq
  \overline{\mathbf{w}}_{j}^{\top} \overline{\mathbf{x}}.
\end{equation}
With this notation in place, the function $f_{\boldsymbol{\theta}} \colon \mathbb{R}^{N_{i}} \to \mathbb{R}^{N_{o}}$ realized by $\boldsymbol{\theta}$ can be expressed as
\begin{equation}
  \label{eq:function-decomposition-by-parameter-classes}
  f_{\boldsymbol{\theta}}(\mathbf{x})
  =
  \sum_{q\in\mathcal{Q}(\boldsymbol{\theta})}
  \boldsymbol{\beta}_{q}(\boldsymbol{\theta})
  \phi_{q}(\mathbf{x})
  +
  \mathbf{r}_{\boldsymbol{\theta}}(\mathbf{x}).
\end{equation}

\begin{proposition}[Characterization of symmetry-equivalence]
  \label[proposition]{prop:characterization-of-symmetry-equivalence}
  Two parameterizations $\boldsymbol{\theta}$ and $\boldsymbol{\xi}$, possibly of different widths, with the same nonaffine positively homogeneous activation of degree~$1$, are symmetry-equivalent if and only if
\begin{equation}
    \boldsymbol{\beta}_{q}(\boldsymbol{\theta})
    =
    \boldsymbol{\beta}_{q}(\boldsymbol{\xi})
    \quad
    \text{for every }
    q \in \mathcal{Q}(\boldsymbol{\theta}) \cup \mathcal{Q}(\boldsymbol{\xi}),
    \qquad
    \mathbf{r}_{\boldsymbol{\theta}}
    =
    \mathbf{r}_{\boldsymbol{\xi}},
  \end{equation}
with $\boldsymbol{\beta}_{q}(\cdot)$ understood to be zero for an absent class.
  \proofdeferred{app-subsubsec:characterization-of-symmetry-equivalence}
\end{proposition}

Within a fixed orbit $\mathcal{O}$, we suppress the parameterization argument $\boldsymbol{\theta}$ and call a class $q$ \emph{essential} if $\boldsymbol{\beta}_{q} \neq \mathbf{0}$ and \emph{nonessential} otherwise.
We write $\mathcal{E} \coloneqq \set{q \given \boldsymbol{\beta}_{q} \neq \mathbf{0}}$ for the set of essential classes.

\section{Parameter symmetries act through three feature primitives}
\label{sec:parameter-symmetries-act-through-three-feature-primitives}

While two symmetry-equivalent parameterizations realize the same function, they can induce drastically different hidden activations.
Overparameterization amplifies this variability, as additional hidden neurons create increasingly many ways to compose parameter symmetries and distribute computation across neurons without ever changing the realized function.
This raises the question of whether the corresponding hidden-activation matrices admit a tractable description.
The key observation of this section is that they do: at the level of hidden activations, arbitrary sequences of parameter symmetries reduce to compositions of just three primitive feature transformations.
Together with the orbit invariants of \cref{subsec:orbit-invariants-and-essential-parameter-classes}, this yields an explicit characterization of hidden-activation matrices within a symmetry orbit.
Up to scaling and duplication, each essential parameter class remains represented by at least one feature throughout the orbit, while nonessential classes may contribute additional features.
For several important activation classes, persistence throughout the orbit extends from essential parameter classes to the individual features shared by its irreducible parameterizations.

\subsection{Feature-level action of parameter symmetries}
\label{subsec:feature-level-action-of-parameter-symmetries}

The parameter symmetries introduced in \cref{sec:functional-parameter-symmetries} take several distinct forms in parameter space: they may introduce canceling groups of neurons, duplicate existing neurons, or exploit algebraic structure of the activation through sign-flipped incoming parameters.
At the level of hidden activations, however, these distinct parameter-space mechanisms reduce to compositions of just three primitive feature transformations, which we introduce next.

Feature \emph{addition} appends features computed by additional neurons.
Collecting these features in $\mathbf{U} \in \mathbb{R}^{K \times P}$, we define $\mathcal{A}_{\mathbf{U}}(\mathbf{H}) \coloneqq [\mathbf{H}^{\top}, \mathbf{U}^{\top}]^{\top}$.
Feature \emph{duplication} repeats existing features.
Given a duplication pattern $\boldsymbol{\nu} \in \mathbb{N}_{>0}^{N_{h}}$, let $\mathbf{D}_{\boldsymbol{\nu}}$ denote the duplication matrix that repeats row $i$ exactly $\nu_{i}$ times (see \cref{def:duplication-matrix}), and define $\mathcal{D}_{\boldsymbol{\nu}}(\mathbf{H}) \coloneqq \mathbf{D}_{\boldsymbol{\nu}} \mathbf{H}$.
Feature \emph{scaling} multiplies each feature by a nonzero scalar.
For $\boldsymbol{\alpha} \in \mathbb{R}_{\neq 0}^{N_{h}}$, we define $\mathcal{S}_{\boldsymbol{\alpha}}(\mathbf{H}) \coloneqq \diag(\boldsymbol{\alpha}) \mathbf{H}$.

The corresponding inverse operations remove appended features, discard all but one copy of duplicated feature rows, and apply reciprocal scaling factors, respectively.
Together, these operations provide an exhaustive characterization of how parameter symmetries act on hidden activations.

\begin{proposition}[Feature-level characterization of parameter symmetries]
  \label[proposition]{prop:feature-level-characterization-of-parameter-symmetries}
  Up to row permutations, positive-scaling and sign-flip symmetries act through feature scaling, while duplicate-neuron groups act through feature duplication.
  Zero-neuron groups and constant neurons act through feature addition followed, where necessary, by duplication.
  Linear-neuron and constant-neuron groups add two oppositely oriented features and duplicate them according to the sizes of their aligned and opposite subgroups.
  Linear-duplicate and constant-duplicate groups act analogously, except that one orientation is already supplied by the reference neuron.
  Reversing any of these parameter symmetries acts through the corresponding inverse primitives.
  \proofdeferred{app-subsec:feature-level-characterization-of-parameter-symmetries}
\end{proposition}

\begin{figure*}[t]
  \centering
  \iflatexml
    \includegraphics[width=\textwidth]{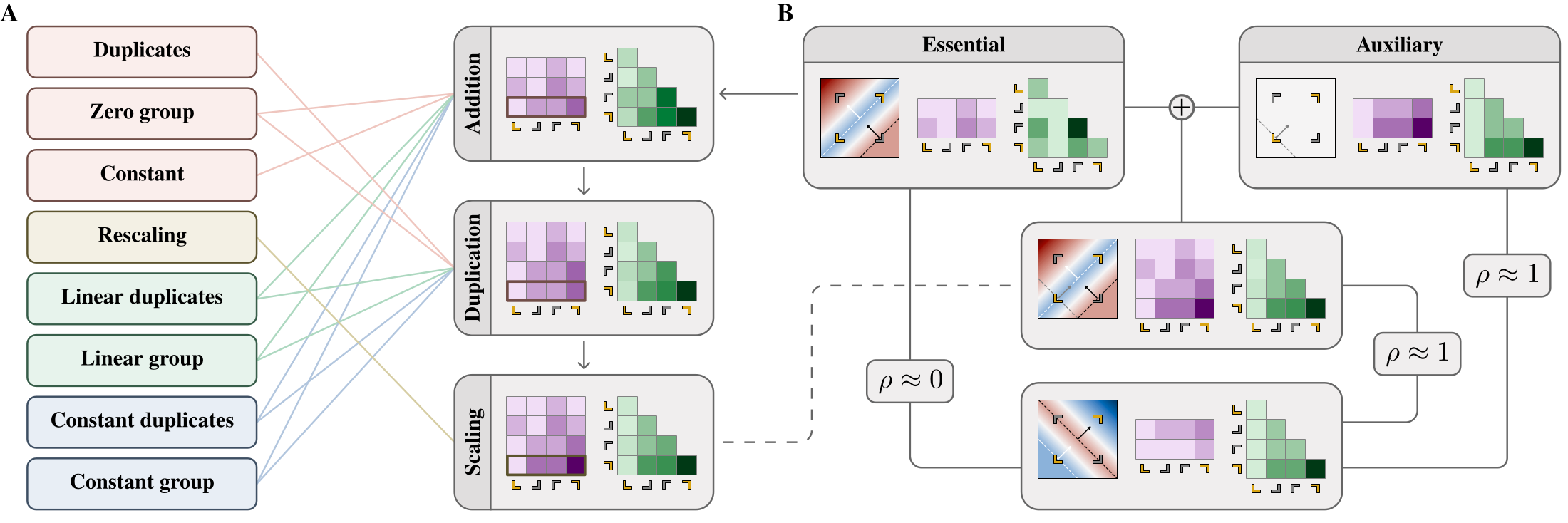}
  \else
    \includegraphics[width=\textwidth]{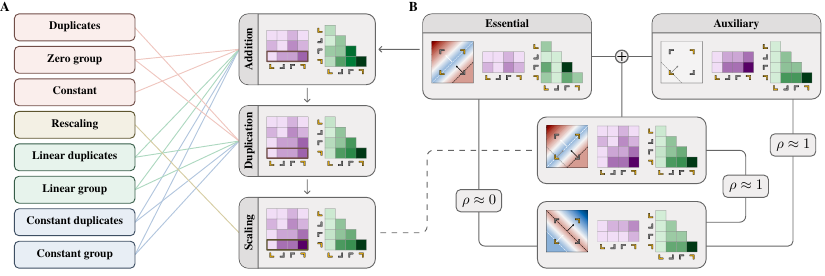}
  \fi
  \caption{\textbf{Primitive feature transformations unify parameter symmetries and separate essential from auxiliary geometry}.
    \textbf{(A)} Parameter symmetries decompose into three primitive feature transformations and their inverses: \emph{addition}, \emph{duplication}, and \emph{scaling}.
    Shown are their effects on hidden activations and \acp{rsm} when adding a zero-readout neuron to a solution in Panel~B of \cref{fig:dissociation}, duplicating it, and rescaling both copies.
    \textbf{(B)} Decomposing the full \ac{rsm} of the resulting network (center) into essential and auxiliary contributions reveals the source of the dissociation in \cref{fig:dissociation}.
    Whereas the essential component is nearly uncorrelated with the \ac{rsm} of an opposite-family solution, the auxiliary component is nearly perfectly correlated with it, driving the full \ac{rsm} toward the reference geometry without changing the function being computed.
  }
  \label{fig:primitives}
\end{figure*}

\subsection{Structure of hidden-activation matrices within a symmetry orbit}
\label{subsec:structure-of-hidden-activation-matrices-within-a-symmetry-orbit}

Having characterized individual parameter symmetries at the feature level, we now derive a structural form satisfied by hidden-activation matrices across an entire symmetry orbit.
As in \cref{subsec:orbit-invariants-and-essential-parameter-classes}, we focus on positively $1$-homogeneous activations in the main text and defer the remaining activation classes to \cref{app-subsec:hidden-activation-matrices-within-a-symmetry-orbit}.
Fix a symmetry orbit, and let $\mathcal{E}$ denote its essential parameter classes.
For each essential class $q \in \mathcal{E}$, choose a unit-norm representative $\overline{\mathbf{w}}_q \in q$, and let $\mathbf{F}_{\mathcal{E}} \in \mathbb{R}^{2\abs{\mathcal{E}} \times P}$ collect the two features
\begin{equation}
  \label{eq:essential-class-orientation-features}
  (\mathbf{f}_{q}^{+})^{\top}
  \coloneqq
  \sigma(\overline{\mathbf{w}}_{q}^{\top} \overline{\mathbf{X}}),
  \qquad
  (\mathbf{f}_{q}^{-})^{\top}
  \coloneqq
  \sigma(-\overline{\mathbf{w}}_{q}^{\top} \overline{\mathbf{X}}).
\end{equation}

\begin{proposition}[Hidden-activation matrices within a symmetry orbit]
  \label[proposition]{prop:hidden-activation-matrices-within-a-symmetry-orbit}
  Fix a symmetry orbit $\mathcal{O}$ with a positively $1$-homogeneous activation and essential parameter classes $\mathcal{E}$, and let $\mathbf{F}_{\mathcal{E}}$ be defined as above.
  Every parameterization $\boldsymbol{\theta}$ of width $N_h$ in $\mathcal{O}$ has a hidden-activation matrix of the form
\begin{equation}
    \label{eq:hidden-activation-matrices-within-a-symmetry-orbit}
    \mathbf{H}
    =
    \diag(\boldsymbol{\alpha})
    \mathbf{D}_{\boldsymbol{\nu}}
    \begin{bmatrix}
      \mathbf{F}_{\mathcal{E},\mathcal{I}} \\
      \mathbf{U}
    \end{bmatrix},
    \qquad
    \boldsymbol{\alpha} \in \mathbb{R}_{>0}^{N_{h}},
    \quad
    \boldsymbol{\nu} \in \mathbb{N}_{>0}^{\abs{\mathcal{I}}+K},
  \end{equation}
for some $K \geq 0$ and $\mathbf{U} \in \mathbb{R}^{K \times P}$, with $\sum_{\ell} \nu_{\ell} = N_{h}$.
  Here, $\mathbf{F}_{\mathcal{E},\mathcal{I}}$ is the submatrix of $\mathbf{F}_{\mathcal{E}}$ indexed by $\mathcal{I} \subseteq \set{1, \dots, 2\abs{\mathcal{E}}}$.
  The index set $\mathcal{I}$ selects the rows of $\mathbf{F}_{\mathcal{E}}$ corresponding to the features represented in $\boldsymbol{\theta}$, with at least one of $(\mathbf{f}_{q}^{+})^{\top}$ and $(\mathbf{f}_{q}^{-})^{\top}$ selected for every $q \in \mathcal{E}$.
  The matrix $\mathbf{U}$ collects pairwise distinct features from nonessential classes and constant neurons.
  \proofdeferred{app-subsec:hidden-activation-matrices-within-a-symmetry-orbit}
\end{proposition}

Every parameterization in the same symmetry orbit must therefore represent each essential parameter class by at least one hidden neuron.
An even stronger result holds for the following three activation classes: activations that are neither even-linear nor constant-odd (e.g., those studied in \textcite{simsek2021geometry}), even activations (e.g., Gaussian), and odd activations (e.g., tanh).
For these classes, persistence extends from essential parameter classes to the individual features themselves.
That is, every feature computed by an irreducible parameterization persists throughout the orbit up to nonzero scaling.

\begin{corollary}
  \label[corollary]{cor:persistence-of-irreducible-features}
  Suppose that $\sigma$ is neither even-linear nor constant-odd, or is purely even or purely odd.
  Let $\boldsymbol{\theta}^{\star}$ be any irreducible parameterization of width $N_{h}^{\star}$ with hidden-activation matrix $\mathbf{H}^{\star}$.
  Then, every parameterization $\boldsymbol{\theta} \sim \boldsymbol{\theta}^{\star}$ of width $N_{h}$ has a hidden-activation matrix of the form
\begin{equation}
    \label{eq:persistence-of-irreducible-features}
    \mathbf{H}
    =
    \diag(\boldsymbol{\alpha})
    \mathbf{D}_{\boldsymbol{\nu}}
    \begin{bmatrix}
      \mathbf{H}^{\star} \\
      \mathbf{U}
    \end{bmatrix},
    \qquad
    \boldsymbol{\alpha} \in \mathbb{R}_{\neq 0}^{N_{h}},
    \quad
    \boldsymbol{\nu} \in \mathbb{N}_{>0}^{N_{h}^\star+K},
  \end{equation}
for some $K \geq 0$ and $\mathbf{U} \in \mathbb{R}^{K \times P}$, with $\sum_{\ell} \nu_{\ell} = N_{h}$.
  The matrix $\mathbf{U}$ can be chosen to have pairwise distinct rows.
  \proofdeferred{app-subsec:persistence-of-irreducible-features}
\end{corollary}

\section{Parameter symmetries dissociate function and representation}
\label{sec:parameter-symmetries-dissociate-function-and-representation}

The feature-level characterization in \cref{sec:parameter-symmetries-act-through-three-feature-primitives} lets us translate variability in individual features across a symmetry orbit into variability of the representational geometry these features collectively induce.
While the function realized by a parameterization remains unchanged, feature addition can introduce new features into its hidden representation, while duplication and scaling can alter the relative contributions of existing ones.
These degrees of freedom are especially consequential because similarity to external reference geometries is often used to draw inferences about the computations supported by a representation, for example when comparing network representations with neural data.
We therefore characterize how the range of similarities to a fixed reference geometry attainable within a symmetry orbit expands with overparameterization, eventually reaching bounds that are independent of the realized function.

\subsection{Essential and auxiliary contributions to representational geometry}
\label{subsec:essential-and-auxiliary-contributions-to-representational-geometry}

Substituting the factorization of the hidden-activation matrix $\mathbf{H}$ established in \cref{subsec:structure-of-hidden-activation-matrices-within-a-symmetry-orbit} into $\mathbf{H}^{\top} \mathbf{H}$ yields a decomposition of representational geometry into essential and auxiliary contributions.

\begin{proposition}[RSM decomposition]
  \label[proposition]{prop:rsm-decomposition-and-reweighting}
  For a parameterization $\boldsymbol{\theta}$ with positively $1$-homogeneous activation, let $\mathbf{f}_{\ell}^{\top}$ and $\mathbf{u}_{k}^{\top}$ denote the rows of $\mathbf{F}_{\mathcal{E},\mathcal{I}}$ and $\mathbf{U}$, respectively, in the factorization of $\mathbf{H}$ in \cref{prop:hidden-activation-matrices-within-a-symmetry-orbit}.
  Then,
\begin{equation}
    \label{eq:essential-auxiliary-rsm-decomposition}
    \mathbf{H}^{\top}\mathbf{H}
    =
    \underbrace{\sum_{\ell=1}^{\abs{\mathcal{I}}}
    \gamma_{\ell} \mathbf{f}_{\ell} \mathbf{f}_{\ell}^{\top}}_{\text{\emph{essential}}}
    +
    \underbrace{\sum_{k=1}^{K}
    \gamma_{\abs{\mathcal{I}}+k} \mathbf{u}_{k} \mathbf{u}_{k}^{\top}}_{\text{\emph{auxiliary}}},
    \qquad
    \boldsymbol{\gamma}
    \coloneqq
    \mathbf{D}_{\boldsymbol{\nu}}^{\top} \boldsymbol{\alpha}^{2}
    \in \mathbb{R}_{>0}^{\abs{\mathcal{I}}+K},
  \end{equation}
where $\boldsymbol{\alpha}^{2}$ is the Hadamard square of $\boldsymbol{\alpha}$.
  For any such realized factorization, every strictly positive $\boldsymbol{\gamma}$ is realizable by a symmetry-equivalent parameterization of the same width, with the feature matrices $\mathbf{F}_{\mathcal{E},\mathcal{I}}$ and $\mathbf{U}$ and duplication pattern $\boldsymbol{\nu}$ fixed.
  \proofdeferred{app-subsec:rsm-decomposition-and-positive-reweighting}
\end{proposition}

This decomposition reveals two complementary mechanisms by which parameter symmetries reshape representational geometry.
First, feature addition can introduce new auxiliary features and hence new rank-one contributions to the geometry $\mathbf{H}^{\top}\mathbf{H}$, provided that the class aggregates $\boldsymbol{\beta}_{q}$ and global residual $\mathbf{r}$ are preserved (\cref{subsec:orbit-invariants-and-essential-parameter-classes}).
Second, duplication and scaling reweight existing contributions.
Duplication does so through integer feature multiplicities and is available even without the positive-scaling symmetry, whereas positive scaling permits \emph{arbitrary} strictly positive weights on the rank-one contributions $\mathbf{f}_{\ell} \mathbf{f}_{\ell}^{\top}$ and $\mathbf{u}_{k} \mathbf{u}_{k}^{\top}$.
Additional width can therefore accommodate more auxiliary features and a broader range of multiplicities, potentially expanding the set of compatible representational geometries.
We next show how this width-dependent freedom in representational geometry translates into an expanding range of similarity scores to a fixed reference geometry.

\subsection{Variability in representational similarity}
\label{subsec:variability-in-representational-similarity}

Fix a symmetry orbit $\mathcal{O}$ and denote the \ac{rsm} induced by a parameterization $\boldsymbol{\theta} \in \mathcal{O}$ on the inputs $\mathbf{X}$ by $\mathbf{M}_{\boldsymbol{\theta}} \coloneqq \mathbf{H}^{\top}\mathbf{H}$.
Let $\mathbf{N} \in \mathbb{R}^{P \times P}$ denote a reference \ac{rsm} computed from responses to the same inputs, such as population responses measured by fMRI, EEG, or neural recordings.
We assume throughout that $\mathbf{N}$ has nonconstant strict upper-triangular entries.
Following common practice \parencite{kriegeskorte2008rsa}, we quantify alignment with this reference geometry by the Pearson correlation $\rho(\mathbf{M}_{\boldsymbol{\theta}},\mathbf{N})$ between the strict upper-triangular entries of $\mathbf{M}_{\boldsymbol{\theta}}$ and $\mathbf{N}$.

Let $\mathcal{R}$ denote the realizable \acp{rsm} within $\mathcal{O}$ whose strict upper-triangular entries are nonconstant:
\begin{equation}
  \label{eq:realizable-rsms}
  \mathcal{R}
  \coloneqq
  \set{
    \mathbf{M}_{\boldsymbol{\theta}}
    \given
    \boldsymbol{\theta} \in \mathcal{O},
    \ \mathbf{M}_{\boldsymbol{\theta}} \text{ has nonconstant strict upper-triangular entries}
  }.
\end{equation}
For each width $N_{h}$, let $\mathcal{R}_{N_{h}} \subseteq \mathcal{R}$ denote the subset of \acp{rsm} realized in $\mathcal{O}_{N_{h}}$.
The attainable similarity scores across all widths and at width $N_{h}$, respectively, are the sets
\begin{equation}
  \label{eq:attainable-similarity-scores}
  \mathcal{S}(\mathbf{N})
  \coloneqq
  \set{
    \rho(\mathbf{M}_{\boldsymbol{\theta}}, \mathbf{N})
    \given
    \mathbf{M}_{\boldsymbol{\theta}} \in \mathcal{R}
  },
  \qquad
  \mathcal{S}_{N_{h}}(\mathbf{N})
  \coloneqq
  \set{
    \rho(\mathbf{M}_{\boldsymbol{\theta}}, \mathbf{N})
    \given
    \mathbf{M}_{\boldsymbol{\theta}} \in \mathcal{R}_{N_{h}}
  },
\end{equation}
which we assume to be nonempty.
The spreads
\begin{equation}
  \Delta_{N_{h}}(\mathbf{N}) \coloneqq \sup \mathcal{S}_{N_{h}}(\mathbf{N}) - \inf \mathcal{S}_{N_{h}}(\mathbf{N}),
  \qquad
  \Delta_{\infty}(\mathbf{N}) \coloneqq \sup \mathcal{S}(\mathbf{N}) - \inf \mathcal{S}(\mathbf{N})
\end{equation}
quantify the variability in attainable similarity scores at width $N_{h}$ and across widths, respectively.

\begin{proposition}[Strict growth until saturation]
  \label[proposition]{prop:strict-growth-until-saturation}
  For a symmetry orbit with positively $1$-homogeneous activation, the realizable \ac{rsm} sets and attainable similarity sets are nested with increasing width:
\begin{equation}
    \label{eq:width-dependent-nesting}
    \mathcal{R}_{N_{h}}
    \subseteq
    \mathcal{R}_{N_{h}+1},
    \qquad
    \mathcal{S}_{N_{h}}(\mathbf{N})
    \subseteq
    \mathcal{S}_{N_{h}+1}(\mathbf{N}).
  \end{equation}
Moreover, $\Delta_{N_{h}+1}(\mathbf{N}) > \Delta_{N_{h}}(\mathbf{N})$ if and only if $\Delta_{N_{h}}(\mathbf{N}) < \Delta_{\infty}(\mathbf{N})$.
  \proofdeferred{app-subsec:strict-growth-until-saturation}
\end{proposition}

Hence, as overparameterization increases, the spread of similarity scores compatible with the same realized function increases strictly until its maximum within the symmetry orbit is reached.
More strikingly, this limit is \emph{independent} of the function being computed: the limiting similarity bounds depend only on the activation $\sigma$, probe inputs $\mathbf{X}$, and reference geometry $\mathbf{N}$.

\begin{proposition}[Function-independent similarity limits]
  \label[proposition]{prop:function-independent-similarity-limits}
  For fixed positively $1$-homogeneous activation $\sigma$, probe inputs $\mathbf{X}$, and reference geometry $\mathbf{N}$, the limiting similarity bounds $\inf \mathcal{S}(\mathbf{N})$ and $\sup \mathcal{S}(\mathbf{N})$ are identical across all symmetry orbits $\mathcal{O}$, regardless of the functions realized by their parameterizations.
  Moreover, within each orbit, these common bounds are recovered at sufficiently large finite width.
  That is, for all sufficiently large $N_{h}$,
\begin{equation}
    \inf \mathcal{S}_{N_{h}}(\mathbf{N})
    =
    \inf \mathcal{S}(\mathbf{N}),
    \qquad
    \sup \mathcal{S}_{N_{h}}(\mathbf{N})
    =
    \sup \mathcal{S}(\mathbf{N}).
  \end{equation}
Consequently, every orbit reaches the same limiting spread $\Delta_{\infty}(\mathbf{N})$ at finite width.
  \proofdeferred{app-subsec:function-independent-similarity-limits}
\end{proposition}

Networks computing entirely different functions can thus approach the same bounds on similarity to a fixed reference geometry through function-preserving reparameterizations.

When $\operatorname{rank}(\overline{\mathbf{X}}) = P$, the attainable similarity range reaches its maximal possible extent.
That is, for every reference geometry $\mathbf{N}$, correlations arbitrarily close to both $1$ and $-1$ are attainable within every symmetry orbit (\cref{prop:alignment-with-affinely-independent-inputs}).
For $P \leq N_{i} + 1$, this rank condition holds generically,
making this case particularly relevant for networks probed with high-dimensional inputs such as images.
Moreover, these limiting bounds are reached at every width $N_{h} \geq N_{h}^{\star} + \frac{P(P-1)}{2} - 1$, where $N_{h}^{\star}$ denotes the minimum width within the symmetry orbit.

Another important case, which does not require this rank
condition, arises when the reference geometry $\mathbf{N}$ is itself the \ac{rsm} of a network with the same activation $\sigma$ evaluated on the same inputs $\mathbf{X}$.
In this case, the common upper similarity limit is $1$, even if the reference network computes an entirely different function (\cref{prop:alignment-with-same-activation-reference-geometries}).
The features of the reference network can be introduced through zero-neuron groups, and their rank-one contributions amplified through duplication and scaling until they dominate the \ac{rsm}, all while preserving the realized function.
The construction in \cref{fig:primitives} illustrates this mechanism with a single feature from an opposite-family solution in \cref{fig:dissociation}, whose rank-one contribution is nearly perfectly correlated with the reference \ac{rsm}.

\section{Identifiability through minimum-norm selection}
\label{sec:identifiability-through-minimum-norm-selection}

The preceding results show that function alone leaves substantial freedom in representational geometry.
Within a single symmetry orbit, the range of attainable similarity scores grows systematically with width and eventually reaches function-independent limits.
When the probe inputs are affinely independent, the same realized function is compatible with representations ranging from near-perfect alignment to near-perfect anti-alignment with any reference geometry.
Together, these results pose a fundamental obstacle to representational identifiability.
Here, however, we show that this degeneracy can be resolved by accounting for how a parameterization is selected within a symmetry orbit.
Specifically, we derive conditions under which minimum-norm selection yields a unique representational geometry, independently of width.

Following the minimum-norm selection rules studied by \textcite{braun2025dissociation} in linear networks, we consider two criteria for selecting among symmetry-equivalent parameterizations.
Within a fixed-width orbit $\mathcal{O}_{N_{h}}$, we call a parameterization $\boldsymbol{\theta} = (\sigma; \mathbf{W}, \mathbf{b}, \mathbf{A})$ a \emph{\ac{mwnp}} or \emph{\ac{mrnp}} if it minimizes, respectively,
\begin{equation}
  \label{eq:minimum-norm-objectives}
  \Omega_{W}(\boldsymbol{\theta})
  \coloneqq
  \norm{\mathbf{W}}_{F}^{2}
  + \norm{\mathbf{b}}^{2}
  + \norm{\mathbf{A}}_{F}^{2},
  \qquad
  \Omega_{H}(\boldsymbol{\theta})
  \coloneqq
  \norm{\mathbf{H}}_{F}^{2}
  + \norm{\mathbf{A}}_{F}^{2}.
\end{equation}
Both criteria balance the cost of computing features against the cost of reading them out, albeit in slightly different ways.
As in previous sections, we focus on positively $1$-homogeneous activations and, here, on \acp{mwnp}.
The corresponding \ac{mrnp} analysis is deferred to \cref{app-sec:identifiability-through-minimum-norm-selection}.

Fix a symmetry orbit with positively $1$-homogeneous activation $\sigma(z) = \delta\abs{z} + mz$, and let $\mathbf{f}_{q}^{+}$ and $\mathbf{f}_{q}^{-}$ denote the feature vectors generated by the two orientations of the unit-norm representative $\overline{\mathbf{w}}_{q}$ of an essential parameter class $q \in \mathcal{E}$, as introduced in \cref{eq:essential-class-orientation-features}.
For each $q \in \mathcal{E}$, the orbit invariant $\boldsymbol{\beta}_{q}$ fixes the aggregate rescaled readout contributed by that class, while leaving open how this aggregate is distributed among neurons of its two opposite orientations.
For each essential class $q \in \mathcal{E}$, define the orientation index sets
\begin{equation}
  \mathcal{J}_{q}^{+}
  \coloneqq
  \set{
    j \in \mathcal{J}_{q}
    \given
    \overline{\mathbf{w}}_{j}
    \in
    \mathbb{R}_{>0} \overline{\mathbf{w}}_{q}
  },
  \qquad
  \mathcal{J}_{q}^{-}
  \coloneqq
  \set{
    j \in \mathcal{J}_{q}
    \given
    \overline{\mathbf{w}}_{j}
    \in
    \mathbb{R}_{<0} \overline{\mathbf{w}}_{q}
  }.
\end{equation}
We consider readout allocations
\begin{equation}
  \sum_{j \in \mathcal{J}_{q}^{+}}
  \norm{\overline{\mathbf{w}}_{j}} \mathbf{a}_{j}
  =
  t_{q} \boldsymbol{\beta}_{q},
  \qquad
  \sum_{j \in \mathcal{J}_{q}^{-}}
  \norm{\overline{\mathbf{w}}_{j}} \mathbf{a}_{j}
  =
  (1-t_{q}) \boldsymbol{\beta}_{q},
  \qquad
  t_{q} \in [0,1],
\end{equation}
since every \ac{mwnp} necessarily has such an allocation (\cref{cor:mwnp-readout-allocation}).
For an allocation of this form, the contribution of essential class
$q \in \mathcal{E}$ to the orbit residual $\mathbf{r}$ is
\begin{equation}
  \label{eq:contribution-to-residual-by-essential-classes}
  m
  \sum_{j \in \mathcal{J}_{q}}
  \mathbf{a}_{j} \overline{\mathbf{w}}_{j}^{\top} \overline{\mathbf{x}}
  =
  m
  \Bigl(
    \sum_{j \in \mathcal{J}_{q}^{+}}
    \norm{\overline{\mathbf{w}}_{j}} \mathbf{a}_{j}
    -
    \sum_{j \in \mathcal{J}_{q}^{-}}
    \norm{\overline{\mathbf{w}}_{j}} \mathbf{a}_{j}
  \Bigr)
  \overline{\mathbf{w}}_{q}^{\top} \overline{\mathbf{x}}
  =
  m (2 t_{q} - 1) \boldsymbol{\beta}_{q} \overline{\mathbf{w}}_{q}^{\top} \overline{\mathbf{x}}.
\end{equation}
Define $\mathcal{T} \subseteq [0,1]^{\abs{\mathcal{E}}}$ as the set of splits $\mathbf{t} = (t_{q})_{q \in \mathcal{E}}$ for which these nonnegative allocations of the class aggregates $\boldsymbol{\beta}_{q}$ between orientations reproduce the entire orbit residual $\mathbf{r}$, that is,
\begin{equation}
  \label{eq:residual-compatible-readout-splits}
  \mathcal{T}
  \coloneqq
  \set[\Big]{
    \mathbf{t} \in [0,1]^{\abs{\mathcal{E}}}
    \given
    m
    \sum_{q \in \mathcal{E}}
    (2t_{q}-1)
    \boldsymbol{\beta}_{q}
    \overline{\mathbf{w}}_{q}^{\top}
    \overline{\mathbf{x}}
    =
    \mathbf{r}(\mathbf{x})
    \quad \text{for every } \mathbf{x} \in \mathbb{R}^{N_{i}}
  },
\end{equation}
and let $\mathcal{T}_{N_{h}} \subseteq \mathcal{T}$ denote the splits realizable using at most $N_{h}$ neurons from essential classes.

\begin{proposition}[Minimum-weight-norm representational geometry]
  \label[proposition]{prop:mwnp-representational-geometry}
  Fix a symmetry orbit $\mathcal{O}$ with positively $1$-homogeneous activation and essential parameter classes $\mathcal{E}$.
  If $\mathcal{T}_{N_{h}} \neq \emptyset$, the minimum of $\Omega_{W}$ over $\mathcal{O}_{N_{h}}$ is attained and equals $2 \sum_{q \in \mathcal{E}} \norm{\boldsymbol{\beta}_{q}}$.
  Moreover, the \acp{rsm} of all \acp{mwnp} in $\mathcal{O}_{N_{h}}$ are precisely
\begin{equation}
    \label{eq:mwnp-representational-geometry}
    \sum_{q \in \mathcal{E}}
    \norm{\boldsymbol{\beta}_{q}}
    \bigl(
      t_{q} \mathbf{f}_{q}^{+}(\mathbf{f}_{q}^{+})^{\top}
      + (1-t_{q}) \mathbf{f}_{q}^{-}(\mathbf{f}_{q}^{-})^{\top}
    \bigr),
  \end{equation}
as $\mathbf{t}$ ranges over $\mathcal{T}_{N_{h}}$.
  \proofdeferred{app-subsec:mwnp-representational-geometry}
\end{proposition}

Under the feasibility condition $\mathcal{T}_{N_{h}} \neq \emptyset$, minimum-weight-norm selection therefore removes both degrees of freedom in representational geometry identified in \cref{subsec:essential-and-auxiliary-contributions-to-representational-geometry}.
Auxiliary contributions to the \ac{rsm} vanish, and the weight with which an essential class $q \in \mathcal{E}$ contributes to the \ac{rsm} is fixed by the orbit invariant  $\boldsymbol{\beta}_{q}$.
Since $\boldsymbol{\beta}_{q}$ is the class's aggregate rescaled readout, the same orbit invariant that determines its contribution to the realized function also determines the weight of its contribution to the \ac{rsm}, directly linking representational geometry to the network's computation.

The only remaining degree of freedom in \cref{eq:mwnp-representational-geometry} is the split $\mathbf{t}$.
For each $q \in \mathcal{E}$, the scalar $t_{q}$ determines how the coefficient $\norm{\boldsymbol{\beta}_{q}}$ is divided between the rank-one terms associated with the two feature vectors $\mathbf{f}_{q}^{+}$ and $\mathbf{f}_{q}^{-}$.
Hence, \acp{mwnp} exhibit a unique \ac{rsm} whenever the residual constraint in \cref{eq:residual-compatible-readout-splits} admits a unique split.
This is guaranteed when $m \neq 0$ and the matrices $\boldsymbol{\beta}_{q}\overline{\mathbf{w}}_{q}^{\top}$, $q \in \mathcal{E}$, are linearly independent.

\begin{corollary}[Minimum-weight-norm identifiability]
  \label[corollary]{cor:mwnp-representational-identifiability}
  In the setting of \cref{prop:mwnp-representational-geometry}, suppose additionally that the matrices $\boldsymbol{\beta}_{q} \overline{\mathbf{w}}_{q}^{\top}$, $q\in\mathcal{E}$, are linearly independent and that $m \neq 0$.
  Then, all \acp{mwnp} in $\mathcal{O}$ of width at least $N_{h}$ induce the same \ac{rsm}.
  \proofdeferred{app-subsec:mwnp-representational-geometry}
\end{corollary}

Although the set of representational geometries within a symmetry orbit can expand considerably with width (\cref{prop:strict-growth-until-saturation}), minimum-weight-norm selection can single out a unique geometry that remains unchanged even as width increases.
In particular, all six analytical ReLU solutions in \cref{fig:dissociation} satisfy the conditions of \cref{cor:mwnp-representational-identifiability} (\cref{ex:mwnp-identifiability-of-analytical-relu-solutions}).
Within each corresponding symmetry orbit, all \acp{mwnp} therefore induce the same \ac{rsm} across all widths $N_h \geq 2$, despite the substantial freedom permitted without minimum-norm selection.
Thus, the very examples that expose the limitations of function alone in constraining representation also demonstrate how a principled selection rule can restore representational identifiability.

\section{Discussion}
\label{sec:discussion}

It is common knowledge that function underdetermines the parameterization of artificial neural networks \parencite{nielsen1990algebraic,sussmann1992uniqueness,kurkova1994equivalent,neyshabur2015pathsgd,dinh2017sharp,elbraechter2019degenerate,phuong2020equivalence}.
Building on work that makes this underdetermination precise by enumerating parameter symmetries in nonlinear networks \parencite{simsek2021geometry,martinelli2024expandcluster}, we examine instead the underdetermination of \emph{representational geometry} (\cref{sec:parameter-symmetries-act-through-three-feature-primitives}).
We characterize how parameter symmetries can reshape representational geometry while leaving function unchanged, and show that the range of similarity scores to a fixed reference systematically expands with overparameterization (\cref{sec:parameter-symmetries-dissociate-function-and-representation}), providing a theoretical account of dissociations observed empirically in prior work \parencite{lampinen2024learned,cloos2025differentiable,bo2025evaluating,lampinen2026representation,theiss2026representation}.
Finally, we demonstrate that the coupling between representation and function can be restored for certain implementations characterized by efficiency constraints (\cref{sec:identifiability-through-minimum-norm-selection}), mirroring a result from the linear regime \parencite{braun2025dissociation}.
Though our work is theoretical at present, it holds consequences for the affordances of neural representations; we comment on three.

\paragraph{Completeness.}
\textcite{simsek2021geometry} establish a \emph{complete} characterization of the global minima manifold for a teacher--student learning problem under specific assumptions on the activation function and the input distribution, showing that \emph{all} zero-population-loss solutions lie in the orbit of the irreducible teacher under the symmetries they consider.
\textcite{martinelli2024expandcluster} extend this framework to additional activation functions that admit further parameter symmetries, but do not establish that the resulting symmetry orbits exhaust the corresponding zero-loss sets.
Further, both studies examine identifiability \emph{within} a layer, and thus do not address degeneracies \emph{across} layers, such as collapse of the representational geometry in an intermediate layer of a deep network \parencite[mechanism (iv)]{grigsby2023hidden}.
Accordingly, our results characterize representational variation generated by these parameter symmetries, without claiming that they exhaust the full fiber of a realized function.

\paragraph{Contravariance and task complexity.}
\emph{Universal} representations have been observed across artificial neural networks trained on different tasks and modalities, and even across different architectures \parencite{vanrossem2024universality,huh2024platonic,chen2025universal}.
One proposed explanation is that the complexity of a function is \emph{contravariant} to the ``dispersion'' (variability) of its implementations: harder problems admit fewer solutions \parencite{cao2024explanatoryb}.
This principle could explain why networks trained on increasingly broad or complex tasks converge toward universal representations \parencite{li2016convergent,bansal2021revisiting,wolfram2025layers}.
Our results, however, demonstrate that contravariance need not hold for overparameterized implementations without additional implementation-level constraints.
In a proportional regime, where model size grows alongside task complexity (as is arguably the case in practice \parencite{bahri2024explaining}), the space of implementations need not narrow as the task becomes more complex.
Contravariance, as a principle at the level of \emph{function} rather than \emph{implementation}, therefore cannot by itself account for representational universality.

\paragraph{Individual differences.}
Idiosyncrasies in representational geometry persist even among models of the same architecture \parencite{schrimpf2020integrative,conwell2024largescale,tuckute2023many,linsley2023performanceoptimized}.
A central goal of cognitive computational neuroscience is nevertheless to isolate differences in neural activity that meaningfully distinguish individuals \parencite{thobani2025modelbrain,feather2025brainmodel}, while recognizing that meaningfulness is necessarily context-dependent \parencite{baker2026use}.
Our characterization of function-preserving variation in representational geometry in \cref{sec:parameter-symmetries-dissociate-function-and-representation} provides a framework for distinguishing inter-individual differences that may be inconsequential from those that reflect genuine differences in computation.
In particular, the normative selection rules of \cref{sec:identifiability-through-minimum-norm-selection} can distinguish individuals along computationally relevant dimensions such as noise robustness and transfer performance \parencite{braun2025dissociation,holton2026humans}.

\begin{ack}
  We thank Flavio Martinelli for helpful discussions that clarified aspects of the parameter symmetries cataloged in \textcite{martinelli2024expandcluster}, and Felix A.~Wichmann for valuable feedback on the manuscript.
MT additionally thanks Felix A.~Wichmann for his continued guidance and support throughout MT's doctoral studies.

MT was supported by the International Max Planck Research School for Intelligent Systems (\mbox{IMPRS-IS}) and the German Research Foundation (DFG): SFB 1233, Robust Vision: Inference Principles and Neural Mechanisms, TP INF, project number: 276693517.
LB was supported by the Allen Institute.
AMS was supported by a Schmidt Science Polymath Award, a Sainsbury Wellcome Centre Core Grant from Wellcome (219627/Z/19/Z), and the Gatsby Charitable Foundation (GAT3850).
EG was supported by the Natural Sciences and Engineering Research Council of Canada (NSERC; RGPIN-2026-07018 and DGECR-2026-00191).
AMS is a CIFAR Fellow in Learning in Machines \& Brains, and EG is a CIFAR Azrieli Global Scholar in Learning in Machines \& Brains and a Canada CIFAR AI Chair.
 \end{ack}

\medskip
\printbibliography

\clearpage
\appendix

\renewcommand{\thefigure}{\Alph{section}.\arabic{figure}}
\counterwithin{figure}{section}
\renewcommand{\thetable}{\Alph{section}.\arabic{table}}
\counterwithin{table}{section}

\addtocontents{toc}{\protect\setcounter{tocdepth}{3}}
\begingroup
  \renewcommand{\contentsname}{Appendix Contents}
  \tableofcontents
\endgroup
\clearpage

\section{Notation}
\label[appendix]{app-sec:notation}

This appendix collects the notation used throughout the paper, organized by topic and sorted alphabetically within each block.
We adopt the following standard conventions:
\begin{itemize}
  \item lowercase letters (e.g., $b_{j}$, $c$, $m$) denote scalars,
  \item bold lowercase letters (e.g., $\mathbf{w}_{j}$, $\mathbf{a}_{j}$, $\mathbf{h}^{\mu}$) denote vectors,
  \item and bold uppercase letters (e.g., $\mathbf{W}$, $\mathbf{A}$, $\mathbf{H}$) denote matrices.
\end{itemize}
When corresponding quantities are defined for both an arbitrary parameterization and an irreducible parameterization, the latter are marked by a superscript star; for example, $\mathbf{w}_{j}^{\star}$, $b_{j}^{\star}$, $\mathbf{a}_{j}^{\star}$ (irreducible) vs.\ $\mathbf{w}_{j}$, $b_{j}$, $\mathbf{a}_{j}$ (arbitrary).

\let\oldarraystretch\arraystretch
\renewcommand{\arraystretch}{1.5}

\subsection*{Indices and dimensions}

\begin{xltabular}{\textwidth}{@{}p{0.24\textwidth}@{\hskip 1em}X@{}}
  $j$                      & Index enumerating neurons \\
  $K$                      & Number of rows of $\mathbf{U}$ \\
  $\mu$                    & Index enumerating inputs \\
  $N_{h}$, $N_{h}^{\star}$ & Width of a parameterization, and minimum width within a symmetry orbit, respectively \\
  $N_{i}$                  & Input dimension \\
  $N_{o}$                  & Output dimension \\
  $P$                      & Number of inputs
\end{xltabular}

\subsection*{Constants, norms, inner products, and operators}

\begin{xltabular}{\textwidth}{@{}p{0.24\textwidth}@{\hskip 1em}X@{}}
  $\mathbf{0}$, $\mathbf{1}$                   & Zero vector or matrix, all-ones vector \\
  $\abs{\cdot}$                                & Absolute value, or set cardinality \\
  $\diag(\cdot)$                               & Diagonal matrix from a vector, or extraction of the diagonal of a matrix \\
  $\langle \mathbf{M}, \mathbf{N} \rangle_{F}$ & Frobenius inner product $\tr(\mathbf{M}^{\top} \mathbf{N})$ \\
  $\norm{\cdot}$                               & Euclidean norm \\
  $\norm{\cdot}_{F}$                           & Frobenius norm \\
  $\rho(\mathbf{M}, \mathbf{N})$               & Pearson correlation between strict upper-triangular entries of two \acsp{rsm}
\end{xltabular}

\subsection*{Network parameters}

\begin{xltabular}{\textwidth}{@{}p{0.24\textwidth}@{\hskip 1em}X@{}}
  $\mathbf{a}_{j}$, $\mathbf{a}_{j}^{\star}$                         & Readout weights of a single neuron \\
  $\mathbf{a}^{\pm}$, $\mathbf{a}^{\mp}$                             & Sum of and difference between aggregate readouts of aligned and opposite subgroups \\
  $\mathbf{a}_{\mathcal{J}}^{\pm}$, $\mathbf{a}_{\mathcal{J}}^{\mp}$ & Corresponding sum and difference restricted to a subset $\mathcal{J}$ of neurons \\
  $\mathbf{A}$, $\mathbf{A}^{\star}$                                 & Readout-weight matrix \\
  $\mathbf{b}$, $\mathbf{b}^{\star}$                                 & Hidden-bias vector \\
  $b_{j}$, $b_{j}^{\star}$                                           & Hidden bias of a single neuron \\
  $f_{\boldsymbol{\theta}}$                                          & Function realized by the network with parameterization $\boldsymbol{\theta}$ \\
  $\boldsymbol{\theta}$, $\boldsymbol{\theta}^{\star}$               & Network parameterization consisting of the nonlinearity $\sigma$, incoming weights $\mathbf{W}$, hidden biases $\mathbf{b}$, and readout weights $\mathbf{A}$ \\
  $\mathbf{w}_{j}$, $\mathbf{w}_{j}^{\star}$                         & Incoming-weight vector of a single neuron \\
  $\overline{\mathbf{w}}_{j}$, $\overline{\mathbf{w}}_{j}^{\star}$   & Incoming-parameter vector $(\mathbf{w}_{j}^{\top}, b_{j})^{\top} \in \mathbb{R}^{N_{i}+1}$ of a single neuron \\
  $\mathbf{W}$, $\mathbf{W}^{\star}$                                 & Incoming-weight matrix
\end{xltabular}

\subsection*{Inputs and hidden activations}

\begin{xltabular}{\textwidth}{@{}p{0.24\textwidth}@{\hskip 1em}X@{}}
  $\mathbf{f}_{\ell}^{\top}$                                         & $\ell$th row of $\mathbf{F}_{\mathcal{E},\mathcal{I}}$ \\
  $(\mathbf{f}_{q}^{+})^{\top}$, $(\mathbf{f}_{q}^{-})^{\top}$       & Rows of $\mathbf{F}_{\mathcal{E}}$ corresponding to the two orientations $\pm\overline{\mathbf{w}}_{q}$ of essential parameter class $q$ \\
  $\mathbf{F}_{\mathcal{E}}$, $\mathbf{F}_{\mathcal{E},\mathcal{I}}$ & Feature matrices for both orientations of every essential parameter class and for the subset of represented orientations indexed by $\mathcal{I}$ \\
  $\mathbf{h}^{\mu}$                                                 & Hidden activation elicited by input $\mathbf{x}^{\mu}$ \\
  $\mathbf{H}$, $\mathbf{H}^{\star}$                                 & Hidden-activation matrix \\
  $\mathcal{I}$                                                      & Index set selecting the rows of $\mathbf{F}_{\mathcal{E}}$ represented in a parameterization \\
  $\mathbf{u}_{k}^{\top}$                                            & $k$th row of $\mathbf{U}$ \\
  $\mathbf{U}$                                                       & Matrix collecting the pairwise distinct additional feature rows contributed by nonessential parameter classes and constant neurons \\
  $\mathbf{v}_{j}^{\top}$, ${\mathbf{v}_{j}^{\star}}^{\top}$         & $j$th feature row of $\mathbf{H}$ or $\mathbf{H}^{\star}$, respectively \\
  $\mathbf{x}^{\mu}$                                                 & $\mu$th input vector \\
  $\overline{\mathbf{x}}$                                            & Augmented input vector $(\mathbf{x}^{\top}, 1)^{\top}$ \\
  $\mathbf{X}$                                                       & Input matrix $[\mathbf{x}^{1}, \dots, \mathbf{x}^{P}]$ collecting all $P$ inputs as columns \\
  $\overline{\mathbf{X}}$                                            & Augmented input matrix $[\overline{\mathbf{x}}^{1}, \dots, \overline{\mathbf{x}}^{P}] = [\mathbf{X}^{\top}, \mathbf{1}]^{\top}$ \\
  $z$, $z_{j}$                                                       & Pre-activation $\overline{\mathbf{w}}^{\top} \overline{\mathbf{x}}$ of a neuron
\end{xltabular}

\subsection*{Activation functions}

\begin{xltabular}{\textwidth}{@{}p{0.24\textwidth}@{\hskip 1em}X@{}}
  $c$                                       & Constant value of the even component of a constant-odd activation \\
  $\delta$                                  & Coefficient of the absolute-value component of a positively $1$-homogeneous activation \\
  $e(z)$, $o(z)$                            & Even and odd components of a function \\
  $\lambda_{-}$, $\lambda_{+}$              & Slopes of a positively $1$-homogeneous activation on the negative and positive half-lines, respectively \\
  $m$                                       & Slope of the linear component of an even-linear activation \\
  $\psi(z)$                                 & ReLU activation $\psi(z) = \max(0, z)$ \\
  $\sigma \colon \mathbb{R} \to \mathbb{R}$ & Nonaffine scalar activation function
\end{xltabular}

\subsection*{Symmetry groups and orbits}

\begin{xltabular}{\textwidth}{@{}p{0.24\textwidth}@{\hskip 1em}X@{}}
  $\mathcal{D}$                                                                  & Index set of a duplicate-neuron group \\
  $\mathcal{K}$                                                                  & Index set of an aligned/opposite group \\
  $\mathcal{N}^{+}$, $\mathcal{N}^{-}$                                           & Index sets of aligned and opposite subgroups of an aligned/opposite group \\
  $\mathcal{O}(\boldsymbol{\theta})$, $\mathcal{O}_{N_{h}}(\boldsymbol{\theta})$ & Symmetry orbit of $\boldsymbol{\theta}$ (at width $N_{h}$) \\
  $\boldsymbol{\theta} \sim \boldsymbol{\xi}$                                    & Symmetry-equivalence of network parameterizations $\boldsymbol{\theta}$ and $\boldsymbol{\xi}$ \\
  $\mathcal{Z}$                                                                  & Index set of a zero-neuron group
\end{xltabular}

\subsection*{Parameter classes and orbit invariants}

\begin{xltabular}{\textwidth}{@{}p{0.24\textwidth}@{\hskip 1em}X@{}}
  $\boldsymbol{\beta}_{q}(\boldsymbol{\theta})$, $\boldsymbol{\beta}_{q}$ & Class aggregate of parameter class $q$ in $\boldsymbol{\theta}$ (the aggregate rescaled readout for positively $1$-homogeneous activations), and its common value on a fixed symmetry orbit \\
  $\mathcal{E}$                                                           & Orbit-invariant set of essential parameter classes, characterized by $\boldsymbol{\beta}_{q} \neq \mathbf{0}$ \\
  $\mathcal{J}_{q}$, $\mathcal{J}_{0}$                                    & Index sets of the neurons in parameter class $q$ and of constant neurons, respectively \\
  $\mathcal{J}_{q}^{+}$, $\mathcal{J}_{q}^{-}$                            & Neurons of $\mathcal{J}_{q}$ whose incoming parameters $\overline{\mathbf{w}}_{j}$ are positive and negative multiples of $\overline{\mathbf{w}}_{q}$, respectively \\
  $\phi_{q}(\mathbf{x})$                                                  & Nonlinear feature associated with parameter class $q$ \\
  $q$                                                                     & Activation-dependent parameter class of a nonconstant neuron \\
  $\mathcal{Q}(\boldsymbol{\theta})$                                      & Set of parameter classes represented by the nonconstant neurons of $\boldsymbol{\theta}$ \\
  $\mathbf{r}_{\boldsymbol{\theta}}$, $\mathbf{r}$                        & Global residual function of $\boldsymbol{\theta}$, and its common value on a fixed orbit \\
  $[\overline{\mathbf{w}}]$                                               & Parameter class of the incoming-parameter vector $\overline{\mathbf{w}}$ \\
  $\overline{\mathbf{w}}_{q}$                                             & Chosen incoming-parameter representative of class $q$, normalized to unit norm in the positively $1$-homogeneous case, and with a positive first nonzero incoming-weight entry in the constant-odd case
\end{xltabular}

\subsection*{Primitive feature transformations}

\begin{xltabular}{\textwidth}{@{}p{0.24\textwidth}@{\hskip 1em}X@{}}
  $\boldsymbol{\alpha}$                                                                             & Vector of nonzero scaling factors \\
  $\boldsymbol{\alpha}^{2}$                                                                         & Hadamard square of $\boldsymbol{\alpha}$ \\
  $\mathcal{A}_{\mathbf{U}}$, $\mathcal{D}_{\boldsymbol{\nu}}$, $\mathcal{S}_{\boldsymbol{\alpha}}$ & Feature addition, duplication, and scaling, respectively \\
  $\mathbf{D}_{\boldsymbol{\nu}}$                                                                   & Binary duplication matrix that repeats the $i$th pre-duplication feature row exactly $\nu_{i}$ times \\
  $\boldsymbol{\gamma} = \mathbf{D}_{\boldsymbol{\nu}}^{\top} \boldsymbol{\alpha}^{2}$              & Effective weights of the rank-one contributions to the \ac{rsm} induced by duplication and scaling \\
  $\boldsymbol{\nu}$                                                                                & Duplication pattern, with $\nu_{i} \in \mathbb{N}_{>0}$ denoting the number of copies of the $i$th pre-duplication feature row
\end{xltabular}

\subsection*{Representational geometry and similarity}

\begin{xltabular}{\textwidth}{@{}p{0.24\textwidth}@{\hskip 1em}X@{}}
  $\cl$                                                        & Closure with respect to the Frobenius norm \\
  $\cone(\cdot)$, $\overline{\cone}(\cdot)$                    & Conic hull operator and its closure, respectively \\
  $\mathcal{C}_{\sigma,\mathbf{X}}$                            & Closed conic hull of $\mathcal{G}_{\sigma,\mathbf{X}}$ \\
  $d$                                                          & Dimension of the span of $\mathcal{C}_{\sigma,\mathbf{X}}$ \\
  $\Delta_{N_{h}}(\mathbf{N})$, $\Delta_{\infty}(\mathbf{N})$  & Spread of attainable similarity scores relative to $\mathbf{N}$ at width $N_{h}$ and across all widths, respectively \\
  $\mathcal{G}_{\sigma,\mathbf{X}}$                            & Set of projected rank-one contributions from all single-neuron features realizable with activation $\sigma$ on $\mathbf{X}$ \\
  $\Hol$                                                       & Subspace of hollow symmetric matrices with zero off-diagonal mean \\
  $\mathbf{M}_{\boldsymbol{\theta}}$                           & \ac{rsm} induced by parameterization $\boldsymbol{\theta}$ \\
  $\mathbf{M}^{\circ}$                                         & Normalized projection of a symmetric matrix $\mathbf{M}$ onto $\Hol$ \\
  $\mathbf{N}$                                                 & Fixed reference \ac{rsm} used for representational comparison \\
  $\Pi_{\Hol}$                                                 & Orthogonal projection onto $\Hol$ \\
  $\mathcal{R}_{N_{h}}$, $\mathcal{R}$                         & Sets of realizable \acp{rsm} within the symmetry orbit at width $N_{h}$ and across all widths, respectively \\
  $\mathcal{S}_{N_{h}}(\mathbf{N})$, $\mathcal{S}(\mathbf{N})$ & Sets of similarity scores to reference $\mathbf{N}$ attainable within the symmetry orbit at width $N_{h}$ and across all widths, respectively
\end{xltabular}

\subsection*{Readout allocation and norm minimization}

\begin{xltabular}{\textwidth}{@{}p{0.24\textwidth}@{\hskip 1em}X@{}}
  $g_{q}$                                      & Minimum of $g_{q}^{+}$ and $g_{q}^{-}$ \\
  $g_{q}^{+}$, $g_{q}^{-}$                     & Norms of the two features $\mathbf{f}_{q}^{+}$ and $\mathbf{f}_{q}^{-}$, respectively, of essential parameter class $q$ \\
  $\kappa(\mathbf{t})$                         & Minimum number of essential-class neurons required to realize the split $\mathbf{t}$ \\
  $\Omega_{H}(\boldsymbol{\theta})$            & Representation-norm objective $\norm{\mathbf{H}}_{F}^{2} + \norm{\mathbf{A}}_{F}^{2}$ \\
  $\Omega_{W}(\boldsymbol{\theta})$            & Weight-norm objective $\norm{\mathbf{W}}_{F}^{2} + \norm{\mathbf{b}}^{2} + \norm{\mathbf{A}}_{F}^{2}$ \\
  $t_{q} \in [0,1]$                            & Fraction of $\boldsymbol{\beta}_{q}$ allocated to orientation $\overline{\mathbf{w}}_{q}$, with the remaining fraction $1-t_{q}$ allocated to $-\overline{\mathbf{w}}_{q}$ \\
  $\mathbf{t} = (t_{q})_{q \in \mathcal{E}}$   & Vector of readout splits across essential parameter classes \\
  $\mathcal{T}$                                & Set of residual-compatible splits, i.e., splits that reproduce the orbit residual $\mathbf{r}$ \\
  $\mathcal{T}_{N_{h}}$                        & Width-feasible subset of $\mathcal{T}$, comprising the splits satisfying $\kappa(\mathbf{t}) \leq N_{h}$ \\
  $\mathcal{T}^{H}$, $\mathcal{T}^{H}_{N_{h}}$ & Splits in $\mathcal{T}$ and $\mathcal{T}_{N_{h}}$, respectively, that use only orientations of minimal feature norm $g_{q}$
\end{xltabular}

\let\arraystretch\oldarraystretch
\clearpage

\section{Parameter symmetries in overparameterized nonlinear networks}
\label[appendix]{app-sec:parameter-symmetries-in-overparameterized-networks}

This appendix provides a formal treatment of the function-preserving parameter symmetries in overparameterized nonlinear networks that are summarized in the main text.
We distinguish three classes of symmetries:
\begin{itemize}
  \item generic reparameterization symmetries (\cref{app-subsec:generic-reparameterization-symmetries}),
\item symmetries induced by overparameterization that are independent of the activation function (\cref{app-subsec:activation-independent-symmetries}),
\item and symmetries induced by overparameterization that depend on algebraic symmetries of the activation function \mbox{(\cref{app-subsec:activation-dependent-symmetries})}.
\end{itemize}
After introducing all parameter symmetries considered in this work, we establish a necessary and sufficient characterization of when two parameterizations are symmetry-equivalent (\cref{app-subsec:parameter-classes-orbit-invariants-and-symmetry-equivalence}).

\subsection{Generic reparameterization symmetries}
\label[appendix]{app-subsec:generic-reparameterization-symmetries}

We briefly revisit three well-known function-preserving symmetries of neural networks:
the permutation symmetry \mbox{(\cref{app-subsubsec:permutation-symmetry})}, the positive-scaling symmetry (\cref{app-subsubsec:positive-scaling-symmetry}), and the sign-flip symmetry (\cref{app-subsubsec:sign-flip-symmetry}).

\subsubsection{Permutation symmetry}
\label[appendix]{app-subsubsec:permutation-symmetry}

For a hidden layer of width $N_{h}$, the ordering of the neurons within that layer is arbitrary.
Let $\mathfrak{S}_{N_{h}}$ denote the \emph{symmetric group} of the set of integers $\set{1, \dots, N_{h}}$, and let $\pi \in \mathfrak{S}_{N_{h}}$ be a permutation.
Reordering the parameters of the layer according to
\begin{equation}
  (\mathbf{w}_{j}, b_{j}, \mathbf{a}_{j})_{j=1}^{N_{h}}
  \mapsto
  (\mathbf{w}_{\pi(j)}, b_{\pi(j)}, \mathbf{a}_{\pi(j)})_{j=1}^{N_{h}}
\end{equation}
does not change the function computed by the layer, since
\begin{equation}
  \sum_{j=1}^{N_{h}} \mathbf{a}_{\pi(j)} \, \sigma(\mathbf{w}_{\pi(j)}^{\top} \mathbf{x} + b_{\pi(j)})
  = \sum_{j=1}^{N_{h}} \mathbf{a}_{j} \, \sigma(\mathbf{w}_{j}^{\top} \mathbf{x} + b_{j}),
  \qquad
  \mathbf{x} \in \mathbb{R}^{N_{i}}.
\end{equation}
This invariance is commonly referred to as the \emph{permutation symmetry}.

\subsubsection{Positive-scaling symmetry}
\label[appendix]{app-subsubsec:positive-scaling-symmetry}

Suppose the activation function $\sigma \colon \mathbb{R} \to \mathbb{R}$ is positively homogeneous of degree~$1$, i.e.,
\begin{equation}
  \sigma(\alpha z) = \alpha \, \sigma(z),
  \qquad
  \alpha > 0.
\end{equation}
Rescaling the parameters of a single neuron via
\begin{equation}
  (\mathbf{w}_{j}, b_{j}, \mathbf{a}_{j}) \mapsto (\alpha \mathbf{w}_{j}, \alpha b_{j}, \alpha^{-1} \mathbf{a}_{j}),
  \qquad
  \alpha > 0,
\end{equation}
does not change the function computed by that neuron since, for $\mathbf{x} \in \mathbb{R}^{N_{i}}$,
\begin{equation}
  \alpha^{-1} \mathbf{a}_{j} \, \sigma(\alpha \mathbf{w}_{j}^{\top} \mathbf{x} + \alpha b_{j})
  = (\alpha^{-1} \alpha) \mathbf{a}_{j} \, \sigma(\mathbf{w}_{j}^{\top} \mathbf{x} + b_{j})
  = \mathbf{a}_{j} \, \sigma(\mathbf{w}_{j}^{\top} \mathbf{x} + b_{j}).
\end{equation}
This invariance is commonly referred to as the \emph{positive-scaling symmetry}.

\subsubsection{Sign-flip symmetry}
\label[appendix]{app-subsubsec:sign-flip-symmetry}

Suppose the activation function $\sigma \colon \mathbb{R} \to \mathbb{R}$ is odd, i.e., $\sigma(-z) = -\sigma(z)$.
Flipping the sign of both the incoming parameters $(\mathbf{w}_{j}, b_{j})$ and the readout weights $\mathbf{a}_{j}$, i.e.,
\begin{equation}
  (\mathbf{w}_{j}, b_{j}, \mathbf{a}_{j}) \mapsto (-\mathbf{w}_{j}, -b_{j}, -\mathbf{a}_{j}),
\end{equation}
does not change the function computed by that neuron since, for $\mathbf{x} \in \mathbb{R}^{N_{i}}$,
\begin{equation}
  -\mathbf{a}_{j} \, \sigma(-\mathbf{w}_{j}^{\top} \mathbf{x} - b_{j})
  = -\mathbf{a}_{j} \, \sigma(-(\mathbf{w}_{j}^{\top} \mathbf{x} + b_{j}))
  = \mathbf{a}_{j} \, \sigma(\mathbf{w}_{j}^{\top} \mathbf{x} + b_{j}).
\end{equation}
Similarly, if $\sigma$ is even, i.e., $\sigma(-z) = \sigma(z)$, flipping the sign of the incoming parameters,
\begin{equation}
  (\mathbf{w}_{j}, b_{j}, \mathbf{a}_{j})
  \mapsto
  (-\mathbf{w}_{j}, -b_{j}, \mathbf{a}_{j}),
\end{equation}
also leaves the realized function invariant.
We refer to this as the \emph{sign-flip symmetry}.

\subsection{Activation-independent overparameterization symmetries}
\label[appendix]{app-subsec:activation-independent-symmetries}

In contrast to the generic function-preserving reparameterization symmetries discussed in \cref{app-subsec:generic-reparameterization-symmetries}, the symmetries considered in this and the subsequent subsections are induced by overparameterization: they rely on the presence of ``redundant'' neurons.
We first review overparameterization symmetries that do \emph{not} depend on the choice of activation function $\sigma \colon \mathbb{R} \to \mathbb{R}$.
In a two-layer, bias-free setting, \textcite{simsek2021geometry} show that these symmetries generate affine subspaces of equivalent networks (an \emph{expansion manifold}) in a teacher--student setup.

This subsection first introduces a nondegeneracy condition on readout weights that rules out decomposable symmetry groups, before defining the three activation-independent symmetry classes: zero-neuron groups, duplicate-neuron groups, and constant neurons.
We then discuss how these symmetry classes give rise to function-preserving realizations (\cref{app-subsubsec:activation-independent-symmetries-function-preserving-realizations}), and demonstrate that the symmetry classes are minimal and mutually distinct  (\cref{app-subsubsec:activation-independent-symmetries-minimality-and-distinctness}).

Compared with prior formulations \parencite{simsek2021geometry,martinelli2024expandcluster}, we impose the following nondegeneracy condition to exclude decomposable cases.

\begin{definition}[Subset-nonzero]
  \label[definition]{def:subset-nonzero}
  Let $\mathcal{I} \subseteq \set{1, \dots, N_{h}}$ be a finite index set.
  We say that the collection of readout weights $\set{\mathbf{a}_{i}}_{i\in\mathcal{I}}$ is \emph{subset-nonzero} if,
  for every nonempty proper subset $\mathcal{J} \subsetneq \mathcal{I}$,
\begin{equation}
    \sum_{j \in \mathcal{J}} \mathbf{a}_{j} \neq \mathbf{0}.
  \end{equation}
\end{definition}

We now formalize the three activation-independent overparameterization symmetry classes considered in this subsection.

\begin{definition}[Zero-neuron group]
  \label[definition]{def:zero-neuron-group}
  Fix $(\mathbf{w}, b) \in \mathbb{R}^{N_{i}} \times \mathbb{R}$ with $\mathbf{w} \neq \mathbf{0}$, and let $\mathcal{Z} \neq \emptyset$ index a set of hidden neurons satisfying $\mathbf{w}_{z} = \mathbf{w}$ and $b_{z} = b$ for all $z \in \mathcal{Z}$, with readout weights $\mathbf{a}_{z}$.
  We call the neurons indexed by $\mathcal{Z}$ a \emph{zero-neuron group} if
\begin{equation}
    \label{eq:zero-neuron-group-condition}
    \sum_{z \in \mathcal{Z}} \mathbf{a}_{z} = \mathbf{0}
    \qquad \text{and} \qquad
    \set{\mathbf{a}_{z}}_{z \in \mathcal{Z}} \ \text{is subset-nonzero.}
  \end{equation}
\end{definition}

\begin{definition}[Duplicate-neuron group]
  \label[definition]{def:duplicate-neuron-group}
  Fix $(\mathbf{w}^{\star}, b^{\star}, \mathbf{a}^{\star}) \in \mathbb{R}^{N_{i}} \times \mathbb{R} \times \mathbb{R}^{N_{o}}$ with $\mathbf{w}^{\star} \neq \mathbf{0}$ and $\mathbf{a}^{\star} \neq \mathbf{0}$, and let $\mathcal{D}$ index a set of at least two hidden neurons satisfying $\mathbf{w}_{d} = \mathbf{w}^{\star}$ and $b_{d} = b^{\star}$ for all $d \in \mathcal{D}$, with readout weights $\mathbf{a}_{d}$.
  We call the neurons indexed by $\mathcal{D}$ a \emph{duplicate-neuron group} if
\begin{equation}
    \label{eq:duplicate-neuron-group-condition}
    \sum_{d \in \mathcal{D}} \mathbf{a}_{d} = \mathbf{a}^{\star}
    \qquad \text{and} \qquad
    \set{\mathbf{a}_{d}}_{d \in \mathcal{D}} \ \text{is subset-nonzero.}
  \end{equation}
\end{definition}

\begin{definition}[Constant neuron]
  \label[definition]{def:constant-neuron}
  A \emph{constant neuron} is a hidden neuron with vanishing incoming weights $\mathbf{w}=\mathbf{0}$.
\end{definition}

\subsubsection{Function-preserving realizations}
\label[appendix]{app-subsubsec:activation-independent-symmetries-function-preserving-realizations}

\paragraph{Zero-neuron groups.}
Adding or removing a zero-neuron group to or from a parameterization leaves the realized function unchanged.
Indeed,
\begin{equation}
  \sum_{z \in \mathcal{Z}} \mathbf{a}_{z} \,
  \sigma(
    \mathbf{w}_{z}^{\top} \mathbf{x}
    + b_{z}
  )
  = \sigma(\mathbf{w}^{\top} \mathbf{x} + b)
  \underbrace{\sum_{z \in \mathcal{Z}} \mathbf{a}_z}_{=\mathbf{0}
  }
  = \mathbf{0}.
\end{equation}

\paragraph{Duplicate-neuron groups.}
The same holds when a neuron with parameters $(\mathbf{w}^{\star}, b^{\star}, \mathbf{a}^{\star})$ is replaced by a duplicate-neuron group, since
\begin{equation}
  \sum_{d \in \mathcal{D}} \mathbf{a}_{d} \,
  \sigma(
    \mathbf{w}_{d}^{\top} \mathbf{x}
    + b_{d}
  )
  = \sigma(
    {\mathbf{w}^{\star}}^{\top} \mathbf{x}
    + b^{\star}
  )
  \underbrace{\sum_{d \in \mathcal{D}} \mathbf{a}_{d}}_{= \mathbf{a}^{\star}
  }
  = \mathbf{a}^{\star} \,
  \sigma(
    {\mathbf{w}^{\star}}^{\top} \mathbf{x}
    + b^{\star}
  ).
\end{equation}
Merging copies of a duplicate-neuron group into a single neuron leaves the realized function invariant as well.

\paragraph{Constant neurons.}
A constant neuron, by contrast, contributes an input-independent offset:
\begin{equation}
  \mathbf{a} \, \sigma(\mathbf{w}^{\top} \mathbf{x} + b)
  = \mathbf{a} \, \sigma(b).
\end{equation}
Thus, for the realized function to remain unchanged, the constant contributions of all such neurons must cancel.
If $\sigma$ is constant-odd (\cref{def:constant-odd-activations}), this offset may also be compensated by constant-neuron groups (\cref{def:constant-neuron-group}) or constant-duplicate-neuron groups (\cref{def:constant-duplicate-neuron-group}).

\subsubsection{Minimality and distinctness}
\label[appendix]{app-subsubsec:activation-independent-symmetries-minimality-and-distinctness}

Requiring that the shared incoming-weight vector be nonzero for zero-neuron and duplicate-neuron groups, and incorporating the subset-nonzero condition from \cref{def:subset-nonzero} into their definitions, guarantees that the resulting taxonomy of activation-independent symmetries is \emph{minimal} in the sense that groups do not decompose into smaller subgroups.

\paragraph{Zero-neuron groups.}
Let $\mathcal{Z}$ index a zero-neuron group.
Then, the subset-nonzero condition guarantees that any nonempty proper subset $\mathcal{J} \subsetneq \mathcal{Z}$ satisfies $\sum_{j \in \mathcal{J}} \mathbf{a}_{j} \neq \mathbf{0}$.
Hence, a zero-neuron group can never contain a proper zero-neuron subgroup.
In other words, \cref{def:zero-neuron-group} singles out the smallest groups of neurons that could be removed from a parameterization without altering its realized function.

\paragraph{Duplicate-neuron groups.}
Let $\mathcal{D}$ index a duplicate-neuron group with shared parameters $(\mathbf{w}^{\star}, b^{\star})$ and aggregate readout weights $\sum_{d \in \mathcal{D}} \mathbf{a}_{d} = \mathbf{a}^{\star} \neq \mathbf{0}$.
Then, there exists no proper duplicate-neuron subgroup \emph{with respect to} $\mathbf{a}^{\star}$,\footnote{In contrast to zero-neuron groups, for duplicate-neuron groups the notion of minimality is only sensible with respect to a fixed reference readout $\mathbf{a}^{\star}$, since \emph{any} subset $\mathcal{J} \subsetneq \mathcal{D}$ indexing at least two neurons again defines a duplicate-neuron group with respect to $\mathbf{a}_{\mathcal{J}}^{\star} \coloneqq \sum_{j \in \mathcal{J}} \mathbf{a}_{j}$.
}
i.e., there does not exist a subset of neurons $\mathcal{J} \subsetneq \mathcal{D}$ such that $\sum_{j \in \mathcal{J}} \mathbf{a}_{j} = \mathbf{a}^{\star}$.
To see this, assume the opposite.
The identity $\sum_{j \in \mathcal{J}} \mathbf{a}_{j} = \mathbf{a}^{\star}$ implies
\begin{equation}
  \sum_{d \in \mathcal{D} \setminus \mathcal{J}} \mathbf{a}_{d}
  = \sum_{d \in \mathcal{D}} \mathbf{a}_{d} - \sum_{j \in \mathcal{J}} \mathbf{a}_{j}
  = \mathbf{a}^{\star} - \mathbf{a}^{\star} = \mathbf{0},
\end{equation}
violating the subset-nonzero condition.
As for zero-neuron groups, the subset-nonzero condition also guarantees that a duplicate-neuron group can never contain a zero-neuron subgroup.

\paragraph{Constant neurons.}
Since constant neurons are individual neurons, rather than groups of neurons, they trivially do not admit a decomposition into subgroups.

\Cref{def:zero-neuron-group,def:duplicate-neuron-group,def:constant-neuron} not only satisfy this minimality property, but also yield a proper taxonomy of activation-independent overparameterization symmetries: the three classes are mutually \emph{distinct}.
Zero-neuron groups and duplicate-neuron groups require the shared incoming weights to be nonzero, which guarantees that both are distinct from constant neurons, since the latter are defined by having vanishing incoming weights $\mathbf{w} = \mathbf{0}$.
Similarly, a zero-neuron group never qualifies as a duplicate-neuron group because the latter requires the aggregate readout $\mathbf{a}^{\star}$ to be nonzero.
Hence, \cref{def:zero-neuron-group,def:duplicate-neuron-group,def:constant-neuron} are all mutually exclusive.

\subsection{Activation-dependent overparameterization symmetries}
\label[appendix]{app-subsec:activation-dependent-symmetries}

The three activation-\emph{independent} symmetries in \cref{app-subsec:activation-independent-symmetries} arise from overparameterization alone, regardless of the activation function $\sigma \colon \mathbb{R} \to \mathbb{R}$.
Additional symmetry groups, composed of ``aligned'' and ``opposite'' subgroups of neurons whose incoming weights and hidden biases agree up to sign flips, emerge when $\sigma$ has algebraic structure relating $\sigma(z)$ and $\sigma(-z)$.
Formalizing these symmetries requires two ingredients:
we first identify two classes of activation functions that exhibit the relevant algebraic structure (\cref{app-subsubsec:even-linear-and-constant-odd-activations}),
and then formalize the notion of a group of neurons splitting into aligned and opposite subgroups (\cref{app-subsubsec:aligned-opposite-groups}).
We then review the corresponding symmetry groups for even-linear activations (\cref{app-subsubsec:symmetries-even-linear-activations}) and constant-odd activations (\cref{app-subsubsec:symmetries-constant-odd-activations}), before discussing the choice of sign in aligned/opposite subgroups (\cref{app-subsubsec:choice-of-sign-aligned-opposite-groups}), their function-preserving realizations (\cref{app-subsubsec:activation-dependent-symmetries-function-preserving-realizations}), and the degenerate cases excluded by the nondegeneracy requirements placed on aligned/opposite subgroups (\cref{app-subsubsec:activation-dependent-symmetries-minimality-and-nondegeneracy}).
Finally, we demonstrate that, when $\sigma$ admits activation-dependent overparameterization symmetries, irreducible parameterizations can differ by more than generic reparameterization symmetries alone (\cref{app-subsubsec:nonuniqueness-irreducible-parameterizations}).

\subsubsection{Even-linear and constant-odd activations}
\label[appendix]{app-subsubsec:even-linear-and-constant-odd-activations}

Recall that any function $\sigma \colon \mathbb{R} \to \mathbb{R}$ can be uniquely decomposed into even and odd components $\sigma(z) = e(z) + o(z)$, where
\begin{equation}
  e(z) = \frac{\sigma(z) + \sigma(-z)}{2}, \qquad
  o(z) = \frac{\sigma(z) - \sigma(-z)}{2},
\end{equation}
satisfying $e(-z) = e(z)$ and $o(-z) = -o(z)$.
Following \textcite{martinelli2024expandcluster}, we distinguish two classes of activations by the structure of their even and odd components.

\begin{definition}[Even-linear activations]
  \label[definition]{def:even-linear-activations}
  An activation function is called \emph{even-linear} if its odd component is linear: $\sigma(z) = e(z) + mz$ for $m \in \mathbb{R}$.
\end{definition}

\begin{definition}[Constant-odd activations]
  \label[definition]{def:constant-odd-activations}
  An activation function is called \emph{constant-odd} if its even component is constant: $\sigma(z) = c + o(z)$ for $c \in \mathbb{R}$.
\end{definition}

These two classes are not mutually exclusive: an activation is both even-linear and constant-odd if and only if it is affine, $\sigma(z) = mz + c$.
As established in \cref{sec:preliminaries-and-setting}, we exclude affine activations throughout so that the two classes are mutually exclusive in our setting.

\subsubsection{Aligned/opposite groups}
\label[appendix]{app-subsubsec:aligned-opposite-groups}

Here, we formalize subgroups of hidden neurons that share incoming weights and hidden biases up to a sign flip as \emph{aligned/opposite groups}.
These underlie the activation-dependent symmetry groups of \textcite{martinelli2024expandcluster}, which we revisit and refine by imposing additional nondegeneracy conditions on the aligned and opposite subgroups.
Much like the subset-nonzero condition for activation-independent symmetries, these conditions prevent symmetry groups from decomposing into smaller ones.
Without them, activation-\emph{dependent} symmetry groups can be broken up into smaller activation-\emph{dependent} symmetry groups or even collapse into disjoint activation-\emph{independent} symmetry groups.
A detailed discussion of these degenerate cases is provided in \cref{app-subsubsec:activation-dependent-symmetries-minimality-and-nondegeneracy}.

\begin{definition}[Aligned/opposite group]
  \label[definition]{def:aligned-opposite-group}
  Fix $(\mathbf{w}, b) \in \mathbb{R}^{N_{i}} \times \mathbb{R}$ with $\mathbf{w} \neq \mathbf{0}$,
  and let $\mathcal{K}$ index a set of hidden neurons with parameters $(\mathbf{w}_{k}, b_{k}, \mathbf{a}_{k})$ for $k \in \mathcal{K}$.
  Let $\mathcal{N}^{+}, \mathcal{N}^{-} \subsetneq \mathcal{K}$ be nonempty, disjoint index sets such that $\mathcal{K} = \mathcal{N}^{+} \sqcup \mathcal{N}^{-}$.
  The group of neurons indexed by $\mathcal{K}$ is called an \emph{aligned/opposite group} with respect to $(\mathbf{w}, b)$ if it splits into an aligned subgroup $\mathcal{N}^{+}$ and an opposite subgroup $\mathcal{N}^{-}$ such that
\begin{equation}
    (\mathbf{w}_{k}, b_{k}) = (\mathbf{w}, b), \quad k \in \mathcal{N}^{+},
    \qquad
    (\mathbf{w}_{k}, b_{k}) = (-\mathbf{w}, -b), \quad k \in \mathcal{N}^{-}.
  \end{equation}
\end{definition}

Note that the splitting is defined only up to a global sign flip: the same set of neurons splits into an aligned subgroup $\mathcal{N}^{-}$ and an opposite subgroup $\mathcal{N}^{+}$ with respect to parameters $(-\mathbf{w}, -b)$.
In \cref{app-subsubsec:choice-of-sign-aligned-opposite-groups}, we show that, for some of the symmetry groups to be introduced in \cref{app-subsubsec:symmetries-even-linear-activations,app-subsubsec:symmetries-constant-odd-activations}, the labels ``aligned'' and ``opposite'' are indeed interchangeable, whereas for others they acquire a semantic meaning that removes this sign ambiguity altogether.

Next, we introduce two nondegeneracy conditions that yield \emph{minimal} activation-\emph{dependent} symmetry groups, paralleling the role of the subset-nonzero condition in ensuring minimality for activation-\emph{independent} symmetry groups (\cref{def:subset-nonzero}).

\begin{definition}[Nondegeneracy of aligned/opposite groups]
  \label[definition]{def:nondegeneracy-aligned-opposite-groups}
  Let $\mathcal{K}$ index an aligned/opposite group with subgroups indexed by $\mathcal{N}^{+}$ and $\mathcal{N}^{-}$.
  For every nonempty subset $\mathcal{J} \subseteq \mathcal{K}$, we define
\begin{equation}
    \mathbf{a}_{\mathcal{J}}^{\pm} \coloneqq \sum_{k \in \mathcal{J} \cap \mathcal{N}^{+}} \mathbf{a}_{k} + \sum_{k \in \mathcal{J} \cap \mathcal{N}^{-}} \mathbf{a}_{k},
    \qquad
    \mathbf{a}_{\mathcal{J}}^{\mp} \coloneqq \sum_{k \in \mathcal{J} \cap \mathcal{N}^{+}} \mathbf{a}_{k} - \sum_{k \in \mathcal{J} \cap \mathcal{N}^{-}} \mathbf{a}_{k}.
  \end{equation}
We write
\begin{equation}
    \mathbf{a}^{\pm} \coloneqq \mathbf{a}_{\mathcal{K}}^{\pm} = \sum_{k \in \mathcal{N}^{+}} \mathbf{a}_{k} + \sum_{k \in \mathcal{N}^{-}} \mathbf{a}_{k},
    \qquad
    \mathbf{a}^{\mp} \coloneqq \mathbf{a}_{\mathcal{K}}^{\mp} = \sum_{k \in \mathcal{N}^{+}} \mathbf{a}_{k} - \sum_{k \in \mathcal{N}^{-}} \mathbf{a}_{k}
  \end{equation}
for the aggregate weights of the full group.
  An aligned/opposite group is called \emph{sum-nondegenerate} if $\mathbf{a}_{\mathcal{J}}^{\pm} \neq \mathbf{0}$ for every nonempty proper subset $\mathcal{J} \subsetneq \mathcal{K}$, and \emph{difference-nondegenerate} if $\mathbf{a}_{\mathcal{J}}^{\mp} \neq \mathbf{0}$ for every such $\mathcal{J}$.
\end{definition}

Both notions are invariant under the global sign flip described above: exchanging $\mathcal{N}^{+}$ and $\mathcal{N}^{-}$ leaves $\mathbf{a}_{\mathcal{J}}^{\pm}$ unchanged and negates $\mathbf{a}_{\mathcal{J}}^{\mp}$.

\subsubsection{Symmetries arising from even-linear activations}
\label[appendix]{app-subsubsec:symmetries-even-linear-activations}

Assume $\sigma(z) = e(z) + mz$ is even-linear in the sense of \cref{def:even-linear-activations}.
For an aligned/opposite group with parameters $(\mathbf{w}, b)$, letting $z_{k} \coloneqq \mathbf{w}_{k}^{\top} \mathbf{x} + b_{k}$ and $z \coloneqq \mathbf{w}^{\top} \mathbf{x} + b$, we have $z_{k} = z$ for $k \in \mathcal{N}^{+}$ and $z_{k} = -z$ for $k \in \mathcal{N}^{-}$.
Thus, the combined contribution of the aligned/opposite group to the layer output is
\begin{equation}
  \label{eq:aligned-opposite-contribution-even-linear}
  \sum_{k} \mathbf{a}_{k} \, \sigma(z_{k})
  = \sum_{k} \mathbf{a}_{k} \, e(z_{k}) + \sum_{k} \mathbf{a}_{k} \, m z_{k}
  = \mathbf{a}^{\pm} e(z) + \mathbf{a}^{\mp} m z.
\end{equation}
Hence, the aggregate readout weights $\mathbf{a}^{\pm}$ and $\mathbf{a}^{\mp}$ control the even and linear components of the group's contribution, respectively.
In particular, an aligned/opposite group with $\mathbf{a}^{\pm} = \mathbf{0}$ generates a linear contribution, while an aligned/opposite group satisfying $\mathbf{a}^{\pm} = \mathbf{a}^{\star}$ reproduces the contribution of a single reference neuron with parameters $(\mathbf{w}, b, \mathbf{a}^{\star})$ up to a residual linear contribution.
This motivates the following two symmetry groups, which correspond to the ``even + linear'' symmetries of \textcite{martinelli2024expandcluster}.

\begin{definition}[Linear-neuron group]
  \label[definition]{def:linear-neuron-group}
  Suppose $\sigma(z) = e(z) + mz$ is even-linear.
  A \emph{linear-neuron group} is a sum-nondegenerate aligned/opposite group with parameters $(\mathbf{w}, b)$ such that $\mathbf{a}^{\pm} = \mathbf{0}$.
\end{definition}

\begin{definition}[Linear-duplicate-neuron group]
  \label[definition]{def:linear-duplicate-neuron-group}
  Suppose $\sigma(z) = e(z) + mz$ is even-linear, and fix parameters $(\mathbf{w}^{\star}, b^{\star}, \mathbf{a}^{\star})$ with $\mathbf{a}^{\star} \neq \mathbf{0}$.
  A \emph{linear-duplicate-neuron group} is a sum-nondegenerate aligned/opposite group with parameters $(\mathbf{w}^{\star}, b^{\star})$ such that $\mathbf{a}^{\pm} = \mathbf{a}^{\star}$.
\end{definition}

\subsubsection{Symmetries arising from constant-odd activations}
\label[appendix]{app-subsubsec:symmetries-constant-odd-activations}

Now let $\sigma(z) = c + o(z)$ be constant-odd in the sense of \cref{def:constant-odd-activations}.
Using the same notation as before, the combined contribution of an aligned/opposite group with parameters $(\mathbf{w}, b)$ is
\begin{equation}
  \label{eq:aligned-opposite-contribution-constant-odd}
  \sum_{k} \mathbf{a}_{k} \, \sigma(z_{k})
  = \sum_{k} \mathbf{a}_{k} \, c + \sum_{k} \mathbf{a}_{k} \, o(z_{k})
  = \mathbf{a}^{\pm} c + \mathbf{a}^{\mp} o(z).
\end{equation}
Analogously to the even-linear case, $\mathbf{a}^{\pm}$ and $\mathbf{a}^{\mp}$ now control the constant and odd components, respectively.
In particular, choosing $\mathbf{a}^{\mp} = \mathbf{0}$ generates a purely constant contribution, while an aligned/opposite group satisfying $\mathbf{a}^{\mp} = \mathbf{a}^{\star}$ reproduces the contribution of a single reference neuron with parameters $(\mathbf{w}, b, \mathbf{a}^{\star})$ up to a residual constant contribution.
This motivates the following two symmetry groups, corresponding to the ``odd (+ constant)'' symmetries of \textcite{martinelli2024expandcluster}.

\begin{definition}[Constant-neuron group]
  \label[definition]{def:constant-neuron-group}
  Suppose $\sigma(z) = c + o(z)$ is constant-odd.
  A \emph{constant-neuron group} is a difference-nondegenerate aligned/opposite group with parameters $(\mathbf{w}, b)$ such that $\mathbf{a}^{\mp} = \mathbf{0}$.
\end{definition}

\begin{definition}[Constant-duplicate-neuron group]
  \label[definition]{def:constant-duplicate-neuron-group}
  Suppose $\sigma(z) = c + o(z)$ is constant-odd, and fix parameters $(\mathbf{w}^{\star}, b^{\star}, \mathbf{a}^{\star})$ with $\mathbf{a}^{\star} \neq \mathbf{0}$.
  A \emph{constant-duplicate-neuron group} is a difference-nondegenerate aligned/opposite group with parameters $(\mathbf{w}^{\star}, b^{\star})$ such that $\mathbf{a}^{\mp} = \mathbf{a}^{\star}$.
\end{definition}

Note that constant-neuron \emph{groups} are distinct from the \emph{individual} constant neurons introduced in \cref{def:constant-neuron}: the former involve aligned/opposite groups with nonzero incoming weights, whereas the latter are individual neurons with vanishing incoming weights.

\subsubsection{Choice of sign in aligned/opposite groups}
\label[appendix]{app-subsubsec:choice-of-sign-aligned-opposite-groups}

As noted after \cref{def:aligned-opposite-group}, the decomposition of an aligned/opposite group into subgroups $\mathcal{N}^{+}$ and $\mathcal{N}^{-}$ is defined only up to a global sign flip of the reference parameters $(\mathbf{w}, b)$.

For nonduplicate symmetry groups, this ambiguity is intrinsic and irrelevant.
Linear-neuron groups and constant-neuron groups are defined by the conditions $\mathbf{a}^{\pm} = \mathbf{0}$ and $\mathbf{a}^{\mp} = \mathbf{0}$, respectively, which are symmetric under exchanging the aligned and opposite subgroups.
In these cases, there is no intrinsic meaning attached to the labels ``aligned'' and ``opposite'': either choice yields the same symmetry group.

The duplicate variants are conceptually different.
They are intended to replicate the behavior of a single reference neuron with parameters $(\mathbf{w}^{\star}, b^{\star}, \mathbf{a}^{\star})$, which induces a distinguished notion of alignment \emph{relative to that neuron}.
This semantic ``orientation'' is present for both linear-duplicate and constant-duplicate groups.
In the linear-duplicate case, the defining condition $\mathbf{a}^{\pm} = \mathbf{a}^{\star}$ happens to be compatible with either choice of aligned/opposite decomposition, so that the sign ambiguity remains at the level of the formal definition.
In contrast, for constant-duplicate groups, the defining condition $\mathbf{a}^{\mp} = \mathbf{a}^{\star} \neq \mathbf{0}$ selects a unique admissible decomposition, since reversing the sign would violate the defining equation.

Thus, while the aligned/opposite decomposition is a priori sign-ambiguous, this ambiguity either plays no role (for nonduplicate groups), matters semantically but not algebraically (for linear-duplicate groups), or is resolved by the defining equations themselves (for constant-duplicate groups).

\subsubsection{Function-preserving realizations}
\label[appendix]{app-subsubsec:activation-dependent-symmetries-function-preserving-realizations}

\paragraph{Even-linear activations.}
For even-linear $\sigma(z) = e(z) + mz$, the combined contribution of an aligned/opposite group with parameters $(\mathbf{w}, b)$ equals $\mathbf{a}^{\pm} e(z) + \mathbf{a}^{\mp} m z$, where $z \coloneqq \mathbf{w}^{\top} \mathbf{x} + b$.
A linear-neuron group ($\mathbf{a}^{\pm} = \mathbf{0}$,  \cref{def:linear-neuron-group}) thus contributes a linear term:
\begin{equation}
  \sum_{k} \mathbf{a}_{k} \, \sigma(z_{k})
  = \mathbf{a}^{\pm} e(z) + \mathbf{a}^{\mp} m z
  = \mathbf{a}^{\mp} m z.
\end{equation}
Replacing an individual neuron with parameters $(\mathbf{w}^{\star}, b^{\star}, \mathbf{a}^{\star})$ by a linear-duplicate-neuron group ($\mathbf{a}^{\pm} = \mathbf{a}^{\star}$, \cref{def:linear-duplicate-neuron-group}) reproduces the reference neuron's contribution up to a residual linear term:
\begin{equation}
  \sum_{k} \mathbf{a}_{k} \, \sigma(z_{k})
  = \mathbf{a}^{\pm} e(z^{\star}) + \mathbf{a}^{\mp} m z^{\star}
  = \mathbf{a}^{\star} \, \sigma(z^{\star}) + (\mathbf{a}^{\mp} - \mathbf{a}^{\star}) \, m z^{\star},
\end{equation}
where $z^{\star} \coloneqq {\mathbf{w}^{\star}}^{\top} \mathbf{x} + b^{\star}$.
For the realized function to remain unchanged, these (residual) contributions must cancel collectively across linear-neuron and linear-duplicate-neuron groups and individual constant neurons.

\paragraph{Even activations.}
When $\sigma$ is even, corresponding to the even-linear case with $m = 0$, both the linear contribution $\mathbf{a}^{\mp} mz$ of a linear-neuron group and the linear residual $(\mathbf{a}^{\mp} - \mathbf{a}^{\star}) \, m z^{\star}$ of a linear-duplicate-neuron group vanish.
Consequently, linear-neuron groups and linear-duplicate-neuron groups are always function-preserving.

\paragraph{Constant-odd activations.}
For constant-odd $\sigma(z) = c + o(z)$, the combined contribution of an aligned/opposite group with parameters $(\mathbf{w}, b)$ is $\mathbf{a}^{\pm} c + \mathbf{a}^{\mp} o(z)$, with $z$ as above.
Adding a constant-neuron group ($\mathbf{a}^{\mp} = \mathbf{0}$, \cref{def:constant-neuron-group}) thus contributes an input-independent offset:
\begin{equation}
  \sum_{k} \mathbf{a}_{k} \, \sigma(z_{k})
  = \mathbf{a}^{\pm} c + \mathbf{a}^{\mp} o(z)
  = \mathbf{a}^{\pm} c.
\end{equation}
Replacing an individual neuron with parameters $(\mathbf{w}^{\star}, b^{\star}, \mathbf{a}^{\star})$ by a constant-duplicate-neuron group ($\mathbf{a}^{\mp} = \mathbf{a}^{\star}$, \cref{def:constant-duplicate-neuron-group}) reproduces the reference neuron's contribution up to a residual constant term:
\begin{equation}
  \sum_{k} \mathbf{a}_{k} \, \sigma(z_{k})
  = \mathbf{a}^{\pm} c + \mathbf{a}^{\mp} o(z^{\star})
  = \mathbf{a}^{\star} \, \sigma(z^{\star}) + (\mathbf{a}^{\pm} - \mathbf{a}^{\star}) \, c,
\end{equation}
with $z^{\star}$ as above.
For the realized function to remain unchanged, these constant contributions must cancel collectively across constant-neuron and constant-duplicate-neuron groups and individual constant neurons (\cref{def:constant-neuron}).

\paragraph{Odd activations.}
When $\sigma$ is odd, corresponding to the constant-odd case with $c = 0$, both the constant contribution $\mathbf{a}^{\pm} c$ of a constant-neuron group and the constant residual $(\mathbf{a}^{\pm} - \mathbf{a}^{\star}) \, c$ of a constant-duplicate-neuron group vanish.
Consequently, constant-neuron groups and constant-duplicate-neuron groups are always function-preserving.

Throughout, a finite collection of additions, replacements, removals, or collapses of the applicable activation-dependent groups and individual constant neurons is regarded as a single function-preserving transformation whenever its net residual contribution vanishes identically.
The constituent operations need not preserve the function in isolation.

\subsubsection{Minimality and nondegeneracy}
\label[appendix]{app-subsubsec:activation-dependent-symmetries-minimality-and-nondegeneracy}

Under our standing exclusion of affine activations, the structural conditions in \cref{def:aligned-opposite-group}, the class-specific nondegeneracy conditions in \cref{def:nondegeneracy-aligned-opposite-groups}, and the requirement $\mathbf{a}^{\star} \neq \mathbf{0}$ in the duplicate variants jointly ensure that the four activation-dependent symmetry classes (\cref{def:linear-neuron-group,def:linear-duplicate-neuron-group,def:constant-neuron-group,def:constant-duplicate-neuron-group}) form a minimal taxonomy of mutually distinct classes disjoint from the activation-independent classes of \cref{app-subsec:activation-independent-symmetries}.
Building on the definitions of \textcite{martinelli2024expandcluster}, these conditions prevent activation-dependent groups from decomposing into smaller activation-dependent or activation-independent groups or degenerating into another activation-dependent class.
This subsection treats each condition in turn, spelling out the degeneracy it rules out.

\paragraph{Nonzero shared incoming weights ($\mathbf{w} \neq \mathbf{0}$).}
If $\mathbf{w} = \mathbf{0}$, every neuron in the aligned/opposite group has incoming weights $\pm \mathbf{w} = \mathbf{0}$ and therefore qualifies as an individual constant neuron in the sense of \cref{def:constant-neuron}.
The combined contribution reduces to $\mathbf{a}^{\pm} e(b) + \mathbf{a}^{\mp} m b$ for an even-linear activation, and to $\mathbf{a}^{\pm} c + \mathbf{a}^{\mp} o(b)$ for a constant-odd activation, in either case indistinguishable from the contribution of an unstructured collection of constant neurons.
Requiring $\mathbf{w} \neq \mathbf{0}$ thus ensures that activation-dependent symmetry groups capture genuinely input-dependent structure.

\paragraph{Two nonempty subgroups ($\mathcal{N}^{+}, \mathcal{N}^{-} \neq \emptyset$).}
If one of the subgroups were empty, all neurons would share the same incoming parameters.
For the nonduplicate classes, $\mathbf{a}^{\pm} = \mathbf{0}$ (linear-neuron group) and $\mathbf{a}^{\mp} = \mathbf{0}$ (constant-neuron group) then both become $\sum_{k} \mathbf{a}_{k} = \mathbf{0}$, the defining condition of a zero-neuron group, whichever subgroup is empty.
For the duplicate classes, by contrast, the resulting activation-independent symmetry depends on which subgroup is empty.
If $\mathcal{N}^{-} = \emptyset$, both $\mathbf{a}^{\pm} = \mathbf{a}^{\star}$ (linear-duplicate-neuron group) and $\mathbf{a}^{\mp} = \mathbf{a}^{\star}$ (constant-duplicate-neuron group) become $\sum_{k} \mathbf{a}_{k} = \mathbf{a}^{\star}$, the defining condition of a duplicate-neuron group of the reference neuron.
If $\mathcal{N}^{+} = \emptyset$, they become $\sum_{k} \mathbf{a}_{k} = \mathbf{a}^{\star}$ and $\sum_{k} \mathbf{a}_{k} = -\mathbf{a}^{\star}$, respectively.
The group then duplicates the flipped neuron $(-\mathbf{w}^{\star}, -b^{\star}, \mathbf{a}^{\star})$ or $(-\mathbf{w}^{\star}, -b^{\star}, -\mathbf{a}^{\star})$, whose contribution differs from that of the reference neuron by $-2m \, \mathbf{a}^{\star} z^{\star}$ or $-2c \, \mathbf{a}^{\star}$, respectively.
Requiring both subgroups to be nonempty thus prevents activation-dependent symmetry groups from reducing to activation-independent ones, possibly applied to a flipped copy of the reference neuron.

\paragraph{Sum- and difference-nondegeneracy.}
For linear-neuron and linear-duplicate-neuron groups, the relevant aggregate of a subset $\mathcal{J} \subseteq \mathcal{K}$ is $\mathbf{a}_{\mathcal{J}}^{\pm}$, whereas for constant-neuron and constant-duplicate-neuron groups it is $\mathbf{a}_{\mathcal{J}}^{\mp}$.
The applicable nondegeneracy condition requires this aggregate to be nonzero for every nonempty proper subset $\mathcal{J} \subsetneq \mathcal{K}$.

Suppose that this condition is violated and, among all nonempty proper subsets with vanishing relevant aggregate, choose the smallest such subset $\mathcal{J}$.
If $\mathcal{J}$ lies entirely within $\mathcal{N}^{+}$ or $\mathcal{N}^{-}$, its relevant aggregate (i.e., $\mathbf{a}_{\mathcal{J}}^{\pm}$ or $\mathbf{a}_{\mathcal{J}}^{\mp}$) agrees, up to sign, with the sum of its readout weights.
By the choice of $\mathcal{J}$, no nonempty proper subset of these readout weights sums to zero, so $\mathcal{J}$ is subset-nonzero (\cref{def:subset-nonzero}).
Because their sum vanishes, the neurons indexed by $\mathcal{J}$ form a zero-neuron group.
If instead $\mathcal{J}$ intersects both $\mathcal{N}^{+}$ and $\mathcal{N}^{-}$, the same choice ensures that it is sum-nondegenerate or difference-nondegenerate, as applicable.
The neurons indexed by $\mathcal{J}$ therefore form a smaller linear-neuron or constant-neuron group.
In either case, the complementary subset $\mathcal{K} \setminus \mathcal{J}$ retains the defining aggregate of the original group (i.e., $\mathbf{0}$ for a nonduplicate group and $\mathbf{a}^{\star}$ for a duplicate group) so that the original group is decomposable and hence not minimal.

Because $\mathcal{N}^{+}$ and $\mathcal{N}^{-}$ are themselves nonempty proper subsets of $\mathcal{K}$, either nondegeneracy condition also requires both these subgroups to have nonzero aggregate readout.
Indeed,
\begin{equation}
  \mathbf{a}^{\pm} = \mathbf{a}^{\mp}
  \iff
  \sum_{k \in \mathcal{N}^{-}} \mathbf{a}_{k} = \mathbf{0},
  \qquad
  \mathbf{a}^{\pm} = -\mathbf{a}^{\mp}
  \iff
  \sum_{k \in \mathcal{N}^{+}} \mathbf{a}_{k} = \mathbf{0}.
\end{equation}
If either equality held, the neurons in one entire subgroup could be decomposed into zero-neuron groups.
In the nonduplicate case, the defining equation would force the aggregate readout of the other subgroup to vanish as well, so the full group would decompose into zero-neuron groups.
In the duplicate case, the remaining subgroup consists entirely of neurons with identical incoming parameters and can therefore be reduced to a single neuron, possibly together with zero-neuron groups.
Thus, in either case, the activation-dependent group would decompose into the activation-independent symmetry groups of \cref{app-subsec:activation-independent-symmetries}, possibly together with a flipped copy of the reference neuron.

Finally, in either duplicate class (i.e., linear or constant), a nonempty proper subset has the same relevant aggregate $\mathbf{a}^{\star}$ as the full group if and only if its nonempty complement has vanishing relevant aggregate.
The applicable nondegeneracy condition therefore also excludes proper duplicate subgroups with respect to the same reference neuron, directly paralleling the minimality argument for activation-independent duplicate-neuron groups in \cref{app-subsubsec:activation-independent-symmetries-minimality-and-distinctness}.

\paragraph{Nonzero reference readout ($\mathbf{a}^{\star} \neq \mathbf{0}$, duplicate variants).}
If $\mathbf{a}^{\star} = \mathbf{0}$, the defining conditions of the linear- and constant-duplicate-neuron groups, $\mathbf{a}^{\pm} = \mathbf{a}^{\star}$ and $\mathbf{a}^{\mp} = \mathbf{a}^{\star}$, would reduce to those of the linear- and constant-neuron groups, $\mathbf{a}^{\pm} = \mathbf{0}$ and $\mathbf{a}^{\mp} = \mathbf{0}$, making the duplicate variants indistinguishable from their nonduplicate counterparts.
Moreover, a reference neuron with $\mathbf{a}^{\star} = \mathbf{0}$ contributes nothing to the realized function, so ``duplicating'' it would carry no semantic content.
Requiring $\mathbf{a}^{\star} \neq \mathbf{0}$ keeps the duplicate and nonduplicate variants disjoint and ensures that duplicate variants genuinely replicate a nontrivial reference neuron.

\subsubsection{Nonuniqueness of irreducible parameterizations}
\label[appendix]{app-subsubsec:nonuniqueness-irreducible-parameterizations}

We now demonstrate that distinct irreducible parameterizations within the same symmetry orbit need not be related by generic reparameterization symmetries (\cref{app-subsec:generic-reparameterization-symmetries}) alone.

\begin{example}[Oppositely oriented irreducible ReLU parameterizations]
  \label[example]{ex:opposite-irreducible-relu-parameterizations}
  Consider a one-hidden-layer network with one-dimensional input, one-dimensional output, and the ReLU activation
\begin{equation}
    \psi(z) \coloneqq \max(0, z).
  \end{equation}
Let $\boldsymbol{\theta}^{+}$ be the width-$3$ parameterization consisting of the neurons
\begin{equation}
    (w_{1}, b_{1}, a_{1}) = (1, 1, 1),
    \qquad
    (w_{2}, b_{2}, a_{2}) = (1, 0, -2),
    \qquad
    (w_{3}, b_{3}, a_{3}) = (1, -1, 1),
  \end{equation}
visualized in the left panel of \cref{fig:opposite-orientations}.
  The three neurons have kinks at $-1$, $0$, and $1$, respectively, and together realize the \emph{tent function}
\begin{equation}
    f_{\boldsymbol{\theta}^{+}}(x)
    = \psi(x+1) - 2 \, \psi(x) + \psi(x-1)
    = \psi(1-\abs{x}).
  \end{equation}

  The same function is realized by the oppositely oriented parameterization $\boldsymbol{\theta}^{-}$ consisting of
\begin{equation}
    \begin{aligned}
      (-w_{1}, -b_{1}, a_{1}) &= (-1, -1, 1),\\
      (-w_{2}, -b_{2}, a_{2}) &= (-1, 0, -2),\\
      (-w_{3}, -b_{3}, a_{3}) &= (-1, 1, 1),
    \end{aligned}
  \end{equation}
shown in the center panel of \cref{fig:opposite-orientations}.
  Indeed, using $\psi(z) - \psi(-z) = z$, we obtain
\begin{equation}
    \begin{aligned}
      & f_{\boldsymbol{\theta}^{+}}(x)
        - f_{\boldsymbol{\theta}^{-}}(x) \\
      ={}& \bigl[ \psi(x+1) - \psi(-x-1) \bigr]
        - 2 \bigl[ \psi(x) - \psi(-x) \bigr]
        + \bigl[ \psi(x-1) - \psi(-x+1) \bigr] \\
      ={}& (x+1) - 2x + (x-1)
        = 0.
    \end{aligned}
  \end{equation}

  Moreover, these two parameterizations are symmetry-equivalent as defined in \cref{subsec:symmetry-orbits-irreducibility-and-overparameterization}.
  To see this, write
\begin{equation}
    z_{1} \coloneqq x+1,
    \qquad
    z_{2} \coloneqq x,
    \qquad
    z_{3} \coloneqq x-1.
  \end{equation}
Define a width-$6$ parameterization $\boldsymbol{\eta}$ consisting, for each $j \in \set{1,2,3}$, of the two neurons
\begin{equation}
    (w_{j}, b_{j}, -a_{j})
    \qquad \text{and} \qquad
    (-w_{j}, -b_{j}, a_{j}).
  \end{equation}
Since ReLU is even-linear (see \cref{app-subsec:even-linear-activations}) and $a_{j} \neq 0$, each such pair is sum-nondegenerate.
  Its aggregate readout vanishes, $a^{\pm}=-a_{j}+a_{j}=0$, so each pair forms a linear-neuron group.
  The contribution of the $j$th group to the function $f_{\boldsymbol{\eta}}$ realized by $\boldsymbol{\eta}$ is
\begin{equation}
    -a_{j} \, \psi(z_{j}) + a_{j} \, \psi(-z_{j})
    = -a_{j}z_{j}.
  \end{equation}
Collectively, these contributions cancel:
\begin{equation}
    f_{\boldsymbol{\eta}}(x)
    = -\sum_{j=1}^{3} a_{j}z_{j}
    = -(x+1) + 2x - (x-1)
    = 0.
  \end{equation}
Thus, $\boldsymbol{\eta}$ consists entirely of symmetry groups and realizes the zero function.
  The parameterization $\boldsymbol{\theta}^{+}$ is therefore symmetry-equivalent to the parameterization $\boldsymbol{\theta}^{+} \mkern-2mu \oplus \boldsymbol{\eta}$ obtained by combining all neurons from $\boldsymbol{\theta}^{+}$ and $\boldsymbol{\eta}$.
  In this combined parameterization, each original neuron $(w_{j}, b_{j}, a_{j})$ of $\boldsymbol{\theta}^{+}$ and its readout-negated copy $(w_{j}, b_{j}, -a_{j})$ of $\boldsymbol{\eta}$ form a zero-neuron group.
  Removing these three groups leaves only the oppositely oriented neurons of $\boldsymbol{\eta}$, which are exactly the neurons of $\boldsymbol{\theta}^{-}$, showing that $\boldsymbol{\theta}^{+} \sim \boldsymbol{\theta}^{-}$.

  \begin{figure}[t]
  \centering
  \includegraphics[width=0.95\textwidth]{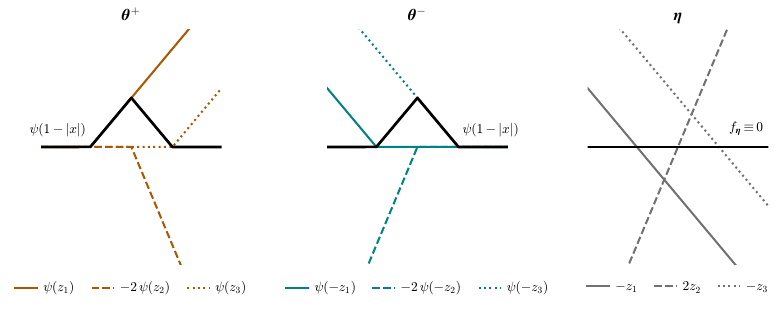}
  \caption{\textbf{Oppositely oriented irreducible parameterizations of the tent function}.
    \textbf{(Left)} A parameterization $\boldsymbol{\theta}^{+}$ of the tent function $\psi(1 - \abs{x}) = \max(0, 1 - \abs{x})$ composed of three ReLU neurons $(w_{j}, b_{j}, a_{j})$, all with positive incoming weight $w_{j} > 0$.
    \textbf{(Center)} The same function, realized by an implementation $\boldsymbol{\theta}^{-}$ consisting of the three neurons $(-w_{j}, -b_{j}, a_{j})$ with negated incoming parameters.
    \textbf{(Right)} The net contributions $-a_{j}z_{j}$ of the three linear-neuron groups in the parameterization $\boldsymbol{\eta}$ used to demonstrate symmetry-equivalence between $\boldsymbol{\theta}^{+}$ and $\boldsymbol{\theta}^{-}$.
    No linear-neuron group vanishes individually, but their contributions sum to zero, so that $f_{\boldsymbol{\eta}}$ realizes the zero function.
  }
  \label{fig:opposite-orientations}
\end{figure}

  The realized tent function $\psi(1 - \abs{x})$ has three distinct kinks with nonzero changes in slope, whereas each one-dimensional ReLU neuron can introduce at most one such kink.
  Any ReLU network realizing the tent function must therefore have width at least $3$.
  Since both $\boldsymbol{\theta}^{+}$ and $\boldsymbol{\theta}^{-}$ attain this minimum, it follows that both parameterizations are irreducible.

  Finally, the two parameterizations cannot be related by neuron permutations and positive scaling.
  Every neuron in $\boldsymbol{\theta}^{+}$ has positive incoming weight, whereas every neuron in $\boldsymbol{\theta}^{-}$ has negative incoming weight, and neither permutation nor positive scaling can reverse this orientation.
\end{example}

\Cref{ex:opposite-irreducible-relu-parameterizations} illustrates three important facts about symmetry-equivalent parameterizations.
\begin{itemize}
    \item The activation-dependent symmetry groups discussed in this subsection need not be function-preserving individually.
    In fact, none of the linear-neuron groups in \cref{ex:opposite-irreducible-relu-parameterizations} is function-preserving when added on its own; only their joint addition yields a function-preserving transformation.

    \item Irreducible parameterizations are \emph{not} unique up to generic reparameterization symmetries (\cref{app-subsec:generic-reparameterization-symmetries}) when the activation admits any of the activation-dependent symmetries discussed in this subsection.

    \item Finally, the construction of the parameterization $\boldsymbol{\eta}$ used to establish symmetry-equivalence between the oppositely oriented parameterizations $\boldsymbol{\theta}^{+}$ and $\boldsymbol{\theta}^{-}$ generalizes to yield a quantitative criterion characterizing when two parameterizations are symmetry-equivalent (see the proof of \cref{prop:characterization-of-symmetry-equivalence-all-activation-classes}).
\end{itemize}

\subsection{Parameter classes, orbit invariants, and symmetry-equivalence}
\label[appendix]{app-subsec:parameter-classes-orbit-invariants-and-symmetry-equivalence}

The definition of symmetry-equivalence established in \cref{subsec:symmetry-orbits-irreducibility-and-overparameterization} of the main text is not directly operational: given parameterizations $\boldsymbol{\theta}$ and $\boldsymbol{\xi}$, we lack a criterion for deciding whether the two are symmetry-equivalent.
To establish such a criterion, we first introduce \emph{parameter classes} and rewrite the realized function $f_{\boldsymbol{\theta}}$ in terms of them (\cref{app-subsubsec:parameter-classes}).
We then show that these classes yield a quantitative, easily checkable criterion for symmetry-equivalence (\cref{app-subsubsec:characterization-of-symmetry-equivalence}).

\paragraph{Terminology.}
Positively homogeneous activations of degree~$1$ form an important subclass of even-linear activations (\cref{cor:positively-homogeneous-functions-are-even-linear}).
Because their inherent positive-scaling symmetry affects the theory developed in this subsection, they must be treated separately from even-linear activations that are not positively $1$-homogeneous.
Accordingly, throughout this subsection, we reserve the term \emph{even-linear} exclusively for even-linear activations that are \emph{not} positively $1$-homogeneous.

\subsubsection{Parameter classes}
\label[appendix]{app-subsubsec:parameter-classes}

As in the main text, we write
\begin{equation}
  \overline{\mathbf{w}}
  \coloneqq
  (\mathbf{w}^{\top}, b)^{\top},
  \qquad
  \overline{\mathbf{x}}
  \coloneqq
  (\mathbf{x}^{\top}, 1)^{\top},
  \qquad
  z
  \coloneqq
  \overline{\mathbf{w}}^{\top} \overline{\mathbf{x}}.
\end{equation}

\begin{definition}[Parameter classes]
  \label[definition]{def:parameter-classes}
  Fix a nonaffine activation $\sigma$ and an incoming-parameter vector $\overline{\mathbf{w}} = (\mathbf{w}^{\top}, b)^{\top}$ with $\mathbf{w} \neq \mathbf{0}$.
  Its \emph{parameter class} is
\begin{equation}
    \label{eq:parameter-classes-by-activation}
    [\overline{\mathbf{w}}] \coloneqq
\begin{cases}
      \set{\alpha \overline{\mathbf{w}} \given \alpha \in \mathbb{R}_{\neq 0}},
      &
      \sigma \text{ is positively homogeneous of degree~} 1,
      \\
      \set{\overline{\mathbf{w}}, -\overline{\mathbf{w}}},
      &
      \sigma \text{ is even-linear or constant-odd},
      \\
      \set{\overline{\mathbf{w}}},
      &
      \text{ otherwise}.
    \end{cases}
  \end{equation}
Mirroring \cref{subsec:orbit-invariants-and-essential-parameter-classes}, we denote the parameter classes represented by a parameterization $\boldsymbol{\theta}$ by
\begin{equation}
    \label{eq:represented-parameter-classes}
    \mathcal{Q}(\boldsymbol{\theta})
    \coloneqq
    \set{
      [\overline{\mathbf{w}}_{j}] \given \mathbf{w}_{j} \neq \mathbf{0}
    },
  \end{equation}
and we let $\mathcal{J}_q \coloneqq \set{j \given \overline{\mathbf{w}}_j \in q}$ index the neurons of class $q \in \mathcal{Q}(\boldsymbol{\theta})$.
\end{definition}

For every $q \in \mathcal{Q}(\boldsymbol{\theta})$, we fix a representative $\overline{\mathbf{w}}_{q} \in q$.
When $\sigma$ is positively homogeneous of degree~$1$, we choose $\overline{\mathbf{w}}_{q}$ to have unit norm, $\norm{\overline{\mathbf{w}}_{q}} = 1$.
When $\sigma$ is constant-odd, we choose the unique element of $q$ whose incoming-weight vector has a positive first nonzero entry.
We write
\begin{equation}
  \label{eq:positive-orientation-index-set}
  \mathcal{J}_{q}^{+}
  \coloneqq
  \set{
    j \in \mathcal{J}_{q}
    \given
    \overline{\mathbf{w}}_{j}
    =
    \alpha \overline{\mathbf{w}}_{q}
    \text{ for some } \alpha > 0
  },
\end{equation}
and, analogously,
\begin{equation}
  \label{eq:negative-orientation-index-set}
  \mathcal{J}_{q}^{-}
  \coloneqq
  \set{
    j \in \mathcal{J}_{q}
    \given
    \overline{\mathbf{w}}_{j}
    =
    -\alpha \overline{\mathbf{w}}_{q}
    \text{ for some } \alpha > 0
  },
\end{equation}
so that $\mathcal{J}_{q} = \mathcal{J}_{q}^{+} \sqcup \mathcal{J}_{q}^{-}$.
Finally, as in the main text, let $\mathcal{J}_{0} \coloneqq \set{j \given \mathbf{w}_{j} = \mathbf{0}}$ index all constant neurons of a parameterization.

\paragraph{Positively homogeneous activations of degree~$\mathbf{1}$.}
Let $\sigma$ be a positively homogeneous activation of degree~$1$.
By \cref{cor:positively-homogeneous-functions-are-even-linear}, we can write $\sigma(z) = \delta \abs{z} + mz$, where
\begin{equation}
  \delta \coloneqq \frac{\lambda_{+} - \lambda_{-}}{2},
  \qquad
  m \coloneqq \frac{\lambda_{+} + \lambda_{-}}{2},
  \qquad
  \lambda_{-}, \lambda_{+} \in \mathbb{R}.
\end{equation}
For $q = \set{\alpha \overline{\mathbf{w}} \given \alpha \in \mathbb{R}_{\neq 0}} \in \mathcal{Q}(\boldsymbol{\theta})$, define
\begin{equation}
  \label{eq:orbit-invariants-positively-homogeneous}
  \phi_{q}(\mathbf{x})
  \coloneqq
  \delta
  \frac{\abs{z}}{\norm{\overline{\mathbf{w}}}},
  \qquad
  \boldsymbol{\beta}_{q}(\boldsymbol{\theta})
  \coloneqq
  \sum_{j \in \mathcal{J}_{q}}
  \norm{\overline{\mathbf{w}}_{j}} \, \mathbf{a}_{j},
\end{equation}
and
\begin{equation}
  \label{eq:global-residual-positively-homogeneous}
  \mathbf{r}_{\boldsymbol{\theta}}(\mathbf{x})
  \coloneqq
  m
  \sum_{j \notin \mathcal{J}_{0}}
  \mathbf{a}_{j} z_{j}
  + \sum_{j \in \mathcal{J}_{0}}
  \mathbf{a}_{j} \sigma(b_{j}).
\end{equation}
Note that $\phi_{q} \colon \mathbb{R}^{N_{i}} \to \mathbb{R}$ depends only on the parameter class $q = \set{\alpha \overline{\mathbf{w}} \given \alpha \in \mathbb{R}_{\neq 0}}$, \emph{not} on the particular representative, since $\abs{z} / \norm{\overline{\mathbf{w}}}$ is unchanged when $\overline{\mathbf{w}}$ is replaced by $\alpha \overline{\mathbf{w}}$ for any $\alpha \neq 0$.

\paragraph{Even-linear activations.}
Let $\sigma(z) = e(z) + mz$, with $e$ even.
For $q = \set{\overline{\mathbf{w}}, -\overline{\mathbf{w}}} \in \mathcal{Q}(\boldsymbol{\theta})$, define
\begin{equation}
  \label{eq:orbit-invariants-even-linear}
  \phi_{q}(\mathbf{x})
  \coloneqq
  e(\overline{\mathbf{w}}^{\top} \overline{\mathbf{x}}),
  \qquad
  \boldsymbol{\beta}_{q}(\boldsymbol{\theta})
  \coloneqq
  \sum_{j \in \mathcal{J}_{q}}
  \mathbf{a}_{j},
\end{equation}
and
\begin{equation}
  \label{eq:global-residual-even-linear}
  \mathbf{r}_{\boldsymbol{\theta}}(\mathbf{x})
  \coloneqq
  m
  \sum_{j \notin \mathcal{J}_{0}}
  \mathbf{a}_{j} z_{j}
  +
  \sum_{j \in \mathcal{J}_{0}}
  \mathbf{a}_{j} \sigma(b_{j}).
\end{equation}
Again, the definition of $\phi_{q}$ is independent of the choice of representative, since $e(-\overline{\mathbf{w}}^{\top} \overline{\mathbf{x}}) = e(\overline{\mathbf{w}}^{\top} \overline{\mathbf{x}})$.

\paragraph{Constant-odd activations.}
Let $\sigma(z) = c + o(z)$, with $o$ odd.
For $q = \set{\overline{\mathbf{w}}_{q}, -\overline{\mathbf{w}}_{q}} \in \mathcal{Q}(\boldsymbol{\theta})$, define
\begin{equation}
  \label{eq:orbit-invariants-constant-odd}
  \phi_{q}(\mathbf{x})
  \coloneqq
  o(\overline{\mathbf{w}}_{q}^{\top}\overline{\mathbf{x}}),
  \qquad
  \boldsymbol{\beta}_{q}(\boldsymbol{\theta})
  \coloneqq
  \sum_{j \in \mathcal{J}_{q}^{+}}
  \mathbf{a}_{j}
  -
  \sum_{j \in \mathcal{J}_{q}^{-}}
  \mathbf{a}_{j},
\end{equation}
and
\begin{equation}
  \label{eq:global-residual-constant-odd}
  \mathbf{r}_{\boldsymbol{\theta}}(\mathbf{x})
  \coloneqq
  c
  \sum_{j \notin \mathcal{J}_{0}}
  \mathbf{a}_{j}
  +
  \sum_{j \in \mathcal{J}_{0}}
  \mathbf{a}_{j} \sigma(b_{j}).
\end{equation}
For constant-odd activations, unlike the positively $1$-homogeneous and even-linear cases, the definitions of $\boldsymbol{\beta}_{q}(\boldsymbol{\theta})$ and $\phi_{q}$ depend on the chosen representative of $q$.
Replacing $\overline{\mathbf{w}}_{q}$ by $-\overline{\mathbf{w}}_{q}$ exchanges $\mathcal{J}_{q}^{+}$ and $\mathcal{J}_{q}^{-}$ in \cref{eq:positive-orientation-index-set,eq:negative-orientation-index-set}, thereby negating $\boldsymbol{\beta}_{q}(\boldsymbol{\theta})$.
Under the same replacement, $\phi_{q}$ changes sign because $o(-\overline{\mathbf{w}}_{q}^{\top} \overline{\mathbf{x}}) = -o(\overline{\mathbf{w}}_{q}^{\top} \overline{\mathbf{x}})$.
Hence, the product $\boldsymbol{\beta}_{q}(\boldsymbol{\theta}) \phi_{q}(\mathbf{x})$ is independent of the representative, but fixing $\overline{\mathbf{w}}_{q}$ is necessary to define the two factors individually.

\paragraph{Activations with neither symmetry property.}
Let $\sigma$ be neither even-linear nor constant-odd.
For $q = \set{\overline{\mathbf{w}}} \in \mathcal{Q}(\boldsymbol{\theta})$, define
\begin{equation}
  \label{eq:orbit-invariants-neither-symmetry-property}
  \phi_{q}(\mathbf{x})
  \coloneqq
  \sigma(\overline{\mathbf{w}}^{\top} \overline{\mathbf{x}}),
  \qquad
  \boldsymbol{\beta}_{q}(\boldsymbol{\theta})
  \coloneqq
  \sum_{j \in \mathcal{J}_{q}}
  \mathbf{a}_{j},
\end{equation}
and
\begin{equation}
  \label{eq:global-residual-neither-symmetry-property}
  \mathbf{r}_{\boldsymbol{\theta}}(\mathbf{x})
  \coloneqq
  \sum_{j \in \mathcal{J}_{0}}
  \mathbf{a}_{j} \sigma(b_{j}),
\end{equation}
which is constant in $\mathbf{x} \in \mathbb{R}^{N_{i}}$.

In all four cases, the definitions of $\boldsymbol{\beta}_{q}(\boldsymbol{\theta})$, $\phi_{q}(\mathbf{x})$, and $\mathbf{r}_{\boldsymbol{\theta}}(\mathbf{x})$ in \cref{eq:orbit-invariants-positively-homogeneous,eq:global-residual-positively-homogeneous,eq:orbit-invariants-even-linear,eq:global-residual-even-linear,eq:orbit-invariants-constant-odd,eq:global-residual-constant-odd,eq:orbit-invariants-neither-symmetry-property,eq:global-residual-neither-symmetry-property} allow us to express the function $f_{\boldsymbol{\theta}}$ realized by the parameterization $\boldsymbol{\theta}$ as follows:
\begin{equation}
  \label{eq:function-decomposition-all-activation-classes}
  f_{\boldsymbol{\theta}}(\mathbf{x})
  =
  \sum_{q \in \mathcal{Q}(\boldsymbol{\theta})}
  \boldsymbol{\beta}_{q}(\boldsymbol{\theta}) \, \phi_{q}(\mathbf{x})
  +
  \mathbf{r}_{\boldsymbol{\theta}}(\mathbf{x}),
  \qquad
  \mathbf{x} \in \mathbb{R}^{N_{i}}.
\end{equation}

\subsubsection{Characterization of symmetry-equivalence}
\label[appendix]{app-subsubsec:characterization-of-symmetry-equivalence}

We now extend the characterization of symmetry-equivalent parameterizations from the main text to all activation classes considered here.

\begin{proposition}[Characterization of symmetry-equivalence]
  \label[proposition]{prop:characterization-of-symmetry-equivalence-all-activation-classes}
  Two parameterizations $\boldsymbol{\theta}$ and $\boldsymbol{\xi}$, possibly of different widths, with the same nonaffine activation $\sigma$, are symmetry-equivalent if and only if
\begin{equation}
    \label{eq:characterization-of-symmetry-equivalence-all-activation-classes}
    \boldsymbol{\beta}_{q}(\boldsymbol{\theta})
    =
    \boldsymbol{\beta}_{q}(\boldsymbol{\xi})
    \quad
    \text{for every }
    q \in
    \mathcal{Q}(\boldsymbol{\theta})
    \cup
    \mathcal{Q}(\boldsymbol{\xi}),
    \qquad
    \mathbf{r}_{\boldsymbol{\theta}}
    =
    \mathbf{r}_{\boldsymbol{\xi}},
  \end{equation}
with $\boldsymbol{\beta}_{q}(\cdot)$ understood to be zero for an absent class.
\end{proposition}

\begin{proof}
  Suppose first that $\boldsymbol{\theta}$ and $\boldsymbol{\xi}$ are symmetry-equivalent, $\boldsymbol{\theta} \sim \boldsymbol{\xi}$.

  \paragraph{Equality of the residuals $\mathbf{r}_{\boldsymbol{\theta}}$ and $\mathbf{r}_{\boldsymbol{\xi}}$.}
  We first show that equality of the aggregates $\boldsymbol{\beta}_{q}(\boldsymbol{\theta})$ and $\boldsymbol{\beta}_{q}(\boldsymbol{\xi})$ implies equality of the residual functions $\mathbf{r}_{\boldsymbol{\theta}}$ and $\mathbf{r}_{\boldsymbol{\xi}}$.
  Assume that $\boldsymbol{\beta}_{q}(\boldsymbol{\theta}) = \boldsymbol{\beta}_{q}(\boldsymbol{\xi})$ for every parameter class $q \in \mathcal{Q}(\boldsymbol{\theta}) \cup \mathcal{Q}(\boldsymbol{\xi})$.
  Since symmetry-equivalence implies $f_{\boldsymbol{\theta}} = f_{\boldsymbol{\xi}}$, the decomposition in \cref{eq:function-decomposition-all-activation-classes} gives
\begin{equation}
    \mathbf{r}_{\boldsymbol{\theta}}(\mathbf{x}) - \mathbf{r}_{\boldsymbol{\xi}}(\mathbf{x})
    =
    \underbrace{f_{\boldsymbol{\theta}}(\mathbf{x}) - f_{\boldsymbol{\xi}}(\mathbf{x})}_{=\mathbf{0}}
    -
    \sum_{q} (\underbrace{\boldsymbol{\beta}_{q}(\boldsymbol{\theta}) - \boldsymbol{\beta}_{q}(\boldsymbol{\xi})}_{=\mathbf{0}}) \, \phi_{q}(\mathbf{x})
    =
    \mathbf{0},
    \qquad
    \mathbf{x} \in \mathbb{R}^{N_{i}},
  \end{equation}
and hence $\mathbf{r}_{\boldsymbol{\theta}} = \mathbf{r}_{\boldsymbol{\xi}}$.

  \paragraph{Equality of the class aggregates $\boldsymbol{\beta}_{q}(\boldsymbol{\theta})$ and $\boldsymbol{\beta}_{q}(\boldsymbol{\xi})$.}
  It remains to show that symmetry-equivalence implies equality of $\boldsymbol{\beta}_{q}(\boldsymbol{\theta})$ and $\boldsymbol{\beta}_{q}(\boldsymbol{\xi})$ for all parameter classes $q$.
  Since the parameterizations $\boldsymbol{\theta}$ and $\boldsymbol{\xi}$ are related through a finite composition of the symmetries in \cref{app-subsec:generic-reparameterization-symmetries,app-subsec:activation-independent-symmetries,app-subsec:activation-dependent-symmetries}, it suffices to show that $\boldsymbol{\beta}_{q}(\boldsymbol{\theta})$ is invariant under each such symmetry.

  The permutation symmetry (\cref{app-subsubsec:permutation-symmetry}) simply permutes neurons and has no effect on the aggregates $\boldsymbol{\beta}_{q}(\boldsymbol{\theta})$. The positive-scaling symmetry (\cref{app-subsubsec:positive-scaling-symmetry}) of positively $1$-homogeneous activations preserves both the parameter class $q$ and its contribution to the corresponding class aggregate $\boldsymbol{\beta}_{q}(\boldsymbol{\theta})$ in \cref{eq:orbit-invariants-positively-homogeneous}, since
\begin{equation}
    \norm{\alpha \overline{\mathbf{w}}_{j}}
    \, \alpha^{-1} \mathbf{a}_{j}
    =
    \norm{\overline{\mathbf{w}}_{j}} \, \mathbf{a}_{j},
    \qquad
    \alpha > 0.
  \end{equation}
For odd activations, the sign-flip symmetry $(\overline{\mathbf{w}}_{j}, \mathbf{a}_{j}) \mapsto (-\overline{\mathbf{w}}_{j}, -\mathbf{a}_{j})$ (\cref{app-subsubsec:sign-flip-symmetry}) moves the $j$th neuron from $\mathcal{J}_{q}^{+}$ to $\mathcal{J}_{q}^{-}$, or vice versa, while negating its readout $\mathbf{a}_{j}$.
  These two changes cancel in the signed difference defining $\boldsymbol{\beta}_{q}(\boldsymbol{\theta})$ in \cref{eq:orbit-invariants-constant-odd}, leaving the aggregate invariant.
  For even activations, the sign-flip symmetry $(\overline{\mathbf{w}}_{j}, \mathbf{a}_{j}) \mapsto (-\overline{\mathbf{w}}_{j}, \mathbf{a}_{j})$ (\cref{app-subsubsec:sign-flip-symmetry}) keeps neuron $j$ in the same parameter class $q$ while leaving its readout $\mathbf{a}_{j}$ unchanged.
  Its contribution to the aggregate $\boldsymbol{\beta}_{q}(\boldsymbol{\theta})$ in \cref{eq:orbit-invariants-even-linear} is therefore unchanged as well.

  In a zero-neuron group (\cref{def:zero-neuron-group}), all neurons share the same incoming-parameter vector $\overline{\mathbf{w}}$ and therefore belong to the same parameter class $q$.
  Because their aggregate readout vanishes, $\sum_{z} \mathbf{a}_{z} = \mathbf{0}$, adding or removing a zero-neuron group does not change $\boldsymbol{\beta}_{q}(\boldsymbol{\theta})$ or any other class aggregate.
  Similarly, replacing a neuron by a duplicate-neuron group (\cref{def:duplicate-neuron-group}), or collapsing such a group back to a single neuron, preserves both the parameter class $q$ and its aggregate $\boldsymbol{\beta}_{q}(\boldsymbol{\theta})$.
  Individual constant neurons (\cref{def:constant-neuron}), by contrast, contribute only to the residual $\mathbf{r}_{\boldsymbol{\theta}}$ and therefore do not affect any aggregate $\boldsymbol{\beta}_{q}(\boldsymbol{\theta})$ to begin with.

  Finally, an aligned/opposite group (\cref{def:aligned-opposite-group}) with incoming-parameter vector $\overline{\mathbf{w}}$ affects only the class aggregate $\boldsymbol{\beta}_{q}(\boldsymbol{\theta})$ associated with parameter class $q = [\overline{\mathbf{w}}]$.
  For positively homogeneous activations of degree~$1$, the group's contribution to this aggregate is
\begin{equation}
    \sum_{k \in \mathcal{K}}
    \norm{\overline{\mathbf{w}}_{k}} \, \mathbf{a}_{k}
    =
    \norm{\overline{\mathbf{w}}}
    \sum_{k \in \mathcal{K}} \mathbf{a}_{k}
    =
    \norm{\overline{\mathbf{w}}} \, \mathbf{a}^{\pm},
  \end{equation}
where $\mathcal{K}$ indexes the neurons of the aligned/opposite group.
  For even-linear activations, the corresponding contribution is
\begin{equation}
    \sum_{k \in \mathcal{K}} \mathbf{a}_{k} = \mathbf{a}^{\pm},
  \end{equation}
and for constant-odd activations, it is
\begin{equation}
    \varepsilon \sum_{k \in \mathcal{N}^{+}} \mathbf{a}_k
    -
    \varepsilon \sum_{k \in \mathcal{N}^{-}} \mathbf{a}_k
    =
    \varepsilon \mathbf{a}^{\mp},
  \end{equation}
where $\varepsilon \in \set{-1,1}$ is such that $\varepsilon \overline{\mathbf{w}} = \overline{\mathbf{w}}_{q}$.
  Linear-neuron groups (\cref{def:linear-neuron-group}) and constant-neuron groups (\cref{def:constant-neuron-group}) satisfy $\mathbf{a}^{\pm} = \mathbf{0}$ and $\mathbf{a}^{\mp} = \mathbf{0}$, respectively, so adding or removing such groups leaves $\boldsymbol{\beta}_{q}(\boldsymbol{\theta})$ unchanged.
  Likewise, linear-duplicate-neuron groups (\cref{def:linear-duplicate-neuron-group}) and constant-duplicate-neuron groups (\cref{def:constant-duplicate-neuron-group}) satisfy $\mathbf{a}^{\pm} = \mathbf{a}^{\star}$ and $\mathbf{a}^{\mp} = \mathbf{a}^{\star}$, respectively, where $\mathbf{a}^{\star}$ denotes the reference neuron's readout.
  Thus, replacing a neuron by such a duplicate group, or collapsing the group back to a single neuron, again leaves $\boldsymbol{\beta}_{q}(\boldsymbol{\theta})$ invariant.

  Conversely, suppose that $\boldsymbol{\theta}$ and $\boldsymbol{\xi}$ satisfy $\boldsymbol{\beta}_{q}(\boldsymbol{\theta}) = \boldsymbol{\beta}_{q}(\boldsymbol{\xi})$ for all parameter classes $q \in \mathcal{Q}(\boldsymbol{\theta}) \cup \mathcal{Q}(\boldsymbol{\xi})$, and that $\mathbf{r}_{\boldsymbol{\theta}} = \mathbf{r}_{\boldsymbol{\xi}}$.
  To prove that such parameterizations are symmetry-equivalent, we generalize the construction of the parameterization $\boldsymbol{\eta}$ from \cref{ex:opposite-irreducible-relu-parameterizations}, used there to establish equivalence between two oppositely oriented ReLU parameterizations $\boldsymbol{\theta}^{+}$ and $\boldsymbol{\theta}^{-}$ realizing the tent function.
  Specifically, we construct a parameterization $\boldsymbol{\eta}$ realizing the zero function, $f_{\boldsymbol{\eta}} \equiv \mathbf{0}$, whose neurons can be partitioned into zero-neuron groups, linear-neuron groups or constant-neuron groups (depending on the activation $\sigma$), and individual constant neurons.
  The parameterization obtained by combining the neurons of $\boldsymbol{\theta}$ and $\boldsymbol{\eta}$ then places every neuron of $\boldsymbol{\theta}$ into a removable symmetry group, leaving precisely the neurons of $\boldsymbol{\xi}$, as we shall see.

  \paragraph{The auxiliary parameterization $\boldsymbol{\eta}$.}
  First, let $\boldsymbol{\eta}$ denote the parameterization formed from all neurons in $\boldsymbol{\theta}$ with their readouts negated, together with all neurons in $\boldsymbol{\xi}$.
  If $\sigma$ is positively homogeneous of degree~$1$, the positive-scaling symmetry allows each nonconstant neuron in $\boldsymbol{\eta}$ to be transformed as
\begin{equation}
    (\overline{\mathbf{w}}_{j}, \mathbf{a}_{j})
    \mapsto
    (
      \norm{\overline{\mathbf{w}}_{j}}^{-1} \overline{\mathbf{w}}_{j},
      \norm{\overline{\mathbf{w}}_{j}}\,\mathbf{a}_{j}
    ).
  \end{equation}
Because positive scaling preserves symmetry-equivalence, parameter classes, class aggregates, and residuals, we may assume without loss of generality that $\boldsymbol{\theta}$ and $\boldsymbol{\xi}$ have been normalized in this way before constructing $\boldsymbol{\eta}$ whenever $\sigma$ is positively $1$-homogeneous.
  The resulting auxiliary parameterization $\boldsymbol{\eta}$ is then normalized as well.
  In that case, within each parameter class $q = [\overline{\mathbf{w}}]$, the incoming-parameter vector of every neuron equals one of the two unit-norm orientations $\pm \norm{\overline{\mathbf{w}}}^{-1} \overline{\mathbf{w}}$.
  The positively $1$-homogeneous and even-linear cases can then be treated identically below.

  By construction of $\boldsymbol{\eta}$,
\begin{equation}
    \mathcal{Q}(\boldsymbol{\eta})
    =
    \mathcal{Q}(\boldsymbol{\theta})
    \cup
    \mathcal{Q}(\boldsymbol{\xi}).
  \end{equation}
Since both the class aggregates $\boldsymbol{\beta}_{q}$ and the residual $\mathbf{r}$ are linear in the readouts $\mathbf{a}_{j}$, we have
\begin{equation}
    \label{eq:vanishing-eta-aggregates-and-residual}
    \boldsymbol{\beta}_{q}(\boldsymbol{\eta})
    =
    \boldsymbol{\beta}_{q}(\boldsymbol{\xi})
    -
    \boldsymbol{\beta}_{q}(\boldsymbol{\theta})
    =
    \mathbf{0}
    \quad
    \text{for every }
    q \in \mathcal{Q}(\boldsymbol{\eta}),
    \qquad
    \mathbf{r}_{\boldsymbol{\eta}}
    =
    \mathbf{r}_{\boldsymbol{\xi}}
    -
    \mathbf{r}_{\boldsymbol{\theta}}
    =
    \mathbf{0}.
  \end{equation}
Hence, by the decomposition in \cref{eq:function-decomposition-all-activation-classes}, $\boldsymbol{\eta}$ realizes the zero function:
\begin{equation}
    \label{eq:eta-realizes-zero-function}
    f_{\boldsymbol{\eta}}(\mathbf{x})
    =
    \sum_{q \in \mathcal{Q}(\boldsymbol{\eta})}
    \boldsymbol{\beta}_{q}(\boldsymbol{\eta}) \, \phi_{q}(\mathbf{x})
    +
    \mathbf{r}_{\boldsymbol{\eta}}(\mathbf{x})
    =
    \mathbf{0},
    \qquad
    \mathbf{x} \in \mathbb{R}^{N_{i}}.
  \end{equation}

  \paragraph{Decomposing $\boldsymbol{\eta}$ into symmetry groups.}
  Next, we show that all nonconstant neurons in $\boldsymbol{\eta}$ can be organized into symmetry groups.
  Fix a parameter class $q \in \mathcal{Q}(\boldsymbol{\eta})$ and choose a nonempty subset $\mathcal{K} \subseteq \mathcal{J}_{q}$ of minimum cardinality whose contribution to $\boldsymbol{\beta}_{q}(\boldsymbol{\eta})$ vanishes.
  Such a subset exists because $\boldsymbol{\beta}_{q}(\boldsymbol{\eta}) = \mathbf{0}$.
  Suppose first that $\mathcal{K}$ contains only neurons whose incoming-parameter vectors are identical.
  Then, their vanishing contribution to $\boldsymbol{\beta}_{q}(\boldsymbol{\eta})$ implies $\sum_{k \in \mathcal{K}} \mathbf{a}_{k} = \mathbf{0}$.
  By minimality of $\mathcal{K}$, the collection $\set{\mathbf{a}_{k}}_{k \in \mathcal{K}}$ is subset-nonzero (\cref{def:subset-nonzero}), so the neurons indexed by $\mathcal{K}$ form a zero-neuron group (\cref{def:zero-neuron-group}).
  Now suppose that $\mathcal{K}$ contains neurons of both orientations, which can only arise when $\sigma$ is even-linear (including the positively $1$-homogeneous case after the normalization above) or constant-odd.
  In this case, the neurons in $\mathcal{K}$ form an aligned/opposite group (\cref{def:aligned-opposite-group}).
  For even-linear activations, the vanishing contribution of $\mathcal{K}$ to $\boldsymbol{\beta}_{q}(\boldsymbol{\eta})$ is equivalent to
  $\mathbf{a}_{\mathcal{K}}^{\pm} = \mathbf{0}$.
  Minimality of $\mathcal{K}$ implies sum-nondegeneracy (\cref{def:nondegeneracy-aligned-opposite-groups}), so the neurons in $\mathcal{K}$ form a linear-neuron group (\cref{def:linear-neuron-group}).
  For constant-odd activations, the corresponding condition is $\mathbf{a}_{\mathcal{K}}^{\mp} = \mathbf{0}$, and minimality implies difference-nondegeneracy (\cref{def:nondegeneracy-aligned-opposite-groups}).
  Hence, the neurons in $\mathcal{K}$ form a constant-neuron group (\cref{def:constant-neuron-group}).
  In every case, $\mathcal{K}$ indexes a symmetry group whose contribution to $\boldsymbol{\beta}_{q}(\boldsymbol{\eta})$ vanishes.
  Removing this group therefore leaves a smaller set $\mathcal{J}_{q} \setminus \mathcal{K}$ of neurons in $q$ whose contribution to the class aggregate $\boldsymbol{\beta}_{q}(\boldsymbol{\eta})$ still vanishes.
  Since $\mathcal{J}_{q}$ is finite, repeating the argument partitions all neurons in $q$ into symmetry groups.
  Consequently, applying the same construction to every $q \in \mathcal{Q}(\boldsymbol{\eta})$ partitions all nonconstant neurons in $\boldsymbol{\eta}$ into zero-neuron groups, together with linear-neuron groups for even-linear activations or constant-neuron groups for constant-odd activations.

  For parameterizations $\boldsymbol{\theta}_{1}$ and $\boldsymbol{\theta}_{2}$ with the same activation, we write $\boldsymbol{\theta}_{1} \oplus \boldsymbol{\theta}_{2}$ for the parameterization obtained by combining their neurons and, when $\boldsymbol{\theta}_{2}$ consists of a subset of the neurons of $\boldsymbol{\theta}_{1}$, $\boldsymbol{\theta}_{1} \ominus \boldsymbol{\theta}_{2}$ for the parameterization obtained by removing those neurons from $\boldsymbol{\theta}_{1}$.

  \paragraph{$\boldsymbol{\theta} \sim \boldsymbol{\theta} \oplus \boldsymbol{\eta}$.}
  We now show that $\boldsymbol{\theta}$ is symmetry-equivalent to $\boldsymbol{\theta} \oplus \boldsymbol{\eta}$.
  Denote the parameterization consisting of all zero-neuron groups in $\boldsymbol{\eta}$ by $\boldsymbol{\eta}_{0}$.
  Since adding zero-neuron groups is function-preserving (\cref{app-subsubsec:activation-independent-symmetries-function-preserving-realizations}), $\boldsymbol{\theta} \sim \boldsymbol{\theta} \oplus \boldsymbol{\eta}_{0}$.
  Because every zero-neuron group has vanishing aggregate readout, $\boldsymbol{\eta} \ominus \boldsymbol{\eta}_{0}$ has the same class aggregates and residual as $\boldsymbol{\eta}$, and hence still realizes the zero function by \cref{eq:eta-realizes-zero-function}.
  By the preceding decomposition of $\boldsymbol{\eta}$, $\boldsymbol{\eta} \ominus \boldsymbol{\eta}_{0}$ consists only of linear-neuron groups or constant-neuron groups, when admitted by the activation $\sigma$, together with individual constant neurons.
  Thus,
\begin{equation}
    \boldsymbol{\theta}
    \sim
    \boldsymbol{\theta} \oplus \boldsymbol{\eta}_{0}
    \sim
    (\boldsymbol{\theta} \oplus \boldsymbol{\eta}_{0})
    \oplus
    (\boldsymbol{\eta} \ominus \boldsymbol{\eta}_{0})
    =
    \boldsymbol{\theta} \oplus \boldsymbol{\eta}.
  \end{equation}

  \paragraph{$\boldsymbol{\theta} \oplus \boldsymbol{\eta} \sim \boldsymbol{\xi}$.}
  It remains to show that $\boldsymbol{\theta} \oplus \boldsymbol{\eta} \sim \boldsymbol{\xi}$.
  By construction of $\boldsymbol{\eta}$, the parameterization $\boldsymbol{\theta} \oplus \boldsymbol{\eta}$ contains all neurons of $\boldsymbol{\xi}$ together with, for every neuron $(\overline{\mathbf{w}}_{j}, \mathbf{a}_{j})$ of $\boldsymbol{\theta}$, both that neuron and a copy $(\overline{\mathbf{w}}_{j}, -\mathbf{a}_{j})$ with negated readout.
  If $\mathbf{w}_{j} = \mathbf{0}$, the pair consists of two constant neurons whose contributions $\mathbf{a}_{j} \sigma(b_{j})$ and $-\mathbf{a}_{j} \sigma(b_{j})$ cancel.
  If $\mathbf{w}_{j} \neq \mathbf{0}$ and $\mathbf{a}_{j} = \mathbf{0}$, each neuron is a singleton zero-neuron group.
  If $\mathbf{w}_{j} \neq \mathbf{0}$ and $\mathbf{a}_{j} \neq \mathbf{0}$, the two neurons together form a zero-neuron group.
  Thus, every such pair can be removed from $\boldsymbol{\theta} \oplus \boldsymbol{\eta}$ while maintaining symmetry-equivalence.
  This leaves precisely the neurons of $\boldsymbol{\xi}$, so $\boldsymbol{\theta} \oplus \boldsymbol{\eta} \sim \boldsymbol{\xi}$.
\end{proof}

\Cref{prop:characterization-of-symmetry-equivalence-all-activation-classes} establishes that the aggregates $\boldsymbol{\beta}_{q}(\boldsymbol{\theta})$ and the residual function $\mathbf{r}_{\boldsymbol{\theta}}$ are invariants of the symmetry orbit containing $\boldsymbol{\theta}$.
We therefore extend the shorthand notation introduced in the main text for positively $1$-homogeneous activations to all activation classes considered here, suppressing the dependence on the particular parameterization $\boldsymbol{\theta}$.

\begin{definition}[Essential parameter classes]
  \label[definition]{def:essential-parameter-classes}
  Fix a symmetry orbit $\mathcal{O}$ and let $\boldsymbol{\theta} \in \mathcal{O}$ be any parameterization.
  We write
\begin{equation}
    \boldsymbol{\beta}_{q}
    \coloneqq
    \boldsymbol{\beta}_{q}(\boldsymbol{\theta}),
    \qquad
    q \in \mathcal{Q}(\boldsymbol{\theta}),
  \end{equation}
and
  \begin{equation}
    \mathbf{r} \coloneqq \mathbf{r}_{\boldsymbol{\theta}},
    \qquad
    \mathcal{E} \coloneqq
    \set{
      q \in \mathcal{Q}(\boldsymbol{\theta})
      \given
      \boldsymbol{\beta}_{q} \neq \mathbf{0}
    }.
  \end{equation}
A parameter class $q$ is called \emph{essential} if $q \in \mathcal{E}$.
\end{definition}

\section{Symmetry properties of common activation functions}
\label[appendix]{app-sec:symmetry-properties-of-common-activation-functions}

Which of the overparameterization symmetries introduced in \cref{app-sec:parameter-symmetries-in-overparameterized-networks} can arise in a given network depends on algebraic properties of its activation function $\sigma$.
Here, we classify every scalar activation available from the \texttt{jax.nn} module \parencite[v0.10.0,][]{jax2018github}, covering the standard activations commonly used in practice, according to whether it is even-linear (\cref{app-subsec:even-linear-activations}), constant-odd (\cref{app-subsec:constant-odd-activations}), positively homogeneous of degree~$1$ (\cref{app-subsec:positively-homogeneous-activations}), or none of the above (\cref{app-subsec:activations-with-neither-symmetry-property}).
These classes are not mutually exclusive, and the complete classification is summarized in \cref{tab:symmetries-of-jax-activation-functions}.
For completeness, this classification includes the identity, although affine activations are excluded from the nonlinear-network setting we consider.

\begin{table}
  \caption{Symmetries of all scalar activation functions in \texttt{jax.nn} (v0.10.0).
    The aliases \texttt{swish} and \texttt{hard\_swish} for \texttt{silu} and \texttt{hard\_silu}, respectively, are omitted.
    We assume $\alpha>0$ (\texttt{celu}, \texttt{elu}, \texttt{leaky\_relu}, \texttt{selu}), $\lambda\geq1$ (\texttt{selu}), and $b>0$ (\texttt{squareplus}).
    For \texttt{leaky\_relu}, we additionally assume $\alpha\neq1$, since $\alpha=1$ gives the identity.
    \textbf{Even+Lin} and \textbf{Const+Odd} report the linear slope $m$ and constant $c$, respectively.
    Odd functions exhibiting the sign-flip symmetry have $c=0$.
    \textbf{Scaling} denotes the positive-scaling symmetry of positively $1$-homogeneous activations.
  }
  \label{tab:symmetries-of-jax-activation-functions}
  \centering
\begin{tabular}{@{}llccc@{}}
    \toprule
    \textbf{Activation}      & $\sigma(z)$                                                                                          & \textbf{Even+Lin}      & \textbf{Const+Odd} & \textbf{Scaling} \\
    \midrule
    \texttt{celu}            & $\begin{cases} \alpha (\mathrm{e}^{\nicefrac{z}{\alpha}} - 1), & z < 0 \\ z, & z \geq 0 \end{cases}$ & --                     & --                 & \xmark           \\[4pt]
    \texttt{elu}             & $\begin{cases} \alpha (\mathrm{e}^{z}-1), & z < 0 \\ z, & z \geq 0 \end{cases}$                      & --                     & --                 & \xmark           \\[4pt]
    \texttt{gelu}            & $z \, \Phi(z)$                                                                                       & $\nicefrac{1}{2}$      & --                 & \xmark           \\[4pt]
    \texttt{hard\_sigmoid}   & $\reluSix(z + 3) / 6$                                                                                  & --                     & $\nicefrac{1}{2}$  & \xmark           \\[4pt]
    \texttt{hard\_silu}      & $z \hardSigmoid(z)$                                                                                  & $\nicefrac{1}{2}$      & --                 & \xmark           \\[4pt]
    \texttt{hard\_tanh}      & $\begin{cases} -1, & z \leq -1 \\ z, & -1 < z < 1 \\ 1, & z \geq 1 \end{cases}$                      & --                     & $0$                & \xmark           \\[4pt]
    \texttt{identity}        & $z$                                                                                                  & $1$                    & $0$                & \cmark           \\[4pt]
    \texttt{leaky\_relu}     & $\begin{cases} \alpha z, & z < 0 \\ z, & z \geq 0 \end{cases}$                                       & $\frac{1 + \alpha}{2}$ & --                 & \cmark           \\[4pt]
    \texttt{log\_sigmoid}    & $-\log(1 + \mathrm{e}^{-z})$                                                                         & $\nicefrac{1}{2}$      & --                 & \xmark           \\[4pt]
    \texttt{mish}            & $z \tanh(\softplus(z))$                                                                              & --                     & --                 & \xmark           \\[4pt]
    \texttt{relu}            & $\max(0, z)$                                                                                         & $\nicefrac{1}{2}$      & --                 & \cmark           \\[4pt]
    \texttt{relu6}           & $\min(\max(0, z), 6)$                                                                                & --                     & --                 & \xmark           \\[4pt]
    \texttt{selu}            & $\lambda \begin{cases} \alpha (\mathrm{e}^{z} - 1), & z < 0 \\ z, & z \geq 0 \end{cases}$            & --                     & --                 & \xmark           \\[4pt]
    \texttt{sigmoid}         & $(1 + \mathrm{e}^{-z})^{-1}$                                                                         & --                     & $\nicefrac{1}{2}$  & \xmark           \\[4pt]
    \texttt{silu}            & $z \sigmoid(z)$                                                                                      & $\nicefrac{1}{2}$      & --                 & \xmark           \\[4pt]
    \texttt{soft\_sign}      & $z / (\abs{z} + 1)$                                                                                  & --                     & $0$                & \xmark           \\[4pt]
    \texttt{softplus}        & $\log(1 + \mathrm{e}^{z})$                                                                           & $\nicefrac{1}{2}$      & --                 & \xmark           \\[4pt]
    \texttt{sparse\_plus}    & $\begin{cases} 0, & z \leq -1 \\ \frac{1}{4}(z+1)^{2}, & -1 < z < 1 \\ z, & z \geq 1 \end{cases}$    & $\nicefrac{1}{2}$      & --                 & \xmark           \\[4pt]
    \texttt{sparse\_sigmoid} & $\begin{cases} 0, & z \leq -1 \\ \frac{1}{2}(z+1), & -1 < z < 1 \\ 1, & z \geq 1 \end{cases}$        & --                     & $\nicefrac{1}{2}$  & \xmark           \\[4pt]
    \texttt{squareplus}      & $\frac{1}{2} (z + \sqrt{z^{2} + b})$                                                                 & $\nicefrac{1}{2}$      & --                 & \xmark           \\[4pt]
    \texttt{tanh}            & $\tanh(z)$                                                                                           & --                     & $0$                & \xmark           \\
    \bottomrule
  \end{tabular}
\end{table}

\subsection{Even-linear activations}
\label[appendix]{app-subsec:even-linear-activations}

We verify that each of the following \texttt{jax.nn} activations is even-linear in the sense of \cref{def:even-linear-activations}, and we report the slope $m$ of its linear component.

\paragraph{GELU.}
For $\sigma(z) = z \, \Phi(z)$,\footnote{Here, GELU refers to the exact formulation rather than the approximation used by default in \texttt{jax.nn}.
}
where $\Phi(z) = \frac{1}{2}(1 + \mathrm{erf}(z / \sqrt{2}))$ is the standard normal cumulative distribution function,
the symmetry $\Phi(-z) = 1 - \Phi(z)$ gives
\begin{equation}
  \sigma(-z) = -z \, \Phi(-z) = -z (1 - \Phi(z)).
\end{equation}
The even component is not constant,
\begin{equation}
  e(z) = \frac{\sigma(z) + \sigma(-z)}{2}
  = \frac{z \, \Phi(z) - z (1 - \Phi(z))}{2}
  = z \, \Phi(z) - \frac{z}{2},
\end{equation}
so GELU is not constant-odd.
On the other hand, the odd component is linear,
\begin{equation}
  o(z) = \frac{\sigma(z) - \sigma(-z)}{2}
  = \frac{z \, \Phi(z) + z (1 - \Phi(z))}{2}
  = \frac{z}{2},
\end{equation}
implying that GELU is even-linear with $m = \frac{1}{2}$.

\paragraph{Hard SiLU.}
Since the hard sigmoid satisfies $\hardSigmoid(z) = \frac{1}{2} + o(z)$ for some odd function $o$, and $\hardSilu(z) = z \, \hardSigmoid(z)$, we have
\begin{equation}
  \hardSilu(z) = z \, o(z) + \frac{z}{2},
\end{equation}
where the product $z \, o(z)$ is even and $\frac{z}{2}$ is odd.
The product $z \, o(z)$ is not constant (e.g., $z \, o(z) = \frac{z}{2}$ for $z \geq 3$), so the hard SiLU is not constant-odd.
The odd component $\frac{z}{2}$ is linear, and the hard SiLU activation function is even-linear with $m = \frac{1}{2}$.

\paragraph{Identity.}
The identity $\sigma(z) = z$ is linear, so it is trivially even-linear with slope $m = 1$.

\paragraph{Leaky ReLU.}
From the definition of leaky ReLU, we have
\begin{equation}
  \sigma(z) =
\begin{cases}
    \alpha z, & z < 0,   \\
    z,        & z \geq 0
  \end{cases}
\qquad \text{and} \qquad
\sigma(-z) =
\begin{cases}
    -z,        & z < 0,    \\
    -\alpha z, & z \geq 0.
  \end{cases}
\end{equation}
For $z \geq 0$, the even component is not constant,
\begin{equation}
  e(z) = \frac{\sigma(z) + \sigma(-z)}{2}
  = \frac{z - \alpha z}{2}
  = \frac{1 - \alpha}{2} z,
\end{equation}
so leaky ReLU is not constant-odd.
The odd component is linear,
\begin{equation}
  o(z) = \frac{\sigma(z) - \sigma(-z)}{2}
  = \frac{z + \alpha z}{2}
  = \frac{1 + \alpha}{2} z,
\end{equation}
so leaky ReLU is even-linear with $m = \frac{1 + \alpha}{2}$.

\paragraph{Log-sigmoid.}
Using standard logarithm identities, the even component can be written as
\begin{equation}
  \begin{aligned}
    e(z) &= \frac{\sigma(z) + \sigma(-z)}{2}
    = \frac{-\log(1 + \mathrm{e}^{-z}) - \log(1 + \mathrm{e}^{z})}{2} \\
    &= -\frac{\log((1 + \mathrm{e}^{-z}) (1 + \mathrm{e}^{z}))}{2}
    = -\frac{\log(2 (1 + \cosh(z)))}{2},
  \end{aligned}
\end{equation}
which is not constant, so the log-sigmoid activation function is not constant-odd.
Similarly, for the odd component:
\begin{equation}
  \begin{aligned}
    o(z) &= \frac{\sigma(z) - \sigma(-z)}{2}
    = \frac{-\log(1 + \mathrm{e}^{-z}) + \log(1 + \mathrm{e}^{z})}{2} \\
    &= \frac{1}{2} \log\mleft(\frac{1 + \mathrm{e}^{z}}{1 + \mathrm{e}^{-z}}\mright)
    = \frac{1}{2} \log\mleft(\frac{\mathrm{e}^{z}(\mathrm{e}^{-z} + 1)}{1 + \mathrm{e}^{-z}}\mright)
    = \frac{\log(\mathrm{e}^{z})}{2}
    = \frac{z}{2}.
  \end{aligned}
\end{equation}
Thus, the log-sigmoid is even-linear with $m = \frac{1}{2}$.

\paragraph{ReLU.}
For $z \geq 0$, we have $\max(0, z) = z$ and $\max(0, -z) = 0$.
Conversely, $z < 0$ gives $\max(0, z) = 0$ and $\max(0, -z) = -z$.
Therefore,
\begin{equation}
  e(z) = \frac{\sigma(z) + \sigma(-z)}{2} = \frac{1}{2}
\begin{cases}
    z,  & z \geq 0, \\
    -z, & z < 0
  \end{cases}
= \frac{\abs{z}}{2},
\end{equation}
which is not constant, so ReLU is not constant-odd.
Similarly, the odd component is linear,
\begin{equation}
  o(z) = \frac{\sigma(z) - \sigma(-z)}{2}
  = \frac{z}{2},
\end{equation}
so that ReLU is even-linear with $m = \frac{1}{2}$.

\begin{figure}[t]
  \centering
  \includegraphics[width=0.95\textwidth]{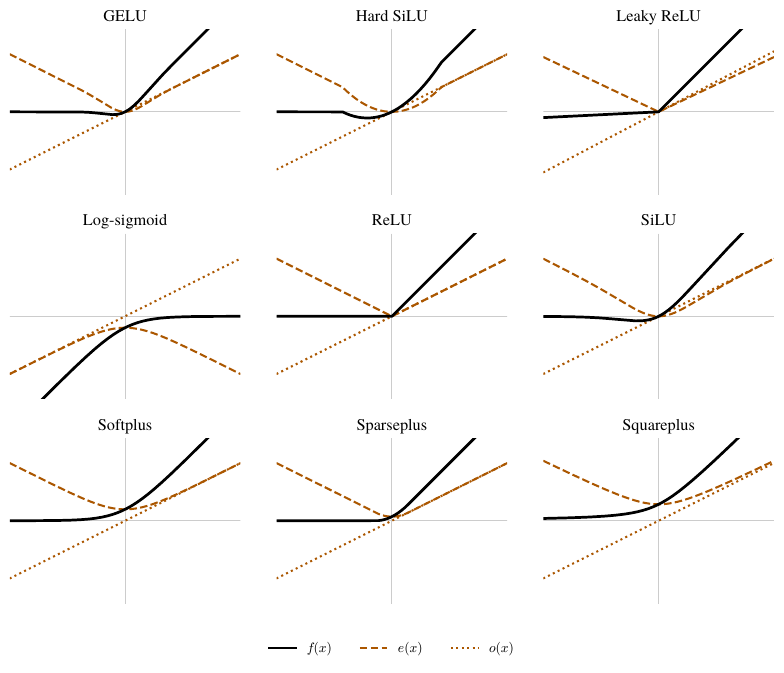}
  \caption{\textbf{Even-linear activations}.
    Activations from \texttt{jax.nn} classified as even-linear, shown with their decomposition into even and odd components.
    Solid dark lines show the activations, dashed lines their even components, and dotted lines their odd components.
    In each case, the odd component is a line passing through the origin.
  }
  \label{fig:activations-even-linear}
\end{figure}

\paragraph{SiLU.}
By definition, $\silu(z) = z \sigmoid(z)$.
Since the sigmoid is constant-odd with $c = \frac{1}{2}$, we have $\sigmoid(z) = \frac{1}{2} + o(z)$ for the odd function $o(z) = \sigmoid(z) - \frac{1}{2}$, and thus SiLU satisfies
\begin{equation}
  \silu(z) = z \, o(z) + \frac{z}{2},
\end{equation}
where the product $z \, o(z)$ is even and $\frac{z}{2}$ is odd.
The product $z \, o(z)$ is not constant,
\begin{equation}
  z \, o(z) = z \, \sigmoid(z) - \frac{z}{2}
  = \frac{z}{1 + \mathrm{e}^{-z}} - \frac{z}{2},
\end{equation}
so SiLU is not constant-odd.
Since the odd component $\frac{z}{2}$ of the sigmoid linear unit is linear, SiLU is even-linear with $m = \frac{1}{2}$.

\paragraph{Softplus.}
Using the identity
\begin{equation}
  \label{eq:softplus-identity}
  \sigma(-z)
  = \log(1 + \mathrm{e}^{-z})
  = \log(1 + \mathrm{e}^{z}) - z
  = \sigma(z) - z
\end{equation}
of the softplus activation $\sigma$, we get
\begin{equation}
  e(z) = \frac{\sigma(z) + \sigma(-z)}{2}
  = \sigma(z) - \frac{z}{2},
\end{equation}
which is not constant, so softplus is not constant-odd.
Using the same identity, the odd component simplifies to
\begin{equation}
  o(z) = \frac{\sigma(z) - \sigma(-z)}{2}
  = \frac{z}{2},
\end{equation}
which is linear, so softplus is even-linear with $m = \frac{1}{2}$.

\paragraph{Sparseplus.}
From the definition of the sparseplus activation function, we have
\begin{equation}
  \sigma(z) =
\begin{cases}
    0,                       & z \leq -1,  \\
    \frac{1}{4} (z + 1)^{2}, & -1 < z < 1, \\
    z,                       & z \geq 1
  \end{cases}
\qquad \text{and} \qquad
  \sigma(-z) =
\begin{cases}
    -z,                       & z \leq -1,  \\
    \frac{1}{4} (-z + 1)^{2}, & -1 < z < 1, \\
    0,                        & z \geq 1.
  \end{cases}
\end{equation}
For $\abs{z} < 1$, we have
\begin{equation}
  e(z) = \frac{\sigma(z) + \sigma(-z)}{2}
  = \frac{(z + 1)^{2} + (-z + 1)^{2}}{8}
  = \frac{z^{2} + 1}{4}
\end{equation}
and
\begin{equation}
  o(z) = \frac{\sigma(z) - \sigma(-z)}{2}
  = \frac{(z + 1)^{2} - (-z + 1)^{2}}{8}
  = \frac{z}{2}.
\end{equation}
Thus, the even component is not constant, and sparseplus is not constant-odd.
The identity $o(z) = \frac{z}{2}$ involving the odd component also holds for $\abs{z} \geq 1$, so sparseplus is even-linear with $m = \frac{1}{2}$.

\paragraph{Squareplus.}
The squareplus activation function satisfies
\begin{equation}
  \sigma(-z) = \frac{-z + \sqrt{z^{2} + b}}{2},
\end{equation}
so the even component equals
\begin{equation}
  e(z) = \frac{\sigma(z) + \sigma(-z)}{2}
  = \frac{(z + \sqrt{z^{2} + b}) + (-z + \sqrt{z^{2} + b})}{4}
  = \frac{\sqrt{z^{2} + b}}{2},
\end{equation}
which is not constant.
Thus, squareplus is not constant-odd. The odd component is linear,
\begin{equation}
  o(z) = \frac{\sigma(z) - \sigma(-z)}{2}
  = \frac{(z + \sqrt{z^{2} + b}) - (-z + \sqrt{z^{2} + b})}{4}
  = \frac{z}{2},
\end{equation}
so squareplus is even-linear with $m = \frac{1}{2}$.

\subsection{Constant-odd activations}
\label[appendix]{app-subsec:constant-odd-activations}

We verify that each of the following \texttt{jax.nn} activations is constant-odd in the sense of \cref{def:constant-odd-activations}, and we report the value $c$ of its constant even component.

\paragraph{Hard sigmoid.}
For $\sigma(z) = \reluSix(z+3) / 6$, where $\reluSix(y) = \min(\max(0, y), 6)$, we have
\begin{equation}
  \sigma(z) =
\begin{cases}
    0,           & z \leq -3,  \\
    (z + 3) / 6, & -3 < z < 3, \\
    1,           & z \geq 3
  \end{cases}
\qquad \text{and} \qquad
  \sigma(-z) =
\begin{cases}
    1,            & z \leq -3,  \\
    (-z + 3) / 6, & -3 < z < 3, \\
    0,            & z \geq 3.
  \end{cases}
\end{equation}
From this, it follows that
\begin{equation}
  e(z) = \frac{\sigma(z) + \sigma(-z)}{2}
  = \frac{(z + 3) + (-z + 3)}{12}
  = \frac{1}{2},
  \qquad \abs{z} < 3.
\end{equation}
By inspection, the same is true for $\abs{z} \geq 3$.
Thus, the hard sigmoid activation function is constant-odd with $c = \frac{1}{2}$.
The odd component satisfies $o(z) = -\frac{1}{2}$ for $z \leq -3$.
Being constant but nonzero on an unbounded interval precludes the odd component from being linear, so the hard sigmoid is not even-linear.

\begin{figure}[t]
  \centering
  \includegraphics[width=0.95\textwidth]{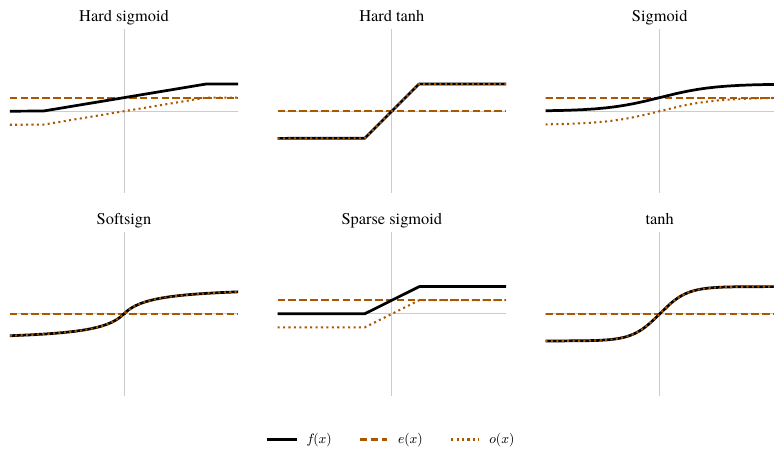}
  \caption{\textbf{Constant-odd activations}.
    Activations from \texttt{jax.nn} classified as constant-odd, shown with their decomposition into even and odd components.
    Solid dark lines show the activations, dashed lines their even components, and dotted lines their odd components.
    In each case, the even component is a horizontal line---at $y=0$ for purely odd functions, where the odd component coincides with the activation itself.
  }
  \label{fig:activations-constant-odd}
\end{figure}

\paragraph{Hard tanh.}
Evidently, the hard tanh activation function is odd, making it constant-odd with $c = 0$.
Since the hard tanh is piecewise linear but not linear, it is not even-linear.

\paragraph{Identity.}
The identity $\sigma(z) = z$ is odd, so it is trivially constant-odd with constant $c = 0$.

\paragraph{Sigmoid.}
The sigmoid $\sigma(z) = (1 + \mathrm{e}^{-z})^{-1}$ satisfies the identity
\begin{equation}
  \label{eq:sigmoid-identity}
  \sigma(-z) = \frac{1}{1 + \mathrm{e}^{z}}
  = \frac{\mathrm{e}^{-z}}{1 + \mathrm{e}^{-z}}
  = 1 - \frac{1}{1 + \mathrm{e}^{-z}}
  = 1 - \sigma(z).
\end{equation}
Hence, the even component is constant,
\begin{equation}
  e(z) = \frac{\sigma(z) + \sigma(-z)}{2}
  = \frac{\sigma(z) + (1 - \sigma(z))}{2}
  = \frac{1}{2},
\end{equation}
so the sigmoid activation function is constant-odd with $c = \frac{1}{2}$.
Applying \cref{eq:sigmoid-identity} once more shows that the odd component is not linear,
\begin{equation}
  o(z) = \frac{\sigma(z) - \sigma(-z)}{2}
  = \frac{\sigma(z) - (1 - \sigma(z))}{2}
  = \sigma(z) - \frac{1}{2},
\end{equation}
so the sigmoid is not even-linear.

\paragraph{Softsign.}
The softsign activation function satisfies
\begin{equation}
  \sigma(-z) = \frac{-z}{\abs{-z} + 1}
  = -\frac{z}{\abs{z} + 1}
  = -\sigma(z),
\end{equation}
and thus is odd, making it constant-odd with $c = 0$.
Since softsign is not itself linear, it is not even-linear.

\paragraph{Sparse sigmoid.}
From the definition of the sparse sigmoid, we have
\begin{equation}
  \sigma(z) =
\begin{cases}
    0,                   & z \leq -1,  \\
    \frac{1}{2} (z + 1), & -1 < z < 1, \\
    1,                   & z \geq 1
  \end{cases}
\qquad \text{and} \qquad
  \sigma(-z) =
\begin{cases}
    1,                    & z \leq -1,  \\
    \frac{1}{2} (-z + 1), & -1 < z < 1, \\
    0,                    & z \geq 1.
  \end{cases}
\end{equation}
For $\abs{z} < 1$, this gives
\begin{equation}
  e(z) = \frac{\sigma(z) + \sigma(-z)}{2}
  = \frac{(z + 1) + (-z + 1)}{4}
  = \frac{1}{2}.
\end{equation}
The same is trivially true for $\abs{z} \geq 1$, so the sparse sigmoid is constant-odd with $c = \frac{1}{2}$.
The odd component
\begin{equation}
  o(z) = \frac{\sigma(z) - \sigma(-z)}{2}
  = \frac{1}{2},
  \qquad
  z \geq 1,
\end{equation}
is constant and nonzero on the half-line $[1, \infty)$, and thus cannot be linear, so the sparse sigmoid is not even-linear.

\paragraph{Tanh.}
The even component of the hyperbolic tangent
\begin{equation}
  \tanh(z) = \frac{
    \mathrm{e}^{z} - \mathrm{e}^{-z}
    }{
    \mathrm{e}^{z} + \mathrm{e}^{-z}
    }
\end{equation}
vanishes:
\begin{equation}
  e(z) = \frac{\sigma(z) + \sigma(-z)}{2}
  = \frac{
    (\mathrm{e}^{z} - \mathrm{e}^{-z}) + (\mathrm{e}^{-z} - \mathrm{e}^{z})
    }{
    2 (\mathrm{e}^{z} + \mathrm{e}^{-z})
  }
  = 0.
\end{equation}
Hence, tanh is odd, making it constant-odd with $c = 0$.
Since the hyperbolic tangent is not itself linear, it is not even-linear.

\subsection{Positively \texorpdfstring{$1$-homogeneous}{1-homogeneous} activations}
\label[appendix]{app-subsec:positively-homogeneous-activations}

We now identify which \texttt{jax.nn} activations are positively homogeneous of degree~$1$.
Recall that a function $\sigma \colon \mathbb{R} \to \mathbb{R}$ is positively homogeneous of degree~$1$ if and only if $\sigma(\alpha z) = \alpha \sigma(z)$ for all $\alpha > 0$ and all $z \in \mathbb{R}$.
The following characterization is standard:

\begin{lemma}[Characterization of positively $1$-homogeneous functions]
  \label[lemma]{lem:characterization-of-positively-homogeneous-functions}
  A function $\sigma \colon \mathbb{R} \to \mathbb{R}$ is positively homogeneous of degree~$1$ if and only if there exist constants $\lambda_{-}, \lambda_{+} \in \mathbb{R}$ such that
\begin{equation}
    \label{eq:characterization-of-positively-homogeneous-functions}
    \sigma(z) =
\begin{cases}
      \lambda_{-} z, & z < 0,    \\
      \lambda_{+} z, & z \geq 0.
    \end{cases}
\end{equation}
\end{lemma}

\begin{proof}
  Any function of the form presented in \cref{eq:characterization-of-positively-homogeneous-functions} is readily verified to be positively homogeneous of degree~$1$.
  Conversely, suppose $\sigma$ is positively homogeneous of degree~$1$.
  For $z < 0$, we have $-z > 0$ and hence
\begin{equation}
    \sigma(z) = -z \, \sigma(-1) = \lambda_{-} z,
    \qquad
    z < 0,
  \end{equation}
with $\lambda_{-} \coloneqq -\sigma(-1)$.
  Similarly, for $z > 0$:
\begin{equation}
    \sigma(z) = z \, \sigma(1) = \lambda_{+} z,
    \qquad
    z > 0,
  \end{equation}
with $\lambda_{+} \coloneqq \sigma(1)$.
  Finally, positive $1$-homogeneity implies $\sigma(0) = \sigma(\alpha \cdot 0) = \alpha \, \sigma(0)$ for every $\alpha > 0$, which requires $\sigma(0) = 0 = \lambda_{+} \cdot 0$.
\end{proof}

This characterization immediately implies that positively $1$-homogeneous activations form a subclass of even-linear activations.

\begin{corollary}
  \label[corollary]{cor:positively-homogeneous-functions-are-even-linear}
  Let $\sigma \colon \mathbb{R} \to \mathbb{R}$ be a positively $1$-homogeneous function as in \cref{eq:characterization-of-positively-homogeneous-functions}.
  Then, $\sigma$ is even-linear with even and odd components
\begin{equation}
    e(z) = \frac{\lambda_{+} - \lambda_{-}}{2} \, \abs{z},
    \qquad
    o(z) = \frac{\lambda_{+} + \lambda_{-}}{2} \, z.
  \end{equation}
In particular, $\sigma$ is nonaffine if and only if $\lambda_{+} \neq \lambda_{-}$.
\end{corollary}

\begin{proof}
  For $z < 0$, \cref{eq:characterization-of-positively-homogeneous-functions} gives
\begin{equation}
    e(z) = \frac{\lambda_{-} z + \lambda_{+} (-z)}{2} = \frac{\lambda_{+} - \lambda_{-}}{2} \, \abs{z}
    \qquad \text{and} \qquad
    o(z) = \frac{\lambda_{-} z - \lambda_{+}(-z)}{2} = \frac{\lambda_{+} + \lambda_{-}}{2} \, z.
  \end{equation}
Similarly, for $z \geq 0$, we have
\begin{equation}
    e(z) = \frac{\lambda_{+} z + \lambda_{-}(-z)}{2} = \frac{\lambda_{+} - \lambda_{-}}{2} \, \abs{z}
    \qquad \text{and} \qquad
    o(z) = \frac{\lambda_{+} z - \lambda_{-}(-z)}{2} = \frac{\lambda_{+} + \lambda_{-}}{2} \, z.
  \end{equation}
Hence, $\sigma$ is even-linear in the sense of \cref{def:even-linear-activations} with $m = \nicefrac{(\lambda_{+} + \lambda_{-})}{2}$.
  The final claim follows directly from the fact that the even component of an affine function $z \mapsto mz + c$ is constant.
\end{proof}

\Cref{lem:characterization-of-positively-homogeneous-functions} also yields the following result:

\begin{corollary}[Unboundedness of positively $1$-homogeneous functions]
  \label[corollary]{cor:positively-homogeneous-functions-are-unbounded-or-zero}
  On each of the half-lines $(-\infty, 0]$ and $[0, \infty)$, a positively homogeneous function of degree~$1$ is either identically zero or unbounded.
\end{corollary}

\Cref{cor:positively-homogeneous-functions-are-unbounded-or-zero} rules out all activation functions except leaky ReLU, ReLU, and the identity from being positively homogeneous of degree~$1$.
These three satisfy the characterization from \cref{lem:characterization-of-positively-homogeneous-functions}, and hence are positively homogeneous of degree~$1$.

\subsection{Activations with neither symmetry property}
\label[appendix]{app-subsec:activations-with-neither-symmetry-property}

The remaining \texttt{jax.nn} activations are neither even-linear nor constant-odd.
We verify this explicitly by computing the even--odd decomposition in each case.

\paragraph{CELU.}
From the definition of CELU, we have
\begin{equation}
  \sigma(z) =
\begin{cases}
    \alpha (\mathrm{e}^{\nicefrac{z}{\alpha}} - 1), & z < 0,   \\
    z,                                              & z \geq 0
  \end{cases}
\qquad \text{and} \qquad
  \sigma(-z) =
\begin{cases}
    -z,                                              & z < 0,    \\
    \alpha (\mathrm{e}^{\nicefrac{-z}{\alpha}} - 1), & z \geq 0.
  \end{cases}
\end{equation}
For $z \geq 0$, this yields
\begin{equation}
  e(z) = \frac{z + \alpha (\mathrm{e}^{\nicefrac{-z}{\alpha}} - 1)}{2}
  \qquad \text{and} \qquad
  o(z) = \frac{z - \alpha (\mathrm{e}^{\nicefrac{-z}{\alpha}} - 1)}{2}.
\end{equation}
The linear term in $e$ implies that $e$ is not constant, and the exponential term in $o$ implies that $o$ is not linear.
Consequently, CELU is neither constant-odd nor even-linear.

\paragraph{ELU.}
From the definition of ELU, we have
\begin{equation}
  \sigma(z) =
\begin{cases}
    \alpha(\mathrm{e}^{z} - 1), & z < 0,   \\
    z,                          & z \geq 0
  \end{cases}
\qquad \text{and} \qquad
  \sigma(-z) =
\begin{cases}
    -z,                          & z < 0,    \\
    \alpha(\mathrm{e}^{-z} - 1), & z \geq 0.
  \end{cases}
\end{equation}
For $z \geq 0$, this yields
\begin{equation}
  e(z) = \frac{z + \alpha (\mathrm{e}^{-z} - 1)}{2}
  \qquad \text{and} \qquad
  o(z) = \frac{z - \alpha (\mathrm{e}^{-z} - 1)}{2}.
\end{equation}
As with CELU, the linear term in $e$ implies that $e$ is not constant, and the exponential term in $o$ implies that $o$ is not linear.
Again, ELU is neither constant-odd nor even-linear.

\paragraph{Mish.}
For brevity, write $s(z) \coloneqq \softplus(z) = \log(1 + \mathrm{e}^{z})$.
For $\sigma(z) = z \tanh(s(z))$, the even and odd components are
\begin{align}
  e(z)
  &= \frac{\sigma(z) + \sigma(-z)}{2}
  = \frac{z}{2} \bigl( \tanh(s(z)) - \tanh(s(-z)) \bigr), \\
  o(z)
  &= \frac{\sigma(z) - \sigma(-z)}{2}
  = \frac{z}{2} \bigl( \tanh(s(z)) + \tanh(s(-z)) \bigr).
\end{align}
As $z \to \infty$, we have $s(z) \to \infty$ and $s(-z) \to 0$.
Therefore,
\begin{equation}
  \lim_{z \to \infty} \frac{e(z)}{z}
  = \frac{1}{2}(1-0)
  = \frac{1}{2},
  \qquad
  \lim_{z \to \infty} \frac{o(z)}{z}
  = \frac{1}{2}(1+0)
  = \frac{1}{2}.
\end{equation}
Thus, the even component grows asymptotically as $z/2$ and is not constant.
On the other hand, as $z \to 0$, both $s(z)$ and $s(-z)$ converge to $\log 2$, yielding
\begin{equation}
  \lim_{z \to 0} \frac{o(z)}{z}
  = \tanh(\log 2)
  = \frac{\mathrm{e}^{2\log 2}-1}{\mathrm{e}^{2\log 2}+1}
  = \frac{\mathrm{e}^{\log 2^{2}}-1}{\mathrm{e}^{\log 2^{2}}+1}
  = \frac{3}{5}.
\end{equation}
Since $o(z)/z$ has different limits as $z \to 0$ and $z \to \infty$, it cannot be constant, and hence the odd component is not linear.
Consequently, Mish is neither constant-odd nor even-linear.

\begin{figure}[t]
  \centering
  \includegraphics[width=0.95\textwidth]{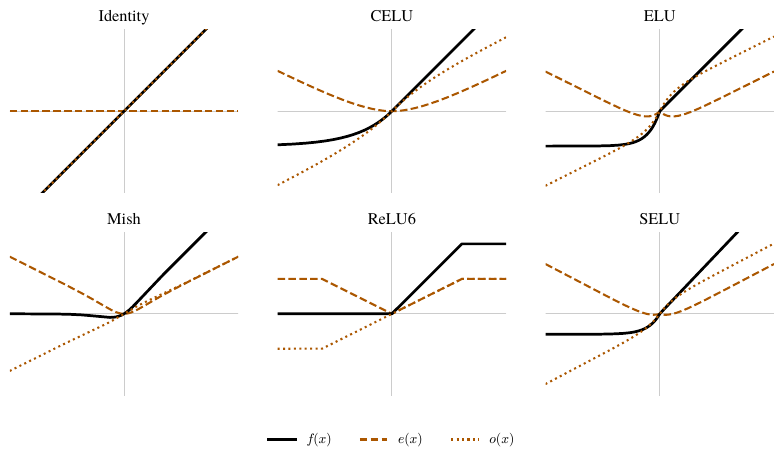}
  \caption{\textbf{Activations with both symmetry properties or neither}.
    Activations from \texttt{jax.nn} that are either both even-linear and constant-odd (the identity) or neither (the asymmetric activations), shown with their decomposition into even and odd components.
    Solid dark lines show the activations, dashed lines their even components, and dotted lines their odd components.
  }
  \label{fig:activations-remaining}
\end{figure}

\paragraph{ReLU6.}
From the definition of ReLU6, we have
\begin{equation}
  \sigma(z) =
\begin{cases}
    0, & z \leq 0,  \\
    z, & 0 < z < 6, \\
    6, & z \geq 6
  \end{cases}
\qquad \text{and} \qquad
\sigma(-z) =
\begin{cases}
    6,  & z \leq -6,  \\
    -z, & -6 < z < 0, \\
    0,  & z \geq 0.
  \end{cases}
\end{equation}
For $0 < z < 6$, the even component is given by
\begin{equation}
  e(z) = \frac{\sigma(z) + \sigma(-z)}{2} = \frac{z}{2},
\end{equation}
which is not constant, so ReLU6 is not constant-odd.
At the same time, the odd component satisfies
\begin{equation}
  o(z) = \frac{\sigma(z) - \sigma(-z)}{2} = 3,
  \qquad
  z \geq 6.
\end{equation}
Being constant and nonzero on an unbounded interval implies that the odd component is not linear, so ReLU6 is not even-linear.

\paragraph{SELU.}
From the definition of SELU, we have
\begin{equation}
  \sigma(z) = \lambda
\begin{cases}
    \alpha(\mathrm{e}^{z} - 1), & z < 0,   \\
    z,                          & z \geq 0
  \end{cases}
\qquad \text{and} \qquad
  \sigma(-z) = \lambda
\begin{cases}
    -z,                          & z < 0,    \\
    \alpha(\mathrm{e}^{-z} - 1), & z \geq 0.
  \end{cases}
\end{equation}
For $z \geq 0$, this yields
\begin{equation}
  e(z) = \lambda \frac{z + \alpha (\mathrm{e}^{-z} - 1)}{2}
  \qquad \text{and} \qquad
  o(z) = \lambda \frac{z - \alpha (\mathrm{e}^{-z} - 1)}{2}.
\end{equation}
The even component is clearly not constant, so SELU is not constant-odd.
At the same time, the odd component is not linear due to the presence of the exponential term, so SELU is not even-linear either.

\section{Feature transformations within symmetry orbits}
\label[appendix]{app-sec:feature-transformations-within-orbits}

This appendix provides proofs and extensions of the results in \cref{sec:parameter-symmetries-act-through-three-feature-primitives}.
We first introduce duplication patterns and duplication matrices and establish identities governing their action on vectors and matrices (\cref{app-subsec:duplication-matrices}).
We then show that every parameter symmetry cataloged in \cref{app-subsec:generic-reparameterization-symmetries,app-subsec:activation-independent-symmetries,app-subsec:activation-dependent-symmetries} induces a finite composition of feature additions, duplications, scalings, and their inverses (\cref{app-subsec:feature-level-characterization-of-parameter-symmetries}).
Next, we extend the characterization of hidden-activation matrices within a symmetry orbit in \cref{prop:hidden-activation-matrices-within-a-symmetry-orbit} to all activation classes considered in this work (\cref{app-subsec:hidden-activation-matrices-within-a-symmetry-orbit}).
Finally, under the activation assumptions of \cref{cor:persistence-of-irreducible-features}, we characterize irreducible parameterizations within a symmetry orbit, establish their uniqueness up to generic reparameterization, and prove that their feature rows persist up to nonzero scaling throughout the orbit (\cref{app-subsec:persistence-of-irreducible-features}).

\subsection{Duplication matrices}
\label[appendix]{app-subsec:duplication-matrices}

To express feature duplication as matrix multiplication, we first record the number of copies of each feature row in a \emph{duplication pattern} and then construct the corresponding \emph{duplication matrix}.

\begin{definition}[Duplication pattern]
  \label[definition]{def:duplication-pattern}
  Let $L \in \mathbb{N}$ denote the number of feature rows before duplication.
  A \emph{duplication pattern} is a vector $\boldsymbol{\nu} = (\nu_{1}, \dots, \nu_{L})^{\top}$, where $\nu_{i} \in \mathbb{N}_{>0}$ denotes the number of copies of the $i$th row after duplication.
  We refer to the resulting number of rows,
\begin{equation}
    N_{h} = \sum_{i=1}^{L} \nu_{i},
  \end{equation}
as the duplication pattern's \emph{induced width}.
\end{definition}

\begin{definition}[Duplication matrix]
  \label[definition]{def:duplication-matrix}
  Given a duplication pattern $\boldsymbol{\nu} \in \mathbb{N}_{>0}^{L}$ with induced width $N_{h}$,
  define the \emph{duplication matrix} $\mathbf{D}_{\boldsymbol{\nu}} \in \set{0,1}^{N_{h} \times L}$ by
\begin{equation}
    \mathbf{D}_{\boldsymbol{\nu}} \coloneqq
\begin{bmatrix}
      \mathbf{1}_{\nu_{1}} & \mathbf{0}           & \cdots & \mathbf{0}         \\
      \mathbf{0}           & \mathbf{1}_{\nu_{2}} & \cdots & \mathbf{0}         \\
      \vdots               & \vdots               & \ddots & \vdots             \\
      \mathbf{0}           & \mathbf{0}           & \cdots & \mathbf{1}_{\nu_{L}}
    \end{bmatrix},
\end{equation}
where each $\mathbf{1}_{\nu_{i}} \in \mathbb{R}^{\nu_{i}}$ is a vector of all ones.
  Equivalently,
\begin{equation}
    \label{eq:duplication-matrix}
    (\mathbf{D}_{\boldsymbol{\nu}})_{rj} \coloneqq
\begin{cases}
      1, & \text{if } \sum_{i=1}^{j-1} \nu_{i} < r \leq \sum_{i=1}^{j} \nu_{i}, \\
      0, & \text{otherwise}.
    \end{cases}
\end{equation}
\end{definition}

\begin{table}[t]
  \caption{Effect of multiplying vectors and matrices by the duplication matrix $\mathbf{D}_{\boldsymbol{\nu}} \in \set{0,1}^{N_{h} \times L}$ and its transpose.
  }
  \label{tab:multiplying-with-duplication-matrices}
  \centering
\begingroup
  \begin{tabular}{@{}lll@{}}
    \toprule
    Input                                    & Transformation                                                                    & Result \\
    \midrule
    $\mathbf{p} \in \mathbb{R}^{L}$          & $\mathbf{D}_{\boldsymbol{\nu}} \mathbf{p} \in \mathbb{R}^{N_{h}}$                 & repeats the $i$th entry of $\mathbf{p}$ exactly $\nu_{i}$ times \\[4pt]
    $\mathbf{M} \in \mathbb{R}^{L \times Q}$ & $\mathbf{D}_{\boldsymbol{\nu}} \mathbf{M} \in \mathbb{R}^{N_{h} \times Q}$        & repeats each row $i$ of $\mathbf{M}$ exactly $\nu_{i}$ times \\[4pt]
    $\mathbf{N} \in \mathbb{R}^{Q \times L}$ & $\mathbf{N} \mathbf{D}_{\boldsymbol{\nu}}^{\top} \in \mathbb{R}^{Q \times N_{h}}$ & repeats each column $i$ of $\mathbf{N}$ exactly $\nu_{i}$ times \\
    \bottomrule
  \end{tabular}
  \endgroup
\end{table}

\begin{remark}[Index shorthand]
  \label[remark]{rem:index-shorthand}
  For a vector $\boldsymbol{\nu} \in \mathbb{R}^{L}$, we denote the sum of its first $j-1$ entries by
\begin{equation}
    \label{eq:index-shorthand}
    \nu_{<j} \coloneqq \sum_{i=1}^{j-1} \nu_{i},
    \qquad
    j = 1, \dots, L+1,
  \end{equation}
with $\nu_{<1} \coloneqq 0$.
  With this shorthand notation, we can rewrite the definition of the duplication matrix $\mathbf{D}_{\boldsymbol{\nu}}$ given in \cref{eq:duplication-matrix} as
\begin{equation}
    (\mathbf{D}_{\boldsymbol{\nu}})_{rj} =
\begin{cases}
      1, & \text{if } \nu_{<j} < r \leq \nu_{<(j+1)}, \\
      0, & \text{otherwise}.
    \end{cases}
  \end{equation}
\end{remark}

Left multiplication by $\mathbf{D}_{\boldsymbol{\nu}}$ repeats entries of a column vector or rows of a matrix according to the duplication pattern $\boldsymbol{\nu}$;
right multiplication by $\mathbf{D}_{\boldsymbol{\nu}}^{\top}$ repeats columns analogously.
These operations are summarized in \cref{tab:multiplying-with-duplication-matrices} and made precise by the following result.

\begin{lemma}[Multiplication by duplication matrices]
  \label[lemma]{lem:multiplication-by-duplication-matrices}
  Let $\boldsymbol{\nu} \in \mathbb{N}_{>0}^{L}$ be a duplication pattern with induced width $N_{h}$.
  Further, let $\mathbf{p} \in \mathbb{R}^{L}$,
  let $\mathbf{M} \in \mathbb{R}^{L \times Q}$,
  and let $\mathbf{N} \in \mathbb{R}^{Q \times L}$.
  Then, whenever $\nu_{<j} < r \leq \nu_{<(j+1)}$, the following identities hold:
\begin{enumerate}
    \item $(\mathbf{D}_{\boldsymbol{\nu}} \mathbf{p})_{r} = p_{j}$,
    \item $(\mathbf{D}_{\boldsymbol{\nu}} \mathbf{M})_{r,:} = \mathbf{M}_{j,:}$,
    \item $(\mathbf{N} \mathbf{D}_{\boldsymbol{\nu}}^{\top})_{:,r} = \mathbf{N}_{:,j}$.
  \end{enumerate}
\end{lemma}

\begin{proof}
  For part~(i), we have
\begin{equation}
    (\mathbf{D}_{\boldsymbol{\nu}} \mathbf{p})_{r}
    = \sum_{i=1}^{L} (\mathbf{D}_{\boldsymbol{\nu}})_{ri} p_{i}.
  \end{equation}
Since the intervals $(\nu_{<i}, \nu_{<(i+1)}]$ partition $(0, N_{h}]$, each $r \in \set{1, \dots, N_{h}}$ lies in exactly one such interval, say for index $j$.
  Then, $(\mathbf{D}_{\boldsymbol{\nu}})_{rj} = 1$ and $(\mathbf{D}_{\boldsymbol{\nu}})_{ri} = 0$ for $i \neq j$, giving $(\mathbf{D}_{\boldsymbol{\nu}} \mathbf{p})_{r} = p_{j}$.
  Part~(ii) follows by applying~(i) to each column of $\mathbf{M}$, and part~(iii) follows by transposing~(ii).
\end{proof}

The next identity describes how weights assigned to individual copies combine when those copies are grouped by their original row.
Each diagonal entry of the resulting weighted Gram matrix is the sum of the weights over one duplication group.

\begin{lemma}[Weighted Gram matrix of a duplication matrix]
  \label[lemma]{lem:weighted-gram-of-duplication-matrix}
  Let $\boldsymbol{\nu} \in \mathbb{N}_{>0}^{L}$ be a duplication pattern with induced width $N_{h}$, and let $\boldsymbol{\omega} \in \mathbb{R}^{N_{h}}$.
  Then, the $\diag(\boldsymbol{\omega})$-weighted Gram matrix of the duplication matrix $\mathbf{D}_{\boldsymbol{\nu}}$ is the diagonal matrix given by
\begin{equation}
    \mathbf{D}_{\boldsymbol{\nu}}^{\top} \diag(\boldsymbol{\omega}) \mathbf{D}_{\boldsymbol{\nu}}
    = \diag(\mathbf{D}_{\boldsymbol{\nu}}^{\top} \boldsymbol{\omega}).
  \end{equation}
\end{lemma}

\begin{proof}
  Let $\mathbf{d}_{j}$ denote the $j$th column of $\mathbf{D}_{\boldsymbol{\nu}}$.
  The entry at position $(r,j)$ of the weighted Gram matrix $\mathbf{D}_{\boldsymbol{\nu}}^{\top} \diag(\boldsymbol{\omega}) \mathbf{D}_{\boldsymbol{\nu}}$ equals
\begin{equation}
    (\mathbf{D}_{\boldsymbol{\nu}}^{\top} \diag(\boldsymbol{\omega}) \mathbf{D}_{\boldsymbol{\nu}})_{rj}
    = \mathbf{d}_{r}^{\top} \diag(\boldsymbol{\omega}) \mathbf{d}_{j}
    = \sum_{i=1}^{N_{h}} \omega_{i} (\mathbf{d}_{r})_{i} (\mathbf{d}_{j})_{i}.
  \end{equation}
This sum accumulates the weights $\omega_{i}$ over all rows $i$ in which columns $r$ and $j$ of $\mathbf{D}_{\boldsymbol{\nu}}$ are both equal to $1$.
  Since each row of $\mathbf{D}_{\boldsymbol{\nu}}$ contains \emph{exactly one} nonzero entry, no two distinct columns can have ones in the same row, i.e., $(\mathbf{d}_{r})_{i} (\mathbf{d}_{j})_{i} = 0$ whenever $r \neq j$.
  For the diagonal entries $r = j$, each entry of $\mathbf{d}_{j}$ is either $0$ or $1$, so $(\mathbf{d}_{j})_{i}^{2} = (\mathbf{d}_{j})_{i}$, and hence
\begin{equation}
    (\mathbf{D}_{\boldsymbol{\nu}}^{\top} \diag(\boldsymbol{\omega}) \mathbf{D}_{\boldsymbol{\nu}})_{jj}
    = \sum_{i=1}^{N_{h}} \omega_{i} (\mathbf{d}_{j})_{i}^{2}
    = \sum_{i=1}^{N_{h}} \omega_{i} (\mathbf{d}_{j})_{i}
    = \mathbf{d}_{j}^{\top} \boldsymbol{\omega}
    = (\mathbf{D}_{\boldsymbol{\nu}}^{\top} \boldsymbol{\omega})_{j},
  \end{equation}
completing the proof.
\end{proof}

\subsection{Feature-level characterization of parameter symmetries}
\label[appendix]{app-subsec:feature-level-characterization-of-parameter-symmetries}

We now prove \cref{prop:feature-level-characterization-of-parameter-symmetries} by translating the cataloged parameter symmetries into operations on feature rows.
The decompositions established in the proof are summarized in \cref{tab:decomposition-of-primitives}.

\begin{table}[t]
  \caption{Decomposition of parameter symmetries into primitive feature transformations.
    Checkmarks indicate the primitives appearing in the displayed decompositions, possibly acting trivially when duplication counts equal $1$.
    For width-changing operations, the table describes the addition or replacement direction; reverse operations use the corresponding inverse primitives.
  }
  \label{tab:decomposition-of-primitives}
  \centering
\begin{tabular}{lccc}
    \toprule
    Symmetry                        & Addition & Duplication & Scaling \\
    \midrule
    Permutation                     & \xmark   & \xmark      & \xmark  \\[2pt]
    Positive scaling                & \xmark   & \xmark      & \cmark  \\[2pt]
    Sign flip                       & \xmark   & \xmark      & \cmark  \\[2pt]
    Zero-neuron group               & \cmark   & \cmark      & \xmark  \\[2pt]
    Duplicate-neuron group          & \xmark   & \cmark      & \xmark  \\[2pt]
    Constant neuron                 & \cmark   & \xmark      & \xmark  \\[2pt]
    Linear-neuron group             & \cmark   & \cmark      & \xmark  \\[2pt]
    Linear-duplicate-neuron group   & \cmark   & \cmark      & \xmark  \\[2pt]
    Constant-neuron group           & \cmark   & \cmark      & \xmark  \\[2pt]
    Constant-duplicate-neuron group & \cmark   & \cmark      & \xmark  \\
    \bottomrule
  \end{tabular}
\end{table}

\begin{proof}[Proof of \cref{prop:feature-level-characterization-of-parameter-symmetries}]
  We verify the claim for the parameter symmetries cataloged in \cref{app-subsec:generic-reparameterization-symmetries,app-subsec:activation-independent-symmetries,app-subsec:activation-dependent-symmetries}.
  Throughout this proof, $\mathbf{H}$ denotes the hidden-activation matrix before the transformation, and $N_{h}$ its number of rows.

  \paragraph{Permutation symmetry.}
  Permuting hidden neurons solely permutes the rows of the hidden-activation matrix $\mathbf{H}$.
  Because \cref{prop:feature-level-characterization-of-parameter-symmetries} is stated up to row permutations of $\mathbf{H}$, we may choose any convenient row ordering throughout this proof.
  We place newly introduced rows after existing ones and duplicate copies of a given row consecutively.

  \paragraph{Positive-scaling symmetry.}
  For positively homogeneous activations of degree~$1$, rescaling the $j$th neuron's incoming weights $\mathbf{w}_{j}$ and hidden bias $b_{j}$ by $\alpha_{j} > 0$, and inversely scaling its readout weights $\mathbf{a}_{j}$ by $\alpha_{j}^{-1}$, multiplies the $j$th row of $\mathbf{H}$ by $\alpha_{j}$.
  This is an instance of feature scaling:
\begin{equation}
    \label{eq:hidden-activations-scaling}
    \mathbf{H}
    \mapsto \diag(\boldsymbol{\alpha}) \mathbf{H},
    \qquad
    \alpha_{i} =
\begin{cases}
      1,          & i \neq j, \\
      \alpha_{j}, & i = j.
    \end{cases}
  \end{equation}

  \paragraph{Sign-flip symmetry.}
  For odd activations, flipping the signs of the $j$th neuron's incoming weights $\mathbf{w}_{j}$, hidden bias $b_{j}$, and readout weights $\mathbf{a}_{j}$ leaves the realized function unchanged.
  Since $\sigma(-z) = -\sigma(z)$, the feature computed by the $j$th neuron changes sign.
  This is another instance of feature scaling:
\begin{equation}
    \label{eq:hidden-activations-sign-flip}
    \mathbf{H}
    \mapsto \diag(\boldsymbol{\alpha}) \mathbf{H},
    \qquad
    \alpha_{i} =
\begin{cases}
      1,  & i \neq j, \\
      -1, & i = j.
    \end{cases}
  \end{equation}
For even activations, by contrast, flipping the incoming weights and hidden bias leaves the feature unchanged because $\sigma(-z) = \sigma(z)$.
  Thus, the sign flip is absorbed by the activation, so $\mathbf{H}$ is unchanged, corresponding to feature scaling by the identity.

  \paragraph{Zero-neuron groups.}
  Since all neurons in a zero-neuron group share incoming parameters $\mathbf{w}$ and $b$, they induce the same feature vector $\mathbf{u} \coloneqq \sigma(\mathbf{w}^{\top} \mathbf{X} + b \mathbf{1}^{\top}) \in \mathbb{R}^{1 \times P}$.
  Introducing a zero-neuron group indexed by $\mathcal{Z}$ therefore amounts to adding this feature once and duplicating it according to the group size:
\begin{equation}
    \label{eq:hidden-activations-zero-neuron-groups}
    \mathbf{H}
    \mapsto \mathbf{D}_{\boldsymbol{\nu}}
\begin{bmatrix}
      \mathbf{H} \\
      \mathbf{u}
    \end{bmatrix},
\qquad
    \boldsymbol{\nu} =
\begin{bmatrix}
      \mathbf{1}_{N_{h}} \\
      \abs{\mathcal{Z}}
    \end{bmatrix}.
  \end{equation}

  \paragraph{Duplicate-neuron groups.}
  Replacing the $j$th neuron with a duplicate-neuron group indexed by $\mathcal{D}$ replicates the $j$th row of $\mathbf{H}$ into $\abs{\mathcal{D}}$ copies.
  This is an instance of feature duplication:
\begin{equation}
    \label{eq:hidden-activations-duplicate-neuron-groups}
    \mathbf{H}
    \mapsto \mathbf{D}_{\boldsymbol{\nu}} \mathbf{H},
    \qquad
    \nu_{i} =
\begin{cases}
      1,                 & i \neq j, \\
      \abs{\mathcal{D}}, & i = j.
    \end{cases}
  \end{equation}

  \paragraph{Constant neurons.}
  Since a constant neuron has vanishing incoming weights, its output $\sigma(b)$ is independent of the input, producing the constant feature vector $\mathbf{u} \coloneqq \sigma(b)\mathbf{1}^{\top} \in \mathbb{R}^{1 \times P}$.
  Adding a constant neuron is therefore a feature addition of the form
\begin{equation}
    \label{eq:hidden-activations-constant-neurons}
    \mathbf{H}
    \mapsto
\begin{bmatrix}
      \mathbf{H} \\
      \mathbf{u}
    \end{bmatrix}.
  \end{equation}

  \paragraph{Linear-neuron groups.}
  For an even-linear activation $\sigma(z) = e(z) + mz$, the aligned and opposite subgroups of a linear-neuron group produce the feature rows associated with the incoming parameters $(\mathbf{w}, b)$ and $(-\mathbf{w}, -b)$, respectively.
  Writing
\begin{equation}
    \label{eq:pre-activation-feature-vector-arbitrary}
    \mathbf{z}
    \coloneqq \mathbf{w}^{\top} \mathbf{X} + b \mathbf{1}^{\top}
    \in \mathbb{R}^{1 \times P}
  \end{equation}
for the pre-activation feature vector, aligned neurons produce
  $\mathbf{u}^{+} \coloneqq e(\mathbf{z}) + m\mathbf{z}$, whereas opposite neurons produce $\mathbf{u}^{-} \coloneqq e(\mathbf{z}) - m\mathbf{z}$, with $e$ applied element-wise.
  These rows need not be distinct.
  Adding a linear-neuron group is therefore a feature addition of these two rows, followed by a feature duplication that replicates $\mathbf{u}^{+}$ for each aligned neuron and $\mathbf{u}^{-}$ for each opposite neuron:
\begin{equation}
    \label{eq:hidden-activations-linear-neuron-groups}
    \mathbf{H}
    \mapsto \mathbf{D}_{\boldsymbol{\nu}}
\begin{bmatrix}
      \mathbf{H}     \\
      \mathbf{u}^{+} \\
      \mathbf{u}^{-}
    \end{bmatrix},
    \qquad
    \boldsymbol{\nu} =
\begin{bmatrix}
      \mathbf{1}_{N_{h}}    \\
      \abs{\mathcal{N}^{+}} \\
      \abs{\mathcal{N}^{-}}
    \end{bmatrix}.
  \end{equation}

  \paragraph{Linear-duplicate-neuron groups.}
  Since aligned neurons share the incoming weights and hidden bias of the $j$th neuron, they replicate the $j$th row of $\mathbf{H}$.
  Opposite neurons have sign-flipped incoming weights and hidden bias.
  Writing
\begin{equation}
    \label{eq:pre-activation-feature-vector-reference}
    \mathbf{z}_{j}
    \coloneqq \mathbf{w}_{j}^{\top} \mathbf{X} + b_{j} \mathbf{1}^{\top}
    \in \mathbb{R}^{1 \times P}
  \end{equation}
for the pre-activation feature vector of the $j$th neuron, opposite neurons produce $\mathbf{u}^{-} \coloneqq e(\mathbf{z}_{j}) - m\mathbf{z}_{j}$.
  Replacing the $j$th neuron with a linear-duplicate-neuron group is therefore a feature addition of $\mathbf{u}^{-}$, followed by a feature duplication that replicates the $j$th row of $\mathbf{H}$ for each aligned neuron and $\mathbf{u}^{-}$ for each opposite neuron:
\begin{equation}
    \label{eq:hidden-activations-linear-duplicate-groups}
    \mathbf{H}
    \mapsto \mathbf{D}_{\boldsymbol{\nu}}
\begin{bmatrix}
      \mathbf{H} \\
      \mathbf{u}^{-}
    \end{bmatrix},
    \qquad
    \boldsymbol{\nu} =
\begin{bmatrix}
      \boldsymbol{\nu}^{+} \\
      \abs{\mathcal{N}^{-}}
    \end{bmatrix},
    \qquad
    \nu_{i}^{+} =
\begin{cases}
      1,                     & i \neq j, \\
      \abs{\mathcal{N}^{+}}, & i = j.
    \end{cases}
  \end{equation}

  \paragraph{Constant-neuron groups.}
  For a constant-odd activation $\sigma(z) = c + o(z)$, aligned and opposite subgroups produce $\mathbf{u}^{+} \coloneqq c\mathbf{1}^{\top} + o(\mathbf{z})$ and $\mathbf{u}^{-} \coloneqq c\mathbf{1}^{\top} - o(\mathbf{z})$, respectively, where $\mathbf{z}$ is defined in \cref{eq:pre-activation-feature-vector-arbitrary} and $o$ is applied element-wise.
  The resulting transformation is otherwise identical to \cref{eq:hidden-activations-linear-neuron-groups}.

  \paragraph{Constant-duplicate-neuron groups.}
  Opposite neurons produce $\mathbf{u}^{-} \coloneqq c\mathbf{1}^{\top} - o(\mathbf{z}_{j})$, where $\mathbf{z}_{j}$ is defined in \cref{eq:pre-activation-feature-vector-reference}.
  The resulting transformation is otherwise identical to \cref{eq:hidden-activations-linear-duplicate-groups}.

  \paragraph{Reverse operations.}
  Each displayed transformation can be reversed, on its image, by applying the corresponding inverse primitives in reverse order.
  In particular, removing an added group first merges duplicated rows and then removes the newly appended rows.
  Collapsing a duplicate-neuron group only merges its copies, whereas collapsing a linear- or constant-duplicate-neuron group additionally removes the appended opposite feature.
  Because both aligned and opposite subgroups are nonempty, all duplication counts are strictly positive: inverse duplication therefore retains at least one copy of every pre-duplication row, while rows are removed only through inverse feature addition.
  Scaling is reversed by reciprocal scaling.

  \paragraph{Function-preserving combinations.}
  As established in \cref{app-subsubsec:activation-independent-symmetries-function-preserving-realizations,app-subsubsec:activation-dependent-symmetries-function-preserving-realizations}, operations involving individual constant neurons or activation-dependent symmetry groups may need to be performed jointly for their residual contributions to cancel.
  For any such function-preserving combination, the induced transformation of hidden activations is obtained by composing the feature-level transformations of its constituent operations.

  Together, these decompositions establish \cref{prop:feature-level-characterization-of-parameter-symmetries}.
\end{proof}

\subsection{Hidden-activation matrices within a symmetry orbit}
\label[appendix]{app-subsec:hidden-activation-matrices-within-a-symmetry-orbit}

The preceding subsection describes how individual parameter symmetries transform hidden-activation matrices.
We next prove the main-text characterization of hidden-activation matrices within a symmetry orbit for positively $1$-homogeneous activations (\cref{prop:hidden-activation-matrices-within-a-symmetry-orbit}) using the orbit invariants established in \cref{app-subsec:parameter-classes-orbit-invariants-and-symmetry-equivalence}, and then extend this characterization to all activation classes considered here.

For each activation class, fix a symmetry orbit $\mathcal{O}$ with essential parameter classes $\mathcal{E}$.
We collect below the corresponding definitions of the matrix $\mathbf{F}_{\mathcal{E}}$ introduced in \cref{subsec:structure-of-hidden-activation-matrices-within-a-symmetry-orbit}, including the positively $1$-homogeneous case already treated in the main text.

\paragraph{Positively homogeneous activations of degree~$\mathbf{1}$.}
If $\sigma$ is positively homogeneous of degree~$1$, each parameter class $q \in \mathcal{E}$ is of the form $q = \set{\alpha \overline{\mathbf{w}} \given \alpha \in \mathbb{R}_{\neq 0}}$.
For each $q \in \mathcal{E}$, choose a unit-norm representative $\overline{\mathbf{w}}_{q} \in q$, and let $\mathbf{F}_{\mathcal{E}} \in \mathbb{R}^{2\abs{\mathcal{E}} \times P}$ collect the two features $(\mathbf{f}_{q}^{+})^{\top} = \sigma(\overline{\mathbf{w}}_{q}^{\top} \overline{\mathbf{X}})$ and $(\mathbf{f}_{q}^{-})^{\top} = \sigma(-\overline{\mathbf{w}}_{q}^{\top} \overline{\mathbf{X}})$.

\paragraph{Even-linear and constant-odd activations.}
If $\sigma$ is even-linear but not positively homogeneous of degree~$1$, or is constant-odd, each parameter class $q$ contains exactly two opposite incoming-parameter vectors $\overline{\mathbf{w}}_{q}$ and $-\overline{\mathbf{w}}_{q}$.
For each $q \in \mathcal{E}$, let $\mathbf{F}_{\mathcal{E}} \in \mathbb{R}^{2\abs{\mathcal{E}} \times P}$ collect the two features $(\mathbf{f}_{q}^{+})^{\top} = \sigma(\overline{\mathbf{w}}_{q}^{\top} \overline{\mathbf{X}})$ and $(\mathbf{f}_{q}^{-})^{\top} = \sigma(-\overline{\mathbf{w}}_{q}^{\top} \overline{\mathbf{X}})$.

\paragraph{Activations with neither symmetry property.}
If $\sigma$ is neither even-linear nor constant-odd, every parameter class $q$ is a singleton $\set{\overline{\mathbf{w}}_{q}}$.
For each $q \in \mathcal{E}$, let $\mathbf{F}_{\mathcal{E}} \in \mathbb{R}^{\abs{\mathcal{E}} \times P}$ collect the feature $\mathbf{f}_{q}^{\top} = \sigma(\overline{\mathbf{w}}_{q}^{\top} \overline{\mathbf{X}})$.

We let each incoming-parameter vector $\pm\overline{\mathbf{w}}_{q}$ specified above contribute a separate row to $\mathbf{F}_{\mathcal{E}}$, even when the oppositely oriented vectors $\overline{\mathbf{w}}_{q}$ and $-\overline{\mathbf{w}}_{q}$ happen to produce identical features $\mathbf{f}_{q}^{+} = \mathbf{f}_{q}^{-}$ on the given inputs $\mathbf{X}$, and we arrange these rows in a fixed order.

\begin{proposition}[Hidden-activation matrices within a symmetry orbit]
  \label[proposition]{prop:hidden-activation-matrices-within-a-symmetry-orbit-all-activation-classes}
  Fix a symmetry orbit $\mathcal{O}$ with nonaffine activation $\sigma$ and essential parameter classes $\mathcal{E}$, and let $\mathbf{F}_{\mathcal{E}}$ be defined as above.
  Every parameterization $\boldsymbol{\theta}$ of width $N_{h}$ in $\mathcal{O}$ has a hidden-activation matrix of the form
\begin{equation}
    \label{eq:hidden-activations-within-a-symmetry-orbit-all-activation-classes}
    \mathbf{H}
    =
    \diag(\boldsymbol{\alpha})
    \mathbf{D}_{\boldsymbol{\nu}}
    \begin{bmatrix}
      \mathbf{F}_{\mathcal{E},\mathcal{I}} \\
      \mathbf{U}
    \end{bmatrix},
    \qquad
    \boldsymbol{\alpha} \in \mathbb{R}_{>0}^{N_{h}},
    \quad
    \boldsymbol{\nu} \in \mathbb{N}_{>0}^{\abs{\mathcal{I}}+K},
  \end{equation}
for some $K \geq 0$ and $\mathbf{U} \in \mathbb{R}^{K \times P}$, with $\sum_{\ell} \nu_{\ell} = N_{h}$.
  Here, $\mathbf{F}_{\mathcal{E},\mathcal{I}}$ is the submatrix of $\mathbf{F}_{\mathcal{E}}$ indexed by $\mathcal{I}$.
  The index set $\mathcal{I}$ selects exactly those rows of $\mathbf{F}_{\mathcal{E}}$ whose associated incoming-parameter vectors are represented by neurons of $\boldsymbol{\theta}$, up to positive scaling when $\sigma$ is positively homogeneous of degree~$1$.
  For every essential class $q \in \mathcal{E}$, at least one of its associated rows is selected.
  The matrix $\mathbf{U}$ collects pairwise distinct features from nonessential classes and constant neurons.
  If $\sigma$ is not positively homogeneous of degree~$1$, $\boldsymbol{\alpha}$ may be chosen as the all-ones vector.
\end{proposition}

\begin{proof}
  Let $\boldsymbol{\theta}$ be any parameterization in $\mathcal{O}$.
  By \cref{prop:characterization-of-symmetry-equivalence-all-activation-classes},
\begin{equation}
    \boldsymbol{\beta}_{q}(\boldsymbol{\theta})
    =
    \boldsymbol{\beta}_{q}
    \neq
    \mathbf{0},
    \qquad
    q \in \mathcal{E}.
  \end{equation}
Since the class aggregate of an unrepresented parameter class is zero, every essential class must be represented in $\boldsymbol{\theta}$, that is, $\mathcal{E} \subseteq \mathcal{Q}(\boldsymbol{\theta})$.

  Suppose first that $\sigma$ is positively homogeneous of degree~$1$.
  For each neuron $j \in \mathcal{J}_{q}$, $q \in \mathcal{E}$, positive $1$-homogeneity gives
\begin{equation}
    \mathbf{H}_{j,:} = \sigma(\overline{\mathbf{w}}_{j}^{\top} \overline{\mathbf{X}}) = \alpha_{j} \, \sigma(\pm\overline{\mathbf{w}}_{q}^{\top} \overline{\mathbf{X}}),
    \qquad \alpha_{j} \coloneqq \norm{\overline{\mathbf{w}}_{j}} > 0,
  \end{equation}
with the sign determined by whether $j \in \mathcal{J}_{q}^{+}$ or $j \in \mathcal{J}_{q}^{-}$.
  Thus, neuron $j$ corresponds to the row of $\mathbf{F}_{\mathcal{E}}$ associated with $\pm\overline{\mathbf{w}}_{q}$, with scaling factor $\alpha_{j}$.

  If $\sigma$ is even-linear but not positively homogeneous of degree~$1$, or is constant-odd, each neuron in an essential class $q$ has one of the two incoming-parameter vectors $\overline{\mathbf{w}}_{q}$ and $-\overline{\mathbf{w}}_{q}$ in $q$.
  Its feature therefore coincides exactly with the corresponding row of $\mathbf{F}_{\mathcal{E}}$.
  We associate each such neuron with that row and set $\alpha_j \coloneqq 1$.
  The same argument applies when $\sigma$ is neither even-linear nor constant-odd, except that each essential class contains a single incoming-parameter vector.

  Thus, in every activation case, each neuron of $\boldsymbol{\theta}$ belonging to an essential parameter class has been associated with a row of $\mathbf{F}_{\mathcal{E}}$.
  Let $\mathcal{I}$ index exactly those rows of $\mathbf{F}_{\mathcal{E}}$ that have been associated with at least one such neuron.
  Since every essential class is represented in $\boldsymbol{\theta}$, the set $\mathcal{I}$ contains at least one index for every $q \in \mathcal{E}$.
  Ordering the rows of $\mathbf{F}_{\mathcal{E},\mathcal{I}}$ as inherited from $\mathbf{F}_{\mathcal{E}}$, let $\nu_{\ell}$ denote the number of neurons associated with the $\ell$th row of $\mathbf{F}_{\mathcal{E},\mathcal{I}}$, for $\ell = 1, \dots, \abs{\mathcal{I}}$.
  By construction, $\nu_{\ell} > 0$.

  Every neuron not yet associated with a row of $\mathbf{F}_{\mathcal{E}}$ either belongs to a nonessential parameter class or is constant.
  Let $\mathbf{U} \in \mathbb{R}^{K \times P}$ collect each of the $K$ \emph{distinct} feature rows generated by these neurons.
  For each remaining neuron $j$, associate it with the unique row of $\mathbf{U}$ equal to its feature row, and set $\alpha_{j} \coloneqq 1$.
  For $k = 1, \dots, K$, let $\nu_{\abs{\mathcal{I}}+k}$ denote the number of neurons associated with the $k$th row of $\mathbf{U}$.
  Since each row of $\mathbf{U}$ represents at least one neuron, these multiplicities are strictly positive.

  Ordering the neurons according to these associations, the definition of the duplication matrix yields \cref{eq:hidden-activations-within-a-symmetry-orbit-all-activation-classes}.
  Since every neuron in the parameterization $\boldsymbol{\theta}$ is counted exactly once,
\begin{equation}
    \sum_{\ell=1}^{\abs{\mathcal{I}}+K}
    \nu_{\ell}
    =
    N_{h}.
  \end{equation}
All scaling factors are positive and, when $\sigma$ is not positively homogeneous of degree~$1$, equal to $1$.
  Restricting the result to positively homogeneous activations of degree~$1$ proves \cref{prop:hidden-activation-matrices-within-a-symmetry-orbit} from the main text.
\end{proof}

\Cref{prop:hidden-activation-matrices-within-a-symmetry-orbit-all-activation-classes} gives a necessary structural description of hidden-activation matrices within a fixed symmetry orbit, but not a converse realizability result.
In particular, it does \emph{not} guarantee that every choice of $\boldsymbol{\alpha}$, $\boldsymbol{\nu}$, $\mathcal{I}$, and $\mathbf{U}$ is realized by a parameterization in $\mathcal{O}$.

\subsection{Persistence of irreducible features}
\label[appendix]{app-subsec:persistence-of-irreducible-features}

For activations that are neither even-linear nor constant-odd, as well as for purely even and purely odd activations, we prove that every feature computed by an irreducible parameterization must persist throughout its symmetry orbit (\cref{cor:persistence-of-irreducible-features}).
We first characterize irreducible parameterizations within an orbit and then describe their remaining parameter freedom.
Finally, we return to this characterization to prove that every irreducible feature persists throughout the orbit.

\begin{lemma}[Irreducible parameterizations within a symmetry orbit]
  \label[lemma]{lem:irreducible-parameterizations-within-a-symmetry-orbit}
  Suppose that $\sigma$ is neither even-linear nor constant-odd, or is purely even or purely odd.
  Fix a symmetry orbit $\mathcal{O}$ with essential parameter classes $\mathcal{E}$ and global residual $\mathbf{r}$.
  A parameterization in $\mathcal{O}$ is irreducible if and only if it contains exactly one neuron from each essential parameter class, no neurons from nonessential parameter classes, and exactly one constant neuron when $\mathbf{r} \neq \mathbf{0}$ and none when $\mathbf{r} = \mathbf{0}$.
\end{lemma}

\begin{proof}
  We first derive a lower bound on the width of any parameterization in $\mathcal{O}$.
  Under the assumptions of the lemma, nonconstant neurons make no contribution to the global residual $\mathbf{r}$.
  For activations that are neither even-linear nor constant-odd, this follows directly from the definition of $\mathbf{r}$ in \cref{eq:global-residual-neither-symmetry-property}.
  For purely even activations, $m=0$ in the even-linear decomposition, so the contribution of nonconstant neurons to $\mathbf{r}$ vanishes in \cref{eq:global-residual-positively-homogeneous,eq:global-residual-even-linear}, regardless of whether $\sigma$ is also positively $1$-homogeneous.
  Similarly, for purely odd activations, $c=0$ in the constant-odd decomposition, so nonconstant neurons make no contribution to $\mathbf{r}$ in \cref{eq:global-residual-constant-odd}.
  Hence, every parameterization $\boldsymbol{\theta}$ in $\mathcal{O}$ satisfies
\begin{equation}
    \mathbf{r}(\mathbf{x})
    =
    \sum_{j \in \mathcal{J}_{0}}
    \mathbf{a}_{j} \, \sigma(b_{j}).
  \end{equation}
Moreover, each essential parameter class $q \in \mathcal{E}$ has nonzero aggregate $\boldsymbol{\beta}_{q}$ and must therefore be represented by at least one nonconstant neuron.
  If $\mathbf{r}\neq\mathbf{0}$, the preceding identity additionally requires at least one constant neuron.
  Hence, the width $N_{h}$ of any parameterization in $\mathcal{O}$ satisfies
\begin{equation}
    \label{eq:irreducible-parameterization-lower-bound}
    N_{h}
    \geq
    \begin{cases}
      \abs{\mathcal{E}},     & \mathbf{r} = \mathbf{0}, \\
      \abs{\mathcal{E}} + 1, & \mathbf{r} \neq \mathbf{0}.
    \end{cases}
  \end{equation}

  We next show that this lower bound is attainable.
  For each essential class $q$, construct a single nonconstant neuron with incoming parameters $\overline{\mathbf{w}}_{q}$, which therefore computes $\phi_{q}$.
  Assigning this neuron readout $\boldsymbol{\beta}_{q}$ reproduces exactly the aggregate associated with $q$ (\cref{eq:orbit-invariants-positively-homogeneous,eq:orbit-invariants-even-linear,eq:orbit-invariants-constant-odd,eq:orbit-invariants-neither-symmetry-property}).
  If $\mathbf{r} \neq \mathbf{0}$, choose a hidden bias $b$ with $\sigma(b) \neq 0$ and add one constant neuron with readout
\begin{equation}
    \mathbf{a}
    \coloneqq
    \frac{\mathbf{r}(\mathbf{0})}{\sigma(b)}.
  \end{equation}
The resulting parameterization has aggregate $\boldsymbol{\beta}_{q}$ for every essential class $q$, zero aggregate for every nonessential class, and global residual $\mathbf{r}$.
  By \cref{prop:characterization-of-symmetry-equivalence-all-activation-classes}, it therefore belongs to the fixed orbit $\mathcal{O}$ and attains the lower bound in \cref{eq:irreducible-parameterization-lower-bound}.
  Since irreducible parameterizations are precisely the minimum-width parameterizations in $\mathcal{O}$, equality in this bound holds exactly when there is one neuron from each essential class, no neurons from nonessential classes, and one constant neuron if and only if $\mathbf{r} \neq \mathbf{0}$.
\end{proof}

The preceding lemma determines the number of neurons in an irreducible parameterization and assigns one neuron to each essential parameter class, with a constant neuron present precisely when the residual is nonzero.
We next determine the freedom that remains in each neuron's parameters when the orbit invariants are held fixed.

\begin{corollary}[Uniqueness up to generic reparameterization]
  \label[corollary]{cor:uniqueness-of-irreducible-parameterizations}
  Suppose that $\sigma$ is neither even-linear nor constant-odd, or is purely even or purely odd.
  Let $\boldsymbol{\theta}^{\star} \sim \boldsymbol{\xi}^{\star}$ be irreducible parameterizations in the same symmetry orbit, with global residual $\mathbf{r}$.
  Up to a permutation of neurons, corresponding nonconstant neurons differ only by the following transformations:
\begin{enumerate}
    \item If $\sigma$ is purely even and positively $1$-homogeneous,
\begin{equation}
        (\overline{\mathbf{w}}, \mathbf{a})
        \mapsto
        (t\overline{\mathbf{w}}, \abs{t}^{-1}\mathbf{a}),
        \qquad
        t \in \mathbb{R}_{\neq 0}.
      \end{equation}
\item If $\sigma$ is purely even but not positively $1$-homogeneous,
\begin{equation}
        (\overline{\mathbf{w}}, \mathbf{a})
        \mapsto
        (\varepsilon\overline{\mathbf{w}}, \mathbf{a}),
        \qquad
        \varepsilon \in \set{-1,1}.
      \end{equation}
\item If $\sigma$ is purely odd,
\begin{equation}
        (\overline{\mathbf{w}}, \mathbf{a})
        \mapsto
        (\varepsilon\overline{\mathbf{w}}, \varepsilon\mathbf{a}),
        \qquad
        \varepsilon \in \set{-1,1}.
      \end{equation}
\item If $\sigma$ is neither even-linear nor constant-odd, corresponding nonconstant neurons have identical parameters.
  \end{enumerate}
If $\mathbf{r}\neq\mathbf{0}$, each parameterization also contains exactly one constant neuron.
  Its hidden bias $b$ may be chosen arbitrarily subject to $\sigma(b) \neq 0$, with readout
\begin{equation}
    \mathbf{a}
    =
    \frac{\mathbf{r}(\mathbf{0})}{\sigma(b)}.
  \end{equation}
If $\mathbf{r}=\mathbf{0}$, neither parameterization contains a constant neuron.
  Conversely, the transformations above, together with the stated freedom in the constant neuron, preserve both the symmetry orbit and irreducibility.
\end{corollary}

\begin{proof}
  By \cref{lem:irreducible-parameterizations-within-a-symmetry-orbit}, both parameterizations contain exactly one neuron from each essential parameter class and, when $\mathbf{r} \neq \mathbf{0}$, exactly one constant neuron.
  We may therefore match their nonconstant neurons by parameter class and, when present, their constant neurons, which determines a permutation of neurons.
  Fix an essential class $q$, and let
\begin{equation}
    (\overline{\mathbf{w}}_{\boldsymbol{\theta}}^{\star}, \mathbf{a}_{\boldsymbol{\theta}}^{\star})
    \qquad \text{and} \qquad
    (\overline{\mathbf{w}}_{\boldsymbol{\xi}}^{\star}, \mathbf{a}_{\boldsymbol{\xi}}^{\star})
  \end{equation}
denote the unique neurons of $\boldsymbol{\theta}^{\star}$ and $\boldsymbol{\xi}^{\star}$, respectively, belonging to $q$.
  Since both parameterizations lie in the same orbit, these neurons must realize the same aggregate $\boldsymbol{\beta}_{q}$.
  We now determine the resulting freedom in their parameters.

  Suppose first that $\sigma$ is purely even and positively homogeneous of degree~$1$.
  Since the two incoming-parameter vectors belong to the same class, there exists $t \in \mathbb{R}_{\neq0}$ such that
\begin{equation}
    \overline{\mathbf{w}}_{\boldsymbol{\xi}}^{\star}
    =
    t\overline{\mathbf{w}}_{\boldsymbol{\theta}}^{\star}.
  \end{equation}
Since a single neuron contributes $\norm{\overline{\mathbf{w}}} \, \mathbf{a}$ to the class aggregate $\boldsymbol{\beta}_{q}$, equality of the aggregates gives
\begin{equation}
    \norm{\overline{\mathbf{w}}_{\boldsymbol{\theta}}^{\star}} \,
    \mathbf{a}_{\boldsymbol{\theta}}^{\star}
    =
    \norm{\overline{\mathbf{w}}_{\boldsymbol{\xi}}^{\star}} \,
    \mathbf{a}_{\boldsymbol{\xi}}^{\star}
    =
    \abs{t}
    \norm{\overline{\mathbf{w}}_{\boldsymbol{\theta}}^{\star}} \,
    \mathbf{a}_{\boldsymbol{\xi}}^{\star},
  \end{equation}
and hence
\begin{equation}
    \mathbf{a}_{\boldsymbol{\xi}}^{\star}
    =
    \abs{t}^{-1}
    \mathbf{a}_{\boldsymbol{\theta}}^{\star}.
  \end{equation}

  If $\sigma$ is purely even but not positively homogeneous of degree~$1$, the two elements of $q$ differ only by sign, so
\begin{equation}
    \overline{\mathbf{w}}_{\boldsymbol{\xi}}^{\star}
    =
    \varepsilon
    \overline{\mathbf{w}}_{\boldsymbol{\theta}}^{\star},
    \qquad
    \varepsilon\in\set{-1,1}.
  \end{equation}
The aggregate of a single neuron equals its readout, and therefore
  $\mathbf{a}_{\boldsymbol{\xi}}^{\star}
   =
   \mathbf{a}_{\boldsymbol{\theta}}^{\star}$.

  If $\sigma$ is purely odd, the incoming parameters in $q$ again differ by a sign, so
\begin{equation}
    \overline{\mathbf{w}}_{\boldsymbol{\xi}}^{\star}
    =
    \varepsilon
    \overline{\mathbf{w}}_{\boldsymbol{\theta}}^{\star},
    \qquad
    \varepsilon\in\set{-1,1},
  \end{equation}
but reversing the incoming parameters of a neuron reverses the sign with which that neuron's readout enters the aggregate $\boldsymbol{\beta}_{q}$.
  Equality of aggregates therefore requires
\begin{equation}
    \mathbf{a}_{\boldsymbol{\xi}}^{\star}
    =
    \varepsilon
    \mathbf{a}_{\boldsymbol{\theta}}^{\star}.
  \end{equation}

  Finally, if $\sigma$ is neither even-linear nor constant-odd, every parameter class is a singleton.
  Hence,
\begin{equation}
    \overline{\mathbf{w}}_{\boldsymbol{\xi}}^{\star}
    =
    \overline{\mathbf{w}}_{\boldsymbol{\theta}}^{\star},
    \qquad
    \mathbf{a}_{\boldsymbol{\xi}}^{\star}
    =
    \mathbf{a}_{\boldsymbol{\theta}}^{\star}.
  \end{equation}

  When $\mathbf{r}\neq\mathbf{0}$, the unique constant neuron is solely responsible for the global residual and therefore satisfies
\begin{equation}
    \mathbf{a} \sigma(b)
    =
    \mathbf{r}(\mathbf{0})
    \neq
    \mathbf{0}.
  \end{equation}
Consequently, $\sigma(b) \neq 0$ and
\begin{equation}
    \mathbf{a}
    =
    \frac{\mathbf{r}(\mathbf{0})}{\sigma(b)},
  \end{equation}
while $b$ may otherwise be chosen arbitrarily.

  Conversely, each listed transformation preserves the corresponding class aggregate $\boldsymbol{\beta}_{q}$, and the stated freedom in the constant neuron preserves the global residual $\mathbf{r}$.
  By \cref{prop:characterization-of-symmetry-equivalence-all-activation-classes}, the transformed parameterization therefore remains in the same symmetry orbit.
  The transformations also preserve the structure characterized in \cref{lem:irreducible-parameterizations-within-a-symmetry-orbit}, and hence preserve irreducibility.
\end{proof}

The preceding corollary compares irreducible parameterizations within the same orbit.
To prove \cref{cor:persistence-of-irreducible-features} from the main text, we now compare an irreducible parameterization with an arbitrary parameterization in its orbit.
The structural characterization in \cref{lem:irreducible-parameterizations-within-a-symmetry-orbit} then allows us to show that every feature row of the former occurs in the latter up to nonzero scaling.

\begin{proof}[Proof of \cref{cor:persistence-of-irreducible-features}]
  By \cref{lem:irreducible-parameterizations-within-a-symmetry-orbit}, $\boldsymbol{\theta}^{\star}$ contains exactly one neuron from each essential parameter class and, when the global residual $\mathbf{r}$ is nonzero, exactly one constant neuron.

  We first associate the neurons of $\boldsymbol{\theta}$ belonging to essential parameter classes with the corresponding neurons of $\boldsymbol{\theta}^{\star}$.
  Fix an essential class $q$, and let $j$ denote the unique neuron of $\boldsymbol{\theta}^{\star}$ belonging to $q$.
  The feature row of every neuron $i$ of $\boldsymbol{\theta}$ in the same class is a nonzero scalar multiple of $\mathbf{H}^{\star}_{j,:}$.
  Specifically, for singleton classes and purely even activations that are not positively $1$-homogeneous, we set $\alpha_{i} = 1$.
  For purely odd activations, we set $\alpha_{i} = 1$ or $\alpha_{i} = -1$ according to whether the incoming parameters of neurons $i$ and $j$ agree or are opposite.
  Finally, for purely even positively $1$-homogeneous activations, we set
\begin{equation}
    \alpha_{i}
    \coloneqq
    \frac{\norm{\overline{\mathbf{w}}_{i}}}
         {\norm{\overline{\mathbf{w}}_{j}^{\star}}}
    > 0.
  \end{equation}
In each case,
\begin{equation}
    \mathbf{H}_{i,:}
    =
    \alpha_{i} \mathbf{H}^{\star}_{j,:}.
  \end{equation}
Since every essential class is represented in $\boldsymbol{\theta}$, each nonconstant row of $\mathbf{H}^{\star}$ has at least one associated neuron.

  Suppose next that $\boldsymbol{\theta}^{\star}$ contains a constant neuron $j$.
  Then, $\mathbf{r} \neq \mathbf{0}$ and $\sigma(b_{j}^{\star}) \neq 0$.
  Because, under the assumptions of \cref{cor:persistence-of-irreducible-features}, only constant neurons contribute to the global residual (see the proof of \cref{lem:irreducible-parameterizations-within-a-symmetry-orbit}), $\boldsymbol{\theta}$ must contain at least one constant neuron with nonzero output.
  For every such neuron $i$, associate it with neuron $j$ and set
\begin{equation}
    \alpha_{i}
    \coloneqq
    \frac{\sigma(b_{i})}{\sigma(b_{j}^{\star})}
    \neq 0.
  \end{equation}
Its feature row then satisfies
\begin{equation}
    \mathbf{H}_{i,:}
    =
    \alpha_{i} \mathbf{H}^{\star}_{j,:}.
  \end{equation}
Consequently, every row of $\mathbf{H}^{\star}$ has at least one associated neuron.
  Let $\nu_{j}>0$ denote the number of neurons associated with the $j$th row, for $j = 1, \dots, N_{h}^{\star}$.

  It remains to account for the neurons of $\boldsymbol{\theta}$ not yet associated with a row of $\mathbf{H}^{\star}$.
  Let $\mathbf{U}$ contain each distinct feature row among these remaining neurons exactly once, and let $K$ denote the number of such rows.
  Associate every remaining neuron $i$ with its matching row of $\mathbf{U}$ and set $\alpha_{i} \coloneqq 1$.
  For $k = 1, \dots, K$, let $\nu_{N_{h}^{\star}+k}>0$ denote the number of neurons associated with the $k$th row of $\mathbf{U}$.

  Every neuron of $\boldsymbol{\theta}$ has now been associated exactly once, so
\begin{equation}
    \boldsymbol{\nu}
    \in
    \mathbb{N}_{>0}^{N_{h}^{\star}+K},
    \qquad
    \sum_{j=1}^{N_{h}^{\star}+K}\nu_{j}
    =
    N_{h}.
  \end{equation}
After ordering the neurons according to these associations, duplicating the rows of
  $[{\mathbf{H}^{\star}}^{\top},\mathbf{U}^{\top}]^{\top}$
  according to $\boldsymbol{\nu}$ and applying the corresponding nonzero scaling factors yields \cref{eq:persistence-of-irreducible-features}.
  By construction, the rows of $\mathbf{U}$ are pairwise distinct.
\end{proof}

\section{Representational geometry and similarity within symmetry orbits}
\label[appendix]{app-sec:representational-geometry-and-similarity-within-orbits}

This appendix provides proofs and extensions of the results in \cref{sec:parameter-symmetries-dissociate-function-and-representation}.
We first derive the decomposition of the \ac{rsm} into essential and auxiliary contributions and establish their independent reweighting under positive scaling (\cref{app-subsec:rsm-decomposition-and-positive-reweighting}).
We then prove that the spread of attainable representational similarity scores with respect to a fixed reference increases strictly with width until it reaches its supremum across widths (\cref{app-subsec:strict-growth-until-saturation}).
Next, we show that the limiting lower and upper similarity bounds are independent of the symmetry orbit and are attained at finite width within every orbit (\cref{app-subsec:function-independent-similarity-limits}).
Finally, we characterize when similarity to a reference geometry can approach $1$, establish conditions under which arbitrary reference geometries admit such alignment, and exhibit an obstruction showing that alignment need not be possible in general (\cref{app-subsec:alignment-with-reference-geometries}).

\subsection{RSM decomposition and positive reweighting}
\label[appendix]{app-subsec:rsm-decomposition-and-positive-reweighting}

We first extend the decomposition in \cref{prop:rsm-decomposition-and-reweighting} for positively $1$-homogeneous activations to the remaining activation classes.

\begin{proposition}[RSM decomposition]
  \label[proposition]{prop:rsm-decomposition-all-activation-classes}
  Consider the factorization of $\mathbf{H}$ in \cref{prop:hidden-activation-matrices-within-a-symmetry-orbit-all-activation-classes}.
  Denote the rows of $\mathbf{F}_{\mathcal{E},\mathcal{I}}$ and $\mathbf{U}$ by $\mathbf{f}_{\ell}^{\top}$ and $\mathbf{u}_{k}^{\top}$, respectively.
  Then,
\begin{equation}
    \mathbf{H}^{\top}\mathbf{H}
    =
    \underbrace{\sum_{\ell=1}^{\abs{\mathcal{I}}}
    \gamma_{\ell} \mathbf{f}_{\ell} \mathbf{f}_{\ell}^{\top}}_{\text{\emph{essential}}}
    +
    \underbrace{\sum_{k=1}^{K}
    \gamma_{\abs{\mathcal{I}}+k} \mathbf{u}_{k} \mathbf{u}_{k}^{\top}}_{\text{\emph{auxiliary}}},
    \qquad
    \boldsymbol{\gamma}
    \coloneqq
    \mathbf{D}_{\boldsymbol{\nu}}^{\top} \boldsymbol{\alpha}^{2}
    \in \mathbb{R}_{>0}^{\abs{\mathcal{I}}+K},
  \end{equation}
where $\boldsymbol{\alpha}^{2}$ denotes the Hadamard square.
  If $\sigma$ is \emph{not} positively $1$-homogeneous, the factorization may be chosen with $\boldsymbol{\alpha} = \mathbf{1}$, in which case $\boldsymbol{\gamma} = \boldsymbol{\nu}$.
\end{proposition}

\begin{proof}
  By \cref{lem:weighted-gram-of-duplication-matrix},
\begin{equation}
    \mathbf{D}_{\boldsymbol{\nu}}^{\top}
    \diag(\boldsymbol{\alpha}^{2})
    \mathbf{D}_{\boldsymbol{\nu}}
    = \diag(
      \mathbf{D}_{\boldsymbol{\nu}}^{\top} \boldsymbol{\alpha}^{2}
    )
    = \diag(\boldsymbol{\gamma}).
  \end{equation}
Hence,
\begin{equation}
    \mathbf{H}^{\top}\mathbf{H}
    = \begin{bmatrix}
      \mathbf{F}_{\mathcal{E},\mathcal{I}} \\
      \mathbf{U}
    \end{bmatrix}^{\top}
      \diag(\boldsymbol{\gamma})
    \begin{bmatrix}
      \mathbf{F}_{\mathcal{E},\mathcal{I}} \\
      \mathbf{U}
    \end{bmatrix}
    = \sum_{\ell=1}^{\abs{\mathcal{I}}}
    \gamma_{\ell} \mathbf{f}_{\ell} \mathbf{f}_{\ell}^{\top}
    + \sum_{k=1}^{K}
      \gamma_{\abs{\mathcal{I}}+k} \mathbf{u}_{k} \mathbf{u}_{k}^{\top}.
  \end{equation}
When $\boldsymbol{\alpha} = \mathbf{1}$, the identity $\mathbf{D}_{\boldsymbol{\nu}}^{\top} \mathbf{1} = \boldsymbol{\nu}$ gives the final claim.
\end{proof}

For the activation classes covered by \cref{cor:persistence-of-irreducible-features}, the same calculation yields a decomposition in terms of the features of any irreducible parameterization.

\begin{corollary}[RSM decomposition relative to an irreducible parameterization]
  \label[corollary]{cor:rsm-decomposition-relative-to-irreducible-parameterization}
  Suppose that $\sigma$ is neither even-linear nor constant-odd, or is purely even or purely odd.
  Let $\mathbf{H}^{\star}$, $\mathbf{U}$, $\boldsymbol{\alpha}$, and $\boldsymbol{\nu}$ be as in \cref{cor:persistence-of-irreducible-features}, and denote the rows of $\mathbf{H}^{\star}$ and $\mathbf{U}$ by ${\mathbf{v}_{j}^{\star}}^{\top}$ and $\mathbf{u}_{k}^{\top}$, respectively.
  Then,
\begin{equation}
    \mathbf{H}^{\top}\mathbf{H}
    = \underbrace{\sum_{j=1}^{N_{h}^{\star}}
    \gamma_{j}\mathbf{v}_{j}^{\star} {\mathbf{v}_{j}^{\star}}^{\top}}_{\text{\emph{essential}}}
    + \underbrace{\sum_{k=1}^{K}
    \gamma_{N_{h}^{\star}+k} \mathbf{u}_{k} \mathbf{u}_{k}^{\top}}_{\text{\emph{auxiliary}}},
    \qquad
    \boldsymbol{\gamma}
    \coloneqq
    \mathbf{D}_{\boldsymbol{\nu}}^{\top} \boldsymbol{\alpha}^{2}
    \in \mathbb{R}_{>0}^{N_{h}^{\star}+K}.
  \end{equation}
\end{corollary}

\begin{proof}
  The stated decomposition follows directly by applying the Gram calculation from the proof of \cref{prop:rsm-decomposition-all-activation-classes} to the factorization of $\mathbf{H}$ in \cref{eq:persistence-of-irreducible-features}.
\end{proof}

To prove \cref{prop:rsm-decomposition-and-reweighting}, it remains to establish the independent reweighting claim for positively $1$-homogeneous activations.

\begin{proof}[Proof of \cref{prop:rsm-decomposition-and-reweighting}]
  Fix a parameterization $\boldsymbol{\theta}$ whose factorization of $\mathbf{H}^{\top}\mathbf{H}$ in \cref{eq:essential-auxiliary-rsm-decomposition} has effective weights $\boldsymbol{\gamma}$, and let $\boldsymbol{\omega} \in \mathbb{R}_{>0}^{\abs{\mathcal{I}}+K}$ be any target effective weights.
  We construct a symmetry-equivalent parameterization $\boldsymbol{\xi} \sim \boldsymbol{\theta}$ whose factorization of $\mathbf{H}^{\top}\mathbf{H}$ in \cref{eq:essential-auxiliary-rsm-decomposition} has effective weights
\begin{equation}
    \mathbf{D}_{\boldsymbol{\nu}}^{\top} \widetilde{\boldsymbol{\alpha}}^{2}
    =
    \boldsymbol{\omega},
    \qquad
    \widetilde{\boldsymbol{\alpha}} \in \mathbb{R}_{>0}^{N_{h}},
  \end{equation}
while all other components of the factorization remain unchanged.

  Let
\begin{equation}
    t_{\ell} \coloneqq
    \sqrt{\frac{\omega_{\ell}}{\gamma_{\ell}}}
    > 0,
    \qquad
    \ell = 1, \dots, \abs{\mathcal{I}} + K,
  \end{equation}
and apply the positive-scaling symmetry
\begin{equation}
    (\mathbf{w}_{j}, b_{j}, \mathbf{a}_{j})
    \mapsto
    (t_{\ell}\mathbf{w}_{j}, t_{\ell}b_{j}, t_{\ell}^{-1}\mathbf{a}_{j})
  \end{equation}
to the $j$th neuron's parameters for every $j = 1, \dots, N_{h}$ such that $(\mathbf{D}_{\boldsymbol{\nu}})_{j\ell} = 1$.
  This yields a parameterization $\boldsymbol{\xi}$ that is symmetry-equivalent to $\boldsymbol{\theta}$.
  The scaling vector $\widetilde{\boldsymbol{\alpha}} \in \mathbb{R}_{>0}^{N_{h}}$ defined by
\begin{equation}
    \widetilde{\alpha}_{j}
    \coloneqq
    t_{\ell} \alpha_{j},
    \qquad
    j = 1, \dots, N_{h},
  \end{equation}
where, for each $j$, $\ell$ is the unique index such that $(\mathbf{D}_{\boldsymbol{\nu}})_{j\ell} = 1$, satisfies
\begin{equation}
    \bigl(
      \mathbf{D}_{\boldsymbol{\nu}}^{\top} \widetilde{\boldsymbol{\alpha}}^{2}
    \bigr)_{\ell}
    = t_{\ell}^{2}
    \bigl(
      \mathbf{D}_{\boldsymbol{\nu}}^{\top} \boldsymbol{\alpha}^{2}
    \bigr)_{\ell}
    = t_{\ell}^{2} \gamma_{\ell}
    = \omega_{\ell}.
  \end{equation}
At the same time, the feature matrices $\mathbf{F}_{\mathcal{E},\mathcal{I}}$ and $\mathbf{U}$, the duplication pattern $\boldsymbol{\nu}$, and the width $N_{h}$ are unchanged, as required.
\end{proof}

\subsection{Strict growth until saturation}
\label[appendix]{app-subsec:strict-growth-until-saturation}

We adopt the notation $\mathcal{R}$, $\mathcal{R}_{N_{h}}$, $\mathcal{S}(\mathbf{N})$, $\mathcal{S}_{N_{h}}(\mathbf{N})$, $\Delta_{N_{h}}(\mathbf{N})$, and $\Delta_{\infty}(\mathbf{N})$ from \cref{subsec:variability-in-representational-similarity}.
Throughout \cref{app-subsec:strict-growth-until-saturation,app-subsec:function-independent-similarity-limits,app-subsec:alignment-with-reference-geometries}, we assume that $\sigma$ is positively $1$-homogeneous.

\paragraph{Pearson correlation as cosine similarity.}
Let $\Hol$ denote the space of symmetric matrices with zero diagonal and zero mean across off-diagonal entries:
\begin{equation}
  \label{eq:vec-space-hollow-symmetric-zero-mean}
  \Hol
  \coloneqq
  \set{
    \mathbf{M} \in \mathbb{R}^{P \times P}
    \given
    \mathbf{M} = \mathbf{M}^{\top},
    \;
    \diag(\mathbf{M}) = \mathbf{0},
    \;
    \mathbf{1}^{\top} \mathbf{M} \mathbf{1} = 0
  }.
\end{equation}
The orthogonal projection $\Pi_{\Hol}$ of a symmetric matrix $\mathbf{M}$ onto $\Hol$ under the Frobenius inner product is given by
\begin{equation}
  \Pi_{\Hol}(\mathbf{M})
  =
  \mathbf{M}
  - \diag(\diag(\mathbf{M}))
  - \frac{
      \langle
        \mathbf{M},
        \mathbf{1}\mathbf{1}^{\top} - \mathbf{I}
    \rangle_{F}
    }{
      P(P-1)
    }
  (\mathbf{1}\mathbf{1}^{\top} - \mathbf{I}).
\end{equation}
Thus, $\Pi_{\Hol}$ sets the diagonal to zero and centers the off-diagonal entries by subtracting their mean.
For any symmetric matrix $\mathbf{M}$ with nonconstant strict upper-triangular entries, define its normalized projection onto $\Hol$ by
\begin{equation}
  \label{eq:centered-normalized-form}
  \mathbf{M}^{\circ}
  \coloneqq
  \frac{\Pi_{\Hol}(\mathbf{M})}
       {\norm{\Pi_{\Hol}(\mathbf{M})}_{F}}.
\end{equation}
Under this normalization, the Pearson correlation $\rho(\mathbf{M}, \mathbf{N})$ between the strict upper-triangular entries of two symmetric matrices $\mathbf{M}$ and $\mathbf{N}$ coincides with the cosine similarity induced by the Frobenius inner product on $\Hol$.

\begin{lemma}[Pearson correlation as cosine similarity]
  \label[lemma]{lem:pearson-as-cosine}
  Let $\mathbf{M},\mathbf{N} \in \mathbb{R}^{P \times P}$ be symmetric matrices with nonconstant strict upper-triangular entries.
  Then, their Pearson correlation is
\begin{equation}
    \rho(\mathbf{M},\mathbf{N})
    = \frac{
      \langle
        \Pi_{\Hol}(\mathbf{M}),
        \Pi_{\Hol}(\mathbf{N})
      \rangle_{F}
    }{
      \norm{\Pi_{\Hol}(\mathbf{M})}_{F}
      \norm{\Pi_{\Hol}(\mathbf{N})}_{F}
    }
    = \langle
      \mathbf{M}^{\circ},
      \mathbf{N}^{\circ}
    \rangle_{F}.
  \end{equation}
\end{lemma}

\begin{proof}
  Let
\begin{equation}
    \hat{m}
    \coloneqq
    \frac{2}{P(P-1)}
    \sum_{\mu<\nu} M_{\mu\nu},
    \qquad
    \hat{n}
    \coloneqq
    \frac{2}{P(P-1)}
    \sum_{\mu<\nu} N_{\mu\nu}
  \end{equation}
denote the means of the strict upper-triangular entries of $\mathbf{M}$ and $\mathbf{N}$, respectively.
  By definition of $\Pi_{\Hol}$,
\begin{equation}
    \bigl[\Pi_{\Hol}(\mathbf{M})\bigr]_{\mu\nu}
    = M_{\mu\nu} - \hat{m},
    \qquad
    \bigl[\Pi_{\Hol}(\mathbf{N})\bigr]_{\mu\nu}
    = N_{\mu\nu} - \hat{n},
    \qquad
    \mu \neq \nu,
  \end{equation}
and
\begin{equation}
    \bigl[\Pi_{\Hol}(\mathbf{M})\bigr]_{\mu\mu}
    =
    \bigl[\Pi_{\Hol}(\mathbf{N})\bigr]_{\mu\mu}
    =
    0.
  \end{equation}

  Hence, the Pearson correlation between the strict upper-triangular entries of $\mathbf{M}$ and $\mathbf{N}$ is
  \begin{equation}
    \rho(\mathbf{M},\mathbf{N})
    =
    \frac{
      \sum_{\mu<\nu}
      (M_{\mu\nu} - \hat{m})(N_{\mu\nu} - \hat{n})
    }{
      \bigl(
        \sum_{\mu<\nu}(M_{\mu\nu}-\hat{m})^{2}
      \bigr)^{1/2}
      \bigl(
        \sum_{\mu<\nu}(N_{\mu\nu}-\hat{n})^{2}
      \bigr)^{1/2}
    }.
  \end{equation}
Since the projected matrices are symmetric with zero diagonal,
\begin{equation}
    \langle
      \Pi_{\Hol}(\mathbf{M}),
      \Pi_{\Hol}(\mathbf{N})
    \rangle_{F}
    = 2 \sum_{\mu<\nu}
    (M_{\mu\nu} - \hat{m})(N_{\mu\nu} - \hat{n}),
  \end{equation}
and similarly
\begin{equation}
    \norm{\Pi_{\Hol}(\mathbf{M})}_{F}^{2}
    = 2 \sum_{\mu<\nu}(M_{\mu\nu} - \hat{m})^{2},
    \qquad
    \norm{\Pi_{\Hol}(\mathbf{N})}_{F}^{2}
    = 2 \sum_{\mu<\nu}(N_{\mu\nu}-\hat{n})^{2}.
  \end{equation}
The factors of $2$ therefore cancel in the cosine ratio, yielding
\begin{equation}
    \rho(\mathbf{M},\mathbf{N})
    = \frac{
      \langle
        \Pi_{\Hol}(\mathbf{M}),
        \Pi_{\Hol}(\mathbf{N})
      \rangle_{F}
    }{
      \norm{\Pi_{\Hol}(\mathbf{M})}_{F}
      \norm{\Pi_{\Hol}(\mathbf{N})}_{F}
    }
    = \langle
      \mathbf{M}^{\circ},
      \mathbf{N}^{\circ}
    \rangle_{F},
  \end{equation}
where the final identity follows from \cref{eq:centered-normalized-form}.
\end{proof}

\paragraph{Scaling and feature addition.}
The following construction shows that, for any parameterization $\boldsymbol{\theta}$ of width $N_{h}$, there exists a parameterization $\boldsymbol{\xi}$ of width $N_{h}+1$ in the same symmetry orbit whose \ac{rsm} equals a rescaled version of the \ac{rsm} of $\boldsymbol{\theta}$ plus an independently weighted rank-one contribution from an added neuron with vanishing readout.

\begin{lemma}[RSM scaling and zero-readout feature addition]
  \label[lemma]{lem:scaling-and-zero-readout-feature-addition}
  Let $\boldsymbol{\theta} \in \mathcal{O}_{N_{h}}$ with \ac{rsm} $\mathbf{M}_{\boldsymbol{\theta}}$, and let $\mathbf{u}^{\top} = \sigma(\overline{\mathbf{w}}^{\top} \overline{\mathbf{X}})$ be any single-neuron feature row on the inputs $\mathbf{X}$.
  For every $\lambda > 0$ and $t \geq 0$, there exists
  $\boldsymbol{\xi} \in \mathcal{O}_{N_{h}+1}$ such that
\begin{equation}
    \label{eq:scaling-and-zero-readout-feature-addition}
    \mathbf{M}_{\boldsymbol{\xi}}
    = \lambda\mathbf{M}_{\boldsymbol{\theta}}
    + t \mathbf{u} \mathbf{u}^{\top}.
  \end{equation}
Moreover, $\lambda \mathbf{M}_{\boldsymbol{\theta}}$ is realizable within $\mathcal{O}_{N_{h}}$.
\end{lemma}

\begin{proof}
  First, apply the positive-scaling symmetry to every neuron of $\boldsymbol{\theta}$:
\begin{equation}
    (\overline{\mathbf{w}}_{j},\mathbf{a}_{j})
    \mapsto
    \bigl(
      \sqrt{\lambda} \, \overline{\mathbf{w}}_{j},
      \lambda^{-1/2} \mathbf{a}_{j}
    \bigr),
    \qquad
    j = 1, \dots, N_{h}.
  \end{equation}
This preserves the realized function while multiplying every feature by $\sqrt{\lambda}$.
  Hence, the hidden-activation matrix is multiplied by $\sqrt{\lambda}$ and its \ac{rsm} by $\lambda$.
  In particular, $\lambda\mathbf{M}_{\boldsymbol{\theta}}$ is realizable within $\mathcal{O}_{N_{h}}$.

  Next, append a single neuron with parameters
\begin{equation}
    (\sqrt{t} \, \overline{\mathbf{w}}, \mathbf{0}),
    \qquad
    t \geq 0.
  \end{equation}
If $t>0$, positive $1$-homogeneity gives the feature $\sqrt{t} \, \mathbf{u}^{\top}$.
  If $t=0$, the feature vanishes because positive $1$-homogeneity forces $\sigma(0) = 0$.
  The appended neuron forms a singleton zero-neuron group when its incoming weights are nonzero and is a zero-readout constant neuron otherwise.
  In either case, its addition preserves the symmetry orbit, while its contribution to the \ac{rsm} is $t\mathbf{u} \mathbf{u}^{\top}$.
  Therefore,
\begin{equation}
    \mathbf{M}_{\boldsymbol{\xi}}
    = \lambda \mathbf{M}_{\boldsymbol{\theta}}
    + t \mathbf{u} \mathbf{u}^{\top},
  \end{equation}
as claimed.
\end{proof}

The preceding lemma provides the mechanism by which increasing width enlarges the set of realizable geometries: increasing the width by one can leave an existing \ac{rsm} unchanged or augment it by an independently weighted rank-one contribution from any realizable single-neuron feature.
We use this to show that the sets of realizable \acp{rsm} and attainable similarity scores are nested with increasing width, and that the spread $\Delta_{N_{h}}(\mathbf{N})$ of attainable similarity scores increases strictly whenever it has not yet reached its orbit-wide limit $\Delta_{\infty}(\mathbf{N})$.

\begin{proof}[Proof of \cref{prop:strict-growth-until-saturation}]
  We divide the proof into three steps.

  \paragraph{Nesting.}
  Let $\mathbf{M} \in \mathcal{R}_{N_{h}}$.
  Setting $\lambda=1$ and $t=0$ in \cref{lem:scaling-and-zero-readout-feature-addition} yields a parameterization in $\mathcal{O}_{N_{h}+1}$ with the same \ac{rsm} $\mathbf{M}$.
  Hence, every \ac{rsm} realizable at width $N_{h}$ is also realizable at width $N_{h}+1$, so
\begin{equation}
    \label{eq:nesting-realizable-rsms}
    \mathcal{R}_{N_{h}}
    \subseteq
    \mathcal{R}_{N_{h}+1}.
  \end{equation}
Applying $\rho(\,\cdot\,, \mathbf{N})$ to these sets then gives
\begin{equation}
    \label{eq:nesting-similarity-scores}
    \mathcal{S}_{N_{h}}(\mathbf{N})
    \subseteq
    \mathcal{S}_{N_{h}+1}(\mathbf{N}).
  \end{equation}

  \paragraph{Strict improvement of the upper bound.}
  Fix a width $N_{h}$, and suppose
\begin{equation}
    s \coloneqq
    \sup \mathcal{S}_{N_{h}}(\mathbf{N})
    < \sup \mathcal{S}(\mathbf{N}).
  \end{equation}
Then, there exists an \ac{rsm} $\mathbf{B}$ realized by a parameterization in the same symmetry orbit at some finite width $L$ such that
\begin{equation}
    \rho(\mathbf{B}, \mathbf{N}) > s.
  \end{equation}
Necessarily, $L > N_{h}$, since nesting would otherwise imply $\rho(\mathbf{B}, \mathbf{N}) \in \mathcal{S}_{N_{h}}(\mathbf{N})$.
  Let $\mathbf{v}_{1}^{\top}, \dots, \mathbf{v}_{L}^{\top}$ denote the feature rows of the parameterization giving rise to $\mathbf{B}$, that is, $\mathbf{B} = \sum_{j}\mathbf{v}_{j} \mathbf{v}_{j}^{\top}$.
  By linearity of the projection $\Pi_{\Hol}$, we have
\begin{equation}
    \mathbf{Y}
    \coloneqq
    \Pi_{\Hol}(\mathbf{B})
    = \sum_{j=1}^{L}\mathbf{G}_{j},
    \qquad
    \mathbf{G}_{j}
    \coloneqq
    \Pi_{\Hol}(\mathbf{v}_{j}\mathbf{v}_{j}^{\top}).
  \end{equation}
Since $\rho(\mathbf{B},\mathbf{N}) = \langle \mathbf{Y}, \mathbf{N}^{\circ} \rangle_{F} / \norm{\mathbf{Y}}_{F}$ by \cref{lem:pearson-as-cosine},
  the choice of $\mathbf{B}$ implies
\begin{equation}
    \label{eq:wider-geometry-strict-improvement}
    \langle \mathbf{Y}, \mathbf{N}^{\circ} \rangle_{F}
    > s \norm{\mathbf{Y}}_{F}.
  \end{equation}
We will show that a single feature from this wider parameterization can already be added to a suitable width-$N_{h}$ parameterization to obtain similarity strictly greater than $s$ at width $N_{h}+1$.

  \emph{Case $s\geq0$.}
  Choose a sequence $\mathbf{M}_{i} \in \mathcal{R}_{N_{h}}$ such that
\begin{equation}
  \lim_{i\to\infty}
    \rho(\mathbf{M}_{i}, \mathbf{N})
    = s.
  \end{equation}
By the same-width scaling construction in \cref{lem:scaling-and-zero-readout-feature-addition}, each $\mathbf{M}_{i}$ may be rescaled within $\mathcal{R}_{N_{h}}$, without changing its similarity to $\mathbf{N}$, so that $\norm{\Pi_{\Hol}(\mathbf{M}_{i})}_{F} = 1$, and hence $\Pi_{\Hol}(\mathbf{M}_{i}) = \mathbf{M}_{i}^{\circ}$.
  Consequently, the sequence $\mathbf{M}_{i}^{\circ}$ lies on the unit sphere of the finite-dimensional vector space $\Hol$ defined in \cref{eq:vec-space-hollow-symmetric-zero-mean}.
  By compactness of this sphere, there exist a subsequence $\mathbf{M}_{i_{k}}^{\circ}$ and a matrix $\mathbf{Q} \in \Hol$ such that $\lim_{k\to\infty} \mathbf{M}_{i_{k}}^{\circ} = \mathbf{Q}$.
  By continuity of the Frobenius norm,
\begin{equation}
    \norm{\mathbf{Q}}_{F}
    = \lim_{k\to\infty}
    \norm{\mathbf{M}_{i_{k}}^{\circ}}_{F}
    = 1.
  \end{equation}
Similarly, continuity of the Frobenius inner product gives
\begin{equation}
    \langle
      \mathbf{Q},
      \mathbf{N}^{\circ}
    \rangle_{F}
    = \lim_{k\to\infty}
    \langle
      \mathbf{M}_{i_{k}}^{\circ},
      \mathbf{N}^{\circ}
    \rangle_{F}
    = \lim_{k\to\infty}
    \rho(\mathbf{M}_{i_{k}},\mathbf{N})
    = s,
  \end{equation}
where the penultimate equality follows from \cref{lem:pearson-as-cosine}.
  Since $s \geq 0$ and $\norm{\mathbf{Q}}_{F} = 1$, the Cauchy--Schwarz inequality gives
\begin{equation}
    \langle \mathbf{Y}, \mathbf{N}^{\circ}-s\mathbf{Q} \rangle_{F}
    = \langle \mathbf{Y}, \mathbf{N}^{\circ} \rangle_{F}
    - s\langle \mathbf{Y}, \mathbf{Q} \rangle_{F}
    \geq
    \langle \mathbf{Y}, \mathbf{N}^{\circ} \rangle_{F}
    - s\norm{\mathbf{Y}}_{F}
    > 0.
  \end{equation}
Using $\mathbf{Y} = \sum_{j} \mathbf{G}_{j}$, we therefore have
\begin{equation}
    \sum_{j=1}^{L}
    \langle
      \mathbf{G}_{j},
      \mathbf{N}^{\circ} - s\mathbf{Q}
    \rangle_{F}
    > 0.
  \end{equation}
Hence, there exists at least one index $j$ such that
\begin{equation}
    \langle
      \mathbf{G}_{j},
      \mathbf{N}^{\circ} - s\mathbf{Q}
    \rangle_{F}
    > 0.
  \end{equation}
For this index $j$, define
\begin{equation}
    \phi(t)
    \coloneqq
    \frac{
      \langle \mathbf{Q} + t\mathbf{G}_{j}, \mathbf{N}^{\circ} \rangle_{F}
    }{
      \norm{\mathbf{Q} + t\mathbf{G}_{j}}_{F}
    }.
  \end{equation}
Since $\norm{\mathbf{Q}}_{F} = 1$ and $\langle \mathbf{Q}, \mathbf{N}^{\circ} \rangle_{F} = s$, we have $\phi(0) = s$.
  Differentiating at $t=0$ gives
\begin{equation}
    \phi'(0)
    = \langle
      \mathbf{G}_{j},
      \mathbf{N}^{\circ} - s\mathbf{Q}
    \rangle_{F}
    > 0.
  \end{equation}
Hence, we can choose $t>0$ sufficiently small that
\begin{equation}
    \frac{
      \langle \mathbf{Q}+t\mathbf{G}_{j}, \mathbf{N}^{\circ} \rangle_{F}
    }{
      \norm{\mathbf{Q}+t\mathbf{G}_{j}}_{F}
    } > s.
  \end{equation}
Continuity and $\mathbf{M}_{i_{k}}^{\circ} \to \mathbf{Q}$ imply that, for all sufficiently large $k$,
\begin{equation}
    \frac{
      \langle
        \mathbf{M}_{i_{k}}^{\circ} + t\mathbf{G}_{j},
        \mathbf{N}^{\circ}
      \rangle_{F}
    }{
      \norm{\mathbf{M}_{i_{k}}^{\circ}
      + t\mathbf{G}_{j}}_{F}
    } > s.
  \end{equation}
By linearity of $\Pi_{\Hol}$ and the definition of $\mathbf{G}_{j}$,
\begin{equation}
    \Pi_{\Hol}
    (
      \mathbf{M}_{i_{k}}
      + t \mathbf{v}_{j} \mathbf{v}_{j}^{\top}
    )
    = \mathbf{M}_{i_{k}}^{\circ}
    + t \mathbf{G}_{j}.
  \end{equation}
Therefore, \cref{lem:pearson-as-cosine} gives, for all sufficiently large $k$,
\begin{equation}
    \rho(
      \mathbf{M}_{i_{k}}
      + t \mathbf{v}_{j} \mathbf{v}_{j}^{\top},
      \mathbf{N}
    )
    =
    \frac{
      \langle
        \mathbf{M}_{i_{k}}^{\circ} + t \mathbf{G}_{j},
        \mathbf{N}^{\circ}
      \rangle_{F}
    }{
      \norm{\mathbf{M}_{i_{k}}^{\circ}
      + t\mathbf{G}_{j}}_{F}
    }
    > s.
  \end{equation}
By \cref{lem:scaling-and-zero-readout-feature-addition}, each such $\mathbf{M}_{i_{k}} + t \mathbf{v}_{j} \mathbf{v}_{j}^{\top}$ is realizable within $\mathcal{O}_{N_{h}+1}$.
  Hence,
\begin{equation}
    \sup\mathcal{S}_{N_{h}+1}(\mathbf{N})
    > s
    = \sup\mathcal{S}_{N_{h}}(\mathbf{N}).
  \end{equation}

  \emph{Case $s<0$.}
  We first show that there exists an index $j$ with $\mathbf{G}_{j} \neq \mathbf{0}$ such that
\begin{equation}
    \label{eq:improving-feature-direction}
    \frac{
      \langle \mathbf{G}_{j}, \mathbf{N}^{\circ} \rangle_{F}
    }{
      \norm{\mathbf{G}_{j}}_{F}
    } > s.
  \end{equation}
Suppose, to the contrary, that every nonzero $\mathbf{G}_{j}$ satisfies $\langle \mathbf{G}_{j}, \mathbf{N}^{\circ} \rangle_{F} \leq s \norm{\mathbf{G}_{j}}_{F}$.
  The same inequality holds trivially when $\mathbf{G}_{j} = \mathbf{0}$.
  Summing over $j$ and using $\mathbf{Y} = \sum_{j} \mathbf{G}_{j}$ gives
\begin{equation}
    \langle \mathbf{Y}, \mathbf{N}^{\circ} \rangle_{F}
    \leq
    s \sum_{j=1}^{L} \norm{\mathbf{G}_{j}}_{F}
    \leq
    s \norm{\mathbf{Y}}_{F},
  \end{equation}
where the second inequality follows from the triangle inequality and $s<0$.
  This contradicts \cref{eq:wider-geometry-strict-improvement}, establishing that there does exist an index $j$ with $\mathbf{G}_{j} \neq \mathbf{0}$ satisfying \cref{eq:improving-feature-direction}.
  Fix any $\mathbf{M}_{0} \in \mathcal{R}_{N_{h}}$, and choose an index $j$ satisfying \cref{eq:improving-feature-direction}.
  For all sufficiently large $t>0$, linearity of $\Pi_{\Hol}$ and \cref{lem:pearson-as-cosine} give
\begin{equation}
    \rho(
      \mathbf{M}_{0}
      + t \mathbf{v}_{j} \mathbf{v}_{j}^{\top},
      \mathbf{N}
    )
    = \frac{
      \langle
        \Pi_{\Hol}(\mathbf{M}_{0}) + t\mathbf{G}_{j},
        \mathbf{N}^{\circ}
      \rangle_{F}
    }{
      \norm{
        \Pi_{\Hol}(\mathbf{M}_{0}) + t\mathbf{G}_{j}
      }_{F}
    }
    =
    \frac{
      \langle
        t^{-1}\Pi_{\Hol}(\mathbf{M}_{0}) + \mathbf{G}_{j},
        \mathbf{N}^{\circ}
      \rangle_{F}
    }{
      \norm{
        t^{-1} \Pi_{\Hol}(\mathbf{M}_{0})+\mathbf{G}_{j}
      }_{F}
    }.
  \end{equation}
Clearly, $\lim_{t\to\infty} t^{-1}\Pi_{\Hol}(\mathbf{M}_{0}) = \mathbf{0}$.
  Since $\mathbf{G}_{j} \neq \mathbf{0}$, continuity of the Frobenius inner product and norm therefore gives
\begin{equation}
    \lim_{t\to\infty}
    \rho(
      \mathbf{M}_{0}
      + t \mathbf{v}_{j} \mathbf{v}_{j}^{\top},
      \mathbf{N}
    )
    = \frac{
      \langle \mathbf{G}_{j}, \mathbf{N}^{\circ} \rangle_{F}
    }{
      \norm{\mathbf{G}_{j}}_{F}
    } > s,
  \end{equation}
where the final inequality follows from \cref{eq:improving-feature-direction}.
  Hence, for some sufficiently large finite $t>0$, we have $\rho(\mathbf{M}_{0} + t \mathbf{v}_{j} \mathbf{v}_{j}^{\top}, \mathbf{N}) > s$.
  By \cref{lem:scaling-and-zero-readout-feature-addition}, this \ac{rsm} is realizable within $\mathcal{O}_{N_{h}+1}$.
  Therefore,
\begin{equation}
    \sup\mathcal{S}_{N_{h}+1}(\mathbf{N})
    > s
    = \sup\mathcal{S}_{N_{h}}(\mathbf{N}).
  \end{equation}

  In both cases, we have shown that $\sup\mathcal{S}_{N_{h}}(\mathbf{N}) < \sup\mathcal{S}(\mathbf{N})$ implies $\sup\mathcal{S}_{N_{h}}(\mathbf{N}) < \sup\mathcal{S}_{N_{h}+1}(\mathbf{N})$.
  Conversely, if $\sup\mathcal{S}_{N_{h}}(\mathbf{N}) = \sup\mathcal{S}(\mathbf{N})$,
  then
\begin{equation}
    \mathcal{S}_{N_{h}}(\mathbf{N})
    \subseteq
    \mathcal{S}_{N_{h}+1}(\mathbf{N})
    \subseteq
    \mathcal{S}(\mathbf{N})
  \end{equation}
implies
\begin{equation}
    \sup\mathcal{S}_{N_{h}+1}(\mathbf{N})
    = \sup\mathcal{S}_{N_{h}}(\mathbf{N})
    = \sup\mathcal{S}(\mathbf{N}).
  \end{equation}
Therefore,
\begin{equation}
    \label{eq:strict-growth-upper-similarity-bound}
    \sup\mathcal{S}_{N_{h}+1}(\mathbf{N})
    >
    \sup\mathcal{S}_{N_{h}}(\mathbf{N})
    \quad\Longleftrightarrow\quad
    \sup\mathcal{S}_{N_{h}}(\mathbf{N})
    <
    \sup\mathcal{S}(\mathbf{N}).
  \end{equation}

  \paragraph{Lower bound and spread.}
  The argument establishing \cref{eq:strict-growth-upper-similarity-bound} depends on the reference geometry $\mathbf{N}$ only through $\mathbf{N}^{\circ}$.
  Replacing $\mathbf{N}^{\circ}$ by $-\mathbf{N}^{\circ}$ negates every similarity score, since
\begin{equation}
    \langle
      \mathbf{M}^{\circ},
      -\mathbf{N}^{\circ}
    \rangle_{F}
    = -\rho(\mathbf{M}, \mathbf{N}).
  \end{equation}
Repeating the argument leading to \cref{eq:strict-growth-upper-similarity-bound} with $-\mathbf{N}^{\circ}$ in place of $\mathbf{N}^{\circ}$ therefore gives
\begin{equation}
    \label{eq:strict-growth-lower-similarity-bound}
    \inf\mathcal{S}_{N_{h}+1}(\mathbf{N})
    <
    \inf\mathcal{S}_{N_{h}}(\mathbf{N})
    \quad\Longleftrightarrow\quad
    \inf\mathcal{S}_{N_{h}}(\mathbf{N})
    >
    \inf\mathcal{S}(\mathbf{N}).
  \end{equation}
Finally, by definition,
\begin{equation}
    \Delta_{\infty}(\mathbf{N})
    - \Delta_{N_{h}}(\mathbf{N})
    = (
      \sup \mathcal{S}(\mathbf{N})
      - \sup \mathcal{S}_{N_{h}}(\mathbf{N})
    )
    + (
      \inf \mathcal{S}_{N_{h}}(\mathbf{N})
      - \inf \mathcal{S}(\mathbf{N})
    ).
  \end{equation}
Both terms are nonnegative because, by definition, $\mathcal{S}_{N_{h}}(\mathbf{N}) \subseteq \mathcal{S}(\mathbf{N})$.
  Hence, $\Delta_{N_{h}}(\mathbf{N}) < \Delta_{\infty}(\mathbf{N})$ if and only if at least one of the upper or lower similarity bounds has not reached its orbit-wide limit.
  Similarly,
\begin{equation}
    \begin{aligned}
      \Delta_{N_{h}+1}(\mathbf{N})
      - \Delta_{N_{h}}(\mathbf{N})
      &=
      (
        \sup \mathcal{S}_{N_{h}+1}(\mathbf{N})
        - \sup \mathcal{S}_{N_{h}}(\mathbf{N})
      )
      \\
      &\quad+
      (
        \inf \mathcal{S}_{N_{h}}(\mathbf{N})
        - \inf \mathcal{S}_{N_{h}+1}(\mathbf{N})
      ).
    \end{aligned}
  \end{equation}
Again, both terms are nonnegative, since $\mathcal{S}_{N_{h}}(\mathbf{N}) \subseteq \mathcal{S}_{N_{h}+1}(\mathbf{N})$ by \cref{eq:nesting-similarity-scores}.
  By \cref{eq:strict-growth-upper-similarity-bound,eq:strict-growth-lower-similarity-bound}, at least one term is strictly positive if and only if at least one of the corresponding orbit-wide bounds has not yet been reached.
  Therefore,
\begin{equation}
    \Delta_{N_{h}+1}(\mathbf{N})
    >
    \Delta_{N_{h}}(\mathbf{N})
    \quad\Longleftrightarrow\quad
    \Delta_{N_{h}}(\mathbf{N})
    <
    \Delta_{\infty}(\mathbf{N}),
  \end{equation}
completing the proof.
\end{proof}

\subsection{Function-independent similarity limits}
\label[appendix]{app-subsec:function-independent-similarity-limits}

\Cref{prop:function-independent-similarity-limits} establishes that the limiting lower and upper similarity bounds are independent of the symmetry orbit and that the corresponding fixed-width bounds stabilize at finite width.
We prove this result by characterizing the closure of the set of normalized projections onto $\Hol$ of \acp{rsm} realizable within a symmetry orbit.
We write $\cl$ for closure with respect to the Frobenius norm and $\overline{\cone}(\mathcal{A})$ for the closed conic hull of a set $\mathcal{A}$.

\paragraph{The closed projected feature cone.}
Define the set of projected single-neuron contributions
\begin{equation}
  \label{eq:projected-feature-generators}
  \mathcal{G}_{\sigma,\mathbf{X}}
  \coloneqq
  \set{
    \Pi_{\Hol}(\mathbf{u}\mathbf{u}^{\top})
    \given
    \mathbf{u}^{\top}
    = \sigma(\overline{\mathbf{w}}^{\top} \overline{\mathbf{X}}),
    \;
    \overline{\mathbf{w}} \in \mathbb{R}^{N_{i}+1}
  }.
\end{equation}
The closed projected feature cone is then
\begin{equation}
  \label{eq:closed-projected-feature-cone}
  \mathcal{C}_{\sigma,\mathbf{X}}
  \coloneqq
  \overline{\cone}\bigl(
    \mathcal{G}_{\sigma,\mathbf{X}}
  \bigr)
  \subseteq \Hol.
\end{equation}
By definition, both $\mathcal{G}_{\sigma,\mathbf{X}}$ and $\mathcal{C}_{\sigma,\mathbf{X}}$ depend only on $\sigma$ and $\mathbf{X}$, and not on any particular symmetry orbit.
Let
\begin{equation}
  \label{eq:projected-feature-cone-dimension}
  d
  \coloneqq
  \dim\spanop\mathcal{C}_{\sigma,\mathbf{X}}
  \leq \dim\Hol
  = \frac{P(P-1)}{2}-1.
\end{equation}
Since $\mathcal{S}(\mathbf{N})$ is nonempty, there exists $\mathbf{M} = \mathbf{H}^{\top}\mathbf{H} \in \mathcal{R}$ with nonconstant strict upper-triangular entries, and hence $\Pi_{\Hol}(\mathbf{M}) \neq \mathbf{0}$.
Writing the rows of $\mathbf{H}$ as $\mathbf{u}_{j}^{\top}$ gives
\begin{equation}
  \Pi_{\Hol}(\mathbf{M})
  =
  \sum_{j}
  \Pi_{\Hol}(\mathbf{u}_{j}\mathbf{u}_{j}^{\top})
  \in
  \cone(\mathcal{G}_{\sigma,\mathbf{X}})
  \subseteq
  \mathcal{C}_{\sigma,\mathbf{X}}.
\end{equation}
Thus, $\mathcal{C}_{\sigma,\mathbf{X}} \neq \set{\mathbf{0}}$, so $d \geq 1$.

The following lemma characterizes the closure of the set of normalized projections onto $\Hol$ of \acp{rsm} attainable within any symmetry orbit.
In particular, for every $N_{h} \geq N_{h}^{\star} + d$, every unit-norm element of $\mathcal{C}_{\sigma, \mathbf{X}}$ lies in the closure of the normalized projections onto $\Hol$ of \acp{rsm} realizable at width $N_{h}$.

\begin{lemma}[Closure of normalized projected geometries within a symmetry orbit]
  \label[lemma]{lem:limiting-normalized-geometries-within-a-symmetry-orbit}
  Fix a symmetry orbit $\mathcal{O}$, and let $N_{h}^{\star}$ denote the minimum width among parameterizations in $\mathcal{O}$.
  Then,
\begin{equation}
    \label{eq:limiting-normalized-geometries}
    \cl\set{
      \mathbf{M}^{\circ}
      \given
      \mathbf{M} \in \mathcal{R}_{N_{h}}
    }
    = \set{
      \mathbf{Y} \in \mathcal{C}_{\sigma,\mathbf{X}}
      \given
      \norm{\mathbf{Y}}_{F} = 1
    },
    \qquad
    N_{h} \geq N_{h}^{\star} + d.
  \end{equation}
\end{lemma}

\begin{proof}
  We prove the set equality in two steps, showing that each set is contained in the other.

  \paragraph{Cone membership.}
  Fix $N_{h}$ and let $\mathbf{M} \in \mathcal{R}_{N_{h}}$.
  Let $\mathbf{v}_{1}^{\top}, \dots, \mathbf{v}_{N_{h}}^{\top}$ denote the feature rows of a parameterization realizing $\mathbf{M}$, i.e., $\mathbf{M} = \sum_{j} \mathbf{v}_{j}\mathbf{v}_{j}^{\top}$.
  Linearity of $\Pi_{\Hol}$ gives
\begin{equation}
    \Pi_{\Hol}(\mathbf{M})
    = \sum_{j=1}^{N_{h}}
    \Pi_{\Hol}(
      \mathbf{v}_{j} \mathbf{v}_{j}^{\top}
    )
    \in
    \cone(
      \mathcal{G}_{\sigma,\mathbf{X}}
    )
    \subseteq
    \mathcal{C}_{\sigma,\mathbf{X}}.
  \end{equation}
Because $\mathbf{M} \in \mathcal{R}_{N_{h}}$, its strict upper-triangular entries are nonconstant, so $\norm{\Pi_{\Hol}(\mathbf{M})}_{F} > 0$.
  Since $\mathcal{C}_{\sigma,\mathbf{X}}$ is a cone,
\begin{equation}
    \mathbf{M}^{\circ}
    =
    \frac{
      \Pi_{\Hol}(\mathbf{M})
    }{
      \norm{\Pi_{\Hol}(\mathbf{M})}_{F}
    }
    \in
    \mathcal{C}_{\sigma,\mathbf{X}},
    \qquad
    \norm{\mathbf{M}^{\circ}}_{F} = 1.
  \end{equation}
The set
\begin{equation}
    \set{
      \mathbf{Y}\in\mathcal{C}_{\sigma,\mathbf{X}}
      \given
      \norm{\mathbf{Y}}_{F}=1
    }
  \end{equation}
is closed because it is the intersection of $\mathcal{C}_{\sigma,\mathbf{X}}$ and the unit sphere in $\Hol$, both of which are closed.
  Therefore,
\begin{equation}
    \label{eq:normalized-projected-geometries-subset-feature-cone}
    \cl\set{
      \mathbf{M}^{\circ}
      \given
      \mathbf{M} \in \mathcal{R}_{N_{h}}
    }
    \subseteq
    \set{
      \mathbf{Y}\in\mathcal{C}_{\sigma,\mathbf{X}}
      \given
      \norm{\mathbf{Y}}_{F} = 1
    }.
  \end{equation}

  \paragraph{Approximation by at most $d$ feature contributions.}
  Fix $\mathbf{Y} \in \mathcal{C}_{\sigma,\mathbf{X}}$ with $\norm{\mathbf{Y}}_{F}=1$.
  Since $\mathcal{C}_{\sigma,\mathbf{X}} = \overline{\cone}(\mathcal{G}_{\sigma,\mathbf{X}})$, there exists a sequence
\begin{equation}
    \set{\mathbf{Y}_{i}}_{i=1}^{\infty}
    \subseteq \cone(\mathcal{G}_{\sigma,\mathbf{X}})
    \qquad\text{such that}\qquad
    \lim_{i\to\infty} \mathbf{Y}_{i} = \mathbf{Y}.
  \end{equation}
Moreover, $\spanop\mathcal{G}_{\sigma,\mathbf{X}} = \spanop\mathcal{C}_{\sigma,\mathbf{X}}$, because the finite-dimensional space $\spanop\mathcal{G}_{\sigma,\mathbf{X}}$ is closed and therefore contains the closure of $\cone(\mathcal{G}_{\sigma,\mathbf{X}})$.
  Hence, the generators lie in a $d$-dimensional vector space.
  By the conic version of Carathéodory's theorem, each $\mathbf{Y}_{i}$ can therefore be expressed using at most $d$ generators:
\begin{equation}
    \mathbf{Y}_{i}
    = \sum_{j=1}^{d}
    t_{ij} \mathbf{G}_{ij},
    \qquad
    t_{ij} \geq 0,
  \end{equation}
where $\mathbf{G}_{ij} \in \mathcal{G}_{\sigma,\mathbf{X}}$.
  If fewer than $d$ generators are required, we set the remaining coefficients to zero.
  Since each $\mathbf{G}_{ij} \in \mathcal{G}_{\sigma,\mathbf{X}}$, for every $i$ and $j$ there exists $\overline{\mathbf{w}}_{ij} \in \mathbb{R}^{N_{i}+1}$ such that
\begin{equation}
    \mathbf{G}_{ij}
    = \Pi_{\Hol}(
      \mathbf{u}_{ij}\mathbf{u}_{ij}^{\top}
    ),
    \qquad
    \mathbf{u}_{ij}^{\top}
    = \sigma(
      \overline{\mathbf{w}}_{ij}^{\top}
      \overline{\mathbf{X}}
    ).
  \end{equation}

  \paragraph{Realization within the orbit.}
  Choose an irreducible parameterization $\boldsymbol{\theta}^{\star} \in \mathcal{O}_{N_{h}^{\star}}$ with \ac{rsm} $\mathbf{M}_{\boldsymbol{\theta}^{\star}}$, and let $\set{\lambda_{i}}_{i=1}^{\infty}$ be any positive sequence with $\lambda_{i} \to 0$.
  Iterating \cref{lem:scaling-and-zero-readout-feature-addition} yields, for every $i \geq 1$, a parameterization in $\mathcal{O}_{N_{h}^{\star} + d}$ with \ac{rsm}
\begin{equation}
    \mathbf{M}_{i}
    \coloneqq
    \lambda_{i}\mathbf{M}_{\boldsymbol{\theta}^{\star}}
    + \sum_{j=1}^{d}
    t_{ij} \mathbf{u}_{ij} \mathbf{u}_{ij}^{\top}.
  \end{equation}
By linearity of $\Pi_{\Hol}$ and the definition of $\mathbf{Y}_{i}$,
\begin{equation}
    \Pi_{\Hol}(\mathbf{M}_{i})
    = \lambda_{i}\Pi_{\Hol}(\mathbf{M}_{\boldsymbol{\theta}^{\star}})
    + \sum_{j=1}^{d}
    t_{ij}
    \Pi_{\Hol}(
      \mathbf{u}_{ij}\mathbf{u}_{ij}^{\top}
    )
    = \lambda_{i}\Pi_{\Hol}(\mathbf{M}_{\boldsymbol{\theta}^{\star}})
    + \mathbf{Y}_{i}.
  \end{equation}
Since $\lambda_{i} \to 0$ and $\mathbf{Y}_{i} \to \mathbf{Y}$, we have $\Pi_{\Hol}(\mathbf{M}_{i}) \to \mathbf{Y}$.
  By continuity of the Frobenius norm,
\begin{equation}
    \lim_{i\to\infty}
    \norm{\Pi_{\Hol}(\mathbf{M}_{i})}_{F}
    = \norm{\mathbf{Y}}_{F}
    = 1.
  \end{equation}
Hence, $\Pi_{\Hol}(\mathbf{M}_{i}) \neq \mathbf{0}$ for all sufficiently large $i$, so $\mathbf{M}_{i} \in \mathcal{R}_{N_{h}^{\star} + d}$ for all such $i$.
  Moreover,
\begin{equation}
    \lim_{i\to\infty}
    \mathbf{M}_{i}^{\circ}
    = \lim_{i\to\infty}
    \frac{
      \Pi_{\Hol}(\mathbf{M}_{i})
    }{
      \norm{\Pi_{\Hol}(\mathbf{M}_{i})}_{F}
    }
    =
    \mathbf{Y},
  \end{equation}
again by continuity.
  Therefore,
\begin{equation}
    \mathbf{Y}
    \in
    \cl\set{
      \mathbf{M}^{\circ}
      \given
      \mathbf{M}\in\mathcal{R}_{N_{h}^{\star} + d}
    }.
  \end{equation}
This proves the reverse inclusion at width $N_{h}^{\star} + d \leq N_{h}$.
  By nesting, $\mathcal{R}_{N_{h}^{\star} + d} \subseteq \mathcal{R}_{N_{h}}$ for every $N_{h} \geq N_{h}^{\star} + d$, so the same inclusion holds at every such width.
  Together with the set inclusion in \cref{eq:normalized-projected-geometries-subset-feature-cone}, this proves \cref{eq:limiting-normalized-geometries}.
\end{proof}

The characterization in \cref{lem:limiting-normalized-geometries-within-a-symmetry-orbit} now yields both claims of \cref{prop:function-independent-similarity-limits}.

\begin{proof}[Proof of \cref{prop:function-independent-similarity-limits}]
  Fix a symmetry orbit $\mathcal{O}$, and let $N_{h}^{\star}$ denote the minimum width among parameterizations in $\mathcal{O}$.
  Let
\begin{equation}
    \mathcal{K}_{\sigma,\mathbf{X}}
    \coloneqq
    \set{
      \mathbf{Y} \in \mathcal{C}_{\sigma,\mathbf{X}}
      \given
      \norm{\mathbf{Y}}_{F} = 1
    }.
  \end{equation}
The set $\mathcal{K}_{\sigma,\mathbf{X}}$ is the intersection of the closed set $\mathcal{C}_{\sigma,\mathbf{X}}$ with the unit sphere in the finite-dimensional space $\Hol$, and is therefore compact.
  Since $\mathcal{C}_{\sigma,\mathbf{X}}$ is a nontrivial cone, $\mathcal{K}_{\sigma,\mathbf{X}}$ is also nonempty.

  Fix $N_{h} \geq N_{h}^{\star} + d$.
  By \cref{lem:limiting-normalized-geometries-within-a-symmetry-orbit},
\begin{equation}
    \cl\set{
      \mathbf{M}^{\circ}
      \given
      \mathbf{M} \in \mathcal{R}_{N_{h}}
    }
    =
    \mathcal{K}_{\sigma,\mathbf{X}}.
  \end{equation}
By \cref{lem:pearson-as-cosine}, $\rho(\mathbf{M},\mathbf{N}) = \langle \mathbf{M}^{\circ}, \mathbf{N}^{\circ} \rangle_{F}$.
  Since $\mathcal{K}_{\sigma,\mathbf{X}}$ is compact and nonempty, the continuous map $\mathbf{Y} \mapsto \langle \mathbf{Y}, \mathbf{N}^{\circ} \rangle_{F}$ attains its minimum and maximum on $\mathcal{K}_{\sigma,\mathbf{X}}$.
  By density of the realizable normalized projections $\mathbf{M}^{\circ}$ in $\mathcal{K}_{\sigma,\mathbf{X}}$,
\begin{equation}
    \inf\mathcal{S}_{N_{h}}(\mathbf{N})
    =
    \inf_{\mathbf{M} \in \mathcal{R}_{N_{h}}}
    \rho(\mathbf{M}, \mathbf{N})
    =
    \inf_{\mathbf{M} \in \mathcal{R}_{N_{h}}}
    \langle \mathbf{M}^{\circ}, \mathbf{N}^{\circ} \rangle_{F}
    =
    \min_{\mathbf{Y} \in \mathcal{K}_{\sigma,\mathbf{X}}}
    \langle \mathbf{Y}, \mathbf{N}^{\circ} \rangle_{F},
  \end{equation}
and similarly,
\begin{equation}
    \sup\mathcal{S}_{N_{h}}(\mathbf{N})
    =
    \max_{\mathbf{Y} \in \mathcal{K}_{\sigma,\mathbf{X}}}
    \langle \mathbf{Y}, \mathbf{N}^{\circ} \rangle_{F}.
  \end{equation}

  By the \emph{cone membership} part of the proof of \cref{lem:limiting-normalized-geometries-within-a-symmetry-orbit}, every $\mathbf{M} \in \mathcal{R}$ satisfies $\mathbf{M}^{\circ} \in \mathcal{K}_{\sigma,\mathbf{X}}$.
  Hence, $\set{\mathbf{M}^{\circ} \given \mathbf{M} \in \mathcal{R}} \subseteq \mathcal{K}_{\sigma,\mathbf{X}}$, and thus
\begin{equation}
    \min_{\mathbf{Y} \in \mathcal{K}_{\sigma,\mathbf{X}}}
    \langle \mathbf{Y}, \mathbf{N}^{\circ} \rangle_{F}
    \leq
    \inf\mathcal{S}(\mathbf{N})
    \leq
    \inf\mathcal{S}_{N_{h}}(\mathbf{N})
    =
    \min_{\mathbf{Y} \in \mathcal{K}_{\sigma,\mathbf{X}}}
    \langle \mathbf{Y}, \mathbf{N}^{\circ} \rangle_{F},
  \end{equation}
where the second inequality follows from $\mathcal{S}_{N_{h}}(\mathbf{N}) \subseteq \mathcal{S}(\mathbf{N})$, forcing
\begin{equation}
    \label{eq:orbit-independent-infima}
    \inf\mathcal{S}(\mathbf{N})
    =
    \inf\mathcal{S}_{N_{h}}(\mathbf{N})
    =
    \min_{\mathbf{Y} \in \mathcal{K}_{\sigma,\mathbf{X}}}
    \langle \mathbf{Y}, \mathbf{N}^{\circ} \rangle_{F}.
  \end{equation}
The same argument yields
\begin{equation}
    \label{eq:orbit-independent-suprema}
    \sup\mathcal{S}(\mathbf{N})
    =
    \sup\mathcal{S}_{N_{h}}(\mathbf{N})
    =
    \max_{\mathbf{Y} \in \mathcal{K}_{\sigma,\mathbf{X}}}
    \langle \mathbf{Y}, \mathbf{N}^{\circ} \rangle_{F}.
  \end{equation}
The right-hand sides in \cref{eq:orbit-independent-infima,eq:orbit-independent-suprema} depend only on $\sigma$, $\mathbf{X}$, and $\mathbf{N}$, and are independent of the symmetry orbit $\mathcal{O}$.
  Hence, $\inf\mathcal{S}(\mathbf{N})$ and $\sup\mathcal{S}(\mathbf{N})$ are identical across all symmetry orbits, and the identities
\begin{equation}
    \inf\mathcal{S}_{N_{h}}(\mathbf{N})
    = \inf\mathcal{S}(\mathbf{N}),
    \qquad
    \sup\mathcal{S}_{N_{h}}(\mathbf{N})
    = \sup\mathcal{S}(\mathbf{N}),
    \qquad
    N_{h} \geq N_{h}^{\star} + d
  \end{equation}
yield
\begin{equation}
    \Delta_{N_{h}}(\mathbf{N})
    = \Delta_{\infty}(\mathbf{N}),
    \qquad
    N_{h} \geq N_{h}^{\star} + d.
  \end{equation}
Hence, every symmetry orbit reaches the common limiting similarity bounds $\inf\mathcal{S}(\mathbf{N})$ and $\sup\mathcal{S}(\mathbf{N})$, and therefore the common limiting spread $\Delta_{\infty}(\mathbf{N})$, at finite width.
\end{proof}

\subsection{Alignment with reference geometries}
\label[appendix]{app-subsec:alignment-with-reference-geometries}

The preceding characterization of the closure of the set of normalized projected \acp{rsm} immediately yields a necessary and sufficient condition under which \acp{rsm} realizable within a symmetry orbit can approach perfect representational similarity with a fixed reference geometry $\mathbf{N}$ at sufficient width.

\begin{corollary}[Alignment with reference geometries]
  \label[corollary]{cor:alignment-with-reference-geometry}
  Fix a symmetry orbit $\mathcal{O}$ with activation $\sigma$, and let $N_{h}^{\star}$ denote the minimum width among parameterizations in $\mathcal{O}$.
  Then,
\begin{equation}
    \sup\mathcal{S}(\mathbf{N}) = 1
    \quad\Longleftrightarrow\quad
    \mathbf{N}^{\circ} \in \mathcal{C}_{\sigma,\mathbf{X}}.
  \end{equation}
Moreover, for every $N_{h} \geq N_{h}^{\star} + d$,
\begin{equation}
    \sup\mathcal{S}_{N_{h}}(\mathbf{N}) = 1
    \quad\Longleftrightarrow\quad
    \mathbf{N}^{\circ} \in \mathcal{C}_{\sigma,\mathbf{X}}.
  \end{equation}
\end{corollary}

\begin{proof}
  Recall from the proof of \cref{prop:function-independent-similarity-limits} that $\mathcal{K}_{\sigma,\mathbf{X}}$ denotes the set of unit-norm elements of $\mathcal{C}_{\sigma,\mathbf{X}}$.
  By \cref{eq:orbit-independent-suprema},
\begin{equation}
    \sup\mathcal{S}(\mathbf{N})
    =
    \sup\mathcal{S}_{N_{h}}(\mathbf{N})
    =
    \max_{\mathbf{Y} \in \mathcal{K}_{\sigma,\mathbf{X}}}
    \langle \mathbf{Y}, \mathbf{N}^{\circ} \rangle_{F},
    \qquad
    N_{h} \geq N_{h}^{\star} + d.
  \end{equation}
Since $\norm{\mathbf{Y}}_{F} = 1$ for $\mathbf{Y} \in \mathcal{K}_{\sigma,\mathbf{X}}$, and $\norm{\mathbf{N}^{\circ}}_{F} = 1$, the Cauchy--Schwarz inequality gives $\langle \mathbf{Y}, \mathbf{N}^{\circ} \rangle_{F} \leq 1$, with equality if and only if $\mathbf{Y} = \mathbf{N}^{\circ}$.
  Hence, the maximum equals $1$ if and only if $\mathbf{N}^{\circ} \in \mathcal{C}_{\sigma,\mathbf{X}}$, proving both claims.
\end{proof}

\Cref{cor:alignment-with-reference-geometry} is agnostic to the source of the reference \ac{rsm}, which may be derived from neural data, another neural network, or any other computational model.
The result depends \emph{only} on the normalized projection $\mathbf{N}^{\circ}$ and the cone $\mathcal{C}_{\sigma,\mathbf{X}}$ determined by $\sigma$ and $\mathbf{X}$.
Intuitively, this cone consists of the projected \acp{rsm} induced by features realizable by neurons with activation function $\sigma$ on the inputs $\mathbf{X}$.

\paragraph{References generated with the same activation.}
Suppose the reference \ac{rsm} is generated by a network that uses the same activation function $\sigma$ and is evaluated on the same inputs $\mathbf{X}$.
Its normalized projection $\mathbf{N}^{\circ}$ then belongs to $\mathcal{C}_{\sigma,\mathbf{X}}$, since the neuron-wise decomposition of $\mathbf{N}^{\circ}$ expresses it as a conic combination of projected single-neuron contributions.
If the reference network has width $L$, the projection admits such a representation with at most $\min\set{d,L}$ contributions.

\begin{proposition}[Same-activation reference geometries]
  \label[proposition]{prop:alignment-with-same-activation-reference-geometries}
  Suppose that $\mathbf{N}$ is the \ac{rsm} of a width-$L$ network with activation $\sigma$, evaluated on $\mathbf{X}$.
  Then, for every symmetry orbit $\mathcal{O}$ with activation $\sigma$,
\begin{equation}
    \label{eq:same-activation-perfect-alignment}
    \sup\mathcal{S}(\mathbf{N}) = 1.
  \end{equation}
Moreover, if $N_{h}^{\star}$ denotes the minimum width among parameterizations in $\mathcal{O}$, then
\begin{equation}
    \label{eq:same-activation-width-bound}
    \sup\mathcal{S}_{N_{h}}(\mathbf{N}) = 1,
    \qquad
    N_{h} \geq N_{h}^{\star} + \min\set{d,L}.
  \end{equation}
\end{proposition}

\begin{proof}
  Let $\mathbf{v}_{j}$ denote the feature of the $j$th hidden neuron of the reference network on the inputs $\mathbf{X}$.
  Then, $\mathbf{N} = \sum_{j=1}^{L} \mathbf{v}_{j}\mathbf{v}_{j}^{\top}$.
  By linearity of the projection $\Pi_{\Hol}$,
\begin{equation}
    \Pi_{\Hol}(\mathbf{N})
    = \sum_{j=1}^{L}
    \Pi_{\Hol}(
      \mathbf{v}_{j}\mathbf{v}_{j}^{\top}
    )
    \in \cone(
      \mathcal{G}_{\sigma,\mathbf{X}}
    )
    \subseteq \mathcal{C}_{\sigma,\mathbf{X}}.
  \end{equation}
By assumption, the strict upper-triangular entries of $\mathbf{N}$ are nonconstant, so $\Pi_{\Hol}(\mathbf{N}) \neq \mathbf{0}$.
  Because $\mathcal{C}_{\sigma,\mathbf{X}}$ is a cone,
\begin{equation}
    \mathbf{N}^{\circ}
    = \frac{
      \Pi_{\Hol}(\mathbf{N})
    }{
      \norm{\Pi_{\Hol}(\mathbf{N})}_{F}
    }
    \in \mathcal{C}_{\sigma,\mathbf{X}}.
  \end{equation}
Hence, \cref{cor:alignment-with-reference-geometry} establishes \cref{eq:same-activation-perfect-alignment} and yields
\begin{equation}
    \label{eq:same-activation-dimension-width-bound}
    \sup\mathcal{S}_{N_{h}}(\mathbf{N}) = 1,
    \qquad
    N_{h} \geq N_{h}^{\star} + d.
  \end{equation}

  To establish \cref{eq:same-activation-width-bound}, it remains to show that $\sup\mathcal{S}_{N_{h}^{\star}+L}(\mathbf{N}) = 1$.
  Choose an irreducible parameterization $\boldsymbol{\theta}^{\star}$ of width $N_{h}^{\star}$ in $\mathcal{O}$, and denote its \ac{rsm} by $\mathbf{M}_{0}$.
  Iteratively applying \cref{lem:scaling-and-zero-readout-feature-addition} yields a symmetry-equivalent parameterization $\boldsymbol{\xi} \sim \boldsymbol{\theta}^{\star}$ of width $N_{h}^{\star} + L$ with \ac{rsm}
\begin{equation}
    \mathbf{M}_{\lambda}
    = \lambda\mathbf{M}_{0}
    + \sum_{j=1}^{L}
    \mathbf{v}_{j} \mathbf{v}_{j}^{\top}
    = \lambda\mathbf{M}_{0}
    + \mathbf{N},
    \qquad
    \lambda > 0.
  \end{equation}
By linearity of $\Pi_{\Hol}$,
\begin{equation}
    \label{eq:same-activation-projection-convergence}
    \lim_{\lambda \downarrow 0}
    \Pi_{\Hol}(\mathbf{M}_{\lambda})
    = \lim_{\lambda \downarrow 0} \bigl(
      \lambda\Pi_{\Hol}(\mathbf{M}_{0})
      + \Pi_{\Hol}(\mathbf{N})
    \bigr)
    =
    \Pi_{\Hol}(\mathbf{N}).
  \end{equation}
Since $\Pi_{\Hol}(\mathbf{N}) \neq \mathbf{0}$, \cref{eq:same-activation-projection-convergence} implies that $\Pi_{\Hol}(\mathbf{M}_{\lambda}) \neq \mathbf{0}$ for all sufficiently small $\lambda > 0$, and hence $\mathbf{M}_{\lambda} \in \mathcal{R}_{N_{h}^{\star}+L}$.
  The normalized projections therefore satisfy
\begin{equation}
    \lim_{\lambda \downarrow 0}
    \mathbf{M}_{\lambda}^{\circ}
    = \frac{
      \Pi_{\Hol}(\mathbf{N})
    }{
      \norm{\Pi_{\Hol}(\mathbf{N})}_{F}
    }
    = \mathbf{N}^{\circ}.
  \end{equation}
By \cref{lem:pearson-as-cosine} and continuity of the Frobenius inner product, we obtain
\begin{equation}
    \lim_{\lambda \downarrow 0}
    \rho(\mathbf{M}_{\lambda}, \mathbf{N})
    =
    \lim_{\lambda \downarrow 0}
    \langle
      \mathbf{M}_{\lambda}^{\circ},
      \mathbf{N}^{\circ}
    \rangle_{F}
    =
    \langle
      \mathbf{N}^{\circ},
      \mathbf{N}^{\circ}
    \rangle_{F}
    = 1.
  \end{equation}
Since Pearson correlation is bounded above by $1$, this implies that $\sup\mathcal{S}_{N_{h}^{\star}+L}(\mathbf{N}) = 1$.
  Combining this with \cref{eq:same-activation-dimension-width-bound} gives $\sup\mathcal{S}_{N_{h}^{\star}+\min\set{d,L}}(\mathbf{N}) = 1$.
  The nesting of the attainable similarity sets $\mathcal{S}_{N_{h}}(\mathbf{N})$ as width increases (\cref{prop:strict-growth-until-saturation}) then establishes \cref{eq:same-activation-width-bound}, completing the proof.
\end{proof}

\Cref{prop:alignment-with-same-activation-reference-geometries} shows once more that representational geometry can vary substantially while the computed function remains unchanged.
Reference features can be added with zero readouts and amplified to make the \ac{rsm} arbitrarily well aligned with $\mathbf{N}$, regardless of the function computed by the reference network.
\Cref{fig:primitives} illustrates this mechanism by adding, with zero readout, a feature from an opposite-family solution (i.e., from the reference network giving rise to $\mathbf{N}$) in \cref{fig:dissociation}.
Duplication and scaling then amplify the feature's rank-one \ac{rsm} contribution, which is nearly perfectly correlated with $\mathbf{N}$, until it dominates the full \ac{rsm} without changing the realized function.

\paragraph{A sufficient condition for arbitrary-reference alignment.}
The preceding result shows that, when the reference geometry $\mathbf{N}$ is generated by a network with activation $\sigma$, similarity to that reference can approach $1$ within every symmetry orbit with the same activation.
For an \emph{arbitrary} reference geometry, however, membership of $\mathbf{N}^{\circ}$ in $\mathcal{C}_{\sigma,\mathbf{X}}$ is not automatic.
We next give a sufficient condition under which $\mathcal{C}_{\sigma,\mathbf{X}}$ does contain the normalized projection of every reference geometry.

\begin{proposition}[Affinely independent probe inputs]
  \label[proposition]{prop:alignment-with-affinely-independent-inputs}
  Suppose that $\operatorname{rank}(\overline{\mathbf{X}}) = P$.
  Then, $\mathcal{C}_{\sigma,\mathbf{X}} = \Hol$.
  Consequently, for any symmetry orbit $\mathcal{O}$ with activation $\sigma$ and minimum width $N_{h}^{\star}$,
\begin{equation}
    \label{eq:affinely-independent-orbit-wide-similarity-bounds}
    \inf\mathcal{S}(\mathbf{N}) = -1,
    \qquad
    \sup\mathcal{S}(\mathbf{N}) = 1.
  \end{equation}
Moreover, for every $N_{h} \geq N_{h}^{\star} + \frac{P(P-1)}{2} - 1$,
\begin{equation}
    \label{eq:affinely-independent-fixed-width-similarity-bounds}
    \inf\mathcal{S}_{N_{h}}(\mathbf{N}) = -1,
    \qquad
    \sup\mathcal{S}_{N_{h}}(\mathbf{N}) = 1.
  \end{equation}
\end{proposition}

\begin{proof}
  Choose $z_{0} \in \mathbb{R}$ such that $\sigma(z_{0}) \neq 0$.
  The rank assumption $\operatorname{rank}(\overline{\mathbf{X}}) = P$ implies that the map
\begin{equation}
    \mathbb{R}^{N_{i}+1} \to \mathbb{R}^{P},
    \quad
    \overline{\mathbf{w}} \mapsto \overline{\mathbf{X}}^{\top} \overline{\mathbf{w}}
  \end{equation}
is surjective.
  Hence, for any $\mathbf{v} \in \mathbb{R}_{\geq0}^{P}$, there exists $\overline{\mathbf{w}} \in \mathbb{R}^{N_{i}+1}$ such that
\begin{equation}
    \frac{z_{0}}{\abs{\sigma(z_{0})}} \mathbf{v}^{\top}
    = (\overline{\mathbf{X}}^{\top} \overline{\mathbf{w}})^{\top}
    = \overline{\mathbf{w}}^{\top} \overline{\mathbf{X}}.
  \end{equation}
Since $\mathbf{v}$ has nonnegative entries, positive $1$-homogeneity gives
\begin{equation}
    \sigma(
      \overline{\mathbf{w}}^{\top}\overline{\mathbf{X}}
    )
    = \frac{\sigma(z_{0})}{\abs{\sigma(z_{0})}}
    \mathbf{v}^{\top}
    = \varepsilon \mathbf{v}^{\top},
    \qquad
    \varepsilon \in \set{-1,1}.
  \end{equation}
Hence,
\begin{equation}
    \label{eq:nonnegative-rank-one-projections-in-feature-cone}
    \Pi_{\Hol}(
      \mathbf{v} \mathbf{v}^{\top}
    )
    \in \mathcal{G}_{\sigma,\mathbf{X}}
    \subseteq \mathcal{C}_{\sigma,\mathbf{X}},
    \qquad
    \mathbf{v} \in \mathbb{R}_{\geq 0}^{P}.
  \end{equation}

  Now, let $\mathbf{Y} \in \Hol$ be any symmetric matrix with zero diagonal and zero-mean off-diagonal entries.
  Set $c \coloneqq -\min_{\mu < \nu} Y_{\mu\nu}$, and let $\mathbf{e}^{\mu} \in \mathbb{R}^{P}$ denote the $\mu$th standard basis vector, $\mu = 1, \dots, P$.
  The matrix
\begin{equation}
    \mathbf{B}
    \coloneqq
    \sum_{\mu<\nu}
    (Y_{\mu\nu} + c)
    (\mathbf{e}^{\mu} + \mathbf{e}^{\nu})
    (\mathbf{e}^{\mu} + \mathbf{e}^{\nu})^{\top}
    \in
    \mathbb{R}^{P \times P}
  \end{equation}
is symmetric by construction and has off-diagonal entries
\begin{equation}
    B_{\mu\nu} = Y_{\mu\nu} + c
    \geq
    0,
    \qquad
    \mu\neq\nu.
  \end{equation}
Thus,
\begin{equation}
    \frac{2}{P(P-1)} \sum_{\mu<\nu} B_{\mu\nu}
    = \frac{2}{P(P-1)} \sum_{\mu<\nu} (Y_{\mu\nu} + c)
    = c + \frac{2}{P(P-1)} \sum_{\mu<\nu} Y_{\mu\nu}
    = c,
  \end{equation}
since $\mathbf{Y}$ has zero-mean off-diagonal entries.
  Projecting $\mathbf{B}$ onto $\mathcal{H}$ sets its diagonal to zero and subtracts $c$ from every off-diagonal entry.
  Therefore,
\begin{equation}
    \Pi_{\Hol}(\mathbf{B})
    = \mathbf{Y}.
  \end{equation}
As each vector $\mathbf{e}^{\mu} + \mathbf{e}^{\nu} \in \mathbb{R}_{\geq0}^{P}$ has nonnegative entries, \cref{eq:nonnegative-rank-one-projections-in-feature-cone} gives
\begin{equation}
    \mathbf{Y}
    = \Pi_{\Hol}(\mathbf{B})
    = \sum_{\mu<\nu}
    (Y_{\mu\nu} + c)
    \Pi_{\Hol}\bigl(
      (\mathbf{e}^{\mu} + \mathbf{e}^{\nu})
      (\mathbf{e}^{\mu} + \mathbf{e}^{\nu})^{\top}
    \bigr)
    \in \mathcal{C}_{\sigma,\mathbf{X}}.
  \end{equation}
Hence, $\Hol \subseteq \mathcal{C}_{\sigma,\mathbf{X}}$, while the reverse inclusion follows directly from the definition of $\mathcal{C}_{\sigma,\mathbf{X}}$ in \cref{eq:closed-projected-feature-cone}.

  It remains to establish \cref{eq:affinely-independent-fixed-width-similarity-bounds}, since \cref{eq:affinely-independent-orbit-wide-similarity-bounds} then follows from $\mathcal{S}_{N_{h}}(\mathbf{N}) \subseteq \mathcal{S}(\mathbf{N})$ and the fact that Pearson correlation takes values in $[-1, 1]$.
  For each $\varepsilon \in \set{-1, 1}$,
\begin{equation}
    \varepsilon\mathbf{N}^{\circ}
    \in \Hol
    = \mathcal{C}_{\sigma,\mathbf{X}},
    \qquad
    \norm{\varepsilon\mathbf{N}^{\circ}}_{F}
    = 1,
  \end{equation}
and hence, by \cref{lem:limiting-normalized-geometries-within-a-symmetry-orbit},
\begin{equation}
    \varepsilon\mathbf{N}^{\circ}
    \in
    \cl\set{
      \mathbf{M}^{\circ}
      \given
      \mathbf{M} \in \mathcal{R}_{N_{h}}
    },
    \qquad
    N_{h} \geq N_{h}^{\star} + d.
  \end{equation}
Thus, for every $\varepsilon \in \set{-1,1}$ and every $N_{h} \geq N_{h}^{\star} + d$, there exists a sequence
\begin{equation}
    \set{\mathbf{M}_{i,\varepsilon}}_{i=1}^{\infty}
    \subseteq
    \mathcal{R}_{N_{h}}
    \qquad
    \text{such that}
    \qquad
    \lim_{i\to\infty}
    \mathbf{M}_{i,\varepsilon}^{\circ}
    =
    \varepsilon \mathbf{N}^{\circ}.
  \end{equation}
By \cref{lem:pearson-as-cosine} and continuity of the Frobenius inner product,
\begin{equation}
    \lim_{i\to\infty} \rho(
      \mathbf{M}_{i,\varepsilon},
      \mathbf{N}
    )
    = \lim_{i\to\infty}
    \langle
      \mathbf{M}_{i,\varepsilon}^{\circ},
      \mathbf{N}^{\circ}
    \rangle_{F}
    = \varepsilon
    \langle
      \mathbf{N}^{\circ},
      \mathbf{N}^{\circ}
    \rangle_{F}
    = \varepsilon.
  \end{equation}
Since Pearson correlation takes values in $[-1,1]$, it follows that
\begin{equation}
    \inf\mathcal{S}_{N_{h}}(\mathbf{N}) = -1,
    \qquad
    \sup\mathcal{S}_{N_{h}}(\mathbf{N}) = 1,
  \end{equation}
for every $N_{h} \geq N_{h}^{\star} + d$.
  Finally, $\mathcal{C}_{\sigma,\mathbf{X}} = \Hol$ implies
\begin{equation}
    d
    = \dim\spanop\mathcal{C}_{\sigma,\mathbf{X}}
    = \dim\Hol
    = \frac{P(P-1)}{2} - 1,
  \end{equation}
which establishes \cref{eq:affinely-independent-fixed-width-similarity-bounds} and completes the proof.
\end{proof}

\Cref{prop:alignment-with-affinely-independent-inputs} gives a sufficient condition under which every reference geometry admits similarities arbitrarily close to both $-1$ and $1$.
When this condition is not satisfied, however, the activation $\sigma$ and probe inputs $\mathbf{X}$ can restrict the closed projected feature cone $\mathcal{C}_{\sigma,\mathbf{X}}$ and thereby prevent arbitrary alignment with particular reference geometries $\mathbf{N}$.

\begin{example}[A reference geometry with nonpositive similarity for ReLU networks]
  \label[example]{ex:reference-geometry-with-nonpositive-similarity-for-relu-networks}
  Consider $\mathbf{N} = \mathbf{u}\mathbf{u}^{\top}$ with $\mathbf{u}^{\top} = (1,0,1)$ and probe inputs
\begin{equation}
    \mathbf{X} = (-1, 0, 1)
    \in
    \mathbb{R}^{3}.
  \end{equation}
Clearly, $\operatorname{rank}(\overline{\mathbf{X}}) = 2 < 3$, so \cref{prop:alignment-with-affinely-independent-inputs} does not apply and $\sup\mathcal{S}(\mathbf{N}) = 1$ is not guaranteed.
  In fact, for every symmetry orbit $\mathcal{O}$ with ReLU activation, we will show that $\sup\mathcal{S}(\mathbf{N}) \leq 0$.

  Write $\psi(z) = \max(0,z)$.
  On the inputs $\mathbf{X}$, a ReLU neuron with weight $w$ and hidden bias $b$ computes the feature $\mathbf{v} \in \mathbb{R}^{3}$ with entries
\begin{equation}
    v_{1} = \psi(b-w),
    \qquad
    v_{2} = \psi(b),
    \qquad
    v_{3} = \psi(b+w).
  \end{equation}
We first establish the inequality
\begin{equation}
    \label{eq:relu-single-neuron-feature-inequality}
    v_{1}v_{2} - 2v_{1}v_{3} + v_{2}v_{3} \geq 0.
  \end{equation}
If $v_{1}v_{3} = 0$, the inequality follows immediately from $v_{1}, v_{2}, v_{3} \geq 0$.
  Otherwise, $v_{1}>0$ and $v_{3}>0$, so
\begin{equation}
    b - w > 0,
    \qquad
    b + w > 0,
  \end{equation}
which implies $b>0$.
  Hence, all three pre-activations are positive, and therefore
\begin{equation}
    v_{1} = b-w,
    \qquad
    v_{2} = b,
    \qquad
    v_{3} = b+w.
  \end{equation}
Thus,
\begin{equation}
    v_{1}v_{2} - 2v_{1}v_{3} + v_{2}v_{3}
    = (b-w)b - 2(b-w)(b+w) + b(b+w)
    = 2w^{2}
    \geq 0,
  \end{equation}
proving \cref{eq:relu-single-neuron-feature-inequality}.

  Now consider an arbitrary ReLU network of width $N_{h}$, and let $\mathbf{v}_{j}^{\top}$ denote the feature computed by neuron $j$ on the probe inputs $\mathbf{X}$.
  The corresponding hidden-activation matrix $\mathbf{H}$ has rows $\mathbf{v}_{j}^{\top}$, so the resulting \ac{rsm} satisfies
\begin{equation}
    \mathbf{M}
    = \mathbf{H}^{\top}\mathbf{H}
    = \sum_{j=1}^{N_{h}}
    \mathbf{v}_{j}\mathbf{v}_{j}^{\top}.
  \end{equation}
Consequently,
\begin{equation}
    M_{\mu\nu}
    = \sum_{j=1}^{N_{h}} v_{j\mu}v_{j\nu},
    \qquad
    1 \leq \mu < \nu \leq 3,
  \end{equation}
where $v_{j\mu}$ denotes the $\mu$th entry of $\mathbf{v}_{j} \in \mathbb{R}^{3}$.
  Hence, \cref{eq:relu-single-neuron-feature-inequality} implies
\begin{equation}
    \label{eq:relu-rsm-off-diagonal-inequality}
    M_{12} - 2M_{13} + M_{23}
    =
    \sum_{j=1}^{N_{h}}
    \bigl(
      v_{j1}v_{j2}
      -2 v_{j1}v_{j3}
      + v_{j2}v_{j3}
    \bigr)
    \geq 0.
  \end{equation}
We have
\begin{equation}
    \Pi_{\Hol}(\mathbf{N})
    =
    \Pi_{\Hol}\left(
      \begin{bmatrix}
        1 & 0 & 1 \\
        0 & 0 & 0 \\
        1 & 0 & 1
      \end{bmatrix}
    \right)
    =
    \frac{1}{3}
    \begin{bmatrix}
      0  & -1 & 2  \\
      -1 & 0  & -1 \\
      2  & -1 & 0
    \end{bmatrix}.
  \end{equation}
Since $\Pi_{\Hol}$ is the Frobenius-orthogonal projection onto $\Hol$ and $\Pi_{\Hol}(\mathbf{N}) \in \Hol$,
\begin{equation}
    \langle
      \Pi_{\Hol}(\mathbf{M}),
      \Pi_{\Hol}(\mathbf{N})
    \rangle_{F}
    =
    \langle
      \mathbf{M},
      \Pi_{\Hol}(\mathbf{N})
    \rangle_{F}
    =
    \frac{2}{3} (
      -M_{12} + 2M_{13} - M_{23}
    )
    \leq 0
  \end{equation}
by \cref{eq:relu-rsm-off-diagonal-inequality}.
  Whenever $\rho(\mathbf{M},\mathbf{N})$ is defined, \cref{lem:pearson-as-cosine} therefore yields
\begin{equation}
    \rho(\mathbf{M},\mathbf{N})
    =
    \frac{
      \langle
        \Pi_{\Hol}(\mathbf{M}),
        \Pi_{\Hol}(\mathbf{N})
      \rangle_{F}
    }{
      \norm{\Pi_{\Hol}(\mathbf{M})}_{F}
      \norm{\Pi_{\Hol}(\mathbf{N})}_{F}
    }
    \leq 0.
  \end{equation}
Since the ReLU network was arbitrary, $\sup\mathcal{S}(\mathbf{N}) \leq 0$ for every symmetry orbit $\mathcal{O}$ with ReLU activation, as claimed.
\end{example}

The fact that no ReLU network can align positively with the reference geometry $\mathbf{N}$ from \cref{ex:reference-geometry-with-nonpositive-similarity-for-relu-networks} has a simple geometric explanation.
The probe inputs satisfy
\begin{equation}
  \mathbf{x}^{2}
  =
  \frac{\mathbf{x}^{1} + \mathbf{x}^{3}}{2}.
\end{equation}
Hence, the pre-activation of any neuron on $\mathbf{x}^{2}$ is the average of its pre-activations on $\mathbf{x}^{1}$ and $\mathbf{x}^{3}$.
If a neuron is active on both $\mathbf{x}^{1}$ and $\mathbf{x}^{3}$, then it is necessarily active on $\mathbf{x}^{2}$ as well, and linearity of ReLU on the positive half-line gives
\begin{equation}
  v_{2}
  =
  \frac{v_{1} + v_{3}}{2}.
\end{equation}
This relation is precisely what gives rise to \cref{eq:relu-single-neuron-feature-inequality}.
If $v_{1}v_{3} = 0$, the inequality follows from nonnegativity of the ReLU features.
Otherwise,
\begin{equation}
  v_{1}v_{2} - 2v_{1}v_{3} + v_{2}v_{3}
  =
  \frac{(v_{1}-v_{3})^{2}}{2}
  \geq 0.
\end{equation}
Thus, the affine dependence among the probe inputs imposes a linear constraint on the neuron-wise contributions to the pairwise similarities between two inputs.
Summing this constraint over neurons yields
\begin{equation}
  \label{eq:relu-rsm-pairwise-similarity-constraint}
  M_{13}
  \leq
  \frac{M_{12} + M_{23}}{2}
\end{equation}
for every \ac{rsm} realizable by a ReLU network on $\mathbf{X}$.
By contrast, the reference geometry $\mathbf{N}$ satisfies
\begin{equation}
  N_{12} = N_{23} = 0,
  \qquad
  N_{13} = 1,
\end{equation}
and therefore violates \cref{eq:relu-rsm-pairwise-similarity-constraint}.
Moreover, the projected reference geometry $\Pi_{\Hol}(\mathbf{N})$ is a negative scalar multiple of the Frobenius normal
\begin{equation}
  \mathbf{L}
  =
  \begin{bmatrix}
    0  & 1 & -2 \\
    1  & 0 & 1  \\
    -2 & 1 & 0
  \end{bmatrix}
\end{equation}
to the half-space defined by \cref{eq:relu-rsm-pairwise-similarity-constraint}, with $\Pi_{\Hol}(\mathbf{N}) = -\nicefrac{1}{3} \, \mathbf{L}$.
For every realizable \ac{rsm},
\begin{equation}
  \langle
    \Pi_{\Hol}(\mathbf{M}),
    \mathbf{L}
  \rangle_{F}
  =
  2 (
    M_{12} - 2M_{13} + M_{23}
  )
  \geq 0,
\end{equation}
and therefore
\begin{equation}
  \langle
    \Pi_{\Hol}(\mathbf{M}),
    \Pi_{\Hol}(\mathbf{N})
  \rangle_{F}
  \leq 0.
\end{equation}

The affine-independence assumption in \cref{prop:alignment-with-affinely-independent-inputs} rules out precisely the input dependencies that give rise to restrictions such as \cref{eq:relu-rsm-pairwise-similarity-constraint}.
By preventing these dependencies from constraining the features, it allows the projected feature cone $\mathcal{C}_{\sigma,\mathbf{X}}$ to span all of $\Hol$ and thereby guarantees arbitrary alignment with any reference geometry.

\section{Identifiability through minimum-norm selection}
\label[appendix]{app-sec:identifiability-through-minimum-norm-selection}

This appendix provides proofs and extensions of the identifiability results in \cref{sec:identifiability-through-minimum-norm-selection}.
We first characterize the representational geometries selected by minimum-weight-norm selection for positively $1$-homogeneous activations, prove a sufficient condition for representational identifiability, and examine its implications for the analytical ReLU solutions illustrated in \cref{fig:dissociation} and discussed in detail in \cref{app-sec:characterization-of-xor-relu-solutions} (\cref{app-subsec:mwnp-representational-geometry}).
We then develop the corresponding minimum-representation-norm analysis, making explicit how the finite input set affects both representational identifiability and existence of a minimizer (\cref{app-subsec:mrnp-representational-geometry}).

\subsection{Minimum weight norm for positively \texorpdfstring{$1$-homogeneous}{1-homogeneous} activations}
\label[appendix]{app-subsec:mwnp-representational-geometry}

Throughout this subsection, fix a symmetry orbit $\mathcal{O}$ with positively $1$-homogeneous activation $\sigma(z) = \delta\abs{z} + mz$, essential parameter classes $\mathcal{E}$, unit-norm representatives $\overline{\mathbf{w}}_{q}$, orbit invariants $\boldsymbol{\beta}_{q}$, and residual $\mathbf{r}$.
We adopt the feature vectors $\mathbf{f}_{q}^{+}$ and $\mathbf{f}_{q}^{-}$, $q \in \mathcal{E}$, and the sets $\mathcal{T}$ and $\mathcal{T}_{N_{h}}$ introduced in \cref{sec:identifiability-through-minimum-norm-selection}.

To establish \cref{prop:mwnp-representational-geometry,cor:mwnp-representational-identifiability}, we first derive an orbit-wide lower bound on the weight-norm objective $\Omega_{W}$ and identify exactly when this bound is attained.

\begin{lemma}[Weight-norm lower bound and equality conditions]
  \label[lemma]{lem:mwnp-lower-bound-and-equality-conditions}
  Every parameterization $\boldsymbol{\theta} \in \mathcal{O}_{N_{h}}$ satisfies
\begin{equation}
    \label{eq:mwnp-orbit-lower-bound}
    \Omega_{W}(\boldsymbol{\theta})
    \geq
    2 \sum_{q \in \mathcal{E}} \norm{\boldsymbol{\beta}_{q}}.
\end{equation}
Equality holds if and only if the following three conditions are satisfied:
\begin{enumerate}
    \item Every neuron outside the essential parameter classes has parameters
\begin{equation}
      \overline{\mathbf{w}}_{j} = \mathbf{0},
      \qquad
      \mathbf{a}_{j} = \mathbf{0}.
    \end{equation}
\item Every neuron in an essential parameter class has equal incoming-parameter and readout norms:
\begin{equation}
      \norm{\overline{\mathbf{w}}_{j}}
      =
      \norm{\mathbf{a}_{j}},
      \qquad
      j \in \mathcal{J}_{q},
      \quad
      q \in \mathcal{E}.
    \end{equation}
    \item For every $q \in \mathcal{E}$, there exist coefficients $c_{j} \geq 0$, $j \in \mathcal{J}_{q}$, such that
\begin{equation}
      \norm{\overline{\mathbf{w}}_{j}} \mathbf{a}_{j}
      = c_{j} \boldsymbol{\beta}_{q},
      \qquad
      \sum_{j \in \mathcal{J}_{q}} c_{j} = 1.
    \end{equation}
  \end{enumerate}
\end{lemma}

\begin{proof}
  By definition of the weight-norm objective $\Omega_{W}(\boldsymbol{\theta})$ in \cref{eq:minimum-norm-objectives}, we have
\begin{equation}
    \label{eq:mwnp-essential-class-lower-bound}
    \Omega_{W}(\boldsymbol{\theta})
    =
    \sum_{j=1}^{N_{h}}
    \bigl(
      \norm{\overline{\mathbf{w}}_{j}}^{2}
      + \norm{\mathbf{a}_{j}}^{2}
    \bigr)
    \geq
    \sum_{q \in \mathcal{E}}
    \sum_{j \in \mathcal{J}_{q}}
    \bigl(
      \norm{\overline{\mathbf{w}}_{j}}^{2}
      + \norm{\mathbf{a}_{j}}^{2}
    \bigr).
  \end{equation}
The AM--GM inequality gives
\begin{equation}
    \label{eq:mwnp-amgm-lower-bound}
    \sum_{q \in \mathcal{E}}
    \sum_{j \in \mathcal{J}_{q}}
    \bigl(
      \norm{\overline{\mathbf{w}}_{j}}^{2}
      + \norm{\mathbf{a}_{j}}^{2}
    \bigr)
    \geq
    2
    \sum_{q \in \mathcal{E}}
    \sum_{j \in \mathcal{J}_{q}}
    \norm{\overline{\mathbf{w}}_{j}}
    \norm{\mathbf{a}_{j}}.
  \end{equation}
For each essential parameter class $q \in \mathcal{E}$, absolute homogeneity of the Euclidean norm and the triangle inequality yield
\begin{equation}
    \label{eq:mwnp-class-aggregate-lower-bound}
    \sum_{j \in \mathcal{J}_{q}}
    \norm{\overline{\mathbf{w}}_{j}}
    \norm{\mathbf{a}_{j}}
    =
    \sum_{j \in \mathcal{J}_{q}}
    \norm[\big]{
      \norm{\overline{\mathbf{w}}_{j}}
      \mathbf{a}_{j}
    }
    \geq
    \norm[\Big]{
      \sum_{j \in \mathcal{J}_{q}}
      \norm{\overline{\mathbf{w}}_{j}}
      \mathbf{a}_{j}
    }
    =
    \norm{\boldsymbol{\beta}_{q}},
  \end{equation}
where the final equality follows from the definition of the class aggregate $\boldsymbol{\beta}_{q}$ in \cref{eq:orbit-invariants-positively-homogeneous}.
  Together, \cref{eq:mwnp-essential-class-lower-bound,eq:mwnp-amgm-lower-bound,eq:mwnp-class-aggregate-lower-bound} establish the inequality in \cref{eq:mwnp-orbit-lower-bound}.

  Equality in \cref{eq:mwnp-orbit-lower-bound} holds if and only if equality holds in \cref{eq:mwnp-essential-class-lower-bound,eq:mwnp-amgm-lower-bound} and, for every $q \in \mathcal{E}$, in \cref{eq:mwnp-class-aggregate-lower-bound}.
  It therefore remains to determine when each of these inequalities is tight.

  Equality in \cref{eq:mwnp-essential-class-lower-bound} holds if and only if every neuron outside the essential parameter classes contributes nothing to $\Omega_{W}$.
  Since each such contribution is
\begin{equation}
    \norm{\overline{\mathbf{w}}_{j}}^{2}
    + \norm{\mathbf{a}_{j}}^{2}
    \geq 0,
  \end{equation}
this is equivalent to
\begin{equation}
    \overline{\mathbf{w}}_{j} = \mathbf{0},
    \qquad
    \mathbf{a}_{j} = \mathbf{0},
  \end{equation}
giving the first condition.

  Equality in \cref{eq:mwnp-amgm-lower-bound} holds if and only if the AM--GM inequality is tight for every neuron in every essential parameter class.
  For each $j \in \mathcal{J}_{q}$,
\begin{equation}
    \bigl(
      \norm{\overline{\mathbf{w}}_{j}}^{2}
      + \norm{\mathbf{a}_{j}}^{2}
    \bigr)
    - 2
    \norm{\overline{\mathbf{w}}_{j}}
    \norm{\mathbf{a}_{j}}
    =
    \bigl(
      \norm{\overline{\mathbf{w}}_{j}}
      - \norm{\mathbf{a}_{j}}
    \bigr)^{2},
  \end{equation}
so equality holds if and only if
\begin{equation}
    \norm{\overline{\mathbf{w}}_{j}}
    =
    \norm{\mathbf{a}_{j}},
    \qquad
    j \in \mathcal{J}_{q},
    \quad
    q \in \mathcal{E},
  \end{equation}
giving the second condition.

  Finally, equality in \cref{eq:mwnp-class-aggregate-lower-bound} is determined by equality in the triangle inequality.
  Since $\boldsymbol{\beta}_{q} \neq \mathbf{0}$ for every essential parameter class $q \in \mathcal{E}$, equality holds if and only if the vectors
\begin{equation}
    \norm{\overline{\mathbf{w}}_{j}}
    \mathbf{a}_{j},
    \qquad
    j \in \mathcal{J}_{q},
  \end{equation}
are nonnegative multiples of $\boldsymbol{\beta}_{q}$.
  Equivalently, there exist coefficients $c_{j} \geq 0$ such that
\begin{equation}
    \norm{\overline{\mathbf{w}}_{j}}
    \mathbf{a}_{j}
    =
    c_{j}\boldsymbol{\beta}_{q},
    \qquad
    j \in \mathcal{J}_{q}.
  \end{equation}
Summing over $j \in \mathcal{J}_{q}$ and using the definition of the class aggregate gives
\begin{equation}
    \boldsymbol{\beta}_{q}
    =
    \boldsymbol{\beta}_{q}
    \sum_{j \in \mathcal{J}_{q}} c_{j}.
  \end{equation}
Since $\boldsymbol{\beta}_{q} \neq \mathbf{0}$, we have $\sum_{j \in \mathcal{J}_{q}} c_{j} = 1$, giving the third condition.

  Conversely, if the three stated conditions hold, then equality holds in \cref{eq:mwnp-essential-class-lower-bound,eq:mwnp-amgm-lower-bound,eq:mwnp-class-aggregate-lower-bound}, and hence in \cref{eq:mwnp-orbit-lower-bound}.
\end{proof}

The third condition in the preceding lemma also determines how the class aggregate $\boldsymbol{\beta}_{q}$ can be distributed between the two orientations $\mathcal{J}_{q}^{+}$ and $\mathcal{J}_{q}^{-}$ of an essential parameter class $q \in \mathcal{E}$.

\begin{corollary}[Readout allocation at the weight-norm lower bound]
  \label[corollary]{cor:mwnp-readout-allocation}
  Suppose that $\boldsymbol{\theta} \in \mathcal{O}_{N_{h}}$ attains the lower bound in \cref{eq:mwnp-orbit-lower-bound}.
  Then, for every essential parameter class $q \in \mathcal{E}$, there exists $t_{q} \in [0,1]$ such that
\begin{equation}
    \sum_{j \in \mathcal{J}_{q}^{+}}
    \norm{\overline{\mathbf{w}}_{j}} \mathbf{a}_{j}
    =
    t_{q} \boldsymbol{\beta}_{q},
    \qquad
    \sum_{j \in \mathcal{J}_{q}^{-}}
    \norm{\overline{\mathbf{w}}_{j}} \mathbf{a}_{j}
    =
    (1-t_{q}) \boldsymbol{\beta}_{q}.
  \end{equation}
\end{corollary}

\begin{proof}
  By \cref{lem:mwnp-lower-bound-and-equality-conditions}, there exist coefficients $c_{j} \geq 0$, $j \in \mathcal{J}_{q}$, such that
\begin{equation}
    \norm{\overline{\mathbf{w}}_{j}} \mathbf{a}_{j}
    =
    c_{j}\boldsymbol{\beta}_{q},
    \qquad
    \sum_{j \in \mathcal{J}_{q}} c_{j} = 1.
  \end{equation}
Setting $t_{q} \coloneqq \sum_{j \in \mathcal{J}_{q}^{+}} c_{j}$ gives $t_{q} \in [0,1]$ and $\sum_{j \in \mathcal{J}_{q}^{-}} c_{j} = 1 - t_{q}$, completing the proof.
\end{proof}

With the preceding lemma and corollary in place, it remains to establish two facts to prove \cref{prop:mwnp-representational-geometry}.
We first show that any parameterization attaining the lower bound in \cref{eq:mwnp-orbit-lower-bound} induces a split $\mathbf{t} \in \mathcal{T}_{N_{h}}$ and has an \ac{rsm} of the form given in \cref{eq:mwnp-representational-geometry}.
We then show that every split $\mathbf{t} \in \mathcal{T}_{N_{h}}$ is realized by a parameterization in $\mathcal{O}_{N_{h}}$ that attains this lower bound.

Before doing so, we make precise the width requirement in the definition of $\mathcal{T}_{N_{h}}$ in \cref{sec:identifiability-through-minimum-norm-selection}.
A split $0 < t_{q} < 1$ requires at least two neurons from class $q$, one of each orientation, whereas a split $t_{q} \in \set{0,1}$ requires only one.
Consequently, the minimum number of neurons from essential parameter classes needed to realize $\mathbf{t} \in \mathcal{T}$ is
\begin{equation}
  \kappa(\mathbf{t})
  \coloneqq
  \abs{\mathcal{E}}
  +
  \abs{\set{q \in \mathcal{E} \given 0 < t_{q} < 1}}.
\end{equation}
With this notation, the set $\mathcal{T}_{N_{h}}$ of splits feasible at width $N_{h}$ can be formally defined as
\begin{equation}
  \mathcal{T}_{N_{h}}
  \coloneqq
  \set{
    \mathbf{t} \in \mathcal{T}
    \given
    \kappa(\mathbf{t}) \leq N_{h}
  }.
\end{equation}

\begin{proof}[Proof of \cref{prop:mwnp-representational-geometry}]
  We first show that every parameterization $\boldsymbol{\theta} \in \mathcal{O}_{N_{h}}$ attaining the lower bound in \cref{eq:mwnp-orbit-lower-bound} induces a split $\mathbf{t} \in \mathcal{T}_{N_{h}}$ for which its \ac{rsm} is given by \cref{eq:mwnp-representational-geometry}.

  Fix $\boldsymbol{\theta} \in \mathcal{O}_{N_{h}}$ attaining the lower bound in \cref{eq:mwnp-orbit-lower-bound}.
  By \cref{cor:mwnp-readout-allocation}, for every $q \in \mathcal{E}$, there exists $t_{q} \in [0,1]$ such that the class aggregate $\boldsymbol{\beta}_{q}$ is allocated between the two orientations $\mathcal{J}_{q}^{+}$ and $\mathcal{J}_{q}^{-}$ according to $t_{q}$ and $1-t_{q}$.
  By the first condition in \cref{lem:mwnp-lower-bound-and-equality-conditions}, every neuron outside the essential parameter classes has parameters $\overline{\mathbf{w}}_{j} = \mathbf{0}$ and $\mathbf{a}_{j} = \mathbf{0}$.
  Hence, the global residual $\mathbf{r}$ in \cref{eq:global-residual-positively-homogeneous} is determined entirely by essential parameter classes, with each class $q \in \mathcal{E}$ contributing $m (2t_{q}-1) \boldsymbol{\beta}_{q} \overline{\mathbf{w}}_{q}^{\top} \overline{\mathbf{x}}$ by \cref{eq:contribution-to-residual-by-essential-classes}.
  Therefore,
\begin{equation}
    \mathbf{r}(\mathbf{x})
    =
    m
    \sum_{q \in \mathcal{E}}
    (2t_{q}-1)
    \boldsymbol{\beta}_{q}
    \overline{\mathbf{w}}_{q}^{\top}
    \overline{\mathbf{x}}.
  \end{equation}
Since $\boldsymbol{\theta} \in \mathcal{O}_{N_{h}}$ has residual $\mathbf{r}$, the definition of $\mathcal{T}$ in \cref{eq:residual-compatible-readout-splits} therefore gives $\mathbf{t} \in \mathcal{T}$.
  Since $\mathbf{t}$ is induced by a parameterization $\boldsymbol{\theta} \in \mathcal{O}_{N_{h}}$ of width $N_{h}$, we have $\mathbf{t} \in \mathcal{T}_{N_{h}}$.

  We next determine the \ac{rsm} induced by $\boldsymbol{\theta}$.
  By the second and third conditions in \cref{lem:mwnp-lower-bound-and-equality-conditions}, for every $j \in \mathcal{J}_{q}$, there exists $c_{j} \geq 0$ such that
\begin{equation}
    \label{eq:mwnp-balanced-feature-weight}
    \norm{\overline{\mathbf{w}}_{j}}^{2}
    =
    \norm{\overline{\mathbf{w}}_{j}}
    \norm{\mathbf{a}_{j}}
    =
    \norm[\big]{
      \norm{\overline{\mathbf{w}}_{j}}
      \mathbf{a}_{j}
    }
    =
    c_{j}\norm{\boldsymbol{\beta}_{q}}.
  \end{equation}
Let $\varepsilon \in \set{+,-}$.
  Since $\overline{\mathbf{w}}_{q}$ has unit norm, positive $1$-homogeneity implies that neuron $j \in \mathcal{J}_{q}^{\varepsilon}$ has feature row
\begin{equation}
    \sigma(
      \overline{\mathbf{w}}_{j}^{\top} \overline{\mathbf{X}}
    )
    =
    \norm{\overline{\mathbf{w}}_{j}} \,
    \sigma(
      \varepsilon \overline{\mathbf{w}}_{q}^{\top} \overline{\mathbf{X}}
    )
    =
    \norm{\overline{\mathbf{w}}_{j}} (\mathbf{f}_{q}^{\varepsilon})^{\top}.
  \end{equation}
Its contribution to $\mathbf{H}^{\top}\mathbf{H}$ is therefore
\begin{equation}
    \norm{\overline{\mathbf{w}}_{j}} (\mathbf{f}_{q}^{\varepsilon})
    \norm{\overline{\mathbf{w}}_{j}} (\mathbf{f}_{q}^{\varepsilon})^{\top}
    =
    c_{j}
    \norm{\boldsymbol{\beta}_{q}}
    \mathbf{f}_{q}^{\varepsilon}
    (\mathbf{f}_{q}^{\varepsilon})^{\top}
  \end{equation}
by \cref{eq:mwnp-balanced-feature-weight}.
  Summing over neurons and using
\begin{equation}
    \sum_{j \in \mathcal{J}_{q}^{+}} c_{j}
    =
    t_{q},
    \qquad
    \sum_{j \in \mathcal{J}_{q}^{-}} c_{j}
    =
    1-t_{q},
  \end{equation}
gives
\begin{equation}
    \mathbf{H}^{\top}\mathbf{H}
    =
    \sum_{q \in \mathcal{E}}
    \norm{\boldsymbol{\beta}_{q}}
    \bigl(
      t_{q}
      \mathbf{f}_{q}^{+}
      (\mathbf{f}_{q}^{+})^{\top}
      +
      (1-t_{q})
      \mathbf{f}_{q}^{-}
      (\mathbf{f}_{q}^{-})^{\top}
    \bigr).
  \end{equation}
This proves the first claim that every parameterization $\boldsymbol{\theta} \in \mathcal{O}_{N_{h}}$ attaining the lower bound in \cref{eq:mwnp-orbit-lower-bound} induces a split $\mathbf{t} \in \mathcal{T}_{N_{h}}$ for which its \ac{rsm} is given by \cref{eq:mwnp-representational-geometry}.

  For the converse, fix any $\mathbf{t} \in \mathcal{T}_{N_{h}}$.
  For each $q \in \mathcal{E}$, introduce neurons with parameters
\begin{equation}
    \label{eq:mwnp-balanced-construction}
    \begin{aligned}
      \overline{\mathbf{w}}_{q,+}
      &=
      \sqrt{
        t_{q}
        \norm{\boldsymbol{\beta}_{q}}
      }\,
      \overline{\mathbf{w}}_{q},
      &
      \mathbf{a}_{q,+}
      &=
      \sqrt{
        \frac{
          t_{q}
        }{
          \norm{\boldsymbol{\beta}_{q}}
        }
      }\,
      \boldsymbol{\beta}_{q},
      \\
      \overline{\mathbf{w}}_{q,-}
      &=
      -\sqrt{
        (1-t_{q})
        \norm{\boldsymbol{\beta}_{q}}
      }\,
      \overline{\mathbf{w}}_{q},
      &
      \mathbf{a}_{q,-}
      &=
      \sqrt{
        \frac{
          1-t_{q}
        }{
          \norm{\boldsymbol{\beta}_{q}}
        }
      }\,
      \boldsymbol{\beta}_{q},
    \end{aligned}
  \end{equation}
omitting the positively oriented neuron when $t_{q}=0$ and the negatively oriented neuron when $t_{q}=1$.
  This construction uses exactly $\kappa(\mathbf{t})$ neurons.
  Since $\mathbf{t} \in \mathcal{T}_{N_{h}}$, and hence $\kappa(\mathbf{t}) \leq N_{h}$, adding $N_{h}-\kappa(\mathbf{t})$ neurons with all parameters set to zero yields a parameterization of width $N_{h}$.
  By construction, the positively and negatively oriented neurons contribute
\begin{equation}
    t_{q} \boldsymbol{\beta}_{q}
    \qquad \text{and} \qquad
    (1-t_{q}) \boldsymbol{\beta}_{q},
  \end{equation}
respectively, to the class aggregate $\boldsymbol{\beta}_{q}$.
  Moreover, $\mathbf{t} \in \mathcal{T}$ ensures that the resulting global residual equals $\mathbf{r}$.
  Since no nonessential parameter class is introduced, \cref{prop:characterization-of-symmetry-equivalence} therefore implies that the constructed parameterization belongs to $\mathcal{O}_{N_{h}}$.

  It remains to show that the constructed parameterization has weight norm $2 \sum_{q \in \mathcal{E}} \norm{\boldsymbol{\beta}_{q}}$ and that its \ac{rsm} is of the form in \cref{eq:mwnp-representational-geometry}.
  The former follows directly from the observation that the constructed parameterization satisfies all three conditions of \cref{lem:mwnp-lower-bound-and-equality-conditions}:
  neurons outside essential parameter classes have parameters $\overline{\mathbf{w}}_{j} = \mathbf{0}$ and $\mathbf{a}_{j} = \mathbf{0}$, those in essential classes have equal incoming-parameter and readout norms,
\begin{equation}
    \begin{aligned}
      \norm{\overline{\mathbf{w}}_{q,+}}
      &=
      \sqrt{
        t_{q}
        \norm{\boldsymbol{\beta}_{q}}
      }
      =
      \sqrt{
        \frac{
          t_{q}
        }{
          \norm{\boldsymbol{\beta}_{q}}
        }
      }\,
      \norm{\boldsymbol{\beta}_{q}}
      =
      \norm{\mathbf{a}_{q,+}},
      \\
      \norm{\overline{\mathbf{w}}_{q,-}}
      &=
      \sqrt{
        (1-t_{q})
        \norm{\boldsymbol{\beta}_{q}}
      }
      =
      \sqrt{
        \frac{
          1-t_{q}
        }{
          \norm{\boldsymbol{\beta}_{q}}
        }
      }\,
      \norm{\boldsymbol{\beta}_{q}}
      =
      \norm{\mathbf{a}_{q,-}},
    \end{aligned}
  \end{equation}
and, for every $q \in \mathcal{E}$,
\begin{equation}
    \norm{\overline{\mathbf{w}}_{q,+}}
    \mathbf{a}_{q,+}
    =
    t_{q}\boldsymbol{\beta}_{q},
    \qquad
    \norm{\overline{\mathbf{w}}_{q,-}}
    \mathbf{a}_{q,-}
    =
    (1-t_{q})\boldsymbol{\beta}_{q},
  \end{equation}
where $t_{q} \geq 0$, $1-t_{q} \geq 0$, and $t_{q} + (1-t_{q}) = 1$, with the corresponding term omitted when an orientation is absent.
  For the latter, positive $1$-homogeneity gives feature rows
\begin{equation}
    \sqrt{
      t_{q}
      \norm{\boldsymbol{\beta}_{q}}
    }\,
    (\mathbf{f}_{q}^{+})^{\top}
    \qquad\text{and}\qquad
    \sqrt{
      (1-t_{q})
      \norm{\boldsymbol{\beta}_{q}}
    }\,
    (\mathbf{f}_{q}^{-})^{\top},
  \end{equation}
so the constructed parameterization realizes the \ac{rsm} in \cref{eq:mwnp-representational-geometry} for the chosen $\mathbf{t}$.
\end{proof}

The preceding proof demonstrates that $\mathcal{T}_{N_{h}} \neq \emptyset$ is necessary and sufficient for the lower bound in \cref{eq:mwnp-orbit-lower-bound} to be attained.
When this condition is satisfied, \cref{prop:mwnp-representational-geometry} characterizes the \acp{rsm} of norm-minimizing parameterizations, rather than the parameterizations themselves.
For a fixed split $\mathbf{t}$, the coefficients $c_{j} \geq 0$ can be redistributed among neurons of either orientation while preserving their sums $t_{q}$ and $1-t_{q}$.
Such redistributions leave the \ac{rsm} unchanged, so distinct \acp{mwnp} can realize the same split $\mathbf{t}$ and thus the same \ac{rsm}.

A separate degree of freedom lies in the split $\mathbf{t}$ itself, since different elements of $\mathcal{T}_{N_{h}}$ can yield different \acp{rsm} through \cref{eq:mwnp-representational-geometry}.
Under the linear-independence condition of \cref{cor:mwnp-representational-identifiability}, the residual constraint in \cref{eq:residual-compatible-readout-splits} admits at most one split.
Consequently, all \acp{mwnp} induce the same \ac{rsm}.

\begin{proof}[Proof of \cref{cor:mwnp-representational-identifiability}]
  Let $\mathbf{t}, \mathbf{s} \in \mathcal{T}$ be two residual-compatible splits.
  Since $\mathbf{t}, \mathbf{s} \in \mathcal{T}$, both satisfy \cref{eq:residual-compatible-readout-splits}.
  Hence, for every $\mathbf{x} \in \mathbb{R}^{N_{i}}$, the expressions corresponding to $\mathbf{t}$ and $\mathbf{s}$ are both equal to $\mathbf{r}(\mathbf{x})$.
  Taking their difference and using $m \neq 0$ gives
\begin{equation}
    \left[
      \sum_{q \in \mathcal{E}}
      (t_{q}-s_{q})
      \boldsymbol{\beta}_{q}
      \overline{\mathbf{w}}_{q}^{\top}
    \right]
    \overline{\mathbf{x}}
    = \mathbf{0},
    \qquad
    \mathbf{x} \in \mathbb{R}^{N_{i}}.
  \end{equation}
Since this expression vanishes for every input $\mathbf{x} \in \mathbb{R}^{N_{i}}$, we have
\begin{equation}
    \sum_{q \in \mathcal{E}}
    (t_{q}-s_{q})
    \boldsymbol{\beta}_{q}
    \overline{\mathbf{w}}_{q}^{\top}
    = \mathbf{0}.
  \end{equation}
Linear independence of the matrices $\boldsymbol{\beta}_{q}\overline{\mathbf{w}}_{q}^{\top}$, $q \in \mathcal{E}$, therefore implies $\mathbf{t} = \mathbf{s}$.
  Hence, $\mathcal{T}$ contains at most one split.
  Since $\mathcal{T}_{N_{h}} \subseteq \mathcal{T}$ and $\mathcal{T}_{N_{h}} \neq \emptyset$, $\mathcal{T}$ contains a unique split $\mathbf{t}$, which belongs to $\mathcal{T}_{N_{h}}$.
  \Cref{prop:mwnp-representational-geometry} therefore gives a unique minimum-weight-norm \ac{rsm}.

  For every $L \geq N_{h}$, this split remains feasible because $\kappa(\mathbf{t}) \leq N_{h} \leq L$ and hence belongs to $\mathcal{T}_{L}$.
  Since $\mathbf{t}$ is the unique element of $\mathcal{T}$, it is also the unique element of $\mathcal{T}_{L}$.
  Applying \cref{prop:mwnp-representational-geometry} again yields the same \ac{rsm}.
\end{proof}

The preceding result guarantees identifiability by showing that the residual constraint admits at most one split.
For even positively $1$-homogeneous activations, neurons of either orientation generate identical features $\mathbf{f}_{q}^{+} = \mathbf{f}_{q}^{-}$.
In this case, every residual-compatible split $\mathbf{t} \in \mathcal{T}$ yields the same \ac{rsm}, so representational identifiability holds irrespective of whether the residual constraint admits multiple splits.

\begin{corollary}[Identifiability for even positively $1$-homogeneous activations]
  \label[corollary]{cor:mwnp-identifiability-even-homogeneous-activations}
  Suppose $\sigma(z) = \delta\abs{z}$ with $\delta \neq 0$, and let $\mathcal{O}$ be a symmetry orbit with residual $\mathbf{r} = \mathbf{0}$.
  At every width $N_{h} \geq \abs{\mathcal{E}}$, the minimum of $\Omega_{W}$ is attained, and all \acp{mwnp} have the same \ac{rsm}, namely
\begin{equation}
    \sum_{q \in \mathcal{E}}
    \norm{\boldsymbol{\beta}_{q}}\,
    \mathbf{f}_{q}^{+} (\mathbf{f}_{q}^{+})^{\top}.
  \end{equation}
\end{corollary}

\begin{proof}
  Since $m = 0$ and $\mathbf{r} = \mathbf{0}$, every $\mathbf{t} \in [0,1]^{\abs{\mathcal{E}}}$ satisfies the residual constraint in \cref{eq:residual-compatible-readout-splits}.
  Choosing $\mathbf{t} \in \set{0,1}^{\abs{\mathcal{E}}}$ gives $\kappa(\mathbf{t}) = \abs{\mathcal{E}}$, so $\mathcal{T}_{N_{h}} \neq \emptyset$ for every $N_{h} \geq \abs{\mathcal{E}}$.
  Moreover, since $\sigma$ is even, $\mathbf{f}_{q}^{+} = \mathbf{f}_{q}^{-}$ for every $q \in \mathcal{E}$.
  Hence, \cref{eq:mwnp-representational-geometry} reduces to
\begin{equation}
    \sum_{q \in \mathcal{E}}
    \norm{\boldsymbol{\beta}_{q}}\,
    \mathbf{f}_{q}^{+}(\mathbf{f}_{q}^{+})^{\top},
  \end{equation}
independently of $\mathbf{t}$.
  The result therefore follows from \cref{prop:mwnp-representational-geometry}.
\end{proof}

For positively $1$-homogeneous activations with $m \neq 0$, \cref{prop:mwnp-representational-geometry} shows that minimum-weight-norm parameterizations need not have a unique \ac{rsm}: when $\mathcal{T}_{N_{h}}$ contains multiple residual-compatible splits, \cref{eq:mwnp-representational-geometry} parameterizes the corresponding \acp{rsm}, which can be distinct.
\Cref{cor:mwnp-representational-identifiability} shows that such a family cannot arise when the matrices $\boldsymbol{\beta}_{q}\overline{\mathbf{w}}_{q}^{\top}$, $q \in \mathcal{E}$, are linearly independent, since the residual constraint then admits at most one split.
Thus, whenever $\mathcal{T}_{N_{h}} \neq \emptyset$, any orbit admitting multiple distinct minimum-weight-norm \acp{rsm} must violate this linear-independence condition.
The following ReLU example illustrates precisely this case.

\begin{example}[Nonunique minimum-weight-norm geometry]
  \label[example]{ex:nonunique-mwnp-tent-geometry}
  Consider the ReLU tent function $f_{\boldsymbol{\theta}^{+}}(x) = \psi(1-\abs{x})$ from \cref{ex:opposite-irreducible-relu-parameterizations}.
  Recall that the parameterization $\boldsymbol{\theta}^{+}$ has incoming parameters and readouts
\begin{equation}
    (\overline{\mathbf{w}}_{1}, a_{1})
    =
    \left(
      \begin{bmatrix}
        1\\
        1
      \end{bmatrix},
      1
    \right),
    \qquad
    (\overline{\mathbf{w}}_{2}, a_{2})
    =
    \left(
      \begin{bmatrix}
        1\\
        0
      \end{bmatrix},
      -2
    \right),
    \qquad
    (\overline{\mathbf{w}}_{3}, a_{3})
    =
    \left(
      \begin{bmatrix}
        1\\
        -1
      \end{bmatrix},
      1
    \right).
  \end{equation}
The three essential parameter classes are
\begin{equation}
    q_{j} = [\overline{\mathbf{w}}_{j}],
    \qquad
    j \in \set{1,2,3}.
  \end{equation}
Choose the unit-norm representatives
\begin{equation}
    \overline{\mathbf{w}}_{q_{1}}
    =
    \frac{1}{\sqrt{2}}
    \begin{bmatrix}
      1\\
      1
    \end{bmatrix},
    \qquad
    \overline{\mathbf{w}}_{q_{2}}
    =
    \begin{bmatrix}
      1\\
      0
    \end{bmatrix},
    \qquad
    \overline{\mathbf{w}}_{q_{3}}
    =
    \frac{1}{\sqrt{2}}
    \begin{bmatrix}
      1\\
      -1
    \end{bmatrix}.
  \end{equation}
Since each essential parameter class contains a single neuron of $\boldsymbol{\theta}^{+}$, the definition of the orbit invariants $\beta_{q}$ in \cref{eq:orbit-invariants-positively-homogeneous} gives
\begin{equation}
    (\beta_{q_{1}},\beta_{q_{2}},\beta_{q_{3}})
    =
    \bigl(\sqrt{2},-2,\sqrt{2}\bigr).
  \end{equation}
Hence,
\begin{equation}
    \beta_{q_{1}}\overline{\mathbf{w}}_{q_{1}}^{\top}
    +
    \beta_{q_{2}}\overline{\mathbf{w}}_{q_{2}}^{\top}
    +
    \beta_{q_{3}}\overline{\mathbf{w}}_{q_{3}}^{\top}
    =
    (1,1)
    -
    2(1,0)
    +
    (1,-1)
    =
    (0,0),
  \end{equation}
so the matrices $\beta_{q} \overline{\mathbf{w}}_{q}^{\top}$ are linearly dependent.

  Since the orbit residual vanishes, the residual constraint in \cref{eq:residual-compatible-readout-splits} becomes
\begin{equation}
    \label{eq:tent-residual-compatible-split-constraint}
    (2t_{q_{1}}-1)(1,1)
    -
    2(2t_{q_{2}}-1)(1,0)
    +
    (2t_{q_{3}}-1)(1,-1)
    =
    (0,0).
  \end{equation}
Equivalently,
\begin{equation}
    \label{eq:tent-residual-compatible-split-first-coordinate}
    t_{q_{1}} - 2t_{q_{2}} + t_{q_{3}} = 0
  \end{equation}
and
\begin{equation}
    \label{eq:tent-residual-compatible-split-second-coordinate}
    t_{q_{1}}-t_{q_{3}} = 0.
  \end{equation}
\Cref{eq:tent-residual-compatible-split-second-coordinate} gives $t_{q_{1}} = t_{q_{3}}$, and substituting this into \cref{eq:tent-residual-compatible-split-first-coordinate} gives $t_{q_{1}} = t_{q_{2}}$.
  Therefore,
\begin{equation}
    \mathcal{T}
    =
    \set{
      (t,t,t)
      \given
      t \in [0,1]
    }.
  \end{equation}
Thus, the residual constraint leaves exactly one degree of freedom in the orientation split.
  For the endpoint splits $t \in \set{0,1}$, each essential class uses only one orientation, so
\begin{equation}
    \kappa(t,t,t)
    =
    3.
  \end{equation}
For every interior split $0 < t < 1$, both orientations are required in all three classes, giving
\begin{equation}
    \kappa(t,t,t)
    =
    6.
  \end{equation}
Hence,
\begin{equation}
    \mathcal{T}_{N_{h}}
    =
    \set{
      (0,0,0),
      (1,1,1)
    }
    \qquad
    \text{for } 3 \leq N_{h} < 6,
  \end{equation}
whereas $\mathcal{T}_{N_{h}} = \mathcal{T}$ for every $N_{h} \geq 6$.

  By \cref{prop:mwnp-representational-geometry}, every $\mathbf{t} \in \mathcal{T}_{N_{h}}$ is realized by an \ac{mwnp}, and all such parameterizations attain the minimum weight norm
\begin{equation}
    2
    \sum_{q \in \mathcal{E}}
    \abs{\beta_{q}}
    =
    4 + 4\sqrt{2}.
  \end{equation}
For any split $\mathbf{t} = (t,t,t) \in \mathcal{T}$, \cref{eq:mwnp-representational-geometry} gives the \ac{rsm}
\begin{equation}
    \mathbf{M}(t)
    \coloneqq
    t\mathbf{M}_{+}
    +
    (1-t)\mathbf{M}_{-},
  \end{equation}
where
\begin{equation}
    \label{eq:tent-endpoint-mwnp-rsms}
    \mathbf{M}_{+}
    \coloneqq
    \sum_{q \in \mathcal{E}}
    \abs{\beta_{q}}\,
    \mathbf{f}_{q}^{+}
    (\mathbf{f}_{q}^{+})^{\top},
    \qquad
    \mathbf{M}_{-}
    \coloneqq
    \sum_{q \in \mathcal{E}}
    \abs{\beta_{q}}\,
    \mathbf{f}_{q}^{-}
    (\mathbf{f}_{q}^{-})^{\top}.
  \end{equation}
For the chosen unit-norm representatives, the two feature orientations $\mathbf{f}_{q}^{+}$ and $\mathbf{f}_{q}^{-}$ have entries
\begin{equation}
    \begin{aligned}
      (\mathbf{f}_{q_{1}}^{+})_{\mu}
      &=
      \frac{1}{\sqrt{2}}\,
      \psi(x^{\mu}+1),
      &
      (\mathbf{f}_{q_{1}}^{-})_{\mu}
      &=
      \frac{1}{\sqrt{2}}\,
      \psi(-x^{\mu}-1),
      \\
      (\mathbf{f}_{q_{2}}^{+})_{\mu}
      &=
      \psi(x^{\mu}),
      &
      (\mathbf{f}_{q_{2}}^{-})_{\mu}
      &=
      \psi(-x^{\mu}),
      \\
      (\mathbf{f}_{q_{3}}^{+})_{\mu}
      &=
      \frac{1}{\sqrt{2}}\,
      \psi(x^{\mu}-1),
      &
      (\mathbf{f}_{q_{3}}^{-})_{\mu}
      &=
      \frac{1}{\sqrt{2}}\,
      \psi(1-x^{\mu}),
    \end{aligned}
    \qquad
    \mu = 1, \dots, P.
  \end{equation}
The geometries $\mathbf{M}_{+}$ and $\mathbf{M}_{-}$ in \cref{eq:tent-endpoint-mwnp-rsms}, corresponding to the endpoints $t=1$ and $t=0$, respectively, are distinct on any input set containing an input $x^{\mu} > 1$.
  At such an input,
\begin{equation}
    (\mathbf{f}_{q}^{+})_{\mu} > 0,
    \qquad
    (\mathbf{f}_{q}^{-})_{\mu} = 0
    \qquad
    \text{for every } q \in \mathcal{E},
  \end{equation}
so $\mathbf{M}_{+} \neq \mathbf{M}_{-}$.

  On such an input set, distinct values of $t$ yield distinct \acp{rsm}, since
\begin{equation}
    \mathbf{M}(t) - \mathbf{M}(s)
    =
    (t-s)
    (\mathbf{M}_{+}-\mathbf{M}_{-})
    \neq
    \mathbf{0}
  \end{equation}
whenever $t \neq s$.
  Thus, minimum-weight-norm selection yields two distinct \acp{rsm} for $3 \leq N_{h} < 6$ and a one-parameter family of distinct \acp{rsm} for every $N_{h} \geq 6$, visualized in \cref{fig:nonunique-mwnp-geometry}.
\end{example}

\begin{figure}[t]
  \centering
  \includegraphics[width=0.95\textwidth]{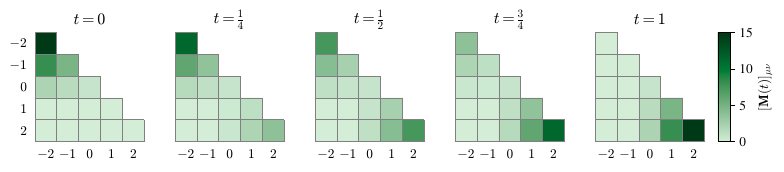}
  \caption{\textbf{Nonunique minimum-weight-norm geometries of the ReLU tent function}.
    The one-parameter family of \acp{rsm} $\mathbf{M}(t) = t\mathbf{M}_{+} + (1-t)\mathbf{M}_{-}$ realized by minimum-weight-norm parameterizations of the ReLU tent function $\psi(1 - \abs{x})$ from \cref{ex:nonunique-mwnp-tent-geometry}, evaluated on the inputs collected in $\mathbf{X} = (-2, -1, 0, 1, 2)$.
    Each panel shows $\mathbf{M}(t)$ for one value of $t \in [0,1]$, corresponding to the residual-compatible split $\mathbf{t} = (t,t,t)$, with rows and columns indexed by the inputs and all panels sharing a common color scale.
    The endpoints $t = 0$ and $t = 1$ use a single orientation per essential class and correspond to the oppositely oriented parameterizations $\boldsymbol{\theta}^{-}$ and $\boldsymbol{\theta}^{+}$ of \cref{fig:opposite-orientations}, rescaled to minimum-weight-norm parameterizations.
    Their \acp{rsm} are mirror images because reversing the inputs swaps the two orientations.
    Every interior value $0 < t < 1$ requires both orientations in all three essential classes.
    The diagonal entry corresponding to $x = 0$ is constant across the family, whereas every other nonzero entry varies with $t$, so distinct values of $t$ yield distinct \acp{rsm}.
}
  \label{fig:nonunique-mwnp-geometry}
\end{figure}

The preceding example shows that minimum-weight-norm selection need not identify a unique representational geometry when the linear-independence condition of \cref{cor:mwnp-representational-identifiability} fails.
We close this analysis by returning to the six analytical ReLU solutions underlying \cref{fig:dissociation}, for which the main text asserts representational identifiability under minimum-weight-norm selection.
For each solution, we verify that the conditions of \cref{cor:mwnp-representational-identifiability} hold, so the selected \ac{rsm} is unique and remains unchanged under further increases in width.

\begin{example}[Identifiability of the analytical ReLU solutions]
  \label[example]{ex:mwnp-identifiability-of-analytical-relu-solutions}
  Consider the six analytical ReLU solutions illustrated in \cref{fig:dissociation} and detailed in \cref{tab:xor-relu-solutions}.
  Following the convention described in \cref{app-subsec:xor-setup}, we hold the output bias $b$ fixed and apply \cref{cor:mwnp-representational-identifiability} to $f_{\boldsymbol{\theta}} - b$.

  For solutions~1, 3, 4, and~6, the two incoming-parameter vectors have the form
\begin{equation}
    \overline{\mathbf{w}}_{1}
    =
    \begin{bmatrix}
      \mathbf{v}\\
      0
    \end{bmatrix}
    \qquad\text{and}\qquad
    \overline{\mathbf{w}}_{2}
    =
    \begin{bmatrix}
      \mathbf{v}\\
      \sqrt{2}
    \end{bmatrix},
    \qquad
    \norm{\mathbf{v}} = 1,
  \end{equation}
up to ordering.
  These vectors are linearly independent and therefore belong to distinct parameter classes
\begin{equation}
    q_{1} = [\overline{\mathbf{w}}_{1}],
    \qquad
    q_{2} = [\overline{\mathbf{w}}_{2}],
  \end{equation}
both of which are essential because their corresponding readouts $a_{1}$ and $a_{2}$ are nonzero.
  Choose the unit-norm representatives
\begin{equation}
    \overline{\mathbf{w}}_{q_{j}}
    =
    \frac{
      \overline{\mathbf{w}}_{j}
    }{
      \norm{\overline{\mathbf{w}}_{j}}
    },
    \qquad
    j \in \set{1,2}.
  \end{equation}
Since each class contains a single neuron of the given solution, \cref{eq:orbit-invariants-positively-homogeneous} gives
\begin{equation}
    \beta_{q_{j}}
    =
    \norm{\overline{\mathbf{w}}_{j}}a_{j},
    \qquad
    \beta_{q_{j}}\overline{\mathbf{w}}_{q_{j}}^{\top}
    =
    a_{j}\overline{\mathbf{w}}_{j}^{\top},
    \qquad
    j \in \set{1,2}.
  \end{equation}
Because $a_{1},a_{2} \neq 0$ and $\overline{\mathbf{w}}_{1},\overline{\mathbf{w}}_{2}$ are linearly independent, the two matrices
  $\beta_{q_{1}}\overline{\mathbf{w}}_{q_{1}}^{\top}$ and
  $\beta_{q_{2}}\overline{\mathbf{w}}_{q_{2}}^{\top}$
  are linearly independent.
  Moreover, the two neurons realize the split $\mathbf{t} = (1,1)$, so $\mathbf{t} \in \mathcal{T}_{2}$ and hence $\mathcal{T}_{2} \neq \emptyset$.

  For solutions~2 and~5, the two incoming-parameter vectors are opposite unit vectors with zero hidden bias and therefore belong to a single essential parameter class.
  Their readout weights are $(1,1)$ and $(-1,-1)$, respectively, giving the scalar class aggregate $\beta_{q} = 2$ for solution~2 and $\beta_{q} = -2$ for solution~5.
  In both cases, the residual vanishes, so \cref{eq:residual-compatible-readout-splits} uniquely fixes $t_{q} = \nicefrac{1}{2}$.
  Since $\kappa(t_{q}) = 2$, this split belongs to $\mathcal{T}_{2}$.
  Moreover, $\beta_{q}\overline{\mathbf{w}}_{q}^{\top} \neq \mathbf{0}$, so the singleton family $\set{\beta_{q}\overline{\mathbf{w}}_{q}^{\top}}$ is linearly independent.

  Since ReLU has $m = \nicefrac{1}{2} \neq 0$, all six solutions thus satisfy the assumptions of \cref{cor:mwnp-representational-identifiability} at width $N_{h} = 2$.
  Each corresponding symmetry orbit therefore has a unique minimum-weight-norm \ac{rsm}, which remains unchanged at every width $N_{h} \geq 2$.
\end{example}

\subsection{Minimum representation norm for positively \texorpdfstring{$1$-homogeneous}{1-homogeneous} activations}
\label[appendix]{app-subsec:mrnp-representational-geometry}

We now turn to \aclp{mrnp}, whose analysis closely parallels that of the minimum-weight-norm case presented in the preceding subsection.
We first derive an orbit-wide lower bound on the representation-norm objective $\Omega_{H}$ and characterize exactly when this bound is attained.
The resulting equality conditions determine the admissible orientation splits and yield the corresponding family of minimum-representation-norm \acp{rsm}, analogous to \cref{prop:mwnp-representational-geometry}.
We then identify conditions under which this family reduces to a unique geometry, similar to \cref{cor:mwnp-representational-identifiability}.
Finally, in contrast to the \ac{mwnp} case, we show that a fixed symmetry orbit need not contain an \ac{mrnp} at all.

Throughout this subsection, we retain the setting and notation of \cref{app-subsec:mwnp-representational-geometry}, including the symmetry orbit $\mathcal{O}$, essential parameter classes $\mathcal{E}$, orbit invariants $\boldsymbol{\beta}_{q}$, residual $\mathbf{r}$, unit-norm representatives $\overline{\mathbf{w}}_{q}$, feature vectors $\mathbf{f}_{q}^{+}$ and $\mathbf{f}_{q}^{-}$, and residual-compatible split set $\mathcal{T}$.

Unlike the weight-norm objective $\Omega_{W}$, the representation-norm objective $\Omega_{H}$ depends on the fixed input matrix $\mathbf{X}$ through the rows $\mathbf{v}_{j}^{\top} = \sigma(\overline{\mathbf{w}}_{j}^{\top} \overline{\mathbf{X}})$ of the hidden-activation matrix $\mathbf{H}$, since
\begin{equation}
  \norm{\mathbf{H}}_{F}^{2}
  =
  \sum_{j=1}^{N_{h}} \norm{\mathbf{v}_{j}}^{2}.
\end{equation}
For each essential parameter class $q \in \mathcal{E}$, define
\begin{equation}
  \label{eq:mrnp-orientation-feature-norms}
  g_{q}^{+}
  \coloneqq
  \norm{\mathbf{f}_{q}^{+}},
  \qquad
  g_{q}^{-}
  \coloneqq
  \norm{\mathbf{f}_{q}^{-}},
  \qquad
  g_{q}
  \coloneqq
  \min\set{g_{q}^{+},g_{q}^{-}}.
\end{equation}

The next lemma is the minimum-representation-norm analog of \cref{lem:mwnp-lower-bound-and-equality-conditions}.
It again gives an orbit-wide lower bound together with necessary and sufficient equality conditions.
Here, achieving the lower bound requires matching the magnitude $\norm{\mathbf{v}_{j}}$ of the feature computed by the $j$th neuron to the magnitude $\norm{\mathbf{a}_{j}}$ of its readout.
For any essential class $q \in \mathcal{E}$ with $g_{q}^{+} \neq g_{q}^{-}$, achieving the lower bound further requires every neuron in that class to have the orientation attaining $g_{q} = \min\set{g_{q}^{+},g_{q}^{-}}$.

\begin{lemma}[Representation-norm lower bound and equality conditions]
  \label[lemma]{lem:mrnp-lower-bound-and-equality-conditions}
  Suppose $g_{q} > 0$ for every essential parameter class $q \in \mathcal{E}$.
  Then, every parameterization $\boldsymbol{\theta} \in \mathcal{O}_{N_{h}}$ satisfies
\begin{equation}
    \label{eq:mrnp-orbit-lower-bound}
    \Omega_{H}(\boldsymbol{\theta})
    \geq
    2
    \sum_{q \in \mathcal{E}}
    g_{q}\norm{\boldsymbol{\beta}_{q}}.
  \end{equation}
Equality holds if and only if the following three conditions are satisfied:
\begin{enumerate}
    \item Every neuron outside the essential parameter classes has
\begin{equation}
      \mathbf{v}_{j} = \mathbf{0},
      \qquad
      \mathbf{a}_{j} = \mathbf{0}.
    \end{equation}

    \item Every neuron in an essential parameter class has equal feature and readout norms:
\begin{equation}
      \norm{\mathbf{v}_{j}}
      =
      \norm{\mathbf{a}_{j}},
      \qquad
      j \in \mathcal{J}_{q},
      \quad
      q \in \mathcal{E}.
    \end{equation}
Moreover, whenever the two orientations of a class $q \in \mathcal{E}$ have different feature norms $g_{q}^{+} \neq g_{q}^{-}$ on $\mathbf{X}$, all neurons in $q$ have the orientation with the smaller feature norm.
    Equivalently,
\begin{equation}
      g_{q}^{+} < g_{q}^{-}
      \implies
      \mathcal{J}_{q}^{-} = \emptyset,
      \qquad
      g_{q}^{-} < g_{q}^{+}
      \implies
      \mathcal{J}_{q}^{+} = \emptyset.
    \end{equation}

    \item For every $q \in \mathcal{E}$, there exist coefficients $c_{j} \geq 0$, $j \in \mathcal{J}_{q}$, such that
\begin{equation}
      \norm{\overline{\mathbf{w}}_{j}}\mathbf{a}_{j}
      =
      c_{j}\boldsymbol{\beta}_{q},
      \qquad
      \sum_{j \in \mathcal{J}_{q}} c_{j} = 1.
    \end{equation}
  \end{enumerate}
\end{lemma}

\begin{proof}
  By definition of the representation-norm objective $\Omega_{H}(\boldsymbol{\theta})$ in \cref{eq:minimum-norm-objectives}, we have
\begin{equation}
    \label{eq:mrnp-essential-class-lower-bound}
    \Omega_{H}(\boldsymbol{\theta})
    =
    \sum_{j=1}^{N_{h}}
    \bigl(
      \norm{\mathbf{v}_{j}}^{2}
      + \norm{\mathbf{a}_{j}}^{2}
    \bigr)
    \geq
    \sum_{q \in \mathcal{E}}
    \sum_{j \in \mathcal{J}_{q}}
    \bigl(
      \norm{\mathbf{v}_{j}}^{2}
      + \norm{\mathbf{a}_{j}}^{2}
    \bigr).
  \end{equation}
The AM--GM inequality gives
\begin{equation}
    \label{eq:mrnp-amgm-lower-bound}
    \sum_{q \in \mathcal{E}}
    \sum_{j \in \mathcal{J}_{q}}
    \bigl(
      \norm{\mathbf{v}_{j}}^{2}
      + \norm{\mathbf{a}_{j}}^{2}
    \bigr)
    \geq
    2
    \sum_{q \in \mathcal{E}}
    \sum_{j \in \mathcal{J}_{q}}
    \norm{\mathbf{v}_{j}}
    \norm{\mathbf{a}_{j}}.
  \end{equation}
For each essential parameter class $q \in \mathcal{E}$, positive $1$-homogeneity and the unit norm of $\overline{\mathbf{w}}_{q}$ give
\begin{equation}
    \norm{\mathbf{v}_{j}}
    =
    \begin{cases}
      \norm{\overline{\mathbf{w}}_{j}} g_{q}^{+},
      & j \in \mathcal{J}_{q}^{+}, \\
      \norm{\overline{\mathbf{w}}_{j}} g_{q}^{-},
      & j \in \mathcal{J}_{q}^{-}.
    \end{cases}
  \end{equation}
Since $g_{q} = \min\set{g_{q}^{+},g_{q}^{-}}$, it follows that
\begin{equation}
    \label{eq:mrnp-orientation-lower-bound}
    \begin{aligned}
      \sum_{j \in \mathcal{J}_{q}}
      \norm{\mathbf{v}_{j}}
      \norm{\mathbf{a}_{j}}
      &=
      g_{q}^{+}
      \sum_{j \in \mathcal{J}_{q}^{+}}
      \norm{\overline{\mathbf{w}}_{j}}
      \norm{\mathbf{a}_{j}}
      +
      g_{q}^{-}
      \sum_{j \in \mathcal{J}_{q}^{-}}
      \norm{\overline{\mathbf{w}}_{j}}
      \norm{\mathbf{a}_{j}}
      \\
      &\geq
      g_{q}
      \sum_{j \in \mathcal{J}_{q}}
      \norm{\overline{\mathbf{w}}_{j}}
      \norm{\mathbf{a}_{j}}.
    \end{aligned}
  \end{equation}
Finally, absolute homogeneity of the Euclidean norm and the triangle inequality yield
\begin{equation}
    \label{eq:mrnp-class-aggregate-lower-bound}
    \sum_{j \in \mathcal{J}_{q}}
    \norm{\overline{\mathbf{w}}_{j}}
    \norm{\mathbf{a}_{j}}
    =
    \sum_{j \in \mathcal{J}_{q}}
    \norm[\big]{
      \norm{\overline{\mathbf{w}}_{j}}
      \mathbf{a}_{j}
    }
    \geq
    \norm[\Big]{
      \sum_{j \in \mathcal{J}_{q}}
      \norm{\overline{\mathbf{w}}_{j}}
      \mathbf{a}_{j}
    }
    =
    \norm{\boldsymbol{\beta}_{q}},
  \end{equation}
where the final equality follows from the definition of the class aggregate $\boldsymbol{\beta}_{q}$ in \cref{eq:orbit-invariants-positively-homogeneous}.
  Together, \cref{eq:mrnp-essential-class-lower-bound,eq:mrnp-amgm-lower-bound,eq:mrnp-orientation-lower-bound,eq:mrnp-class-aggregate-lower-bound} establish \cref{eq:mrnp-orbit-lower-bound}.

  Equality in \cref{eq:mrnp-orbit-lower-bound} holds if and only if equality holds in \cref{eq:mrnp-essential-class-lower-bound,eq:mrnp-amgm-lower-bound} and, for every $q \in \mathcal{E}$, in \cref{eq:mrnp-orientation-lower-bound,eq:mrnp-class-aggregate-lower-bound}.
  It therefore remains to determine when each inequality is tight.

  Equality in \cref{eq:mrnp-essential-class-lower-bound} holds if and only if every neuron outside the essential parameter classes contributes nothing to $\Omega_{H}$.
  Since each such contribution is
\begin{equation}
    \norm{\mathbf{v}_{j}}^{2}
    +
    \norm{\mathbf{a}_{j}}^{2}
    \geq 0,
  \end{equation}
this is equivalent to
\begin{equation}
    \mathbf{v}_{j} = \mathbf{0},
    \qquad
    \mathbf{a}_{j} = \mathbf{0},
  \end{equation}
for every neuron outside the essential parameter classes, giving the first condition.

  Equality in \cref{eq:mrnp-amgm-lower-bound} holds if and only if the AM--GM inequality is tight for every neuron in every essential parameter class.
  For each $j \in \mathcal{J}_{q}$,
\begin{equation}
    \bigl(
      \norm{\mathbf{v}_{j}}^{2}
      + \norm{\mathbf{a}_{j}}^{2}
    \bigr)
    -
    2
    \norm{\mathbf{v}_{j}}
    \norm{\mathbf{a}_{j}}
    =
    \bigl(
      \norm{\mathbf{v}_{j}}
      -
      \norm{\mathbf{a}_{j}}
    \bigr)^{2},
  \end{equation}
so equality holds if and only if
\begin{equation}
    \norm{\mathbf{v}_{j}}
    =
    \norm{\mathbf{a}_{j}},
    \qquad
    j \in \mathcal{J}_{q},
    \quad
    q \in \mathcal{E}.
  \end{equation}
Let $\varepsilon \in \set{+,-}$.
  Because $g_{q}^{\varepsilon} \geq g_{q} > 0$ and every neuron in $\mathcal{J}_{q}^{\varepsilon}$ has $\norm{\overline{\mathbf{w}}_{j}} > 0$, this equality also implies $\norm{\mathbf{a}_{j}} > 0$ for every neuron in an essential parameter class.
  Consequently, equality in \cref{eq:mrnp-orientation-lower-bound} holds if and only if every represented orientation of class $q$ has feature norm $g_{q}$.
  Equivalently,
\begin{equation}
    g_{q}^{+} < g_{q}^{-}
    \implies
    \mathcal{J}_{q}^{-} = \emptyset,
    \qquad
    g_{q}^{-} < g_{q}^{+}
    \implies
    \mathcal{J}_{q}^{+} = \emptyset,
  \end{equation}
completing the second condition.

  Finally, equality in \cref{eq:mrnp-class-aggregate-lower-bound} is determined by equality in the triangle inequality.
  Since $\boldsymbol{\beta}_{q} \neq \mathbf{0}$ for every essential parameter class $q \in \mathcal{E}$, equality holds if and only if the vectors
\begin{equation}
    \norm{\overline{\mathbf{w}}_{j}}
    \mathbf{a}_{j},
    \qquad
    j \in \mathcal{J}_{q},
  \end{equation}
are nonnegative multiples of $\boldsymbol{\beta}_{q}$.
  Equivalently, there exist coefficients $c_{j} \geq 0$ such that
\begin{equation}
    \norm{\overline{\mathbf{w}}_{j}}
    \mathbf{a}_{j}
    =
    c_{j}\boldsymbol{\beta}_{q},
    \qquad
    j \in \mathcal{J}_{q}.
  \end{equation}
Summing over $j \in \mathcal{J}_{q}$ and using the definition of the class aggregate gives
\begin{equation}
    \boldsymbol{\beta}_{q}
    =
    \boldsymbol{\beta}_{q}
    \sum_{j \in \mathcal{J}_{q}} c_{j}.
  \end{equation}
Since $\boldsymbol{\beta}_{q} \neq \mathbf{0}$, we have $\sum_{j \in \mathcal{J}_{q}} c_{j} = 1$, giving the third condition.

  Conversely, if the three stated conditions hold, then equality holds in \cref{eq:mrnp-essential-class-lower-bound,eq:mrnp-amgm-lower-bound,eq:mrnp-orientation-lower-bound,eq:mrnp-class-aggregate-lower-bound}, and hence in \cref{eq:mrnp-orbit-lower-bound}.
\end{proof}

As in the minimum-weight-norm case, the equality conditions in \cref{lem:mrnp-lower-bound-and-equality-conditions} again determine how the class aggregate $\boldsymbol{\beta}_{q}$ can be distributed between the two orientations of each essential parameter class at the representation-norm lower bound.

\begin{corollary}[Readout allocation at the representation-norm lower bound]
  \label[corollary]{cor:mrnp-readout-allocation}
  Suppose that $\boldsymbol{\theta} \in \mathcal{O}_{N_{h}}$ attains the lower bound in \cref{eq:mrnp-orbit-lower-bound}.
  Then, for every essential parameter class $q \in \mathcal{E}$, there exists $t_{q} \in [0,1]$ such that
\begin{equation}
    \sum_{j \in \mathcal{J}_{q}^{+}}
    \norm{\overline{\mathbf{w}}_{j}} \mathbf{a}_{j}
    =
    t_{q} \boldsymbol{\beta}_{q},
    \qquad
    \sum_{j \in \mathcal{J}_{q}^{-}}
    \norm{\overline{\mathbf{w}}_{j}} \mathbf{a}_{j}
    =
    (1-t_{q}) \boldsymbol{\beta}_{q}.
  \end{equation}
Moreover,
\begin{equation}
    g_{q}^{+} < g_{q}^{-}
    \implies
    t_{q} = 1,
    \qquad
    g_{q}^{-} < g_{q}^{+}
    \implies
    t_{q} = 0.
  \end{equation}
\end{corollary}

\begin{proof}
  By \cref{lem:mrnp-lower-bound-and-equality-conditions}, there exist coefficients $c_{j} \geq 0$, $j \in \mathcal{J}_{q}$, such that
\begin{equation}
    \norm{\overline{\mathbf{w}}_{j}} \mathbf{a}_{j}
    =
    c_{j}\boldsymbol{\beta}_{q},
    \qquad
    \sum_{j \in \mathcal{J}_{q}} c_{j} = 1.
  \end{equation}
Setting
\begin{equation}
    t_{q}
    \coloneqq
    \sum_{j \in \mathcal{J}_{q}^{+}} c_{j}
  \end{equation}
gives $t_{q} \in [0,1]$ and $\sum_{j \in \mathcal{J}_{q}^{-}} c_{j} = 1-t_{q}$, yielding the two allocation identities.
  If $g_{q}^{+} < g_{q}^{-}$, \cref{lem:mrnp-lower-bound-and-equality-conditions} gives $\mathcal{J}_{q}^{-} = \emptyset$, so $t_{q}=1$.
  Similarly, if $g_{q}^{-} < g_{q}^{+}$, then $\mathcal{J}_{q}^{+} = \emptyset$ and hence $t_{q}=0$.
\end{proof}

The preceding corollary shows that every parameterization attaining the representation-norm lower bound $2 \sum_{q \in \mathcal{E}} g_{q}\norm{\boldsymbol{\beta}_{q}}$ induces a residual-compatible split $\mathbf{t}$ whose coordinates $t_{q}$ are additionally constrained whenever the two features $\mathbf{f}_{q}^{+}$ and $\mathbf{f}_{q}^{-}$ of an essential class $q$ have different magnitudes.
Accordingly, let $\mathcal{T}^{H} \subseteq \mathcal{T}$ denote the residual-compatible splits $\mathbf{t} \in \mathcal{T}$ satisfying, for every $q \in \mathcal{E}$,
\begin{equation}
  \label{eq:mrnp-orientation-restrictions}
  t_{q}
  =
  \begin{cases}
    1, & \text{if } g_{q}^{+} < g_{q}^{-}, \\
    0, & \text{if } g_{q}^{-} < g_{q}^{+},
  \end{cases}
\end{equation}
with no additional restriction imposed on $t_{q}$ when $g_{q}^{+} = g_{q}^{-}$.
As in the minimum-weight-norm case, we further restrict to splits realizable at width $N_{h}$ by defining
\begin{equation}
  \label{eq:mrnp-width-feasible-readout-splits}
  \mathcal{T}_{N_{h}}^{H}
  \coloneqq
  \set{
    \mathbf{t} \in \mathcal{T}^{H}
    \given
    \kappa(\mathbf{t}) \leq N_{h}
  }
  =
  \mathcal{T}^{H} \cap \mathcal{T}_{N_{h}}.
\end{equation}
Thus, $\mathcal{T}_{N_{h}}^{H}$ contains precisely the residual-compatible splits $\mathbf{t}$ that can be realized at width $N_{h}$ using only orientations permitted by the equality conditions of \cref{lem:mrnp-lower-bound-and-equality-conditions}.

We can now characterize the \acp{rsm} of all parameterizations attaining the representation-norm lower bound, in direct analogy with \cref{prop:mwnp-representational-geometry}.

\begin{proposition}[Minimum-representation-norm representational geometry]
  \label[proposition]{prop:mrnp-representational-geometry}
  Fix a symmetry orbit $\mathcal{O}$ with positively $1$-homogeneous activation and essential parameter classes $\mathcal{E}$, and suppose $g_{q} > 0$ for every $q \in \mathcal{E}$.
  If $\mathcal{T}_{N_{h}}^{H} \neq \emptyset$, the minimum of $\Omega_{H}$ over $\mathcal{O}_{N_{h}}$ is attained and equals
\begin{equation}
    2
    \sum_{q \in \mathcal{E}}
    g_{q}\norm{\boldsymbol{\beta}_{q}}.
  \end{equation}
Moreover, the \acp{rsm} of all \acp{mrnp} in $\mathcal{O}_{N_{h}}$ are precisely
\begin{equation}
    \label{eq:mrnp-representational-geometry}
    \sum_{q \in \mathcal{E}}
    \frac{\norm{\boldsymbol{\beta}_{q}}}{g_{q}}
    \bigl(
      t_{q}\mathbf{f}_{q}^{+}(\mathbf{f}_{q}^{+})^{\top}
      +
      (1-t_{q})\mathbf{f}_{q}^{-}(\mathbf{f}_{q}^{-})^{\top}
    \bigr),
  \end{equation}
as $\mathbf{t}$ ranges over $\mathcal{T}_{N_{h}}^{H}$.
\end{proposition}

\begin{proof}
  We first show that every parameterization $\boldsymbol{\theta} \in \mathcal{O}_{N_{h}}$ attaining the lower bound in \cref{eq:mrnp-orbit-lower-bound} induces a split $\mathbf{t} \in \mathcal{T}_{N_{h}}^{H}$ for which its \ac{rsm} is given by \cref{eq:mrnp-representational-geometry}.

  Fix $\boldsymbol{\theta} \in \mathcal{O}_{N_{h}}$ attaining the lower bound in \cref{eq:mrnp-orbit-lower-bound}.
  By \cref{cor:mrnp-readout-allocation}, for every $q \in \mathcal{E}$, there exists $t_{q} \in [0,1]$ such that
\begin{equation}
    \sum_{j \in \mathcal{J}_{q}^{+}}
    \norm{\overline{\mathbf{w}}_{j}}\mathbf{a}_{j}
    =
    t_{q}\boldsymbol{\beta}_{q},
    \qquad
    \sum_{j \in \mathcal{J}_{q}^{-}}
    \norm{\overline{\mathbf{w}}_{j}}\mathbf{a}_{j}
    =
    (1-t_{q})\boldsymbol{\beta}_{q}.
  \end{equation}
By the first condition in \cref{lem:mrnp-lower-bound-and-equality-conditions}, every neuron outside the essential parameter classes has $\mathbf{a}_{j} = \mathbf{0}$ and therefore does not contribute to the global residual.
  Hence,
\begin{equation}
    \mathbf{r}(\mathbf{x})
    =
    m
    \sum_{q \in \mathcal{E}}
    (2t_{q}-1)
    \boldsymbol{\beta}_{q}
    \overline{\mathbf{w}}_{q}^{\top}
    \overline{\mathbf{x}}.
  \end{equation}
Since $\boldsymbol{\theta} \in \mathcal{O}_{N_{h}}$ has residual $\mathbf{r}$, the definition of $\mathcal{T}$ in \cref{eq:residual-compatible-readout-splits} gives $\mathbf{t} \in \mathcal{T}$.
  Moreover, \cref{cor:mrnp-readout-allocation} ensures that $\mathbf{t}$ satisfies the orientation restrictions in \cref{eq:mrnp-orientation-restrictions}, so $\mathbf{t} \in \mathcal{T}^{H}$.
  Since $\mathbf{t}$ is induced by a parameterization of width $N_{h}$, we have $\kappa(\mathbf{t}) \leq N_{h}$ and therefore $\mathbf{t} \in \mathcal{T}_{N_{h}}^{H}$.

  We next determine the \ac{rsm} induced by $\boldsymbol{\theta}$.
  By the second condition in \cref{lem:mrnp-lower-bound-and-equality-conditions}, every neuron $j \in \mathcal{J}_{q}$ in every essential parameter class $q \in \mathcal{E}$ satisfies $\norm{\overline{\mathbf{w}}_{j}} g_{q} = \norm{\mathbf{v}_{j}} = \norm{\mathbf{a}_{j}}$.
  By the third condition, there exist coefficients $c_{j} \geq 0$, $j \in \mathcal{J}_{q}$, such that $\norm{\overline{\mathbf{w}}_{j}} \mathbf{a}_{j} = c_{j} \boldsymbol{\beta}_{q}$.
  Combining these gives
\begin{equation}
    \norm{\overline{\mathbf{w}}_{j}}^{2} g_{q}
    =
    \norm{\overline{\mathbf{w}}_{j}}
    \norm{\mathbf{a}_{j}}
    =
    \norm[\big]{
      \norm{\overline{\mathbf{w}}_{j}}
      \mathbf{a}_{j}
    }
    =
    c_{j}\norm{\boldsymbol{\beta}_{q}}.
  \end{equation}
Equivalently,
\begin{equation}
    \label{eq:mrnp-balanced-feature-weight}
    \norm{\overline{\mathbf{w}}_{j}}^{2}
    =
    \frac{
      c_{j}\norm{\boldsymbol{\beta}_{q}}
    }{
      g_{q}
    }.
  \end{equation}
Let $\varepsilon \in \set{+,-}$.
  Positive $1$-homogeneity implies that neuron $j \in \mathcal{J}_{q}^{\varepsilon}$ has feature row
\begin{equation}
    \mathbf{v}_{j}^{\top}
    =
    \sigma(\overline{\mathbf{w}}_{j}^{\top} \overline{\mathbf{X}})
    =
    \norm{\overline{\mathbf{w}}_{j}}
    \sigma(\varepsilon \overline{\mathbf{w}}_{q}^{\top} \overline{\mathbf{X}})
    =
    \norm{\overline{\mathbf{w}}_{j}}
    (\mathbf{f}_{q}^{\varepsilon})^{\top}.
  \end{equation}
Its contribution to $\mathbf{H}^{\top}\mathbf{H}$ is therefore
\begin{equation}
    \mathbf{v}_{j}\mathbf{v}_{j}^{\top}
    =
    c_{j}
    \frac{\norm{\boldsymbol{\beta}_{q}}}{g_{q}}
    \mathbf{f}_{q}^{\varepsilon}
    (\mathbf{f}_{q}^{\varepsilon})^{\top}
  \end{equation}
by \cref{eq:mrnp-balanced-feature-weight}.
  Summing over neurons and using
\begin{equation}
    \sum_{j \in \mathcal{J}_{q}^{+}} c_{j}
    =
    t_{q},
    \qquad
    \sum_{j \in \mathcal{J}_{q}^{-}} c_{j}
    =
    1-t_{q},
  \end{equation}
gives
\begin{equation}
    \mathbf{H}^{\top}\mathbf{H}
    =
    \sum_{q \in \mathcal{E}}
    \frac{\norm{\boldsymbol{\beta}_{q}}}{g_{q}}
    \bigl(
      t_{q}
      \mathbf{f}_{q}^{+}
      (\mathbf{f}_{q}^{+})^{\top}
      +
      (1-t_{q})
      \mathbf{f}_{q}^{-}
      (\mathbf{f}_{q}^{-})^{\top}
    \bigr).
  \end{equation}
This proves the first claim that every parameterization $\boldsymbol{\theta} \in \mathcal{O}_{N_{h}}$ attaining the lower bound in \cref{eq:mrnp-orbit-lower-bound} induces a split $\mathbf{t} \in \mathcal{T}_{N_{h}}^{H}$ for which its \ac{rsm} is given by \cref{eq:mrnp-representational-geometry}.

  For the converse, fix any $\mathbf{t} \in \mathcal{T}_{N_{h}}^{H}$.
  For each $q \in \mathcal{E}$, introduce neurons with parameters
\begin{equation}
    \label{eq:mrnp-balanced-construction}
    \begin{aligned}
      \overline{\mathbf{w}}_{q,+}
      &=
      \sqrt{
        \frac{
          t_{q}\norm{\boldsymbol{\beta}_{q}}
        }{
          g_{q}
        }
      }\,
      \overline{\mathbf{w}}_{q},
      &
      \mathbf{a}_{q,+}
      &=
      \sqrt{
        \frac{
          t_{q}g_{q}
        }{
          \norm{\boldsymbol{\beta}_{q}}
        }
      }\,
      \boldsymbol{\beta}_{q},
      \\
      \overline{\mathbf{w}}_{q,-}
      &=
      -
      \sqrt{
        \frac{
          (1-t_{q})\norm{\boldsymbol{\beta}_{q}}
        }{
          g_{q}
        }
      }\,
      \overline{\mathbf{w}}_{q},
      &
      \mathbf{a}_{q,-}
      &=
      \sqrt{
        \frac{
          (1-t_{q})g_{q}
        }{
          \norm{\boldsymbol{\beta}_{q}}
        }
      }\,
      \boldsymbol{\beta}_{q},
    \end{aligned}
  \end{equation}
omitting the positively oriented neuron when $t_{q} = 0$ and the negatively oriented neuron when $t_{q} = 1$.
  This construction uses exactly $\kappa(\mathbf{t})$ neurons.
  Since $\mathbf{t} \in \mathcal{T}_{N_{h}}^{H}$, and hence $\kappa(\mathbf{t}) \leq N_{h}$, adding $N_{h}-\kappa(\mathbf{t})$ neurons with all parameters set to zero yields a parameterization of width $N_{h}$.

  By construction,
\begin{equation}
    \norm{\overline{\mathbf{w}}_{q,+}}
    \mathbf{a}_{q,+}
    =
    t_{q}\boldsymbol{\beta}_{q},
    \qquad
    \norm{\overline{\mathbf{w}}_{q,-}}
    \mathbf{a}_{q,-}
    =
    (1-t_{q})\boldsymbol{\beta}_{q}.
  \end{equation}
Thus, the two orientations together reproduce the class aggregate $\boldsymbol{\beta}_{q}$.
  Moreover, $\mathbf{t} \in \mathcal{T}$ ensures that the resulting global residual equals $\mathbf{r}$.
  Since no nonessential parameter class is introduced, \cref{prop:characterization-of-symmetry-equivalence} therefore implies that the constructed parameterization belongs to $\mathcal{O}_{N_{h}}$.

  It remains to show that the constructed parameterization has representation norm $2 \sum_{q \in \mathcal{E}} g_{q}\norm{\boldsymbol{\beta}_{q}}$ and that its \ac{rsm} is of the form in \cref{eq:mrnp-representational-geometry}.
  The former follows by verifying the three conditions of \cref{lem:mrnp-lower-bound-and-equality-conditions}.
  The padding neurons have $\mathbf{v}_{j} = \mathbf{0}$ and $\mathbf{a}_{j} = \mathbf{0}$.
  Since $\mathbf{t} \in \mathcal{T}^{H}$, every represented orientation of class $q$ has feature norm $g_{q}$ on $\mathbf{X}$.
  Hence, the feature and readout norms of the constructed essential neurons satisfy
\begin{equation}
    \begin{aligned}
      \norm{\mathbf{v}_{q,+}}
      &=
      \sqrt{
        \frac{
          t_{q}\norm{\boldsymbol{\beta}_{q}}
        }{
          g_{q}
        }
      }\,
      g_{q}
      =
      \sqrt{
        t_{q}g_{q}\norm{\boldsymbol{\beta}_{q}}
      }
      =
      \norm{\mathbf{a}_{q,+}},
      \\
      \norm{\mathbf{v}_{q,-}}
      &=
      \sqrt{
        \frac{
          (1-t_{q})\norm{\boldsymbol{\beta}_{q}}
        }{
          g_{q}
        }
      }\,
      g_{q}
      =
      \sqrt{
        (1-t_{q})g_{q}\norm{\boldsymbol{\beta}_{q}}
      }
      =
      \norm{\mathbf{a}_{q,-}}.
    \end{aligned}
  \end{equation}
Together with
\begin{equation}
    \norm{\overline{\mathbf{w}}_{q,+}}
    \mathbf{a}_{q,+}
    =
    t_{q}\boldsymbol{\beta}_{q},
    \qquad
    \norm{\overline{\mathbf{w}}_{q,-}}
    \mathbf{a}_{q,-}
    =
    (1-t_{q})\boldsymbol{\beta}_{q},
  \end{equation}
where $t_{q} \geq 0$, $1-t_{q} \geq 0$, and $t_{q}+(1-t_{q})=1$, with the corresponding term omitted when an orientation is absent, all three equality conditions are satisfied.
  Thus, the constructed parameterization does indeed attain the lower bound in \cref{eq:mrnp-orbit-lower-bound}.

  For the \ac{rsm} characterization, positive $1$-homogeneity gives feature rows
\begin{equation}
    \sqrt{
      \frac{
        t_{q}\norm{\boldsymbol{\beta}_{q}}
      }{
        g_{q}
      }
    }\,
    (\mathbf{f}_{q}^{+})^{\top}
    \qquad\text{and}\qquad
    \sqrt{
      \frac{
        (1-t_{q})\norm{\boldsymbol{\beta}_{q}}
      }{
        g_{q}
      }
    }\,
    (\mathbf{f}_{q}^{-})^{\top},
  \end{equation}
so the constructed parameterization realizes the \ac{rsm} in \cref{eq:mrnp-representational-geometry} for the chosen $\mathbf{t}$.
\end{proof}

Under the positivity assumption $g_{q} > 0$, the preceding proof shows that $\mathcal{T}_{N_{h}}^{H} \neq \emptyset$ is necessary and sufficient for the lower bound in \cref{eq:mrnp-orbit-lower-bound} to be attained.
As in the minimum-weight-norm case, the resulting \acp{rsm} contain no auxiliary contributions.
Unlike in that case, however, neurons from nonessential classes need not have zero incoming parameters.
Any such neuron with $\mathbf{v}_{j} = \mathbf{0}$ and $\mathbf{a}_{j} = \mathbf{0}$ has zero representation-norm cost regardless of its incoming parameters.

\Cref{prop:mrnp-representational-geometry} further shows that minimum-representation-norm selection fixes the weights of essential feature contributions jointly through the orbit invariants $\boldsymbol{\beta}_{q}$ and the feature norms $g_{q}$ on $\mathbf{X}$.
The only remaining freedom is the residual-compatible allocation between orientations of equal feature norm $g_{q}^{+} = g_{q}^{-}$.
This freedom disappears either when one orientation has strictly smaller feature norm for every essential class or when the residual constraint uniquely determines the split.

\begin{corollary}[Minimum-representation-norm identifiability]
  \label[corollary]{cor:mrnp-representational-identifiability}
  In the setting of \cref{prop:mrnp-representational-geometry}, suppose additionally that either
\begin{enumerate}
    \item $g_{q}^{+} \neq g_{q}^{-}$ for every $q \in \mathcal{E}$, or
    \item $m \neq 0$ and the matrices $\boldsymbol{\beta}_{q} \overline{\mathbf{w}}_{q}^{\top}$, $q \in \mathcal{E}$, are linearly independent.
  \end{enumerate}
Then, all \acp{mrnp} in $\mathcal{O}$ of width at least $N_{h}$ induce the same \ac{rsm}.
\end{corollary}

\begin{proof}
  Under the first condition, \cref{eq:mrnp-orientation-restrictions} uniquely fixes every coordinate $t_{q}$, so $\mathcal{T}^{H}$ contains at most one split.
  Under the second condition, the argument in the proof of \cref{cor:mwnp-representational-identifiability} shows that $\mathcal{T}$, and hence $\mathcal{T}^{H} \subseteq \mathcal{T}$, contains at most one split.
  In either case, $\mathcal{T}_{N_{h}}^{H} \neq \emptyset$ guarantees that this unique split exists and is feasible at width $N_{h}$.
  \Cref{prop:mrnp-representational-geometry} therefore gives a unique minimum-representation-norm \ac{rsm}.
  The argument in the proof of \cref{cor:mwnp-representational-identifiability} then applies identically to show that this split remains feasible and unique at every larger width, yielding the same \ac{rsm}.
\end{proof}

The preceding result guarantees identifiability by ensuring that $\mathcal{T}^{H}$ contains at most one split.
For even positively $1$-homogeneous activations, neurons of either orientation generate identical features $\mathbf{f}_{q}^{+} = \mathbf{f}_{q}^{-}$ and hence satisfy $g_{q}^{+} = g_{q}^{-} = g_{q}$.
In this case, every split $\mathbf{t} \in \mathcal{T}_{N_{h}}^{H}$ yields the same \ac{rsm}, so representational identifiability holds irrespective of whether $\mathcal{T}^{H}$ contains multiple splits.

\begin{corollary}[Identifiability for even positively $1$-homogeneous activations]
  \label[corollary]{cor:mrnp-identifiability-even-homogeneous-activations}
  Suppose $\sigma(z) = \delta\abs{z}$ with $\delta \neq 0$, and let $\mathcal{O}$ be a symmetry orbit with residual $\mathbf{r} = \mathbf{0}$.
  If $g_{q} > 0$ for every $q \in \mathcal{E}$, then at every width $N_{h} \geq \abs{\mathcal{E}}$, the minimum of $\Omega_{H}$ is attained, and all \acp{mrnp} have the same \ac{rsm}, namely
\begin{equation}
    \sum_{q \in \mathcal{E}}
    \frac{\norm{\boldsymbol{\beta}_{q}}}{g_{q}}\,
    \mathbf{f}_{q}^{+}(\mathbf{f}_{q}^{+})^{\top}.
  \end{equation}
\end{corollary}

\begin{proof}
  Since $m = 0$ and $\mathbf{r} = \mathbf{0}$, every $\mathbf{t} \in [0,1]^{\abs{\mathcal{E}}}$ satisfies the residual constraint in \cref{eq:residual-compatible-readout-splits}.
  Choosing $\mathbf{t} \in \set{0,1}^{\abs{\mathcal{E}}}$ gives $\kappa(\mathbf{t}) = \abs{\mathcal{E}}$, so $\mathcal{T}_{N_{h}} \neq \emptyset$ for every $N_{h} \geq \abs{\mathcal{E}}$.
  Moreover, since $\sigma$ is even, $\mathbf{f}_{q}^{+} = \mathbf{f}_{q}^{-}$ and thus $g_{q}^{+} = g_{q}^{-} = g_{q}$ for every $q \in \mathcal{E}$.
  Hence, $\mathcal{T}_{N_{h}}^{H} = \mathcal{T}_{N_{h}} \neq \emptyset$ for every $N_{h} \geq \abs{\mathcal{E}}$, while $\mathbf{f}_{q}^{+} = \mathbf{f}_{q}^{-}$ reduces \cref{eq:mrnp-representational-geometry} to
\begin{equation}
    \sum_{q \in \mathcal{E}}
    \frac{\norm{\boldsymbol{\beta}_{q}}}{g_{q}}\,
    \mathbf{f}_{q}^{+}(\mathbf{f}_{q}^{+})^{\top},
  \end{equation}
independently of $\mathbf{t}$.
  The result therefore follows from \cref{prop:mrnp-representational-geometry}.
\end{proof}

The positivity assumption $g_{q} > 0$ in \cref{prop:mrnp-representational-geometry} excludes essential parameter classes for which one of the features $\mathbf{f}_{q}^{+}$ or $\mathbf{f}_{q}^{-}$ vanishes.
The following example shows why this assumption matters: when an orbit $\mathcal{O}$ contains such an essential parameter class, the minimum of the representation-norm objective $\Omega_{H}$ over $\mathcal{O}$ may fail to exist.

\begin{example}[Nonexistence of an \ac{mrnp}]
  \label[example]{ex:mrnp-nonattainment-invisible-feature}
  Let $\psi(z) \coloneqq \max(0,z)$ and consider the family of symmetry-equivalent width-$1$ parameterizations
\begin{equation}
    \boldsymbol{\theta}_{\alpha}
    =
    (\psi;\alpha,-2\alpha,\alpha^{-1}),
    \qquad
    \alpha > 0,
  \end{equation}
which realize $f_{\boldsymbol{\theta}_{\alpha}}(x) = \psi(x-2)$.
  On the inputs collected in $\mathbf{X} = (-1, 0, 1)$, the hidden-activation matrix of each parameterization consists of the single feature row
\begin{equation}
    \mathbf{v}_{\alpha}^{\top}
    =
    \psi(\overline{\mathbf{w}}_{\alpha}^{\top}\overline{\mathbf{X}}),
    \qquad
    \overline{\mathbf{w}}_{\alpha}
    =
    \begin{bmatrix}
      \alpha\\
      -2\alpha
    \end{bmatrix}.
  \end{equation}
Since
\begin{equation}
    \overline{\mathbf{w}}_{\alpha}^{\top}\overline{\mathbf{X}}
    =
    \begin{bmatrix}
      \alpha & -2\alpha
    \end{bmatrix}
    \begin{bmatrix}
      -1 & 0 & 1 \\
      1  & 1 & 1
    \end{bmatrix}
    =
    \begin{bmatrix}
      -3\alpha & -2\alpha & -\alpha
    \end{bmatrix},
  \end{equation}
the feature $\mathbf{v}_{\alpha}$ vanishes for every $\alpha > 0$.
  Hence,
\begin{equation}
    \Omega_{H}(\boldsymbol{\theta}_{\alpha})
    =
    \norm{\mathbf{H}}_{F}^{2}
    +
    \norm{\mathbf{A}}_{F}^{2}
    =
    \alpha^{-2}
    \longrightarrow
    0
    \qquad
    \text{as } \alpha \to \infty.
  \end{equation}
Consequently, $\Omega_{H}$ can be made arbitrarily small within this width-$1$ symmetry orbit.
  Since $\Omega_{H} \geq 0$, any minimum would therefore have to equal zero.
  But $\Omega_{H} = 0$ would require $\mathbf{A} = \mathbf{0}$ and hence force the realized function to vanish everywhere, contradicting the fact that $\psi(x-2)$ is nonzero for $x > 2$.
  Therefore, this symmetry orbit contains no \ac{mrnp}.

  We next relate this failure of existence to the positivity assumption $g_{q} > 0$ in \cref{prop:mrnp-representational-geometry}.
  The single neuron represents the unique essential parameter class $q$ of the symmetry orbit.
  Choosing the unit-norm representative
\begin{equation}
    \overline{\mathbf{w}}_{q}
    =
    \frac{1}{\sqrt{5}}
    \begin{bmatrix}
      1\\
      -2
    \end{bmatrix},
  \end{equation}
the two features generated on $\mathbf{X}$ by the two orientations of $q$ are
\begin{equation}
    \mathbf{f}_{q}^{+}
    =
    \begin{bmatrix}
      0\\
      0\\
      0
    \end{bmatrix},
    \qquad
    \mathbf{f}_{q}^{-}
    =
    \frac{1}{\sqrt{5}}
    \begin{bmatrix}
      3\\
      2\\
      1
    \end{bmatrix}.
  \end{equation}
Hence,
\begin{equation}
    g_{q}^{+} = 0,
    \qquad
    g_{q}^{-} > 0,
    \qquad
    g_{q} = 0,
  \end{equation}
so the positivity assumption in \cref{prop:mrnp-representational-geometry} fails.

  Moreover, the residual constraint uniquely fixes $t_{q} = 1$.
  This split also satisfies the orientation restriction in \cref{eq:mrnp-orientation-restrictions}, since $g_{q}^{+} < g_{q}^{-}$, and requires only one neuron.
  Thus, the corresponding split is feasible at width $1$, even though the representation-norm objective has no minimizer in the width-$1$ orbit.
  This shows that feasibility of $\mathcal{T}_{N_{h}}^{H}$ alone does not guarantee existence of an \ac{mrnp} when $g_{q} = 0$.
\end{example}

The reason $g_{q} = 0$ can prevent existence of an \ac{mrnp} is that it breaks the usual trade-off between feature magnitude $\norm{\mathbf{v}_{j}}$ and readout magnitude $\norm{\mathbf{a}_{j}}$ under positive scaling.
For a neuron $j$ in an essential class $q$, the representation-norm objective contains the two contributions $\norm{\mathbf{v}_{j}}^{2}$ and $\norm{\mathbf{a}_{j}}^{2}$.
Ordinarily, positive scaling trades off these two terms: scaling up a neuron's incoming parameters $\overline{\mathbf{w}}_{j}$ increases $\norm{\mathbf{v}_{j}}$, while simultaneously scaling down its readout weights $\mathbf{a}_{j}$ decreases $\norm{\mathbf{a}_{j}}$, producing a finite optimum.
When $g_{q} = 0$, however, an orientation of class $q$ generates the zero feature on $\mathbf{X}$, so any neuron in that orientation satisfies $\norm{\mathbf{v}_{j}} = 0$ regardless of the scale of its incoming parameters.
The incoming parameters can therefore be scaled up arbitrarily while the readout weights $\mathbf{a}_{j}$ approach $\mathbf{0}$, driving the neuron's contribution to $\norm{\mathbf{A}}_{F}^{2}$ toward zero without increasing its contribution to $\norm{\mathbf{H}}_{F}^{2}$.
The condition $g_{q} > 0$ rules out precisely this failure of the scaling trade-off.

\section{Numerical and analytical characterization of ReLU networks solving the XOR task}
\label[appendix]{app-sec:characterization-of-xor-relu-solutions}

This appendix provides mathematical details supporting the ReLU networks illustrated in \cref{fig:dissociation}.
We first specify the \acs{xor} dataset, the two-neuron ReLU architecture, and the \ac{bce} training objective used throughout the analysis (\cref{app-subsec:xor-setup}).
We then describe how the six qualitatively distinct ReLU solutions to the \acs{xor} task shown in Panel~B of \cref{fig:dissociation} were identified through a gradient-descent sweep (\cref{app-subsec:xor-gradient-descent-sweep-and-clustering}), derive closed-form expressions for these solutions in a geometric parameterization (\cref{app-subsec:xor-closed-form-characterization}), and prove that each can approach zero \ac{bce} loss in the limit of parameter rescaling (\cref{app-subsec:xor-approaching-zero-loss-by-scaling}).

\subsection{Setup}
\label[appendix]{app-subsec:xor-setup}

The \ac{xor} dataset consists of the $P = 4$ points $\set{-1, +1}^{2}$, where same-sign inputs receive label $0$ and opposite-sign inputs receive label $1$:
\begin{equation}
  \label{eq:xor-dataset}
  \mathbf{X} =
  \begin{bmatrix}
    -1 & +1 & -1 & +1 \\
    -1 & -1 & +1 & +1
  \end{bmatrix},
  \qquad
  \mathbf{y} =
  \begin{bmatrix}
    0 & 1 & 1 & 0
  \end{bmatrix}.
\end{equation}
We consider two-neuron ReLU networks with scalar output,
\begin{equation}
  \label{eq:xor-network}
  f_{\boldsymbol{\theta}}(\mathbf{x})
  \coloneqq a_{1} \, \psi(\mathbf{w}_{1}^{\top} \mathbf{x} + b_{1})
  + a_{2} \, \psi(\mathbf{w}_{2}^{\top} \mathbf{x} + b_{2})
  + b,
  \qquad
  \psi(z) = \max(0, z),
\end{equation}
where $\mathbf{w}_{j} \in \mathbb{R}^{2}$, $b_{j} \in \mathbb{R}$ are the incoming weights and hidden biases, $a_{j} \in \mathbb{R}$ are the readout weights, $b \in \mathbb{R}$ is the output bias, and $\boldsymbol{\theta} = (\psi; \mathbf{w}_{1}, \mathbf{w}_{2}, b_{1}, b_{2}, a_{1}, a_{2}, b)$ denotes the network's parameterization.
The output bias $b$ lies outside the model class introduced in \cref{sec:preliminaries-and-setting}.
It does not enter the hidden-activation matrix and remains unchanged under the parameter symmetries applied to these networks in \cref{fig:dissociation,fig:primitives}.
The results of \cref{sec:functional-parameter-symmetries,sec:parameter-symmetries-act-through-three-feature-primitives,sec:parameter-symmetries-dissociate-function-and-representation} therefore apply directly to the hidden layer, whose contribution to the network output is $f_{\boldsymbol{\theta}} - b$.
Training minimizes the mean \ac{bce} loss
\begin{equation}
  \label{eq:xor-bce-loss}
  \mathcal{L}(\boldsymbol{\theta})
  \coloneqq -\frac{1}{4} \sum_{\mu=1}^{4}
  \bigl[
    y^\mu \log p^\mu + (1 - y^\mu) \log(1 - p^\mu)
  \bigr],
  \qquad
  p^\mu = \varsigma(f_{\boldsymbol{\theta}}(\mathbf{x}^\mu)),
\end{equation}
where $\varsigma(z) \coloneqq (1 + \mathrm{e}^{-z})^{-1}$ denotes the logistic sigmoid.

\paragraph{Geometric parameterization.}
Each hidden neuron with nonzero incoming weights $\mathbf{w}_{j} \neq \mathbf{0}$ admits a geometric parameterization in terms of the angle $\phi_{j}$ and signed distance $d_{j}$ of its activation boundary, together with a gain factor $g_{j} > 0$.
Concretely, the pre-activation of the $j$th neuron is
\begin{equation}
  \label{eq:xor-geometric-pre-activation}
  z_{j}(\mathbf{x})
  \coloneqq g_{j}
  \bigl(
  \mathbf{n}(\phi_{j})^{\top} \mathbf{x} - d_{j}
  \bigr),
  \qquad
  \mathbf{n}(\phi) \coloneqq (\cos \phi, \sin \phi)^{\top},
\end{equation}
so that the zero level set $z_{j}(\mathbf{x}) = 0$ is the line with unit normal $\mathbf{n}(\phi_{j})$ at signed distance $d_{j}$ from the origin.
This corresponds to the usual affine parameterization via $\mathbf{w}_{j} = g_{j} \, \mathbf{n}(\phi_{j})$ and $b_{j} = -g_{j} d_{j}$.

\subsection{Gradient-descent sweep and solution clustering}
\label[appendix]{app-subsec:xor-gradient-descent-sweep-and-clustering}

\begin{figure}[t]
  \centering
  \includegraphics[width=0.75\textwidth]{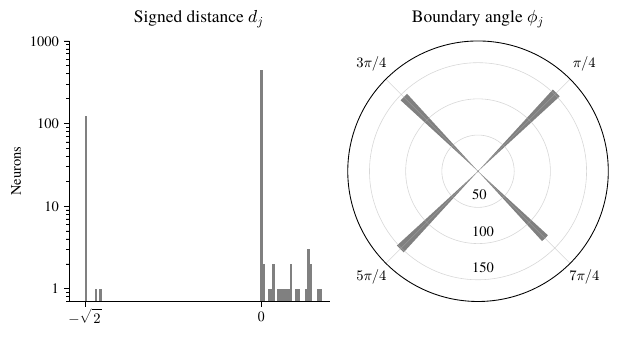}
  \caption{\textbf{Boundary geometry of the $592$ hidden neurons across $296$ converged networks}.
    Ticks indicate the canonical values of the six solution types listed in \cref{tab:xor-relu-solutions}.
    Histogram counts are shown on a logarithmic scale for the signed distances $d_{j}$ (left) and a linear scale for the angles $\phi_{j}$ (right).
    Of the $592$ neurons, $19$ have signed distances distinct from both dominant values, while $2$ have angles distinct from all four dominant values.
  }
  \label{fig:xor-sweep-geometry}
\end{figure}

We trained \texttt{N\_SEEDS}~$= 1000$ two-neuron ReLU networks on the \ac{xor} dataset from independent random initializations, drawing each weight and bias i.i.d.\ from $\mathcal{U}(-1/\sqrt{2},\, +1/\sqrt{2})$, corresponding to the default initialization in Equinox's \texttt{nn.Linear} module.
Networks were optimized with full-batch Adam at learning rate $\eta = 0.1$ for \texttt{N\_STEPS}~$= 10^{7}$ steps, minimizing the \ac{bce} loss detailed in \cref{eq:xor-bce-loss}.
A run was deemed converged if its final loss fell below \texttt{TARGET\_LOSS}~$= 10^{-12}$; $296$ of the $1000$ runs converged.
The sweep was run locally on the JAX CPU backend on a MacBook Air with an Apple M4 chip, $10$ CPU cores, and $16$~GB of unified memory; it required approximately $9$ minutes of wall-clock time and used at most $0.5$~GB of peak process memory.
All experiments were implemented in JAX \parencite[v0.10.0,][]{jax2018github} using the Equinox neural-network library \parencite[v0.13.8,][]{kidger2021equinox}.

\begin{figure}[t!]
  \centering
  \iflatexml
    \includegraphics[width=0.8\textwidth]{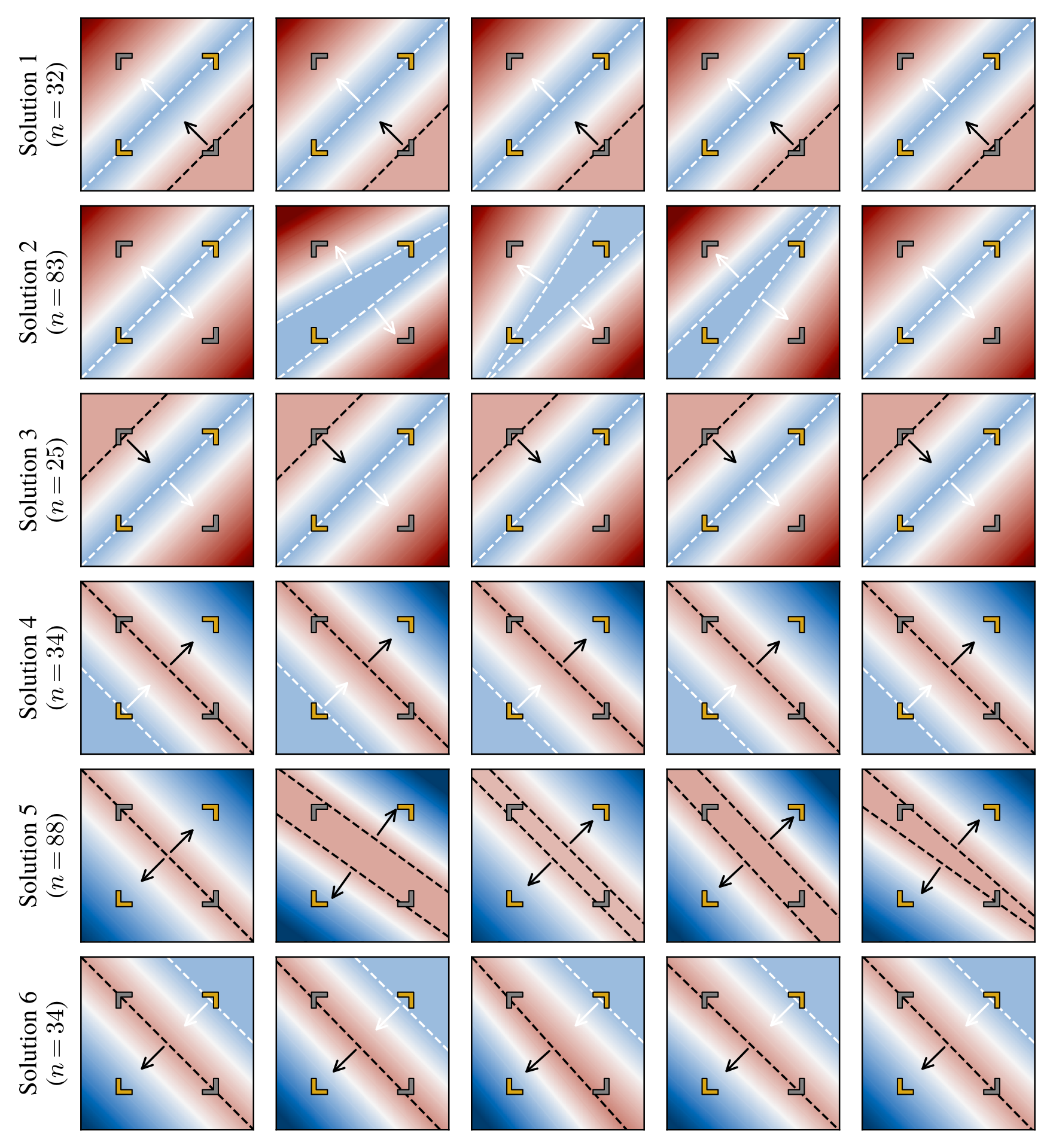}
  \else
    \includegraphics[width=0.8\textwidth]{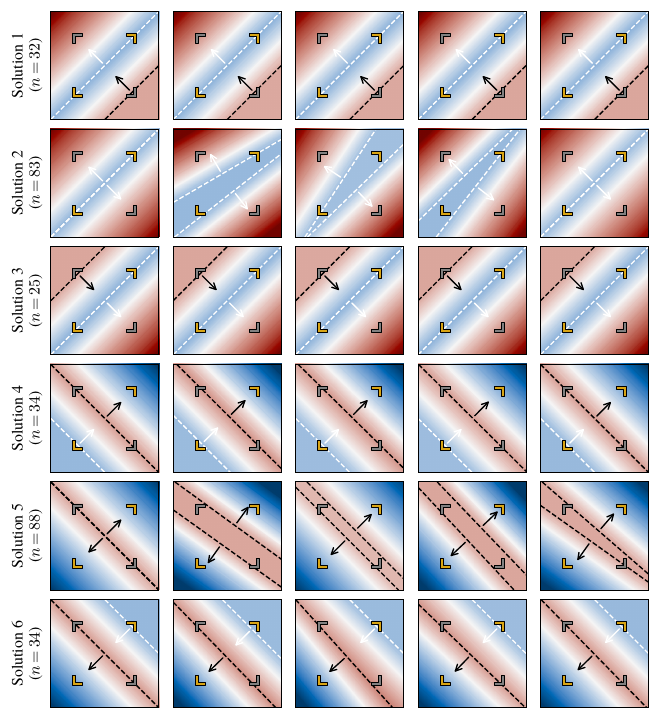}
  \fi
  \caption{\textbf{Cluster medoids and maximally dissimilar members}.
    The leftmost panel in each row shows the cluster medoid; moving right, each panel shows the member farthest from those already selected.
    Because training produces substantially larger readout weights than the closed-form representatives in \cref{tab:xor-relu-solutions}, the readouts of each network are rescaled for visualization to match the logit scale of \cref{eq:xor-logits}.
  }
  \label{fig:xor-sweep-clusters}
\end{figure}

For each converged network, we converted the affine hidden parameters to the geometric parameterization of \cref{eq:xor-geometric-pre-activation}:
\begin{equation}
  g_{j} = \norm{\mathbf{w}_{j}},
  \qquad
  \phi_{j} = \mathrm{atan2}(w_{j,2}, w_{j,1}),
  \qquad
  d_{j} = -\frac{b_{j}}{g_{j}}.
\end{equation}
Before comparing networks, we accounted for two structural symmetries: the positive-scaling symmetry of ReLU and the permutation symmetry of hidden neurons.
We removed positive scaling from the neuron representation and quotiented out hidden-neuron permutations in the network-level distance.
For the former, we absorbed the gain $g_{j}$ into the readout and defined the effective readout weight $\beta_{j} \coloneqq a_{j}g_{j}$.
Each hidden neuron is then represented by
\begin{equation}
  \mathbf{s}_{j}
  \coloneqq
  (\cos\phi_{j}, \sin\phi_{j}, d_{j}, \beta_{j}),
\end{equation}
where the unit-circle representation of $\phi_{j}$ avoids branch-cut artifacts at $\phi = \pm\pi$.
A network is thus represented by its two neuron descriptors $\mathbf{s}_{1},\mathbf{s}_{2}$ and its output bias $b$.

For two networks $\boldsymbol{\theta}$ and $\boldsymbol{\xi}$, let $\mathbf{s}_{j}(\boldsymbol{\theta})$ and $\mathbf{s}_{j}(\boldsymbol{\xi})$ denote their neuron descriptors and let $b_{\boldsymbol{\theta}}$ and $b_{\boldsymbol{\xi}}$ denote their output biases.
We define
\begin{equation}
  \label{eq:xor-quotient-metric}
  d_{\mathrm{quot}}(\boldsymbol{\theta}, \boldsymbol{\xi})
  \coloneqq
  \min_{\pi\in \mathfrak{S}_{2}}
  \Bigl(
    (b_{\boldsymbol{\theta}}-b_{\boldsymbol{\xi}})^{2}
    +
    \sum_{j=1}^{2}
    \norm{
      \mathbf{s}_{j}(\boldsymbol{\theta})
      -
      \mathbf{s}_{\pi(j)}(\boldsymbol{\xi})
    }^{2}
  \Bigr)^{1/2}.
\end{equation}
This distance optimally matches the two hidden neurons before computing their Euclidean distance and therefore removes the hidden-neuron permutation symmetry.
Because the neuron descriptors are themselves invariant under positive scaling, $d_{\mathrm{quot}}$ is invariant under both symmetries.

Before computing distances, we standardized each descriptor dimension to zero mean and unit variance to prevent differences in scale from dominating the comparison.
In particular, the unit-circle coordinates are bounded in $[-1,1]$, whereas signed distances and effective readout weights span wider ranges.
For each neuron-level dimension, the standardization statistics were pooled across both hidden-neuron positions, ensuring that the same transformation is applied regardless of neuron ordering and preserving the permutation invariance of $d_{\mathrm{quot}}$.
We then performed average-linkage hierarchical agglomerative clustering using $d_{\mathrm{quot}}$ and cut the resulting dendrogram at the largest gap in its sequence of merge distances.
This yielded six clusters, with $88$, $83$, $34$, $34$, $32$, and $25$ networks.
Each cluster corresponds to one of the six solution types listed in \cref{tab:xor-relu-solutions}.
Solutions~5 and~2, whose two activation boundaries pass through the origin, account for the two largest clusters, with $88$ and $83$ networks, followed by solutions~4, 6, 1, and~3 with $34$, $34$, $32$, and $25$ networks, respectively.
\Cref{fig:xor-sweep-clusters} shows a small set of diverse representatives of each cluster, selected using medoid-seeded farthest-first (maxmin) sampling.

The corresponding six solution types form two families of three, which we call \emph{diagonal} and \emph{anti-diagonal} according to the orientation of their activation boundaries.
Representatives are shown in Panel~B of \cref{fig:dissociation}, with all six solution types displayed individually in \cref{fig:xor-solutions}.
These types characterize the dominant solutions obtained in this sweep; we do not claim that they exhaust the two-neuron ReLU \ac{xor} solutions attainable under other optimizers, learning rates, initialization distributions, or training protocols.

\subsection{Closed-form characterization of the six dominant solution types}
\label[appendix]{app-subsec:xor-closed-form-characterization}

Each of the six solution types identified by gradient descent admits a convenient closed-form representative using the geometric parameterization introduced in \cref{eq:xor-geometric-pre-activation}.
We choose representatives with unit gain ($g_{1} = g_{2} = 1$), so that $\mathbf{w}_{j} = \mathbf{n}(\phi_{j})$ and $b_{j} = -d_{j}$.
\Cref{tab:xor-relu-solutions} lists their complete specifications, and \cref{fig:xor-solutions} visualizes the networks they define.

\begin{table}[t]
  \caption{Closed-form representatives of the six ReLU solution types for \ac{xor} identified through the gradient-descent sweep.
    Each row specifies the angles $(\phi_{1}, \phi_{2})$, signed distances $(d_{1}, d_{2})$, readout weights $(a_{1}, a_{2})$, and output bias $b$.
    All solutions listed here use unit gain.
  }
  \label{tab:xor-relu-solutions}
  \centering
\begin{tabular}{clcccc}
    \toprule
    \# & Family        & $(\phi_{1}, \phi_{2})$             & $(d_{1}, d_{2})$ & $(a_{1}, a_{2})$ & $b$                             \\
    \midrule
    1  & Diagonal      & $(\frac{3\pi}{4}, \frac{3\pi}{4})$ & $(0, -\sqrt{2})$ & $(2, -1)$        & $\phantom{-}\frac{1}{\sqrt{2}}$ \\[4pt]
    2  & Diagonal      & $(\frac{3\pi}{4}, \frac{7\pi}{4})$ & $(0, 0)$         & $(1, 1)$         & $-\frac{1}{\sqrt{2}}$           \\[4pt]
    3  & Diagonal      & $(\frac{7\pi}{4}, \frac{7\pi}{4})$ & $(-\sqrt{2}, 0)$ & $(-1, 2)$        & $\phantom{-}\frac{1}{\sqrt{2}}$ \\[4pt]
    4  & Anti-diagonal & $(\frac{\pi}{4}, \frac{\pi}{4})$   & $(-\sqrt{2}, 0)$ & $(1, -2)$        & $-\frac{1}{\sqrt{2}}$           \\[4pt]
    5  & Anti-diagonal & $(\frac{\pi}{4}, \frac{5\pi}{4})$  & $(0, 0)$         & $(-1, -1)$       & $\phantom{-}\frac{1}{\sqrt{2}}$ \\[4pt]
    6  & Anti-diagonal & $(\frac{5\pi}{4}, \frac{5\pi}{4})$ & $(0, -\sqrt{2})$ & $(-2, 1)$        & $-\frac{1}{\sqrt{2}}$           \\
    \bottomrule
  \end{tabular}
\end{table}

\paragraph{Reduction to one dimension.}
A key structural property shared by all six solutions is that the two weight vectors $\mathbf{w}_{1}$ and $\mathbf{w}_{2}$ are collinear.
Specifically:
\begin{itemize}
  \item Solutions~1--3 have angles $\phi_{j} \in \set{\nicefrac{3\pi}{4}, \nicefrac{7\pi}{4}}$.
  Since $\mathbf{n}(\nicefrac{7\pi}{4}) = -\mathbf{n}(\nicefrac{3\pi}{4})$,
  both weight vectors lie on the line spanned by $\mathbf{n}(\nicefrac{3\pi}{4})$.
  The network output depends on $\mathbf{x}$ only through the \emph{diagonal} projection $u_{d} = \mathbf{n}(\nicefrac{3\pi}{4})^{\top} \mathbf{x} = (x_{2} - x_{1}) / \sqrt{2}$.
\item Solutions~4--6 use angles $\phi_{j} \in \set{\nicefrac{\pi}{4}, \nicefrac{5\pi}{4}}$, and the output depends only on the \emph{anti-diagonal} projection $u_{a} = \mathbf{n}(\nicefrac{\pi}{4})^{\top} \mathbf{x} = (x_{1} + x_{2}) / \sqrt{2}$.
\end{itemize}
In both cases, the four \ac{xor} points project to exactly three values $u \in \set{-\sqrt{2}, 0, +\sqrt{2}}$, with two points collapsing onto $u = 0$.
For the diagonal family, the label-$1$ points project to $u_{d} = \pm \sqrt{2}$ and the label-$0$ points to $u_{d} = 0$.
For the anti-diagonal family, the roles are reversed, with the label-$0$ points at $u_{a} = \pm \sqrt{2}$ and the label-$1$ points at $u_{a} = 0$.

\begin{figure}[t]
  \centering
  \iflatexml
    \includegraphics[width=0.85\textwidth]{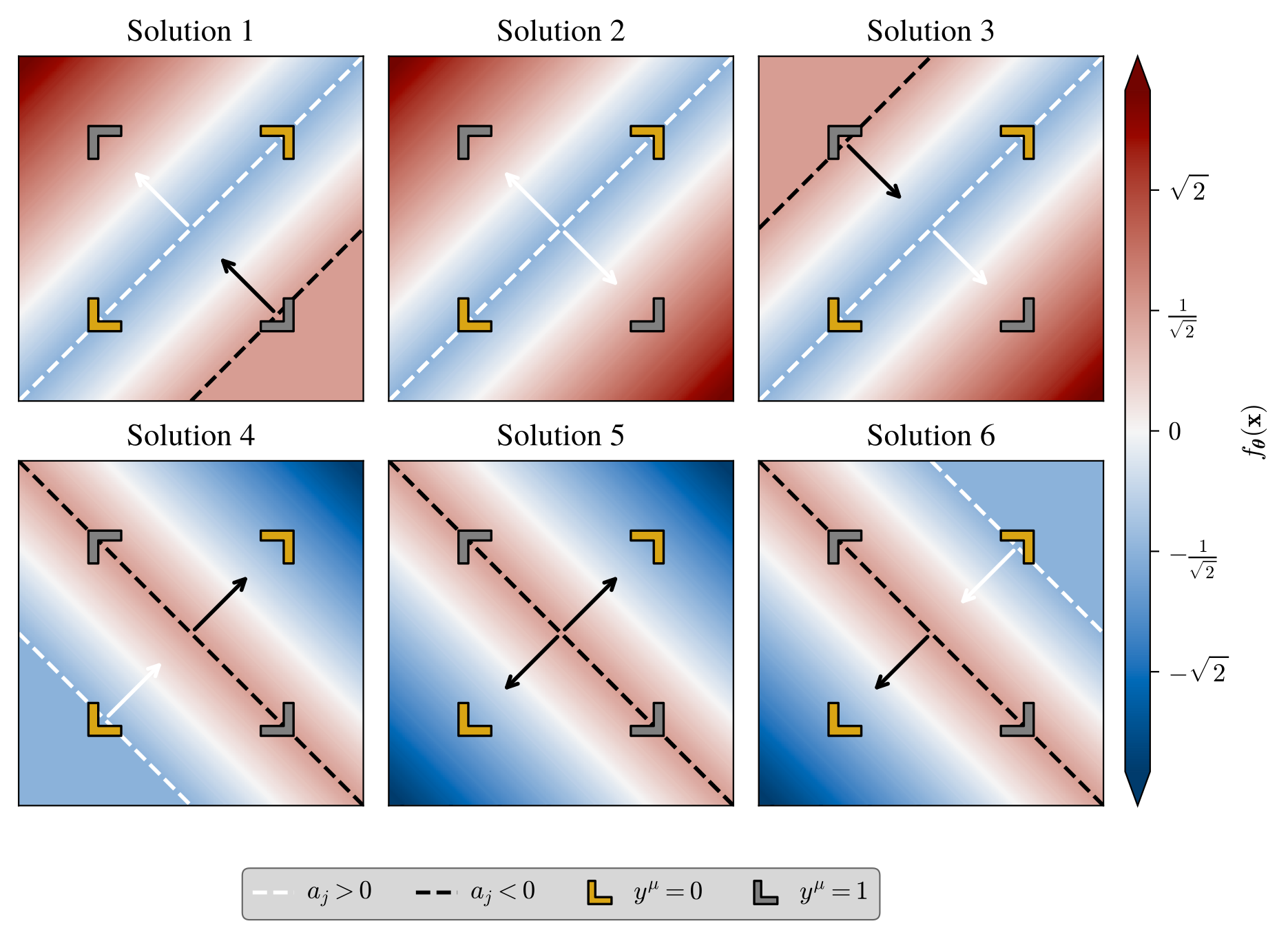}
  \else
    \includegraphics[width=0.85\textwidth]{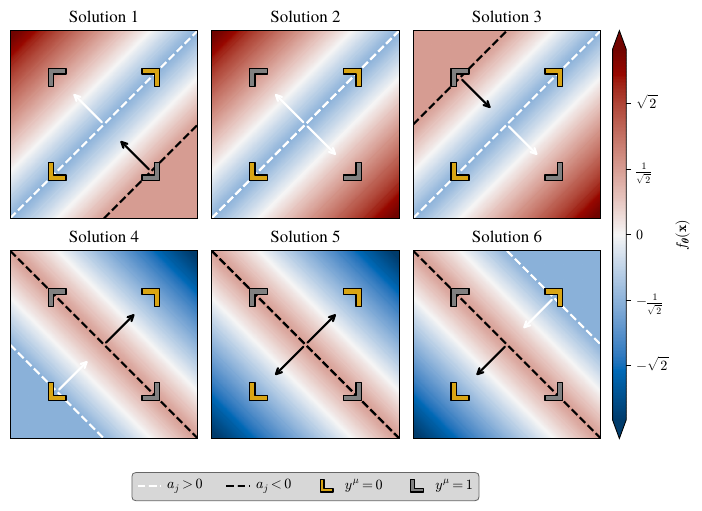}
  \fi
  \caption{\textbf{The six closed-form ReLU solutions of \cref{tab:xor-relu-solutions} at unit gain}.
    Solutions~1--3 (top row) form the diagonal family, while solutions~4--6 (bottom row) form the anti-diagonal family.
    Within each family, the activation boundaries share a common orientation, while the solutions differ in their signed distances $d_j$ and readout weights $a_j$.
    Each panel shows the logit $f_{\boldsymbol{\theta}}(\mathbf{x})$ of \cref{eq:xor-network} as a heatmap, the activation boundaries $z_j(\mathbf{x})=0$ as dashed lines with unit normals $\mathbf{n}(\phi_j)$, and the four \ac{xor} data points.
    All six networks assign every data point a logit of magnitude $1/\sqrt{2}$.
  }
  \label{fig:xor-solutions}
\end{figure}

\paragraph{Logits on the \ac{xor} data.}
One can verify by direct computation that every solution in \cref{tab:xor-relu-solutions} produces the same logit values on the four \ac{xor} points:
\begin{equation}
  \label{eq:xor-logits}
  f_{\boldsymbol{\theta}}(\mathbf{x}^{\mu}) =
  \begin{cases}
    +\frac{1}{\sqrt{2}}, & y^{\mu} = 1, \\
    -\frac{1}{\sqrt{2}}, & y^{\mu} = 0,
  \end{cases}
  \qquad
  \mu = 1, \dots, 4.
\end{equation}
That is, all six networks classify \ac{xor} correctly, with every data point receiving the same logit magnitude $1 / \sqrt{2}$.
We illustrate this for one solution per family.

\emph{Solution~2} (diagonal family):
The network computes
\begin{equation}
  f_{\boldsymbol{\theta}}(\mathbf{x}) = \psi(u_{d}) + \psi(-u_{d}) - \frac{1}{\sqrt{2}} = \abs{u_{d}} - \frac{1}{\sqrt{2}},
\end{equation}
where $u_{d} = (x_{2} - x_{1}) / \sqrt{2}$.
On the data: $f = -1/\sqrt{2}$ at $u_{d} = 0$ (label~$0$) and $f = +1 / \sqrt{2}$ at $u_{d} = \pm \sqrt{2}$ (label~$1$).

\emph{Solution~5} (anti-diagonal family):
The network computes
\begin{equation}
  f_{\boldsymbol{\theta}}(\mathbf{x}) = -\psi(u_{a}) - \psi(-u_{a}) + \frac{1}{\sqrt{2}} = -\abs{u_{a}} + \frac{1}{\sqrt{2}},
\end{equation}
where $u_{a} = (x_{1} + x_{2}) / \sqrt{2}$.
On the data: $f = +1 / \sqrt{2}$ at $u_{a} = 0$ (label~$1$) and $f = -1 / \sqrt{2}$ at $u_{a} = \pm \sqrt{2}$ (label~$0$).

\subsection{Approaching zero BCE loss by parameter scaling}
\label[appendix]{app-subsec:xor-approaching-zero-loss-by-scaling}

Because the logistic sigmoid maps $\mathbb{R}$ into $(0, 1)$ without reaching the endpoints, no finite parameter vector $\boldsymbol{\theta}$ achieves exactly zero \ac{bce} loss.
However, the parameter vector $\boldsymbol{\theta}$ of each of the six solutions in \cref{tab:xor-relu-solutions} sits on a \emph{scaling ray} along which the loss decreases monotonically to zero.

\begin{proposition}
  \label[proposition]{prop:xor-bce-scaling-ray}
  Let $\boldsymbol{\theta}$ be any of the six solutions from \cref{tab:xor-relu-solutions}, and define the scaled parameterization
\begin{equation}
    \label{eq:xor-scaling-ray}
    \boldsymbol{\theta}_{\tau}
    \coloneqq (\psi; \tau \mathbf{w}_{1}, \tau \mathbf{w}_{2}, \tau b_{1}, \tau b_{2}, a_{1}, a_{2}, \tau b),
    \qquad
    \tau > 0.
  \end{equation}
Then, $f_{\boldsymbol{\theta}_{\tau}}(\mathbf{x}^{\mu}) = \tau \, f_{\boldsymbol{\theta}}(\mathbf{x}^{\mu})$, for $\mu = 1, \dots, 4$, and the \ac{bce} loss along this ray satisfies
\begin{equation}
    \label{eq:xor-bce-along-ray}
    \mathcal{L}(\tau) \coloneqq
    \mathcal{L}(\boldsymbol{\theta}_{\tau})
    = \log(1 + \exp(-\tau / \sqrt{2})),
  \end{equation}
which is strictly decreasing in $\tau$, with $\lim_{\tau \to \infty} \mathcal{L}(\tau) = 0$.
\end{proposition}

\begin{proof}
  Under the scaling in \cref{eq:xor-scaling-ray}, the network output satisfies $f_{\boldsymbol{\theta}_{\tau}}(\mathbf{x}) = \tau \, f_{\boldsymbol{\theta}}(\mathbf{x})$, since ReLU is positively homogeneous of degree~$1$.
  By \cref{eq:xor-logits}, the logit at each data point thus has magnitude $\tau / \sqrt{2}$ with the appropriate sign.
  For a label-$1$ point, the \ac{bce} contribution is $-\log \varsigma(\tau / \sqrt{2})$.
  For a label-$0$ point, it is $-\log(1 - \varsigma(-\tau / \sqrt{2})) = -\log \varsigma(\tau / \sqrt{2})$, where the last step uses the identity $1 - \varsigma(z) = \varsigma(-z)$ of the logistic sigmoid.
  Since all four terms are identical, the mean \ac{bce} is $\mathcal{L}(\tau) = \log(1 + \exp(-\tau / \sqrt{2}))$.
  The map $\tau \mapsto -\tau / \sqrt{2}$ is strictly decreasing, while $\exp$ and $\log$ are strictly increasing;
  hence $\mathcal{L}(\tau)$ is strictly decreasing.
  The fact that $\lim_{\tau \to \infty} \mathcal{L}(\tau) = 0$ is evident.
\end{proof}

\end{document}